\documentclass[12pt]{article}
\usepackage{amsfonts}

\usepackage{pifont}

\usepackage{amsmath,verbatim,amssymb,epsfig,color}
\usepackage{bm}
\usepackage{graphicx,subfigure}
\usepackage{url,lscape}
\usepackage{mathrsfs}
\usepackage{float}
\usepackage{cite}
\usepackage{natbib}
\usepackage{algorithm}
\usepackage{amsthm}
\usepackage[colorlinks=true,
linkcolor=blue,citecolor=black]{hyperref}
\newtheorem{theorem}{Theorem}[section]

\newtheorem{assumption}{Assumption}
\newtheorem{lemma}{Lemma}[section]

\def\eqx"#1"{{\label{#1}}}
\def\eqn"#1"{{\ref{#1}}}

\def\squarebox#1{\hbox to #1{\hfill\vbox to #1{\vfill}}}
\def\boxit#1{\vbox{\hrule\hbox{\vrule\kern6pt
			\vbox{\kern6pt#1\kern6pt}\kern6pt\vrule}\hrule}}

\usepackage[symbol]{footmisc}
\renewcommand{\thefootnote}{\fnsymbol{footnote}}

\makeatletter 
\@addtoreset{equation}{section}
\makeatother  

\def\theequation{\thesection.\arabic{equation}}

\newcommand{\bmb}{\bm \beta}
\newcommand{\bmd}{\bm \delta}

\newcommand{\hb}{\widehat{\bmb}}

\newcommand{\hd}{\widehat{\bmd}}

\newcommand{\bR}{\mbox{\bf R}}
\newcommand{\bZ}{\mbox{\bf Z}}

\newcommand{\bH}{\mbox{\bf H}}

\newcommand{\bz}{\mbox{\bf z}}

\newcommand{\be}{\mbox{\bf e}}

\newcommand{\bA}{\mbox{\bf A}}

\newcommand{\bI}{\mbox{\bf I}}

\newcommand{\bu}{\mbox{\bf u}}

\newcommand{\bd}{\mbox{\bf d}}

\newcommand{\bS}{\mbox{\bf S}}

\newcommand{\rmd}{\mbox{\rm d}}
\newcommand{\bv}{\mbox{\bf v}}

\newcommand{\bbeta}{\boldsymbol{\beta}}

\newcommand{\bdelta}{\boldsymbol{\delta}}

\newcommand{\bxi}{\boldsymbol{\xi}}

\newcommand{\bSigma}{\boldsymbol{\Sigma}}
\newcommand{\balpha}{\boldsymbol{\alpha}}

\newcommand{\bzeta}{\boldsymbol{\zeta}}

\newcommand{\bgamma}{\boldsymbol{\gamma}}

\newcommand{\bvarphi}{\boldsymbol{\varphi}}

\newcommand{\mcT}{{\cal T}}
\newcommand{\mcS}{{\cal S}}
\newcommand{\mcZ}{{\cal Z}}
\newcommand{\mcI}{{\cal I}}
\newcommand{\mbR}{{\mathbb R}}
\def\wh{\widehat}
\def\wt{\widetilde}
\def\rmT{{\top}}
\def\rmo{{\rm o}}

\begin{document}
\baselineskip=24pt
\begin{center}
{\Large \bf Estimation and inference for
transfer learning with high-dimensional quantile regression}
\end{center}
\vspace{0.2in}
\begin{center}
\renewcommand*{\thefootnote}{\fnsymbol{footnote}}
{\large Jiayu Huang, Mingqiu Wang, and Yuanshan Wu}
\footnote{Jiayu Huang is Assistant Professor, School of Mathematics and Statistics, Central China Normal University,
Wuhan, Hubei 430079, P. R. China (E-mail: huang803@ccnu.edu.cn).
Mingqiu Wang is Professor, School of Statistics and Data Science, Qufu Normal University,
  Qufu, Shandong 273165,  P. R. China (E-mail: mqwang@qfnu.edu.cn).
Yuanshan Wu, the corresponding author, is  Professor, School of Statistics and Mathematics,
       Zhongnan University of Economics and Law,
       Wuhan, Hubei 430073, P. R. China (E-mail: wu@zuel.edu.cn).
}
\end{center}


\addcontentsline{toc}{section}{Abstract}
\begin{center}{\bf Abstract}
\end{center}
Transfer learning has become an essential technique to exploit information from the source domain to boost performance of the
target task. Despite the prevalence in  high-dimensional data, heterogeneity and heavy tails are insufficiently accounted for by current transfer learning approaches and thus may undermine the resulting performance. We propose a transfer learning procedure in the framework of high-dimensional quantile regression models to accommodate heterogeneity and heavy tails in the source and target domains. We establish error bounds of transfer learning estimator based on delicately selected transferable source domains, showing that lower error bounds can be achieved for critical selection criterion and larger sample size of source tasks. We further propose valid confidence interval and hypothesis test procedures for individual component of high-dimensional quantile regression coefficients by advocating a double transfer learning estimator, which is one-step debiased estimator for the transfer learning estimator wherein the technique of transfer learning is designed again. By adopting data-splitting technique, we advocate a transferability detection approach that guarantees to circumvent negative transfer and identify transferable sources with high probability. Simulation results demonstrate that the proposed method exhibits some favorable and compelling performances and the practical utility is further illustrated by analyzing a real example.

\vspace{0.5cm} \noindent{\bf Keywords}:
Debiased estimator;
High-dimensional inference;
Non-smooth loss function;
Robustness;
Uncertainty quantification

\section{Introduction}
\label{sec1}
\noindent
Transfer learning, which aims to enhance the performance of the target task by
transferring knowledge learned in different but related source
tasks, has been demonstrated to be effective for a wide range of real-world applications
and attracted increasing attention in communities of machine learning and statistics.
To name a few, specific areas of applications include computer vision 
\citep{kulis2011what,duan2012learning},
natural language processing 
\citep{wang2011heterogeneous,prettenhofer2010cross,zhou2014hybrid},
online display web advertising 
\citep{perlich2014machine},  drug sensitivity prediction 
\citep{turki2017transfer},
clinical trial 
\citep{bellot2019boosting,wang2022robust} and so on.
See surveys of transfer learning in 
\citet{pan2010survey}, \citet{weiss2016survey}, and \citet{zhuang2021comprehensive}
for more examples and discussions.
Furthermore, 
\citet{taylor2009transfer} and 
\citet{zhu2020transfer} surveyed
work of embedding  transfer learning in reinforcement learning and deep reinforcement learning, respectively.

Due to the complexity of real applications, there may exist heterogeneity
in the target domain as well as the source domains. In addition, the
outliers or heavy tails in the observations may be encountered in some applications.
Totally ignoring or even insufficiently accounting for these potential issues may undermine the performance of
transfer learning. Quantile regression 
\citep{koenker1978regression,koenker2005quantile}
is a natural and effective tool for tackling the heterogeneity and heavy-tailed model errors
simultaneously.

\subsection{Main Contributions}
We propose a transfer learning approach
to accounting for the heterogeneity and heavy-tailed observations under the framework of
quantile regression models.  Towards
reaching these goals, we made the following specific contributions.

We develop the transfer learning procedure for high-dimensional
quantile regression models and establish  error bounds of the
transfer learning estimator, showing that sharper error bounds
can be achieved when the sample sizes of the transferable source
datasets are much larger than that of the target, a typical scenario
in the transfer learning. We further showcase that the proposed transfer learning
estimator achieves the nearly weak oracle property of variable selection procedure.

Based on the transfer learning estimator,
we construct confidence interval and hypothesis test procedures for individual component of high-dimensional quantile regression coefficients by
advocating a double transfer learning estimator,
which is the one-step debiased estimator for the transfer learning estimator
via orthogonal projection onto the high-dimensional nuisance parametric space
wherein the technique of transfer learning is designed again.
We establish  asymptotic normality of the one-step estimator, whose asymptotic variance
can be consistently estimated through histogram approximation
of the model error density, which is an inherent but unknown quantity involved in the variance of quantile regression parameter estimation.
It is common to employ the bootstrap method to estimate the variance, which however may be computationally intensive and extremely challenging to investigate the consistency of bootstrapped estimator in the regimes of transfer learning
and high dimensions. As a useful alternative, we adopt nonparametric approach to directly estimating the model error density and show the consistency of the resulting variance estimation.

It is critical in transfer learning to decide
which source datasets are transferable and thus helpful for performance enhancement
of  high-dimensional quantile regression in the target task.
We evaluate the similarity between
the target and source tasks via the difference of
their quantile regression coefficients
and obtain the transferable source datasets if the
corresponding difference is less than a delicately designed
level by temporally assuming that these coefficients were known.
Consequently, the transferable source datasets are allowed to be varying with
different quantile levels, a more flexible scenario in practice while aiming to
address the heterogeneity.
We further propose a procedure to identify the transferable source datasets.
By properly choosing the level of difference, we show the consistency of
 detection procedure and conclude that the detected transferable sources are contained
in the truly transferable ones with probability tending to one.
As a result, the proposed transfer learning procedure can be applied in practice with
theoretical guarantee.

\subsection{Related Literature}
Recently, there are substantial advances in theoretical understanding on
the transfer learning on the high-dimensional regression models.
\citet{bastani2021predicting} studied transfer learning approach to
high-dimensional linear regression models by exploiting a single source
dataset, unnecessitating the detection of transferable source datasets
among candidate ones.
\citet{li2022transferA} proposed a two-step method for estimation and prediction of transfer learning on high-dimensional linear regression models. Specifically, a detection
algorithm for consistently identifying the transferable but unknown source datasets was developed and then a contrast regularized estimator anchoring the target task was proposed.
\citet{li2022transferB} further extended the two-step method to
transfer learning on the Gaussian graphical models with false discovery rate control.
\citet{tian2022transfer} studied transfer learning on high-dimensional generalized linear models
and constructed the corresponding confidence interval for individual regression parameter.
\citet{yang2020analysis} utilized the two-layer linear neural network to
combine both source and target datasets
and derived the precise asymptotic limit of the prediction risk of transfer learning
on the high-dimensional linear models.
\citet{cai2021transfer} established an adaptive and  minimax
method for nonparametric classification in the regime of transfer learning.
\citet{xu2021learning} proposed a novel robust multi-task estimator that improves the efficiency of downstream decisions by leveraging better predictive models with lower sample complexity.
\citet{li2021targeting} proposed a two-way data fusion strategy that integrates heterogeneous data from diverse populations and from multiple healthcare institutions via a federated
transfer learning approach.
\citet{zhou2021doubly} proposed the doubly robust transfer learning
to accommodate the label scarcity and covariate shift in the binary regression
with low-dimensional predictors in the target task.
\citet{wang2022robust} studied a transfer learning approach to
time-to-event data under high-dimensional Cox proportional hazards models
by borrowing summary information from a set of health care centers without sharing patient-level information.
\citet{he2022transfer} considered transfer learning problem in estimating undirected semi-parametric graphical model.
\citet{lin2022on} studied transfer
learning on the functional linear regression models and established the optimal convergence rates for excess risk.

Despite the prevalence in  high-dimensional data,
heterogeneity and heavy tails tend to be insufficiently accounted for by current transfer learning approaches and thus may
lead to downstream performances in practice.
Quantile regression, naturally offering a convenient approach to
modelling the heterogenous impact of covariates on the
conditional distribution of response in a comprehensive way
as well as hedging against the possible outliers and heavy-tailed model errors,
is widely used in applications. On the other hand,
although there is a large body of literature on the
least squares methods with $\ell_1$-regularization 
\citep{tibshirani1996regression} in the past two decades,
the advance in  $\ell_1$-penalized high-dimensional quantile regression
is relatively lagging, mainly due to that
the piecewise linear quantile loss is non-smooth and
has curvature concentrated at the origin point and thus demanding
alternatively different analyzing tools.
\citet{belloni2011penalized} established error bounds  for
 $\ell_1$-penalized estimator of high-dimensional quantile regression models,
which is further elaborated via adaptively weighted $\ell_1$-penalty
\citep{zheng2015globally}.
\citet{alquier2019estimation} established
estimation bounds and sharp oracle inequalities of regularized procedures
with Lipschitz loss functions,
which include the piecewise linear quantile loss function as a special case.
For statistical inference and confidence interval for individual parameter
in high-dimensional quantile regression models,
\citet{bradic2017uniform} advocated a debiased method wherein the
asymptotic variance is estimated via regression rank scores from the
perspective of dual problem of quantile loss function.
\citet{belloni2019valid} developed an inference procedure
based on post-selection estimator and orthogonal score functions
to make effect of the estimation of nuisance parameter negligible.
This work aims to provide thorough studies on the transfer learning
for high-dimensional quantile regression models, including
the error bounds for prediction and asymptotic distribution theory for inference.

\subsection{Organization and Notation}
Section \ref{sec2} introduces the transfer learning
method for  high-dimensional quantile regression models when the transferable source datasets are known and establishes error bounds of the transfer learning estimator.
We construct confidence interval and hypothesis test procedures in
Section \ref{sec3} and propose the procedure to identify the transferable source datasets in
Section \ref{sec4}. We conduct the extensive simulations in Section \ref{sec5}.
 A real example is analyzed in Section \ref{sec6}, and
Section \ref{sec7} concludes with some remarks.
In the Supplementary Material,
we gather the proofs of the theoretical results,  supporting lemmas,
in Appendices \ref{appendixA}--\ref{appendixC},
and report additional simulation results
in Appendix \ref{appendixD}.

For ease of exposition, let $\ell_0$-norm $\vert\cdot\vert_0$ denote the number of nonzero components of a vector, and $\vert\cdot\vert_r$  denote the $\ell_r$-norm with $1\le r\le\infty$.
Let $\Vert\cdot\Vert_1$ and $\Vert\cdot\Vert_{\infty}$ denote the $1$-norm and $\infty$-norm of a matrix, respectively,
and let $\Vert\cdot\Vert_{\max}$ denote the maximum absolute value among the entries of a matrix.
Set the operators $\lambda_{\min}(\cdot)$ and $\lambda_{\max}(\cdot)$ as the minimum and maximum eigenvalues of a symmetric matrix, respectively.
For a vector $\bv$ and an index set $\mathcal{I}$,
let $\bv_{\mathcal{I}}$ denote the subvector consisting of the components of $\bv$ with their indices in $\mathcal{I}$.
Similarly, for a matrix $\bA$ and two index sets $\mathcal{I}_1$ and $\mathcal{I}_2$,
let $\bA_{\mathcal{I}_1\mathcal{I}_2}$ denote the submatrix consisting of the rows and columns of $\bA$ with their indices in $\mathcal{I}_1$
and $\mathcal{I}_2$, respectively.
In addition, for any vector $\bv$, let $v_j$ denote the $j$-th component of
$\bv$ and $\bv_{-j}$ denote the subvector formed by the remaining components.
For a matrix $\bA$, we can analogously
define $A_{j_1,j_2},\ \bA_{j_1,-j_2},\ \bA_{-j_1,j_2}$ and $\bA_{-j_1,-j_2}$
for a row index $j_1$ and a column index $j_2$.
Let $\Gamma_{k,\theta}(\bv)$ denote the vector obtained by inserting a scalar $\theta$
between the $(k-1)$-th and $k$-th components of a vector $\bv$ whereas the trivial case with $k$ equal to $1$ or the length of $\bv$ can be accordingly defined.
For two positive series $d^{\dag}_n$ and $d^{*}_n$,
denote $d^{\dag}_n\lesssim d^{*}_n$ if there exists a universal positive constant $c>0$ such that
$d^{\dag}_n\le c d^{*}_n$. Denote $d^{\dag}_n\asymp d^{*}_n$ if
$d^{\dag}_n\lesssim d^{*}_n$ and $d^{*}_n\lesssim d^{\dag}_n$.
Set $I(\cdot)$ as the indicator function.
With a slight abuse of notation, let $\vert \mathcal{A}\vert$ denote the  cardinality of
set $\mathcal{A}$.
Let $\Phi(\cdot)$ denote the cumulative distribution
function of the standard normal variable, $\Phi^{-1}(\cdot)$  the inverse function of $\Phi(\cdot)$, and $\phi(\cdot)$ the
corresponding probability density function.  Let $\xrightarrow{\mathcal{L}}$ denote the convergence in distribution.

\section{Transfer Learning for Quantile Regression}
\label{sec2}
In the target population,
let $Y_0$ denote the response variable of interest which is associated with a
$p$-vector  covariate $\bZ_0$. For $\tau\in (0,1)$ fixed,
the $\tau$-th conditional quantile function of the response variable $Y_0$ given $\bZ_0$ is defined by
$$ Q_{\tau}(Y_0\mid \bZ_0)=\inf\{t\in\mbR: P(Y_0\leq t\mid \bZ_0)\geq \tau\}.$$
We consider the linear quantile regression model that takes the form
$$Q_{\tau}(Y_0\mid \bZ_0)=\bZ_0^{\rmT}\bbeta_{0},$$
where $\bbeta_{0}$ is the unknown regression coefficient of interest.
For notational simplicity,
we suppress the dependence of quantile regression parameter $\bbeta_{0}$
on the specific quantile level $\tau$ throughout the paper but we should keep in mind that
related statistical conclusion is $\tau$-dependent, thus providing comprehensive and
global interpretation by varying $\tau$ over $(0,1)$.
Assume that the target samples  $(Y_{0i}, \bZ_{0i}), i=1,\ldots,n_0,$
are independent and identically distributed (i.i.d.) according to the target
population $(Y_0, \bZ_0)$.

We consider the multiple source transfer learning problem based on
quantile regression for the heterogeneous datasets.
Specifically, for $k=1,\ldots K$,
let $(Y_k, \bZ_k)$ denote the $k$-th source population  and its i.i.d. copies
with size of $n_k$ be denoted by $(Y_{ki}, \bZ_{ki}), i=1,\ldots,n_k.$
We further consider the linear quantile regression model for the source populations
$$Q_{\tau}(Y_k\mid \bZ_k)=\bZ_k^{\rmT}\bbeta_k,$$
where $\bbeta_k$  is the $p$-vector of unknown regression coefficients.
In the  high-dimensional realm,  the dimension $p$ is much larger than the sample sizes $n_k, k=0,1,\ldots, K$.
Let $\bbeta_{0}=(\beta_{01},\ldots,\beta_{0p})^{\rmT}$ and
$\mcS_0=\{j:\beta_{0j}\neq 0, 1\le j\le p\}$.
The cardinality of set $\mcS_0$ is denoted by $s_0$, i.e., $\vert \mcS_0\vert=\vert\bbeta_{0}\vert_0=s_0.$

If the $k$-th source quantile regression model enjoys a certain level of similarity to the target one, it may have the benefit of accommodating the target task
via properly transferring and exploiting the $k$-th source dataset.
We consider the sparsity of the difference between $\bbeta_{0}$ and $\bbeta_{k}$
as a criterion to capture the informative level of the $k$-th source.
Specifically, let the $k$-th contrast be
$\bdelta_{k}=\bbeta_{0}-\bbeta_{k},$
and
the $k$-th source data is viewed as $h$-transferable  if
$\vert\bmd_{k}\vert_1\leq h,$ where $h$ is a non-negative constant to be specified.
Denote the index set of all $h$-transferable source datasets as $\mcT_h=\{k:  \vert\bdelta_{k}\vert_1\leq h, 1\leq k\leq K\}$. In general,
if $h$ is taken to be large, negative transfer tends to happen as some
non-transferable sources may be recruited. On the contrary,
more restrictive $h$ takes the risk of excluding some transferable sources.
Consequently, striking a balance between $h$ and $|\mcT_h|$ is required by properly choosing $h$. 

For ease of exposition, we define a total loss function for the  combined datasets
as
\begin{eqnarray*}
 L(\bbeta; \mcT)=\frac{1}{n_{\mcT}}\sum_{k\in{\mcT}}\sum\limits_{i=1}^{n_k}\rho_{\tau}(Y_{ki}-\bZ_{ki}^{\rmT}\bbeta)
\end{eqnarray*}
for any set $\mcT\subset\{0,1,\ldots,K\}$,
where $n_{\mcT}=\sum_{k\in{\mcT}}n_k$ and $\rho_{\tau}(t)=t\{\tau-I(t\leq0)\}$ for $t\in \mathbb{R}$ is the piecewise linear quanitle loss function.


We first introduce an oracle transfer learning estimator when $\mcT_h$ is temporarily assumed to be known. It is motivated by the algorithms in 
\citet{bastani2021predicting}, \citet{li2022transferA}, and \citet{tian2022transfer}.
Specifically, we first develop an initial $\ell_1$-penalized quantile regression estimator
which is defined by
\begin{eqnarray}
\label{op_trans}
\wh\bbeta_{\mcT_h}=\arg\min_{\bbeta}\left\{L(\bbeta;\{0\}\cup\mcT_h)+
\lambda_{\bbeta}\vert\bbeta\vert_1\right\},
\end{eqnarray}
where $\lambda_{\bbeta}$ is the tuning parameter. Although $\wh\bbeta_{\mcT_h}$
is very crude estimator, it is obtained by combining the target and $h$-transferable
source datasets and beneficial to estimating the target model parameter if properly utilized.  Second, we adopt the empirical quantile loss for the target data only to
make $\wh\bbeta_{\mcT_h}$ anchor to the target via correcting the contrast as follows,
\begin{eqnarray}
\label{op_debias}
\hd_{\mcT_h}=\arg\min_{\bmd}\left\{L\left(\wh\bbeta_{\mcT_h}+
\bmd;\{0\}\right)+\lambda_{\bmd}\vert\bmd\vert_1\right\},
\end{eqnarray}
where $\lambda_{\bmd}$ is the tuning parameter.
Thus, we propose the transfer learning estimator for $\bbeta_0$ in the target quantile regression model by using $\wh\bbeta_0=\wh\bbeta_{\mcT_h}+\hd_{\mcT_h}$.
Before we establish the convergence rate of $\wh\bbeta_0$, we need to
examine to which $\wh\bbeta_{\mcT_h}$ should converge and what the rate is.
Based on the estimating equation theory and by choosing suitable tuning parameter $\lambda_{\bbeta}$, the estimator $\wh\bbeta_{\mcT_h}$ in (\ref{op_trans}) converges to its probabilistic limit $\bbeta^*_{\mcT_h}$, which is defined
via the expected subgradient equation
\begin{eqnarray}
\label{alpha_Th}
E\left\{\bS(\bbeta^*_{\mcT_h};\{0\}\cup\mcT_h)\right\}={\bf 0},
\end{eqnarray}
where
\begin{eqnarray*}
\bS(\bbeta; \mcT)=\nabla L(\bbeta; \mcT)=-\frac{1}{n_{\mcT}}\sum_{k\in \mcT}\sum_{i=1}^{n_k}\bZ_{ki}\left\{\tau-I\left(Y_{ki}-\bZ_{ki}^{\rmT}\bbeta\le 0\right)\right\}
\end{eqnarray*}
for any $\mcT\subset\{0,\ldots,K\}$.

We study the rate of convergence for the estimator $\wh\bbeta_{\mcT_h}$.
Let $F_k(y\mid \bz)$ and $f_k(y\mid \bz)$ denote the conditional distribution and density functions
of $Y_k$ given $\bZ_k=\bz$ for $k=0,\ldots, K$, respectively. The domain of $\bZ_k$ is $\mcZ_k$.
Let $\dot f_k(y\mid \bz)$ denote the derivative of $f_k(y\mid \bz)$ with respect to $y$.
\begin{assumption}\label{C1}
There exists a positive constant $m_0$ such that
$$\max_{0\le k\le K}\left\{\sup_{y\in\mbR,\mathbf{z} \in \mcZ_k} f_k(y\mid \mathbf{z} ),\sup_{y\in\mbR,\mathbf{z} \in \mcZ_k}\left\vert
\dot f_k(y\mid \bz)\right\vert \right\}\le m_0.$$
\end{assumption}

\begin{assumption}\label{C2}
$\left\vert \bZ_k\right\vert_{\infty}$ is bounded by some constant $B$ with probability one
and the largest eigenvalue of $E\left(\bZ_{k}\bZ^{\rmT}_{k}\right)$ is also bounded by $B^2$
for $k=0,\ldots,K.$
\end{assumption}

\begin{assumption}\label{C3}
There exists a positive constant $\kappa_0$  such that
\begin{eqnarray*}
\inf_{\bdelta\in\mathcal{D}(a_n,\mcS)} \frac
{\bdelta^\rmT E\left\{f_k(\bZ_k^{\rmT}\bbeta_k\mid \bZ_k)\bZ_k\bZ_k^{\rmT}\right\}\bdelta}{\vert\bdelta\vert^2_2}\ge \kappa_0
{\ \rm and \ }
\inf_{\bdelta\in\mathcal{D}(a_n,
\mcS)} \frac
{ \vert\bdelta\vert_2^3}{E\left(\vert\bZ_k^{\rmT}\bdelta\vert^3\right)}\ge \kappa_0,
\end{eqnarray*}
for $k=1,\ldots,K$ with some positive quantity  $a_n$
and index set $\mcS$  to be specified, where
$$ \mathcal{D}(a_n,
\mcS)=\left\{\bu:\vert\bu_{\mcS^c}\vert_1\le 3\vert\bu_\mcS\vert_1+a_n\right\}. $$
\end{assumption}

\begin{assumption}\label{C4}
Set $\pi_k=n_k/(n_{\mcT_h}+n_0)$ and
denote
$$ \wt \bSigma_{\mcT_h}=\sum_{k\in\{0\}\cup\mcT_h}\pi_kE\left[\int^1_0f_{k}
\left(\bZ_{k}^{\rmT}\bbeta_0+t\bZ_{k}^{\rmT}(\bbeta^*_{\mcT_h}-\bbeta_0)\mid \bZ_{k}\right)\rmd t\bZ_{k}\bZ_{k}^{\rmT}\right] $$
and
$$ \wt \bSigma_k=E\left[\int^1_0f_k\left(\bZ_{k}^{\rmT}\bbeta_0+t\bZ_{k}^{\rmT}(\bbeta_k-\bbeta_0)\mid \bZ_{k}\right)
\rmd t\bZ_{k}\bZ_{k}^{\rmT}\right]. $$
Assume that $\sup_{k\in\mcT_h}\Vert\wt\bSigma_{\mcT_h}^{-1}\wt\bSigma_k\Vert_1\le {C}$
for some constant $C>0$.
\end{assumption}

Assumption \ref{C1} concerns the boundedness and smoothness of density function of model error, which
is satisfied with a large class of distributions,
including the light-tailed normal distributions, the
 heavy-tailed Cauchy distributions, and the skewed extreme value distributions.
Assumption \ref{C2} states the boundness of covariates in the target and source populations,
a common and just a technique condition in high-dimensional data analysis 
\citep{ning2017general,yu2021confidence}.
It is can be
replaced with a tail condition by assuming the sub-Gaussian covariates
\citep{li2022transferA} with extra derivations involved.
Indeed, in our simulations in Section \ref{sec5}, we explore setting that
$\bZ_k$ is multivariate normal distribution and thus
 $\left\vert \bZ_{k}\right\vert_{\infty}$ is
not bounded with probability one.
Assumption on the bounded maximum eigenvalue of covariance matrix is a
basic assumption in high-dimensional data analysis.
Assumption \ref{C3} imposes a generalized version of restricted eigenvalue condition and restricted nonlinear impact
in a more flexible form. Distinctive index set $\mcS$ and $a_n$ are adaptive to the theoretical results of estimators in different steps.
If $a_n\to 0$ as $n_0,n_k, p\to\infty$ for $k\in\{1,\ldots, K\}$,
 assumption \ref{C3} approximately mimics  the classical restricted eigenvalue condition and restricted nonlinear impact. In addition, we use common lower bounds $\kappa_0$
  and upper bounds $B$ as well for notational simplicity.
Assumption \ref{C4} makes sure that the probabilistic limit $\bbeta^*_{\mcT_h}$ falls in the $\ell_1$-ball centered
at $\bbeta_0$ with radius $Ch$. In the case of linear mean model with homogeneous design $\bZ_{k}$, it
can be simplified to the non-singularity of the covariance matrix of covariates, a very mild assumption. While considering the cases of linear mean regression  
\citep{li2022transferA}, generalized linear model 
\citep{tian2022transfer}, or quantile regression model with heterogeneous
covariates,
assumption \ref{C4} is needed to characterize the differences between the target and source covariates if a faster convergence rate is desired,
which is measured by the constant $C$ in the considered quantile regression model.
Unless otherwise stated, the limits are taken as $(n_0, n_1, \ldots, n_K, p)\to \infty$.

\begin{theorem}
\label{theorem2.1}
{\rm (Convergence rate of $\wh\bbeta_{\mcT_h}$)}
Under assumptions \ref{C1}--\ref{C4} with $a_n=4Ch$ and $\mcS=\mcS_0$ in assumption \ref{C3},
if $\lambda_{\bbeta}=C_{\bbeta}B\sqrt{\log p/(n_{\mcT_h}+n_0)}$ with $C_{\bbeta}\ge 4$ and
$$ \sqrt{\frac{s_0\log p}{n_{\mcT_h}+n_0}}+\sqrt{h}\left(\frac{\log p}{n_{\mcT_h}+n_0}\right)^{1/4}\to 0, $$
then
\begin{eqnarray*}
\left\vert\wh\bbeta_{\mcT_h}-\bbeta^*_{\mcT_h}\right\vert_2\lesssim \sqrt{\frac{s_0\log p}{n_{\mcT_h}+n_0}}+\sqrt{h}\left(\frac{\log p}{n_{\mcT_h}+n_0}\right)^{1/4}
\end{eqnarray*}
and
\begin{equation*}
\left\vert\wh\bbeta_{\mcT_h}-\bbeta^*_{\mcT_h}\right\vert_1
\lesssim  s_0\sqrt{\frac{\log p}{n_{\mcT_h}+n_0}}+h
\end{equation*}
with probability tending to one.
\end{theorem}

We propose the estimator $\wh\bbeta_{\mcT_h}$ in (\ref{op_trans}) by leveraging
 the target dataset as well as the $h$-transferable source ones.
Its probabilistic limit $\bbeta^*_{\mcT_h}$ is a pooled version of parameters $\{\bbeta_k\}_{k\in\{0\}\cup \mcT_h}$.
Even if $h$ could be small sufficiently, the inevitable bias arises from the source datasets, which is sparsely corrected
in (\ref{op_debias}). The error bounds of the resulting transfer learning estimator $\wh\bbeta_0$ are established as follows.

\begin{theorem}
\label{theorem2.2}
{\rm (Convergence rate of $\wh\bbeta_0$ with assumption \ref{C4})} Set
$\lambda_{\bbeta}=C_{\bbeta}B(n_{\mcT_h}+n_0)^{-1/2}\sqrt{\log p}$ and $\lambda_{\bdelta}=C_{\bdelta}B\sqrt{\log p/n_0}$
with $C_{\bbeta}\ge4$ and $C_{\bdelta}\ge4$.
Under assumptions \ref{C1}--\ref{C4}
with
$$ a_n=20 Ch+
\frac{(192+72C_{\bbeta}) Bs_0}{\kappa_0}\sqrt\frac{\log p}{n_{\mcT_h}+n_0}+24 \sqrt{\frac{(4+2C_{\bbeta})BChs_0}{\kappa_0}}\left(\frac{\log p}{n_{\mcT_h}+n_0}\right)^{1/4} $$
and $\mcS=\mcS_0$ in assumption \ref{C3},
if
$$ \sqrt{\frac{s_0\log p}{n_{\mcT_h}+n_0}}+\sqrt{h}\left(\frac{\log p}{n_{\mcT_h}+n_0}\right)^{1/4}\to 0 $$
and
$$ \sqrt{\frac{\log p}{n_0}}\left(s_0\sqrt{\frac{\log p}{n_{\mcT_h}+n_0}}+h\right)\to 0, $$
then
\begin{eqnarray}
\label{beta-l2-bound}
\left\vert\wh\bbeta_0-\bbeta_0\right\vert_2\lesssim \sqrt h\left(\frac{\log p}{n_0}\right)^{1/4}+\left(\frac{s_0\log p}{n_0}\right)^{1/4}\left(\frac{s_0\log p}{n_{\mcT_h}+n_0}\right)^{1/4}
\end{eqnarray}
and
\begin{equation}
\label{beta-l1-bound}
\left\vert\wh\bbeta_0-\bbeta_0\right\vert_1
\lesssim  s_0\sqrt{\frac{\log p}{n_{\mcT_h}+n_0}}+h,
\end{equation}
with probability tending to one.
\end{theorem}

Actually, assumption \ref{C4} is not a necessary condition to establish the convergence of $\wh\bbeta_0$.
When assumption \ref{C4} is absent, we also show the convergence with the sacrifice of
a slower convergence rate and a potential larger order of the tuning parameter.
The reason lies in the circumvention of the pooled version of the parameter $\bbeta^*_{\mcT_h}$, since it
may no longer be close to $\bbeta_0$ enough.
The specific statement is concluded as follows.

\begin{theorem}
\label{theorem2.3}
{\rm (Convergence rate of $\wh\bbeta_0$ without assumption \ref{C4})}
Under assumptions \ref{C1}--\ref{C3} with
\begin{eqnarray*}
a_n&=&4 h+
\frac{1344 Bs_0}{\kappa_0}\sqrt\frac{\log p}{n_{\mcT_h}+n_0}
 +48\sqrt{\frac{3Bhs_0}{\kappa_0}}\left(\frac{\log p}{n_{\mcT_h}+n_0}\right)^{1/4}
 \\&&+\frac{1152 m_0B^2s_0h}{\kappa_0}+48\sqrt{\frac{2m_0s_0}{\kappa_0}}Bh
\end{eqnarray*}
and $\mcS=\mcS_0$ in assumption \ref{C3},
if $\lambda_{\bbeta}=4B\sqrt{\log p/(n_{\mcT_h}+n_0)}+4m_0B^2h$, $\lambda_{\bdelta}=4B\sqrt{\log p/n_0}$,
$$ \sqrt{\frac{s_0\log p}{n_{\mcT_h}+n_0}}+\sqrt{s_0}h\to 0 $$
and
$$ \sqrt{\frac{\log p}{n_0}}\left(\sqrt{\frac{s_0\log p}{n_{\mcT_h}+n_0}}+\sqrt{s_0}h\right)\to 0, $$
then
\begin{eqnarray}
\label{beta-l1-bound*}
\left\vert\wh\bbeta_0-\bbeta_0\right\vert_1\lesssim s_0\sqrt{\frac{\log p}{n_{\mcT_h}+n_0}}+s_0h
\end{eqnarray}
and
\begin{equation}
\label{beta-l2-bound*}
\left\vert\wh\bbeta_0-\bbeta_0\right\vert_2\lesssim
\sqrt {s_0h}\left(\frac{\log p}{n_0}\right)^{1/4}+\left(\frac{s_0\log p}{n_0}\right)^{1/4}\left(\frac{s_0\log p}{n_{\mcT_h}+n_0}\right)^{1/4}
\end{equation}
with probability tending to one.
\end{theorem}

It is worth  noting that when $h=o\left(s_0\sqrt{\log p/n_0}\right)$ and $n_0=o\left(n_{\mcT_h}\right)$, the results in Theorems \ref{theorem2.2}
indicate faster convergence rates than the error bounds of the classical
$\ell_1$-penalized quantile regression parameter estimator without help of the source datasets 
\citep{belloni2011penalized}.
Moreover, when $h=o\left(\sqrt{\log p/n_0}\right)$ and $n_0=o\left(n_{\mcT_h}\right)$,
Theorems \ref{theorem2.3} elucidates that
the source datasets truly improve the theoretical performance of the transfer learning estimator even assumption \ref{C4} is not imposed. Obviously,
 the rate enhancement is mainly due to the size of transferable source datasets $n_{\mcT_h}$
 which can offset the similarity level $h$. In addition,
 $n_0=o\left(n_{\mcT_h}\right)$ is a typical scenario
in transfer learning and the rate of $h$ can be tuned.

We further investigate the nearly weak oracle property of the  proposed quantile regression transfer learning estimator $\wh \bbeta_0$.
We introduce the population versions for the Hessian-type matrices of the loss functions.
Denote
\begin{eqnarray*}
\bH_0(\bbeta)=\frac{\partial}{\partial \bbeta}E\left\{\bS(\bbeta;\{0\})\right\}
=E\left\{f_0\left(\bZ_{0}^\rmT\bbeta\mid \bZ_{0}\right)\bZ_{0}\bZ_{0}^\rmT\right\}.
\end{eqnarray*}
Analogously,  denote
\begin{eqnarray*}
\bH_{k}(\bbeta)
=E\left\{f_k\left(\bZ_{k}^\rmT\bbeta\mid \bZ_{k}\right)\bZ_{k}\bZ_{k}^{\rmT}\right\},\ k=1,\ldots,K
\end{eqnarray*}
and
\begin{eqnarray*}
\bH_{\mcT_h}(\bbeta)=\sum_{k\in\{0\}\cup\mcT_h}\pi_k\bH_k(\bbeta).
\end{eqnarray*}
For ease of exposition, briefly denote $\bH^{*}_0=\bH_0\left(\bbeta_0\right)$ and
$\bH^*_{\mcT_h}=\bH_{\mcT_h}\left(\bbeta^*_{\mcT_h}\right)$.
Some regular assumptions are required to support the nearly weak oracle property.

\begin{assumption}\label{C5}
There exists an index set $\mcS_{\mcT_h}\subset\{1,\ldots,p\}$ such that
$\left\vert\mcS_{\mcT_h}\right\vert\le s'$
and $\left\vert\bbeta^*_{\mcT_h,\mcS_{\mcT_h}^c}\right\vert_1\le h'$ with $h'=o(1)$.
\end{assumption}
\begin{assumption}\label{C6}
Assume that
$$ \max\left\{\left\Vert\bH^*_{0,\bar{\mcS}^c\bar{\mcS}}\left(\bH^{*}_{0,\bar{\mcS}\bar{\mcS}}\right)^{-1}\right\Vert_\infty,
\left\Vert\bH^{*}_{\mcT_h,\mcS_{\mcT_h}^c\mcS_{\mcT_h}}\left(\bH^{*}_{\mcT_h,\mcS_{\mcT_h}\mcS_{\mcT_h}}\right)^{-1}\right\Vert_\infty\right\}< 1-\gamma $$
for some constant $\gamma\in(0,1)$, where  $\bar \mcS=\mcS_{\mcT_h}\cup\mcS_0$.
\end{assumption}

\begin{assumption}\label{C7}
Assume that
$$ \sqrt{s_0}\left(\frac{\log p}{n_0}\right)^{1/4}\max\left\{h^2,\frac{\log p}{n_{\mcT_h}+n_0}\right\}^{1/4}=o\left(\min_{j\in\mcS_0}\vert\beta_{0j}\vert\right). $$
\end{assumption}

Assumption \ref{C5} supposes the approximate sparsity of the pooled parameter $\bbeta^*_{\mcT_h}$.
Lemma \ref{lemmaA1} 
in the Supplementary Material
entails that the $\ell_1$-norm of $\bbeta^*_{\mcT_h}$ is always bounded under
assumption \ref{C4}, and thereby we can always find a decomposition in assumption \ref{C5}.
In general, $h'$ is supposed to converge at a faster rate and $s'$ is supposed  to be not too large in order
to achieve the nearly weak oracle property.
Assumption \ref{C6} states the strong irrepresentable condition 
\citep{zhao2006on}
in the framework of high-dimensional quantile regression
for transfer learning procedure.
The minimum signal strength in assumption \ref{C7} is a common assumption
in high-dimensional signal detection.

\begin{theorem}
\label{theorem2.4}
Under assumptions \ref{C1}--\ref{C3} and \ref{C5}--\ref{C6},
with $$ a_n=4h',\ \mcS=\mcS_{\mcT_h} $$
and $$ a_n=0,\ \mcS=\bar{\mcS} $$
in assumption \ref{C3},
if
$$ \lambda_{\bbeta}\ge \left\{\frac{8D}{\gamma}\sqrt{\frac{s'}{\log p}}+ \frac{32\sqrt{2}}{\gamma}\right\}B\sqrt{\frac{\log p}{n_{\mcT_h}+n_0}} $$
and
$$ \lambda_{\bdelta}\ge \left\{\frac{6\wt D}{\gamma}\sqrt{\frac{s_0+s'}{\log p}}+ \frac{24\sqrt{2}}{\gamma}\right\}B\sqrt{\frac{\log p}{n_0}}, $$
where the positive constants $D$ and $\wt D$ are respectively determined by
 (\ref{determineD}) 
and
 (\ref{determineD_deb}) 
in the Supplementary Material,
$$ \sqrt{s'\lambda_{\bbeta}}\to0,\ \frac{s'\log p}{\lambda_{\bbeta}\left(n_{\mcT_h}+n_0\right)}\to 0,\ \sqrt{(s'+s_0)\lambda_{\bdelta}}\to0, \frac{(s'+s_0)\log p}{\lambda_{\bdelta}n_0}\to 0,
\ {\rm and}\ h'<\frac{\gamma\lambda_{\bbeta}}{8m_0B^2},$$
then
$\wh\bbeta_{0,\bar{\mcS}^c}={\bf 0}$
holds with probability more than $1-24p^{-1}-4p^{-2}$.
\end{theorem}

In viewing conditions in Theorem \ref{theorem2.4}, $h'$ in assumption \ref{C5} is supposed to be dominated by the tuning parameter $\lambda_{\bbeta}$.
As a trade-off, the approximate sparsity of the pooled parameter $\bbeta_{\mcT_h}^*$ is supposed  to be not too large
by restricting that $\sqrt{s'\lambda_{\bbeta}}\to0$.
If we take
$$\lambda_{\bbeta}=C_{\bbeta}B\sqrt{\frac{\log p}{n_{\mcT_h}+n_0}}\ {\rm and}\
\lambda_{\bdelta}=C_{\bdelta}B\sqrt{\frac{\log p}{n_0}} $$
for proper constants $C_{\bbeta}$ and $C_{\bdelta}$,  conditions on tuning parameters and
approximate sparse level in Theorem \ref{theorem2.4} can be satisfied.

\begin{theorem}
\label{theorem2.5}
Under assumptions \ref{C1}--\ref{C3} and \ref{C7} with
\begin{eqnarray*}
a_n&=&4 h+
\frac{1344 Bs_0}{\kappa_0}\sqrt\frac{\log p}{n_{\mcT_h}+n_0}
 +48\sqrt{\frac{3Bhs_0}{\kappa_0}}\left(\frac{\log p}{n_{\mcT_h}+n_0}\right)^{1/4}
 \\&&+\frac{1152 m_0B^2s_0h}{\kappa_0}+48\sqrt{\frac{2m_0s_0}{\kappa_0}}Bh
\end{eqnarray*}
and $\mcS=\mcS_0$ in assumption \ref{C3},
if $\lambda_{\bbeta}=4B\sqrt{\log p/(n_{\mcT_h}+n_0)}+4m_0B^2h$, $\lambda_{\bdelta}=4B\sqrt{\log p/n_0}$,
$$ \sqrt{\frac{s_0\log p}{n_{\mcT_h}+n_0}}+\sqrt{s_0}h\to 0 $$
and
$$ \sqrt{\frac{\log p}{n_0}}\left(\sqrt{\frac{s_0\log p}{n_{\mcT_h}+n_0}}+\sqrt{s_0}h\right)\to 0, $$
then
\begin{eqnarray*}
P\left(\mcS_0\subset\wh\mcS_0\right)\ge 1-2p^{-1}-2p^{-2},
\end{eqnarray*}
where $\wh\bbeta_0=(\wh\beta_{01},\ldots,\wh\beta_{0p})^\rmT$ and $\wh\mcS_0=\{j:\wh\beta_{0j}\neq 0\}$.
\end{theorem}

Combining  Theorems \ref{theorem2.4} and \ref{theorem2.5},
it immediately holds
$$ \mcS_0\subset \wh\mcS_0\subset \mcS_0\cup\mcS_{\mcT_h} $$
with probability tending to one, achieving the
nearly weak oracle property of variable selection in the proposed
transfer learning procedure.
This indicates that the performance of the recovery is influenced by the
parameters of the $h$-transferable source datasets.
It is worth noting that if the structure of the $h$-transferable source datasets is similar to that of the target
such that $\mcS_{\mcT_h}\subset \mcS_0$, the weak oracle property can be eventually attained.

\section{Confidence Interval and Hypothesis Test}
\label{sec3}
Based on the transfer learning estimator $\wh\bbeta_0$, we would like to propose a
statistical inference procedure
to construct confidence interval and hypothesis test of the $\beta_{0m}$,
the $m$-th quantile regression coefficient of $\bbeta_0$.
Without loss of generality, rewrite  $\bbeta_0=(\beta_{0m},\bbeta_{0,-m}^\rmT)^\rmT$
and accordingly $\bS(\bbeta_0; \{0\})=(S_m(\bbeta_0;\{0\}),\bS_{-m}^\rmT(\bbeta_0;\{0\}))^\rmT$.
Viewing $\bbeta_{0,-m}$ as the nuisance parameter, an orthogonal function employed to construct the confidence interval
can be derived via the theory of semiparametric efficiency theory.
By removing the linear projection of $S_m(\bbeta_0;\{0\})$ onto the plane spanned by $\bS_{-m}(\bbeta_0;\{0\})$,
the linear orthogonal score function is denoted by
\begin{eqnarray}\label{orthogonal}
S^{\rmo}_{0m}(\beta_{0m}\mid \bbeta_{0,-m})=S_m(\bbeta_0;\{0\})-\bS^{\rmT}_{-m}(\bbeta_0;\{0\})\bgamma_{0m},
\end{eqnarray}
where $\bgamma_{0m}$ is the projection direction at the population level to be determined as follows.
The Neyman orthogonal condition requires the pathwise derivative of $S^{\rmo}_{0m}(\beta_{0m}\mid \bbeta_{0,-m})$ vanishes to zero
along the submodel $t\mapsto \bbeta_{0,-m}+t(\bbeta_{-m}-\bbeta_{0,-m})$
 for any $\bbeta_{-m}\in\mathbb{R}^{p-1}$ to ensure the effect of nuisance parameter negligible,
which is given by
\begin{eqnarray*}
\frac{\partial}{\partial t}E\left\{S^{\rmo}_{0m}(\beta_{0m}\mid \bbeta_{0,-m}+t(\bbeta_{-m}-\bbeta_{0,-m}))\right\}\big\vert_{t=0}=0.
\end{eqnarray*}
As a result, we obtain
\begin{eqnarray}\label{proj}
\bgamma_{0m}=\left(\bH_{0,-m,-m}^{*}\right)^{-1}\bH_{0,-m,m}^{*}.
\end{eqnarray}
Set $\wh e_{ki}=Y_{ki}-\bZ_{ki}^{\rmT}\wh\bbeta_0,\ i=1,\ldots,n_k, \ k\in\{0\}\cup\mcT_h$
as the residuals for the fitted values of the transfer learning estimator $\wh\bbeta_0$. Set $\wh\be_k=(\wh e_{k1},\ldots, \wh e_{kn_k})^{\rmT}$.
Let  $\wh{\rm var}(\wh \be_k)$ and $\wh Q_\tau(\wh \be_k)$ denote the sample variance and the lower $\tau$-th sample quantile of $\{\wh e_{ki}\}_{i=1}^{n_k}$,
respectively.
We employ the approximation technique with the Powell bandwidth $b_{k}$
\citep{koenker2005quantile,dai2021inference} to estimate $\bH_k(\bbeta_0), k\in\{0\}\cup\mcT_h$, by
\begin{eqnarray*}
\label{hat_Hk}
\wh \bH_k=\frac1{n_k}\sum_{i=1}^{n_k}\wh f_{ki}\left(\wh\bbeta_0\right) \bZ_{ki}\bZ_{ki}^{\rmT},
\end{eqnarray*}
where
$\wh f_{ki}(\bbeta)=
{I\left(\vert Y_{ki}-\bZ_{ki}^{\rmT}\bbeta\vert\le b_k\right)}/{2b_k}
$
and $b_k$ is the bandwidth. We choose
\begin{eqnarray*}
b_k&=&\left(\Phi^{-1}(\tau+\wt b_k)-\Phi^{-1}(\tau-\wt b_k)\right)\min\left\{\sqrt{\wh{\rm var}(\wh \be_k)},\frac{\wh Q_{0.75}(\wh \be_k)-\wh Q_{0.25}(\wh \be_k)}{1.34}\right\},\\
\wt b_k&=&n_k^{-1/3}\{\Phi^{-1}(0.975)\}^{2/3}
\left(\frac{1.5\{\phi(\Phi^{-1}(\tau))\}^2}{2\{\Phi^{-1}(\tau)\}^2+1}\right)^{1/3}
\end{eqnarray*}
in numerical studies.

Set $ \wh\bH_{\mcT_h}=\sum_{k\in\{0\}\cup\mcT_h}\pi_k\wh\bH_{k}$,
which can be viewed as the combined estimator for $\bH_{0}^{*}$ by leveraging the
transferable source datasets.
Mimicking the transfer learning approach and motivated by definition of projection
direction in (\ref{proj}), we first proposed a crude estimator for $\bgamma_{0m}$,
which is given by
\begin{eqnarray}
\label{gamma_trans}
\wh\bgamma_{\mcT_h,m}=\arg\min_{\bgamma\in\mathbb{R}^{p-1}}
\left\{\frac12\bgamma^\rmT\wh\bH_{\mcT_h,-m,-m}\bgamma-
\wh\bH_{\mcT_h,-m,m}\bgamma+\lambda_m\vert\bgamma\vert_1\right\},
\end{eqnarray}
where $\lambda_m$ is the tuning parameter.
Second, the refined estimator is defined by
$\wh\bgamma_{0,m}=\wh\bgamma_{\mcT_h,m}+\wh\bzeta_m$, where
\begin{eqnarray}
\label{gamma_debias}
&&\hspace{0.2in}\wh\bzeta_m
\\&=&\arg\min_{\bzeta\in\mathbb{R}^{p-1}}
\left\{\frac1{2}(\wh\bgamma_{\mcT_h,m}+\bzeta)^\rmT\wh\bH_{0,-m,-m}
(\wh\bgamma_{\mcT_h,m}+\bzeta)-\wh\bH_{0,-m,m}(\wh\bgamma_{\mcT_h,m}+\bzeta)
+\lambda_m'\vert\bzeta\vert_1\right\}\nonumber
\end{eqnarray}
with tuning parameter $\lambda_m'$. Optimization problems
(\ref{gamma_trans}) and (\ref{gamma_debias}) can be solved efficiently
using the traditional lasso approach.
Before establishing the convergence rate of $\wh\bgamma_{0,m}$, we show the convergence of $\wh\bH_{\mcT_h}$ to $\bH_{\mcT_h}(\bbeta_0)$ as follows.
\begin{theorem}
\label{theorem3.1}
Under conditions in Theorem \ref{theorem2.4} and assumption \ref{C4},
it holds that
\begin{eqnarray*}
&&\left\Vert\wh \bH_{\mcT_h}-\bH_{\mcT_h}(\bbeta_0)\right\Vert_{\max}
\\&=&O_P\left(\sum_{k\in \{0\}\cup\mcT_h}\pi_k\Bigg\{
b_k+\frac{1}{b_k}\left(s_0\sqrt{\frac{\log p}{n_{\mcT_h}+n_0}}+h+\sqrt{\frac{s_0+s'}{n_k}}+\sqrt{\frac{\log p}{n_k}}\right)\Bigg\}\right).
\end{eqnarray*}
Moreover, for a special choice of
\begin{eqnarray*}
b_k=O\left(\left(\frac{\log p}{ n_k}\right)^{1/4}+s_0^{1/2}\left(\frac{\log p}{n_{\mcT_h}+n_0}\right)^{1/4}+h^{1/2}+
\left(\frac{s_0+s'}{n_k}\right)^{1/4}\right),
\end{eqnarray*}
it reduces to
\begin{eqnarray*}
\left\Vert\wh \bH_{\mcT_h}-\bH_{\mcT_h}(\bbeta_0)\right\Vert_{\max}
=O_P\left(s_0^{1/2}\left(\frac{\log p}{n_{\mcT_h}+n_0}\right)^{1/4}+h^{1/2}\right).
\end{eqnarray*}
\end{theorem}

We analogously define $\bgamma_{\mcT_h,m}=\left(\gamma_{\mcT_h,m,1},\ldots,\gamma_{\mcT_h,m,p-1}\right)^\rmT
=\left(\bH^*_{\mcT_h,-m,-m}\right)^{-1}\bH^*_{\mcT_h,-m,m}$.
Denote $\mcS_{\mcT_h,m}=\{j:\gamma_{\mcT_h,m,j}\neq 0\}$
with $s_{\mcT_h,m}=\left\vert\mcS_{\mcT_h,m}\right\vert$.

\begin{assumption}\label{C8}
There exists a constant $\kappa_{0m}>0$  such that
\begin{eqnarray*}
\inf_{\bdelta\in\mathcal{D}\left(0,\mcS_{\mcT_h,m}\right)} \frac
{\bdelta^\rmT \bH^{*}_{\mcT_h,-m,-m}\bdelta}{\vert\bdelta\vert^2_2}\ge \kappa_{0m}.
\end{eqnarray*}
\end{assumption}

\begin{assumption}\label{C9}
There exists a constant $\kappa_{0m}>0$  such that
\begin{eqnarray*}
\inf_{\bdelta\in\mathcal{D}\left(0,\bar{\mcS}_{\mcT_h,m}\right)} \frac
{\bdelta^\rmT \bH^*_{0,-m,-m}\bdelta}{\vert\bdelta\vert^2_2}\ge \kappa_{0m},
\end{eqnarray*}
 where $\bar{\mcS}_{\mcT_h,m}=\mcS_0\cup \mcS_{\mcT_h,m}$.
\end{assumption}

Assumptions \ref{C8} and \ref{C9} are  mimicking assumption \ref{C3} on the
restricted eigenvalue conditions for (\ref{gamma_trans}) and (\ref{gamma_debias})
to establish the error bounds of $\wh\bgamma_{0,m}-\bgamma_{0,m}$
in the following theorem.

\begin{theorem}
\label{theorem3.2}
Set
\begin{eqnarray*}
b_k=O\left(\left(\frac{\log p}{ n_k}\right)^{1/4}+s_0^{1/2}\left(\frac{\log p}{n_{\mcT_h}+n_0}\right)^{1/4}+h^{1/2}+
\left(\frac{s_0+s'}{n_k}\right)^{1/4}\right),\ k\in\{0\}\cup\mcT_h.
\end{eqnarray*}
Assume that
$$ \left(s_0+s_{\mcT_h,m}\right)\left(s_0^{1/2}\left(\frac{\log p}{n_{\mcT_h}+n_0}\right)^{1/4}+h^{1/2}\right)=o(1). $$
Under conditions in Theorem \ref{theorem2.4} and assumptions \ref{C4} and \ref{C8}--\ref{C9},
if
$$\lambda_m\ge4C_0(1+r_{\mcT_h,m})\left(s_0^{1/2}\left(\frac{\log p}{n_{\mcT_h}+n_0}\right)^{1/4}+(s_0h)^{1/4}\left(\frac{\log p}{n_{\mcT_h}+n_0}\right)^{1/8}+h^{1/2}\right)$$
with $r_{\mcT_h,m}=\left\vert\bgamma_{\mcT_h,m}\right\vert_1$ and
$$ \lambda_m'\ge 4C'\left(r_{0,m}\left(s_0^{1/2}\left(\frac{\log p}{n_{\mcT_h}+n_0}\right)^{1/4}+h^{1/2}\right)+\frac{\lambda_ms_{\mcT_h,m}}{h_0}\right) $$
with $r_{0,m}=\left\vert\bgamma_{0,m}\right\vert_1$ where the positive constants $C_0$ and $C'$ are respectively specified by 
(\ref{determineC})
and 
(\ref{determineC'}) in the Supplementary Material,
then it holds that
$$ \left\vert\wh\bgamma_{0,m}-\bgamma_{0,m}\right\vert_2=
O_P\left(\lambda_m'\sqrt{s_0+s_{\mcT_h,m}}+\lambda_m\sqrt{s_{\mcT_h,m}}\right) $$
and
$$ \left\vert\wh\bgamma_{0,m}-\bgamma_{0,m}\right\vert_1=
O_P\left(\lambda_m'\left(s_0+s_{\mcT_h,m}\right)+\lambda_ms_{\mcT_h,m}\right). $$
\end{theorem}

Having obtained the projection direction estimator $\wh\bgamma_{0,m}$ for $\bgamma_{0,m}$ and the transfer learning estimator $\wh\bbeta_{0,-m}$ (excluding the
$m$-th entry of $\wh\bbeta_0$) for the nuisance parameter $\bbeta_{0,-m}$, we define an estimated orthogonal score function of $\beta_{0m}$ for (\ref{orthogonal}) as follows,
\begin{eqnarray}\label{estorthogonal}
\wh{S}^{\rmo}_{0m}(\beta_{0m}\mid \wh\bbeta_{0,-m})=S_m((\beta_{0m}, \wh\bbeta^{\rmT}_{0,-m})^{\rmT};\{0\})-\bS^{\rmT}_{-m}((\beta_{0m}, \wh\bbeta^{\rmT}_{0,-m})^{\rmT};\{0\})\wh\bgamma_{0m}.
\end{eqnarray}
Considering $\wh\beta_{0m}$ as an initial estimator for the parameter $\beta_{0m}$ of interest, the one-step estimator based on
(\ref{estorthogonal}) can be deduced as
\begin{eqnarray*}
\wt\beta_{0m}&=&\wh\beta_{0m}-\left(\frac{\partial
\wh{S}^{\rmo}_{0m}(\beta_{m}\mid \wh\bbeta_{0,-m})}{\partial \beta_{m}}\Bigg\vert_{\beta_m=\wh\beta_{0m}}\right)^{-1}\wh{S}^{\rmo}_{0m}(\wh\beta_{0m}\mid \wh\bbeta_{0,-m})
\\&=&\wh\beta_{0m}-\wh H^{-1}_{0,m\mid -m}\wh{S}^{\rmo}_{0m}(\wh\beta_{0m}\mid \wh\bbeta_{0,-m}),
\end{eqnarray*}
where $\partial(\cdot)/\partial \beta_{m} $ is the subgradient operator and  $\wh H_{0,m\mid -m}=\wh H_{0,m,m}-\wh\bH^{\rmT}_{0,-m,m}\wh\bgamma_{0,m}$.
We advocate the transfer learning technique twice to obtain
$\wh\beta_{0m}$ and $\wh\bgamma_{0m}$, respectively, based on which
the one-step debiased estimator $\wt\beta_{0m}$ is deduced. Thus,
we refer to $\wt\beta_{0m}$
as the double transfer learning estimator.

We are in a position to
derive the asymptotic distribution of $\wt\beta_{0m}$ if inference procedure
for $\beta_{0m}$ is desirable. For simplicity, denote $\wh \bSigma_k=\sum_{i=1}^{n_k}\bZ_{ki}\bZ_{ki}^{\rmT}\tau(1-\tau)/n_k$ and
$\wh\bSigma_{\mcT_h}=\sum_{k\in\{0\}\cup\mcT_h}\pi_k\wh\bSigma_k$, whose
population counterparts are denoted by $\bSigma_k=E(\wh \bSigma_k)$ and $\bSigma_{\mcT_h}=E(\wh \bSigma_{\mcT_h})$, respectively. We pose some similarities
between the target and transferable sources.

\begin{assumption}\label{C10}
There exists $h_1$ such that
\begin{eqnarray*}
\max\left\{\left\vert H^*_{\mcT_h,m,m}-H^*_{0,m,m}\right\vert,
\left\vert\left(\bH^*_{\mcT_h,m,-m}-\bH^{*}_{0,m,-m}\right)
\bgamma_{0,m}\right\vert\right\}\le h_1.
\end{eqnarray*}
\end{assumption}

\begin{assumption}\label{C11}
There exists $h_{\max}$ such that
$\left\Vert \bSigma_{\mcT_h}-\bSigma_0\right\Vert_{\max}\le h_{\max}.$
\end{assumption}

Assumptions \ref{C10} and \ref{C11} are imposed to guarantee the consistency of transfer learning estimator of the asymptotic variance 
if $h_1$ and $h_{\max}$
are assumed to be convergent to $0$ at certain rates.
Note that $h_{\max}$ can be set as $0$ if
the distributions of covariates  across the target and transferable sources
are identical.
Define $\wh\bvarphi_{0,m}=\Gamma_{m,1}(-\wh\bgamma_{0,m})$.

\begin{theorem}
\label{theorem3.3}
Assume
$$\frac{s_0\log p}{\sqrt{n_0}}=o(1),\ h=o\left(n_0^{-1/4}\right)\ {\rm and}\ \max\{h_{\max},h_1\}=o(1).$$
Under conditions in Theorems \ref{theorem2.2} and \ref{theorem3.2} and assumptions \ref{C10}--\ref{C11},
if
$$\sqrt{n_0}\left(\lambda_m'\left(s_0+s_{\mcT_h,m}\right)+\lambda_ms_{\mcT_h,m}\right)\left(s_0\sqrt{\frac{\log p}{n_0}}+h\right)=o(1),$$
then it holds
\begin{eqnarray*}
\frac{\sqrt{n_0}\wh H_{0,m\mid -m}(\wt\beta_{0m}-\beta_{0m})}{\wh \sigma_m}\xrightarrow{\mathcal{L}}N(0,1),
\end{eqnarray*}
where 
$\wh\sigma_m^2=\wh\bvarphi_{0,m}^\rmT \wh\bSigma_{\mcT_h}\wh\bvarphi_{0,m}$.
\end{theorem}

Consequently, we can construct valid confidence interval and hypothesis test procedures for
$\beta_{0m}$ of interest.

\section{Transferability Detection}
\label{sec4}
As yet our (double) transfer learning procedure is based on the knowledge of the transferable set $\mcT_h$. In fact, the transferable set is usually unknown in practice.
It is therefore imperative to provide a data-derived method to identify the
transferable set $\mcT_h$.
Motivated by \citet{tian2022transfer}, we propose an approach to detecting the transferable set $\mcT_h$ in the framework of high-dimensional quantile regression models. Without loss of generality, suppose $n_0$ is even and randomly and equally divide the target data into
two sub-datasets:  one for training denoted by ${\cal V}_{\rm tr}$ and the other for test by ${\cal V}_{\rm te}$.
Thus $|{\cal V}_{\rm tr}|=|{\cal V}_{\rm te}|=n_0/2$. Denote the quantile loss function of the target based on observations in $\mathcal{A}\subset \{1,\ldots,n_0\}$ as
$$
L_0(\bbeta;\mathcal{A})=\frac{1}{\vert \mathcal{A}\vert}\sum_{i\in \mathcal{A}}\rho_{\tau}\left(Y_{0i}-\bZ_{0i}^{\rmT}\bbeta\right).
$$

The lasso estimator  based solely on the target training dataset ${\cal V}_{\rm tr}$ is defined by
$$
\wh\bbeta_{\rm lasso}=\arg\min_{\bmb}\left\{L_0(\bbeta;{\cal V}_{\rm tr})+\lambda_0\vert\bbeta\vert_1\right\},
$$
where $\lambda_0$ is the tuning parameter.
For $k=1,\ldots, K,$ viewing the $k$-th source dataset
$\{(\bZ_{ki},Y_{ki}), i=1,\ldots,n_k\}$ as transferable, combining
it with the target training dataset
$\{(\bZ_{0i},Y_{0i}), i\in {\cal V}_{\rm tr}\}$,
and mimicking the transfer learning procedure in (\ref{op_trans})
to obtain $\wh{\bbeta}_{\mcT_h}$,
we can also obtain a crude transfer learning estimator.
Specifically, we define
$$ \wh{\bbeta}_{{\cal V}_{\rm tr},k}=\arg\min_{\bbeta} \left\{ \frac{2n_k}{2n_k+n_0} L(\bbeta;\{k\})+\frac{n_0}{2n_k+n_0}L_0(\bbeta;{\cal V}_{\rm tr})+ \lambda_{\bbeta,k}\vert\bbeta\vert_1\right\}, $$
where $\lambda_{\bbeta,k}$ is the tuning parameter.
If the quantile loss on the target test dataset
at $\wh{\bbeta}_{{\cal V}_{\rm tr},k}$ is not substantially greater than that at $\hb_{\rm lasso}$, the $k$-th
source can be intuitively considered as transferable. In particular,
for a given $\varepsilon_0>0$, we thus define
$$\wh{\mcT}=\left\{1\le k\le K: L_0\left(\wh{\bbeta}_{{\cal V}_{\rm tr},k};{\cal V}_{\rm te}\right) \leq (1+\varepsilon_0)L_0\left(\hb_{\rm lasso};{\cal V}_{\rm te}\right)\right\}$$
and use it to estimate the transferable set $\mcT_h$.
The consistency of detection procedure is established in the Supplementary Material.

\section{Simulation Studies}
\label{sec5}
We conduct simulations to evaluate the finite-sample performances of the
proposed transfer learning method (Trans) under various combinations of $h$ and $\vert\mcT_h\vert$, wherein the transferable source datasets are identified through the detection procedure in
Section \ref{sec4}.
We also fit the $\ell_1$-penalized quantile regression model
solely on the target data (Non-Trans) to investigate the potential loss
if source datasets, transferable or non-transferable, are discarded.
On the contrary, we apply the proposed transfer learning method
on all source and target datasets by viewing all sources as transferable
(Pooled-Trans); thus utilizing non-transferable sources may be detrimental to the
resulting performance. As a benchmark,
we apply the proposed transfer learning method on the
transferable source (without the need of detection procedure) and target datasets (Oracle-Trans).

The $\ell_1$-penalized quantile regression is implemented via R package
{\verb"hqreg"} 
\citep{yi2017semismooth}. We adopt five-fold cross-validation procedure
to select all the tuning parameters and set $\varepsilon_0=0.01$ in the recovery of
transferable source set.
We consider $p=1000$, $n_0=200$, $n_k=300, k=1, \ldots, K$ with $K=20$, and $h=3, 6, 12$
throughout simulations and report the averaged $\ell_2$ estimation errors based on
$1000$ replications.

\begin{figure}
\centerline{
\includegraphics[height=0.75\textwidth,width=\textwidth]{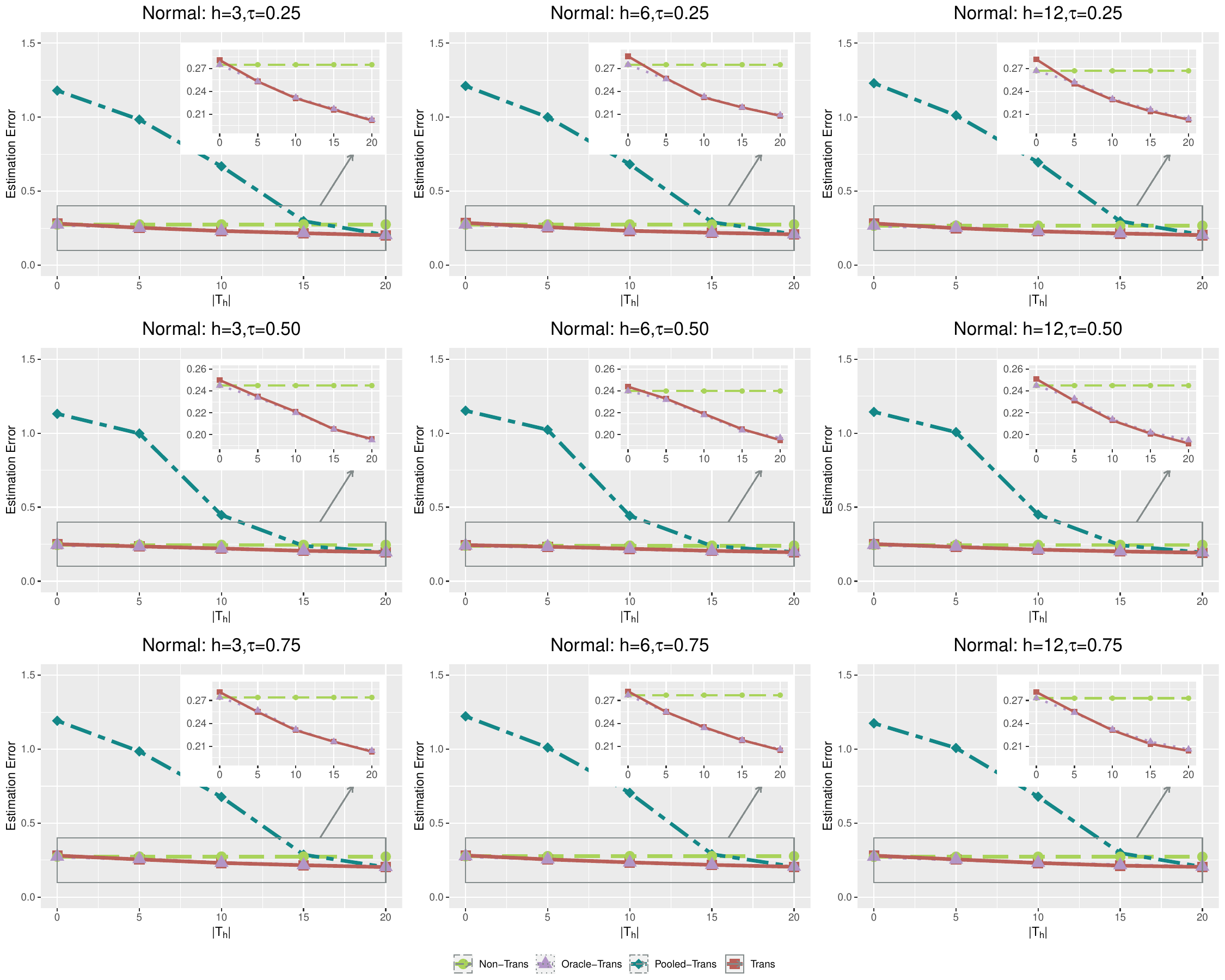}}
\caption{Averaged $\ell_2$ estimation errors with heterogenous normal model errors.}
\label{normal}
\end{figure}

We consider
the following heterogenous high-dimensional quanitle regression models from which
the target and source datasets were independently generated,
\begin{eqnarray}\label{SimModel}
Y_{k}=\bZ_k^\rmT\bbeta_k+Z_{k1}^2\eta_{k\tau}, k=0,\ldots, K,
\end{eqnarray}
where covariate vector $\bZ_0\sim N({\bf 0},(0.5^{|l-j|})_{l,j=1}^{p})$,
and $\bZ_k\sim N({\bf 0},(0.5^{|l-j|})_{l,j=1}^{p}+\bxi_k\bxi_k^\rmT)$ with
$\bxi_k\sim N({\bf 0},0.3^2\bI_p)$, $k=1,\ldots,K$,
  and $\bI_p$ denotes the $p$-dimensional unit matrix.
$Z_{k1}$ is the first component of $\bZ_k$, and $\eta_{k\tau}$ is independent of
$\bZ_k$ with $\tau$-th quantile zero. We set $\bbeta_0=(3,3,3,3,3, {\bf 0}_{p-5}^{\rmT})^{\rmT}$ for the target quantile regression parameter of interest and thus $s_0=5$.
For simplicity, let $\bd_k$ denote a random $p$-vector with
every component independently generated from $-1$ or $1$ with probability $0.5$.
We further construct $h$-transferable source datasets and non-transferable source datasets
under scenarios $|\mcT_h|= 0, 5, 10, 15, 20$.
In particular, for transferable source datasets, we set
$\bbeta_k=\bmb_0+h/p\cdot\bd_k$, implying $\vert\bbeta_k-\bmb_0\vert_1\le h$.
For non-transferable source $k\in\mcT_h^c$, we
first randomly generated index set $\mcI_k$ of size $s_0$ from $\{2s_0 + 1, \ldots, p\}$ and set the $j$-th component of quantile regression coefficient $\bbeta_k$ as
\begin{equation*}
\beta_{kj}=\begin{cases}
1+2h/p\cdot(-1)^{d_{kj}},& j\in\{s_0+1,\ldots,2s_0\}\cup \mcI_k,\\
2h/p\cdot(-1)^{d_{kj}},& \text{otherwise},\\
\end{cases}
\end{equation*}
where $d_{kj}$ is the $j$-th element of $\bd_k$. Basic calculations yield
$\vert\bbeta_k-\bbeta_0\vert_1\ge 2h+3s_0-12s_0h/p$.

We consider heterogenous light-tailed model errors when $\eta_{k\tau}$
were independently from the shifted standard normal distribution with
three different quantile levels $\tau=0.25, 0.5$, and $0.75$ to
more comprehensively examine covariate effects.
Simulation results in Figure \ref{normal}
show that, across the quartiles, Non-Trans approach indeed exhibits higher $\ell_2$ estimation errors if one does not utilize any source dataset especially when it is transferable. On the other hand, Pooled-Trans approach which utilizes all source datasets,
whether transferable or non-transferable, leads to remarkably
downstream performances while mistakenly using more non-transferable sources.
 Trans and Oracle-Trans approaches perform equally well,
implying that the detection procedure can effectively identify the
transferable and non-transferable sources. Meanwhile, the proposed transfer
learning method can greatly reduce estimation errors compared to Non-Trans approach
due to the benefit of properly exploiting the source datasets.
There are considerable decreases in estimation errors of Trans, Oracle-Trans,
and Pooled-Trans approaches when the transferable sources increase among the $K=20$ candidate ones. For each method, the performances are very similar when criterion $h$ is varying from $3, 6,$ and $12$, perhaps due to that $h/p$ is relatively tiny.

We further consider heterogenous heavy-tailed model errors wherein
$\eta_{k\tau}$
were independently from the shifted $t_2$ distribution
and heterogenous skewed model errors from the Gumbel. Simulation results
are respectively displayed in the Supplementary Material, showing that
the proposed transfer learning approach delivers promising performances,
a substantial profit from the ability of quantile regression to address the
heterogeneous heavy-tailed or skewed model errors as well as
the ability of the transfer learning method to
leverage the source datasets properly.
On the other hand, we also demonstrate the usefulness of
the proposed transferability detection procedure which can
recover the truly transferable sources with very high probability. We postpone the simulation
results to the Supplementary Material.

\begin{figure}[tbp]
\includegraphics[height=0.75\textwidth,width=\textwidth]{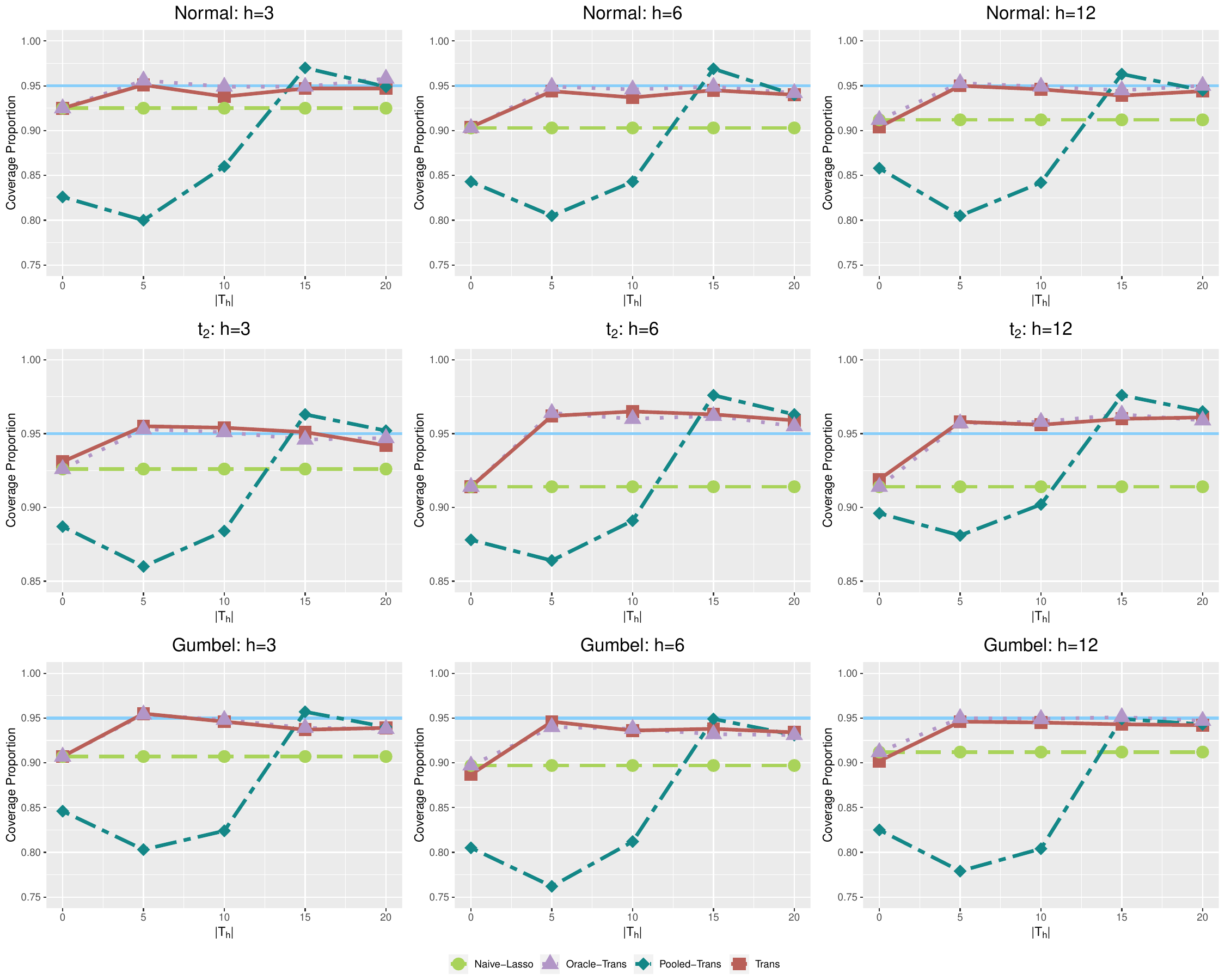}
\caption{Coverage proportion of confidence interval for the $1$-st component
 of quantile regression coefficients with $\tau=0.75$.}
\label{cp1_75}
\end{figure}

We next assess the finite-sample performance of the proposed (double) transfer learning
inference procedure in terms of the coverage proportion
among $1000$ replications for single component of
the high-dimensional quantile regression coefficients.
To be concise, we fix the quantile level of $\tau$ at $0.75$ and
consider regression coefficients of the $1$-st (signal) and $50$-th (non-signal)
covariates, respectively, in different scenarios.
We only report simulation results for signal case
in Figures \ref{cp1_75} and defer ones for non-signal case
to  the Supplementary Material. 
In general, across different heterogenous model errors,
the proposed Trans approach performs as well as Oracle-Trans approach,
and their coverage proportions of the two methods are apparently
closer to the nominal $95\%$ than that of Non-Trans approach.
As for Pooled-Trans approach, when $\vert\mcT_h\vert$ is relatively smaller,
more non-transferable sources are mistakenly exploited in the transfer learning procedure
and thus the coverage proportion is farther away from the nominal $95\%$ than
the other three methods.
As $\vert\mcT_h\vert$ increases to $20$,
negative transfer is alleviated or even avoided and
the coverage proportion becomes closer to the nominal $95\%$.
As for as the size and power analysis, we report the simulation results
in the Supplementary Material, showing that the proposed Trans
approach can achieve higher power while maintaining the size properly. On the contrary,
both
the Non-Trans approach that does not utilize any sources
and the Pooled-Trans approach that utilizes all the sources without discrimination
deliver inferior  performances.

\section{A Real Example}
\label{sec6}
To illustrate the practical utility of the proposed method,
we conduct statistical inference on a survey of China Family Panel Studies (http://www.isss.pku.edu.cn/cfps/),
which aims to investigate the relationship among economy, society, population, education and health in China
by respectively collecting data at the individual, family and community levels.
The panel studies have finished five rounds of follow-up investigation during 2010, 2012, 2014, 2016, and 2018.
The nearest raw dataset of individual level in 2018 contains 37355 individuals and 1346 variables.

\begin{table}
\centering
\caption{Sample sizes of source provinces and the transferable sources detected by the proposed method}
{\setlength{\tabcolsep}{7.8mm}
\begin{tabular}{lccccccccc}
\hline	
\hline
&&\multicolumn{3}{c}{Transferability}\\
\cline{3-5}
Source&Sample size&$\tau=0.25$&$\tau=0.50$&$\tau=0.75$\\
\hline
Yunnan     &576 & \ding{55}  & \ding{51} & \ding{51}    \\
Sichuan    &790 & \ding{51}  & \ding{51}  & \ding{51}   \\
Hubei      &228 & \ding{51}  & \ding{51}  & \ding{51}   \\
Jiangxi    &292 & \ding{51}  & \ding{51}  & \ding{51}    \\
Anhui      &309 & \ding{51}  & \ding{51}  & \ding{51}    \\
Jiangsu    &297 & \ding{55} &  \ding{55} & \ding{51}    \\
Shanghai   &771 & \ding{51} & \ding{51}  & \ding{51}    \\
Guangdong  &1481& \ding{51}  & \ding{51}  & \ding{51}   \\
\hline
\label{trans_source}
\end{tabular}}
\end{table}

As an essentially economic and sociological issue concerned by the public,
the total individual income is used to measure the quality of daily life.
Consequently, we are primarily interested to investigate which factors affect individual income for the adults between 18 and 60 years old in the fifth round of survey in 2018.
The total individual income includes total wages and all types of government subsidies in 2018.
To mitigate the impact arisen by the apparent wealth inequality, we take the logarithm of the total individual income as the response.
We consider the Hunan province as the target, and view the provinces near the Changjiang River
and Zhujiang River as
the sources, which consists of Yunnan, Sichuan, Hubei, Jiangxi, Anhui, Jiangsu,  Shanghai and Guangdong.
We employ the distance correlation learning \citep{li2012feature}
and the fused Kolmogorov filter \citep{mai2015the} to
select $200$ most marginally related covariates, respectively.
By merging these $400$ selected covariates, we obtain
$p=265$ ones.
The sample size of the target Hunan province is $n_0=403$.
Table \ref{trans_source} shows the sample sizes of source provinces and the transferable
sources detected by the proposed method.
It can be seen that transferable sources are quantile-adaptive
and the income groups at the third quartile
are all transferable, showing that there may exist some
similarity among the relatively high income groups.
On the contrary, the Hunan, Yunnan and Jiangsu provinces
possess different geographical conditions such as
arable land, flood disaster and transportation, which however are usually
viewed as agriculture-related factors that primarily constrain
the relatively low income group. Thus, the Yunan and Jiangsu provinces
are not detected as transferable sources while inferring the
income group at the first quartile in the Hunan province.

\begin{table}[tbp]
\centering
\begin{footnotesize}
\caption{
Estimation and hypothesis test for target Hunan province in the China Family Panel Studies}
{\setlength{\tabcolsep}{0.75mm}
\def\arraystretch{1.3}
\begin{tabular}{lcrrrrrrrrrrrrrr}
\hline	
\hline
&&\multicolumn{4}{c}{\hspace{0in}Trans}
&&\multicolumn{4}{c}{\hspace{0in}Pooled-Trans}
&&\multicolumn{4}{c}{\hspace{0in}Non-Trans}\\
\cline{3-6}
\cline{8-11}
\cline{13-16}
 Covariate     &$\tau$     &$\wh\bbeta$ &$\wt\bbeta$ &ESE &$p$-value  &&     $\wh\bbeta$ &$\wt\bbeta$ &ESE &$p$-value  &&    $\wh\bbeta$ &$\wt\bbeta$ &ESE &$p$-value  \\
\hline
Age&0.25&  0.000&  0.458& 0.171& {\bf 0.007}&& $-0.001$& $-0.075$& 0.100& 0.455&& 0.000 &$-0.030$& 0.068& 0.656   \\
   &0.50& $-0.017$& $-0.027$& 0.065& 0.672&& $-0.022$&  0.001& 0.060& 0.990&& 0.000 &$-0.021$& 0.061& 0.728   \\
   &0.75& $-0.034$& $-0.079$& 0.063& 0.211&& $-0.034$& $-0.079$& 0.063& 0.211&& 0.000 &$-0.024$& 0.059& 0.691   \\
Urban&0.25&  0.007&  0.333& 0.116& {\bf 0.004}&&  0.018&  0.305& 0.100& {\bf 0.002}&& 0.000 &$-0.114$& 0.092& 0.216   \\
&0.50&  0.000& $-0.082$& 0.099& 0.406&&  0.001& $-0.076$& 0.103& 0.458&& 0.000 &$-0.040$& 0.072& 0.580   \\
&0.75&  0.004& $-0.086$& 0.123& 0.487&&  0.004& $-0.086$& 0.123& 0.487&& 0.000 & 0.010& 0.092& 0.918   \\
Degree&0.25&  0.017&  0.104& 0.030& $<${\bf 0.001}&&  0.023&  0.104& 0.028& $<${\bf 0.001}&& 0.018 & 0.085& 0.042& {\bf 0.041}   \\
&0.50&  0.019&  0.081& 0.027& {\bf 0.003}&&  0.022&  0.084& 0.027& {\bf 0.002}&& 0.018 & 0.074& 0.031& {\bf 0.018}   \\
&0.75&  0.018&  0.099& 0.028& $<${\bf 0.001}&&  0.018&  0.099& 0.028& $<${\bf 0.001}&& 0.018 & 0.133& 0.058& {\bf 0.022}   \\
Reading&0.25&  0.000&  0.111& 0.039& {\bf 0.004}&&  0.000& $-0.087$& 0.058& 0.129&& 0.007 & 0.014& 0.008& 0.071   \\
&0.50&  0.000&  0.055& 0.008& $<${\bf 0.001}&&  0.000&  0.055& 0.008& $<${\bf 0.001}&& 0.007 & 0.033& 0.011& {\bf 0.002}   \\
&0.75&  0.000&  0.026& 0.008& {\bf 0.002}&&  0.000&  0.026& 0.008& {\bf 0.002}&& 0.007 & 0.048& 0.014& $<${\bf 0.001}   \\
Housing&0.25&  0.000& $-0.087$& 0.043& {\bf 0.045}&&  0.000& $-0.099$& 0.040& {\bf 0.014}&& 0.000 &$-0.037$& 0.048& 0.437   \\
&0.50&  0.000& $-0.010$& 0.041& 0.805&&  0.000&  0.005& 0.039& 0.889&& 0.000 & 0.004& 0.040& 0.918   \\
&0.75&  0.000&  0.022& 0.039& 0.571&&  0.000&  0.022& 0.039& 0.571&& 0.000 & 0.039& 0.047& 0.413   \\
\hline
\label{real_ci}
\end{tabular}}
\end{footnotesize}
\end{table}

To be concise, we report the effects of five covariates of interest
and the associated $p$-values on the total income at different quantiles
in Table \ref{real_ci}, where column ``ESE" stands for
the estimated standard error based on the double transfer learning approach.
For the Trans approach,
covariates Degree and Reading have significantly
positive effects on the total income crossing three quartiles;
higher degree and more reading tend to create more wealth, indicating
education is an important avenue to increase income.
On the other hand, younger, living in the rural residence
and pessimistic attitude to
housing exerts significantly negative
effects only on the first quartile of the total income. It means
that Age, Urban residence and grasping Housing policy make difference
for the relatively low income group.
For the Pooled-Trans and Non-Trans approaches, although
the interpretations are general consistent with that of the Trans approach,
there exists some missed significant covariates at different quantile levels.
As a conclusion, we advocate the government to develop education
to enhance the total income for all individuals and
promote some effective policies to accelerate urbanization and
reduce the housing burden for individuals whose income
lies below the first quartile.

\section{Conclusion}
\label{sec7}
Modern complicated high-dimensional data often display heterogeneity due to heteroscedastic variance or inhomogeneous covariate effects. On the other hand,
heavy-tailed model error or outlier in the response is another major concern in practice.
Heterogeneity and heavy-tailed phenomenon could be even more prevalent in the transfer learning
as the target and multiple source domains are involved.
Yet these concerns are not addressed sufficiently by current transfer learning approaches
and thus may downgrade the resulting performance in practice if they are not properly taken into account.
Numerical studies show that the proposed procedure delivers promising and compelling performances.

We have adopted the $\ell_1$ distance of coefficient difference
between the target and source regression models to measure
their similarity. It has been demonstrated as a straightforward and meaningful
criterion. The cosine similarity \citep{xia2015learning,park2020methodology,gu2022mathematical}, which is free of the magnitude and relatively loose,
between regression coefficients of the target and source models could be viewed as
an alternative measure and certainly deserves further exploration.
Distributionally robust optimization \citep{chen2020distributionally,zhen2021mathematical}
 can be considered as
an effective approach to enhancing the transferability
of  model parametric estimation that is robust to distribution perturbation
in terms of the Wasserstein distance.
Investigations along this direction
that utilizes the Wasserstein distance as a criterion to select the
source domains in the transfer learning may be merit future research.



\section*{Supplement Material}
The online supplementary  material contains
technical proofs of theorems, consistency of the transferability detection procedure, some
supporting lemmas, and  additional simulation results.

\clearpage
\newpage

\begin{center}
{\Large Supplementary  material for
``Estimation and inference for
transfer learning with high-dimensional quantile regression"
}
\end{center}

\tableofcontents

Appendices contain
technical proofs of theorems and some
supporting lemmas as well as additional simulation results. Specifically, we establish the convergence rate for specific $\mcT_h$ in Appendix \ref{appendixA}, confidence interval in Appendix \ref{appendixB},
and detection consistency in Appendix \ref{appendixC}.
Appendix \ref{appendixD} displays additional simulation results.

\appendix
\section{Convergent Rate}
\label{appendixA}
\setcounter{equation}{0}
\def\theequation{A.\arabic{equation}}

\subsection{Proof of Theorem \ref{theorem2.1}}

\begin{lemma}\label{lemmaA1}
Under assumptions \ref{C1} and \ref{C4}, we have
\begin{eqnarray*}
\vert\bbeta^*_{\mcT_h}-\bbeta_0\vert_1\le Ch.
\end{eqnarray*}
\end{lemma}

\begin{proof}
It follows from the definition of $\bbeta^*_{\mcT_h}$ in (\ref{alpha_Th}) that
$$ \sum_{k\in\{0\}\cup\mcT_h}\pi_kE\left[\bZ_{k}\left\{P
\left(Y_{k}-\bZ_{k}^{\rmT}\bbeta^*_{\mcT_h}\le 0\mid \bZ_{k}\right)-P\left(Y_{k}-\bZ_{k}^{\rmT}\bbeta_k\le 0\mid \bZ_{k}\right)\right\}\right]={\bf 0}, $$
which implies that
\begin{eqnarray*}
&&\sum_{k\in\{0\}\cup\mcT_h}\pi_kE\left[\bZ_{k}\left\{P\left(Y_{k}-\bZ_{k}^{\rmT}\bbeta^*_{\mcT_h}\le 0\mid \bZ_{k}\right)-P\left(Y_{k}-\bZ_{k}^{\rmT}\bbeta_0\le 0\mid \bZ_{k}\right)\right\}\right]\\
&=&\sum_{k\in\mcT_h}\pi_kE\left[\bZ_{k}\left\{P\left(Y_{k}-\bZ_{k}^{\rmT}\bbeta_k\le 0\mid \bZ_{k}\right)-P\left(Y_{k}-\bZ_{k}^{\rmT}\bbeta_0\le 0\mid \bZ_{k}\right)\right\}\right].
\end{eqnarray*}
Therefore, under assumption \ref{C1}, it holds by the Newton--Leibniz formula that
\begin{eqnarray*}
&&\sum_{k\in\{0\}\cup\mcT_h}\pi_kE\left[\int^1_0f_{k}
\left(\bZ_{k}^{\rmT}\bbeta_0+t\bZ_{k}^{\rmT}(\bbeta^*_{\mcT_h}-\bbeta_0)\mid \bZ_{k}\right)\rmd t\bZ_{k}\bZ_{k}^{\rmT}\right](\bbeta^*_{\mcT_h}-\bbeta_0)\\
&=&\sum_{k\in\mcT_h}\pi_kE\left[\int^1_0f_{k}
\left(\bZ_{k}^{\rmT}\bbeta_0+t\bZ_{k}^{\rmT}(\bbeta_k-\bbeta_0)\mid \bZ_{k}\right)\rmd t\bZ_{k}\bZ_{k}^{\rmT}\right](\bbeta_k-\bbeta_0).
\end{eqnarray*}
As a result, following assumption \ref{C4} we have
\begin{eqnarray*}
\vert\bbeta^*_{\mcT_h}-\bbeta_0\vert_1&=&\left\vert\sum_{k\in\mcT_h}\pi_k\wt \bSigma_{\mcT_h}^{-1}
\wt \bSigma_k(\bbeta_k-\bbeta_0)\right\vert_1\\
&\le& \sum_{k\in\mcT_h}\pi_k
\sup_{k\in\mcT_h}\left\Vert\wt\bSigma_{\mcT_h}^{-1}\wt\bSigma_k\right\Vert_1\vert\bbeta_k-\bbeta_0\vert_1\\
&\le& Ch,
\end{eqnarray*}
which ends the proof of Lemma \ref{lemmaA1}.
\end{proof}
	
Denote
$$ \delta L(\bu; \{0\}\cup\mcT_h)= L(\bbeta^*_{\mcT_h}+\bu; \{0\}\cup\mcT_h)- L(\bbeta^*_{\mcT_h}; \{0\}\cup\mcT_h)
- \bS(\bbeta^*_{\mcT_h}; \{0\}\cup\mcT_h)^\rmT\bu.$$

\begin{lemma}\label{lemmaA2}
Assume $h=o(1)$. Under assumptions \ref{C1} and \ref{C3}--\ref{C4}, it holds that
\begin{eqnarray}
\label{Hessian-lower-bound}
E\left\{\delta L(\bu; \{0\}\cup\mcT_h)\right\}\ge \min\left\{\frac{\kappa_0}4\vert\bu\vert_2^2,\frac{3\kappa_0^3}{16m_0}\vert\bu\vert_2\right\}.
\end{eqnarray}
Moreover, under assumptions \ref{C1}--\ref{C4}, the sample version satisfies the following inequality
\begin{eqnarray*}
\delta L(\bu; \{0\}\cup\mcT_h)
\ge \min\left\{\frac{\kappa_0}4\vert\bu\vert_2^2,\frac{3\kappa_0^3}{16m_0}\vert\bu\vert_2\right\}-L\sqrt{\frac{\log p}{n_{\mcT_h}+n_0}}\vert\bu\vert_1,
\end{eqnarray*}
with probability more than $1-p^{-2}$, for any  $\bu\in \mathcal{D}(a_n,\mcS)$.
\end{lemma}

\begin{proof}
By the Knight equation
\begin{eqnarray*}
\rho_\tau(u-v)-\rho_\tau(u)=-v\{\tau-I(u\le 0)\}+\int_0^v \{I(u\le z)-I(u\le 0)\}\rmd z,
\end{eqnarray*}
it holds for any  $\bu\in \mathcal{D}(a_n,\mcS)$ that
\begin{eqnarray*}
&&\delta L(\bu; \{0\}\cup\mcT_h)\\
&=& L(\bbeta^*_{\mcT_h}+\bu; \{0\}\cup\mcT_h)- L(\bbeta^*_{\mcT_h}; \{0\}\cup\mcT_h)- \bS(\bbeta^*_{\mcT_h}; \{0\}\cup\mcT_h)^\rmT\bu\\
&=&\frac1{n_{\mcT_h}+n_0}\sum_{k\in\{0\}\cup \mcT_h}\sum_{i=1}^{n_k}\Big[\rho_\tau\left(Y_{ki}-\bZ_{ki}^{\rmT}(\bbeta^*_{\mcT_h}+\bu)\right)\\
&&\hspace{0.5in}-\rho_\tau\left(Y_{ki}-\bZ_{ki}^{\rmT}\bbeta^*_{\mcT_h}\right)-\left\{\tau-I\left(Y_{ki}-\bZ_{ki}^{\rmT}\bbeta^*_{\mcT_h}\le 0\right)\right\}\bZ_{ki}^{\rmT}\bu\Big]\\
&=&\frac1{n_{\mcT_h}+n_0}\sum_{k\in\{0\}\cup \mcT_h}\sum_{i=1}^{n_k}\int^{\mathbf{Z}_{ki}^{\rmT}\mathbf{u}}_0\Big\{I\left(Y_{ki}-\bZ_{ki}^{\rmT}\bbeta^*_{\mcT_h}\le z\right)-I\left(Y_{ki}-\bZ_{ki}^{\rmT}\bbeta^*_{\mcT_h}\le 0\right)\Big\}\rmd z.
\end{eqnarray*}
Therefore, under assumption \ref{C1} and based on the Fubini theorem,
we can derive that
\begin{eqnarray*}
&&E\{\delta L(\bu; \{0\}\cup\mcT_h)\}\\&=&\sum_{k\in\{0\}\cup \mcT_h}\pi_kE\left[\int^{\mathbf{Z}_{k}^{\rmT}\mathbf{u}}_0\left\{I\left(Y_{k}-\bZ_{k}^{\rmT}\bbeta^*_{\mcT_h}\le z\right)-I\left(Y_{k}-\bZ_{k}^{\rmT}\bbeta^*_{\mcT_h}\le 0\right)\right\}\rmd z\right]\\
&=&\sum_{k\in\{0\}\cup \mcT_h}\pi_kE\left[\int^{\mathbf{Z}_{k}^{\rmT}\mathbf{u}}_0\left\{F_k\left(\bZ_{k}^{\rmT}\bbeta^*_{\mcT_h}+z\mid \bZ_{k} \right)-F_k\left(\bZ_{k}^{\rmT}\bbeta^*_{\mcT_h}\mid \bZ_{k}\right)\right\}\rmd z\right].
\end{eqnarray*}
\begin{itemize}
\item[(i)] When $\bZ_{k}^{\rmT}\bu\ge 0$, the integrand has lower bound as
\begin{eqnarray*}
&&F_k\left(\bZ_{k}^{\rmT}\bbeta^*_{\mcT_h}+z\mid \bZ_{k}\right)-F_k\left(\bZ_{k}^{\rmT}\bbeta^*_{\mcT_h}\mid \bZ_{k}^{\rmT}\right)\\
&=&f_k\left(\bZ_{k}^{\rmT}\bbeta^*_{\mcT_h}\mid \bZ_{k}\right)z+\frac12\dot f_k\left(\bZ_{k}^{\rmT}\bbeta^*_{\mcT_h}+z^\dag\mid \bZ_{k}\right)z^2\\
&\ge&f_k\left(\bZ_{k}^{\rmT}\bbeta^*_{\mcT_h}\mid \bZ_{k}\right)z-\frac12m_0z^2,
\end{eqnarray*}
where $z^\dag\in [0,z]$.
\item[(ii)] When $\bZ_{k}^{\rmT}\bu\le 0$, the integrand has upper bound as
\begin{eqnarray*}
&&F_k(\bZ_{k}^{\rmT}\bbeta^*_{\mcT_h}+z\mid \bZ_{k})-F_k(\bZ_{k}^{\rmT}\bbeta^*_{\mcT_h}\mid \bZ_{k})\\
&=&f_k(\bZ_{k}^{\rmT}\bbeta^*_{\mcT_h}\mid \bZ_{k})z+\frac12\dot f_k(\bZ_{k}^{\rmT}\bbeta^*_{\mcT_h}+z^\dag\mid \bZ_{k})z^2\\
&\le&f_k(\bZ_{k}^{\rmT}\bbeta^*_{\mcT_h}\mid \bZ_{k})z+\frac12m_0z^2,
\end{eqnarray*}
where $z^\dag\in [z,0]$.
\end{itemize}
As a result, if $h$ is sufficiently small, it yields that
\begin{eqnarray*}
&&E\left\{\delta L\left(\bu; \{0\}\cup\mcT_h\right)\right\}\\&\ge&\sum_{k\in\{0\}\cup \mcT_h}\pi_k\left[\frac12\bu^\rmT E\left\{f_k\left(\bZ_{k}^{\rmT}\bbeta^*_{\mcT_h}\mid \bZ_{k}\right)\bZ_{k}\bZ_{k}^{\rmT}\right\}\bu-\frac{m_0}6E\left\{\vert\bZ_{k}^{\rmT}\bu \vert^3\right\}\right]\\
&=&\sum_{k\in\{0\}\cup \mcT_h}\pi_k\left[\frac12\bu^\rmT E\left\{f_k\left(\bZ_{k}^{\rmT}\bbeta_0\mid \bZ_{k}\right)\bZ_{k}\bZ_{k}^{\rmT}\right\}\bu-\frac{m_0}6E\left\{\vert\bZ_{k}^{\rmT}\bu \vert^3\right\}\right]\\
&&+\sum_{k\in\{0\}\cup \mcT_h}\pi_k\left[\frac12\bu^\rmT E\left\{\left(f_k\left(\bZ_{k}^{\rmT}\bbeta^*_{\mcT_h}\mid \bZ_{k}\right)-f_k\left(\bZ_{k}^{\rmT}\bbeta_0\mid \bZ_{k}\right)\right)\bZ_{k}\bZ_{k}^{\rmT}\right\}\bu\right]\\
&\ge& \left(\frac{\kappa_0}2-\frac{m_0B^3Ch}2\right) \vert\bu\vert_2^2-\frac{m_0}{6\kappa_0}\vert\bu\vert_2^3
\ge \frac{3\kappa_0}8 \vert\bu\vert_2^2-\frac{m_0}{6\kappa_0}\vert\bu\vert_2^3,
\end{eqnarray*}
where the last inequality follows from assumptions \ref{C3}--\ref{C4}.
We define the maximum radius $r_0$ for the margin condition as
\begin{eqnarray*}
r_0=\sup_{r>0}\left\{r:\forall \bu\in\mathcal{D}(a_n,\mcS)\cap\{\bu\in \mathbb{R}^p: \vert\bu\vert_2\leq r\}, E\{\delta L(\bu ;\{0\}\cup\mcT_h)\}\ge \frac14\kappa_0\vert\bu\vert^2_2\right\}.
\end{eqnarray*}
If $\bu\in \mathcal{D}(a_n,\mcS)$ and $\vert\bu\vert_2\le 3\kappa_0^2/(4m_0)$, then $\kappa_0\vert\bu\vert_2^2/8-m_0\vert\bu\vert_2^3/(6\kappa_0)\ge0$ and
$E\{\delta L(\bu;\{0\}\cup\mcT_h)\}\ge \kappa_0\vert\bu\vert_2^2/4$,
which implies $r_0\ge 3\kappa_0^2/(4m_0)$.
For any $\bu$ such that $\bu\in \mathcal{D}(a_n,\mcS)$ and $\vert\bu\vert_2>r_0$, then denote $w_0=r_0/\vert\bu\vert_2<1$.
On the other hand,
$\vert w_0\bu_{\mcS^c}\vert_1\le 3\vert w_0\bu_\mcS\vert_1+w_0a_n\le 3\vert w_0\bu_\mcS\vert_1+a_n$.
Thus, we have $w_0\bu\in \mathcal{D}(a_n,\mcS)$ and thereby $E\{\delta L(w_0\bu ;\{0\}\cup\mcT_h)\}\ge \kappa_0\vert w_0\bu\vert^2_2/4$
by the definition of $r_0$.
Note that $E\{ \bS(\bbeta^*_{\mcT_h}; \{0\}\cup\mcT_h)\}={\bf 0}$ by definition of $\bbeta^*_{\mcT_h}$.
Thus,
\begin{eqnarray*}
&&E\left\{\delta L\left(\bu; \{0\}\cup\mcT_h\right)\right\}\\&=&E\left\{ L\left(\bbeta^*_{\mcT_h}+\bu; \{0\}\cup\mcT_h\right)\right\}-E\left\{ L\left(\bbeta^*_{\mcT_h}; \{0\}\cup\mcT_h\right)\right\}-E\left\{ \bS\left(\bbeta^*_{\mcT_h}; \{0\}\cup\mcT_h\right)\right\}^\rmT\bu\\
&=&E\left\{L\left(\bbeta^*_{\mcT_h}+\bu; \{0\}\cup\mcT_h\right)\right\}-E\left\{ L\left(\bbeta^*_{\mcT_h}; \{0\}\cup\mcT_h\right)\right\}.
\end{eqnarray*}
By the convexity of the objective function $ L(\balpha; \{0\}\cup\mcT_h)$, we have
\begin{eqnarray*}
E\left\{ L\left(\bbeta^*_{\mcT_h}+w_0\bu; \{0\}\cup\mcT_h\right)\right\}\le w_0E\left\{  L\left(\bbeta^*_{\mcT_h}+\bu; \{0\}\cup\mcT_h\right)\right\}+(1-w_0)E\left\{ L\left(\bbeta^*_{\mcT_h}; \{0\}\cup\mcT_h\right)\right\}.
\end{eqnarray*}
As a result,
\begin{eqnarray*}
E\left\{\delta L\left(\bu;\{0\}\cup\mcT_h\right)\right\}&=&E\left\{ L\left(\bbeta^*_{\mcT_h}+\bu; \{0\}\cup\mcT_h\right)\right\}-E\left\{ L\left(\bbeta^*_{\mcT_h}; \{0\}\cup\mcT_h\right)\right\}\\
&\ge &\frac1{w_0}\left[E\left\{ L\left(\bbeta^*_{\mcT_h}+a_0\bu; \{0\}\cup\mcT_h\right)\right\}-E\left\{  L\left(\bbeta^*_{\mcT_h}; \{0\}\cup\mcT_h\right)\right\}\right]\\
&\ge&\frac{\kappa_0}{4w_0}\vert w_0\bu\vert_2^2=\frac{\kappa_0r_0}4\vert\bu\vert_2\ge \frac{3\kappa_0^3}{16m_0}\vert\bu\vert_2.
\end{eqnarray*}
Thus we conclude that
\begin{eqnarray*}
E\left\{\delta L(\bu; \{0\}\cup\mcT_h)\right\}\ge \min\left\{\frac{\kappa_0}4\vert\bu\vert_2^2,\frac{3\kappa_0^3}{16m_0}\vert\bu\vert_2\right\}.
\end{eqnarray*}
On the other hand, under assumption \ref{C2}, note that
$$ \int^{\mathbf{Z}_{ki}^{\rmT}\mathbf{u}}_0\left\{I\left(Y_{ki}-\bZ_{ki}^{\rmT}\bbeta^*_{\mcT_h}\le z\right)-I\left(Y_{ki}-\bZ_{ki}^{\rmT}\bbeta^*_{\mcT_h}\le 0\right)\right\}\rmd z $$
takes value in $[0,B\vert\bu\vert_1]$, which implies that
the integration is a sub-Gaussian random variable with parameter at most $B\vert\bu\vert_1/2$ 
\citep{wainwright2019high}.
Moreover, the Hoeffding inequality entails that
\begin{eqnarray*}
&&P\left(\delta L\left(\bu; \{0\}\cup\mcT_h\right)-E\left(\delta  L\left(\bu; \{0\}\cup\mcT_h\right)\right)\le -t\right)\\&=&
P\Bigg(\frac1{n_{\mcT_h}+n_0}\sum_{k\in\{0\}\cup \mcT_h}\sum_{i=1}^{n_k}\Bigg\{\int^{\mathbf{Z}_{ki}^{\rmT}\mathbf{u}}_0\left\{I\left(Y_{ki}-\bZ_{ki}^{\rmT}\bbeta^*_{\mcT_h}\le z\right)-I\left(Y_{ki}-\bZ_{ki}^{\rmT}\bbeta^*_{\mcT_h}\le 0\right)\right\}\rmd z\\
&&\hspace{0.5in}-E\int^{\mathbf{Z}_{k}^{\rmT}\mathbf{u}}_0\left\{I\left(Y_{k}-\bZ_{k}^{\rmT}\bbeta^*_{\mcT_h}\le z\right)-I\left(Y_{k}-\bZ_{k}^{\rmT}\bbeta^*_{\mcT_h}\le 0\right)\right\}\rmd z\Bigg\}\le -t\Bigg)\\
&\le &\exp\left(-\frac{2\left(n_{\mcT_h}+n_0\right)t^2}{B^2\vert\bu\vert_1^2}\right).
\end{eqnarray*}
Taking $t=B\vert\bu\vert_1\sqrt{\log p/(n_{\mcT_h}+n_0)}$, it concludes that
\begin{eqnarray*}
\delta L\left(\bu; \{0\}\cup\mcT_h\right)&\ge& E\left(\delta L\left(\bu; \{0\}\cup\mcT_h\right)\right)-B\sqrt{\frac{\log p}{n_{\mcT_h}+n_0}}\vert\bu\vert_1\\
&\ge& \min\left\{\frac{\kappa_0}4\vert\bu\vert_2^2,\frac{3\kappa_0^3}{16m_0}\vert\bu\vert_2\right\}-B\sqrt{\frac{\log p}{n_{\mcT_h}+n_0}}\vert\bu\vert_1,
\end{eqnarray*}
with probability more than $1-p^{-2}$.
\end{proof}

\begin{lemma}
\label{lemmaA3}
Under assumption \ref{C2}, it holds that
\begin{eqnarray}
\label{grad*}
\left\vert \bS\left(\bbeta^*_{\mcT_h};\{0\}\cup\mcT_h\right)\right\vert_\infty\le 2B\sqrt{\frac{\log p}{n_{\mcT_h}+n_0}},
\end{eqnarray}
with probability at least $1-2p^{-1}$.
Moreover, under assumptions
\ref{C1} and \ref{C2}, for any $\bbeta\in\mathbb{R}^p$, we have
\begin{eqnarray*}
\label{grad}
\left\vert  \bS\left(\bbeta;\{0\}\cup\mcT_h\right)\right\vert_\infty\le 2B\sqrt{\frac{\log p}{n_{\mcT_h}+n_0}}+m_0B^2\sum_{k\in\{0\}\cup\mcT_h}\pi_k\vert\bbeta-\bbeta_k\vert_1,
\end{eqnarray*}
with probability at least $1-2p^{-1}$.
\end{lemma}

\begin{proof}
Denote $\bZ_{ki}=(Z_{ki,1},\ldots,Z_{ki,p})^{\rmT}$ for $i=1,\ldots, n_k$ with $k\in\{0\}\cup \mcT_h$.
Under assumption \ref{C2}, for any $j\in\{1,\ldots, p\}$, we have
$$ \left\vert\left\{\tau-I\left(Y_{ki}-\bZ_{ki}^{\rmT}\bbeta^*_{\mcT_h}\le 0\right)\right\}Z_{ki,j}\right\vert \le B. $$
Following from the Hoeffding inequality, we have
\begin{eqnarray}
\label{grad_hoef}
&&P\left(\left\vert\bS\left(\bbeta^*_{\mcT_h};\{0\}\cup\mcT_h\right)\right\vert_\infty>t\right)\nonumber\\
&=&P\Bigg(\Bigg\vert\frac{1}{n_{\mcT_h}+n_0}\sum_{k\in\{0\}\cup\mcT_h}\sum_{i=1}^{n_k}\left\{\tau-I\left(Y_{ki}-\bZ_{ki}^{\rmT}\bbeta^*_{\mcT_h}\ge 0\right)\right\}\bZ_{ki}\nonumber\\
&&\hspace{1in}-E\left[\frac{1}{n_{\mcT_h}+n_0}\sum_{k\in\{0\}\cup\mcT_h}\sum_{i=1}^{n_k}\left\{\tau-I\left(Y_{ki}-\bZ_{ki}^{\rmT}\bbeta^*_{\mcT_h}\le 0\right)\right\}\bZ_{ki}\right]\Bigg\vert_\infty>t\Bigg)\nonumber\\
&\le&\sum_{j=1}^pP\Bigg(\Bigg\vert\frac{1}{n_{\mcT_h}+n_0}\sum_{k\in\{0\}\cup\mcT_h}\sum_{i=1}^{n_k}\left\{\tau-I\left(Y_{ki}-\bZ_{ki}^{\rmT}\bbeta^*_{\mcT_h}\le 0\right)\right\}Z_{ki,j}\nonumber\\
&&\hspace{1in}-E\left[\frac{1}{n_{\mcT_h}+n_0}\sum_{k\in\{0\}\cup\mcT_h}\sum_{i=1}^{n_k}\left\{\tau-I\left(Y_{ki}-\bZ_{ki}^{\rmT}\bbeta^*_{\mcT_h}\le 0\right)\right\}Z_{ki,j}\right]\Bigg\vert>t\Bigg)\nonumber\\
&\le &2p\exp\left(-\frac{(n_{\mcT_h}+n_0)t^2}{2B^2}\right).
\end{eqnarray}
Set $t=2B\sqrt{\log p/(n_{\mcT_h}+n_0)}$ and (\ref{grad*}) is shown.
On the other hand, we have
$$ \left\vert\left\{\tau-I\left(Y_{ki}-\bZ_{ki}^{\rmT}\bbeta\le 0\right)\right\}Z_{ki,j}\right\vert \le B, $$
for any $\bbeta\in\mathbb{R}^p$.
Therefore, the Hoeffding inequality implies that
\begin{eqnarray*}
&&P\left(\left\vert \bS(\bbeta;\{0\}\cup\mcT_h)\right\vert_\infty>t+m_0B^2\sum_{k\in\{0\}\cup\mcT_h}\pi_k\vert\bbeta-\bbeta_k\vert_1\right)\\
&\le &P\Bigg(\Bigg\vert\frac{1}{n_{\mcT_h}+n_0}\sum_{k\in\{0\}\cup\mcT_h}\sum_{i=1}^{n_k}\left\{\tau-I\left(Y_{ki}-\bZ_{ki}^{\rmT}\bbeta\le 0\right)\right\}\bZ_{ki}\\
&&\hspace{1in}-E\left[\sum_{k\in\{0\}\cup\mcT_h}\pi_k\left\{\tau-I\left(Y_{ki}-\bZ_{ki}^{\rmT}\bbeta\le 0\right)\right\}\bZ_{ki}\right]\Bigg\vert_\infty\\&&
>t+m_0B^2\sum_{k\in\{0\}\cup\mcT_h}\pi_k\vert\bbeta-\bbeta_k\vert_1-\left\vert E\left[\sum_{k\in\{0\}\cup\mcT_h}\pi_k\left\{\tau-I\left(Y_{k}-\bZ_{k}^{\rmT}\bbeta\le 0\right)\right\}\bZ_{k}\right]\right\vert_\infty\Bigg)\\
&\le&P\Bigg(\Bigg\vert\frac{1}{n_{\mcT_h}+n_0}\sum_{k\in\{0\}\cup\mcT_h}\sum_{i=1}^{n_k}\left\{\tau-I\left(Y_{ki}-\bZ_{ki}^{\rmT}\bbeta\le 0\right)\right\}\bZ_{ki}\\
&&\hspace{1in}-E\left[\sum_{k\in\{0\}\cup\mcT_h}\pi_k\left\{\tau-I\left(Y_{ki}-\bZ_{ki}^{\rmT}\bbeta\le 0\right)\right\}\bZ_{ki}\right]\Bigg\vert_\infty\\&&
>t+\sum_{k\in\{0\}\cup\mcT_h}\pi_k\left[m_0B^2\vert\bbeta-\bbeta_k\vert_1-\left\vert E\left[\left\{F_k\left(\bZ_{k}^{\rmT}\bbeta_k\mid \bZ_{k}\right)-F_k\left(\bZ_{k}^{\rmT}\bbeta\mid \bZ_{k}\right)\right\}\bZ_{k}\right]\right\vert_\infty\right]\Bigg)\\
&\le&\sum_{j=1}^pP\Bigg(\Bigg\vert\frac{1}{n_{\mcT_h}+n_0}\sum_{k\in\{0\}\cup\mcT_h}\sum_{i=1}^{n_k}\left\{\tau-I\left(Y_{ki}-\bZ_{ki}^{\rmT}\bbeta\le 0\right)\right\}Z_{ki,j}\\
&&\hspace{1in}-E\left[\sum_{k\in\{0\}\cup{\mcT_h}}\pi_k\left\{\tau-I\left(Y_{ki}-\bZ_{ki}^{\rmT}\bbeta\le 0\right)\right\}Z_{ki,j}\right]\Bigg\vert>t\Bigg)\\
&\le &2p\exp\left(-\frac{(n_{\mcT_h}+n_0)t^2}{2B^2}\right).
\end{eqnarray*}
The proof is completed by setting $t=2B\sqrt{\log p/(n_{\mcT_h}+n_0)}$.
\end{proof}

{\bf Proof of Theorem \ref{theorem2.1}}
\begin{proof}
Denote $\wh\bu_{\mcT_h}=\wh\bbeta_{\mcT_h}-\bbeta^*_{\mcT_h}$.
Lemma \ref{lemmaA3} implies that the event
$\lambda_{\bbeta}\ge 2\left\vert \bS(\bbeta^*_{\mcT_h}; \{0\}\cup\mcT_h)\right\vert_\infty$ occurs with probability more than $1-2p^{-1}$,
on which we suppose that
\begin{eqnarray}
\label{suppose}
\left\vert\wh\bu_{\mcT_h}\right\vert_2&\ge& \frac{16B}{\kappa_0}\sqrt{\frac{s_0\log p}{n_{\mcT_h}+n_0}}+\frac6{\kappa_0}\lambda_{\bbeta}\sqrt{s_0}
+\sqrt{\left(4B\left(\frac{\log p}{n_{\mcT_h}+n_0}\right)^{1/2}+2\lambda_{\bbeta}\right)\frac{4Ch}{\kappa_0}}\nonumber\\
&=&\frac{(16+6C_{\bbeta})B}{\kappa_0}\sqrt{\frac{s_0\log p}{n_{\mcT_h}+n_0}}+2\sqrt{\frac{(4+2C_{\bbeta})BCh}{\kappa_0}}\left(\frac{\log p}{n_{\mcT_h}+n_0}\right)^{1/4}.
\end{eqnarray}
Since $L\left(\bbeta^*_{\mcT_h}+\wh\bu_{\mcT_h}; \{0\}\cup\mcT_h\right)+\lambda_{\bbeta}\left\vert\bbeta^*_{\mcT_h}+\wh\bu_{\mcT_h}\right\vert_1\le L\left(\bbeta^*_{\mcT_h};\{0\}\cup \mcT_h\right)+\lambda_{\bbeta}\left\vert\bbeta^*_{\mcT_h}\right\vert_1$
and  following from Lemma \ref{lemmaA1}, we have
\begin{eqnarray}
\label{cone}
&&\delta  L\left(\wh\bu_{\mcT_h}; \{0\}\cup\mcT_h\right)\nonumber\\&=& L\left(\bbeta^*_{\mcT_h}+\wh\bu_{\mcT_h}; \{0\}\cup\mcT_h\right)- L\left(\bbeta^*_{\mcT_h}; \{0\}\cup\mcT_h\right)-  \bS\left(\bbeta^*_{\mcT_h}; \{0\}\cup\mcT_h\right)^\rmT\wh\bu_{\mcT_h}\nonumber\\
&\le& -\lambda_{\bbeta}\left\vert\bbeta^*_{\mcT_h}+\wh\bu_{\mcT_h}\right\vert_1
+\lambda_{\bbeta}\left\vert\bbeta^*_{\mcT_h}\right\vert_1+\frac12\lambda_{\bbeta}\left\vert\wh\bu_{\mcT_h}\right\vert_1\nonumber\\
&= &-\lambda_{\bbeta}\left(\left\vert\wh\bbeta_{\mcT_h,\mcS_0}\right\vert_1+\left\vert\wh\bbeta_{\mcT_h,\mcS_0^c}\right\vert_1\right)
+\lambda_{\bbeta}\left(\left\vert\bbeta^*_{\mcT_h,\mcS_0}\right\vert_1
+\left\vert\bbeta^*_{\mcT_h,\mcS_0^c}\right\vert_1\right)+\frac12\lambda_{\bbeta}\left\vert\wh\bu_{\mcT_h}\right\vert_1\nonumber\\
&= &\lambda_{\bbeta}\left(\left\vert\bbeta^*_{\mcT_h,\mcS_0}\right\vert_1-\left\vert\wh\bbeta_{\mcT_h,\mcS_0}\right\vert_1\right)
-\lambda_{\bbeta}\left(\left\vert\bbeta^*_{\mcT_h,\mcS^c_0}\right\vert_1+\left\vert\wh\bbeta_{\mcT_h,\mcS^c_0}\right\vert_1\right)
+2\lambda_{\bbeta}\left\vert\bbeta^*_{\mcT_h,\mcS^c_0}\right\vert_1+\frac12\lambda_{\bbeta}\left\vert\wh\bu_{\mcT_h}\right\vert_1\nonumber\\ &\le&\lambda_{\bbeta}\left\vert\wh\bu_{\mcT_h,\mcS_0}\right\vert_1-\lambda_{\bbeta}\left\vert\wh\bu_{\mcT_h,\mcS_0^c}\right\vert_1
+2\lambda_{\bbeta}\left\vert\bbeta^*_{\mcT_h,\mcS_0^c}-\bbeta_{0,\mcS_0^c}\right\vert_1
+\frac12\lambda_{\bbeta}\left\vert\wh\bu_{\mcT_h,\mcS_0^c}\right\vert_1+\frac12\lambda_{\bbeta}\left\vert\wh\bu_{\mcT_h,\mcS_0}\right\vert_1\nonumber\\
&\le&\frac32\lambda_{\bbeta}\left\vert\wh\bu_{\mcT_h,\mcS_0}\right\vert_1-\frac12\lambda_{\bbeta}\left\vert\wh\bu_{\mcT_h,\mcS_0^c}\right\vert_1+2\lambda_{\bbeta} Ch.
\end{eqnarray}
Denote
\begin{eqnarray*}
G(\bu)=L\left(\bbeta^*_{\mcT_h}+\bu;\{0\}\cup\mcT_h\right)- L\left(\bbeta^*_{\mcT_h};\{0\}\cup\mcT_h\right)+\lambda_{\bbeta}\left(\left\vert\bbeta^*_{\mcT_h}+\bu\right\vert_1-\left\vert\bbeta^*_{\mcT_h}\right\vert_1\right).
\end{eqnarray*}
Since $\delta L(\wh\bu_{\mcT_h};\{0\}\cup\mcT_h)\ge0$, then
$\wh\bu_{\mcT_h}\in\mathcal{D}(4Ch,\mcS_0)$. Thus,
there exists $t\in(0,1]$ such that $\wt\bu_{\mcT_h}=t\wh\bu_{\mcT_h}$ with $$\left\vert\wt\bu_{\mcT_h}\right\vert_2=\frac{(16+6C_{\bbeta})B}{\kappa_0}\sqrt{\frac{s_0\log p}{n_{\mcT_h}+n_0}}+2\sqrt{\frac{(4+2C_{\bbeta})BCh}{\kappa_0}}\left(\frac{\log p}{n_{\mcT_h}+n_0}\right)^{1/4}.$$
The convexity of $G(\bu)$ entails that
\begin{eqnarray}
\label{Gu}
G\left(\wt\bu_{\mcT_h}\right)=G\left(t\wh\bu_{\mcT_h}+(1-t){\bf 0}\right)\le tG\left(\wh\bu_{\mcT_h}\right)+(1-t)G\left({\bf 0}\right)=tG\left(\wh\bu_{\mcT_h}\right)\le 0.
\end{eqnarray}
On the other hand, we have
\begin{eqnarray*}
\left\vert\wt\bu_{\mcT_h,\mcS_0^c}\right\vert_1=t\left\vert\wh\bu_{\mcT_h,\mcS_0^c}\right\vert_1
\le 3t\left\vert\wh\bu_{\mcT_h,\mcS_0}\right\vert_1+4t Ch
\le 3\left\vert\wt\bu_{\mcT_h,\mcS_0}\right\vert_1+4 Ch.
\end{eqnarray*}
Thus, it yields that $\wt \bu_{\mcT_h}\in\mathcal{D}(4Ch,\mcS_0)$.
Consequently, combining Lemmas \ref{lemmaA2} and \ref{lemmaA3}, it holds  that
\begin{eqnarray*}
G\left(\wt\bu_{\mcT_h}\right)&=& L\left(\bbeta^*_{\mcT_h}+\wt\bu_{\mcT_h};\{0\}\cup\mcT_h\right)- L\left(\bbeta^*_{\mcT_h};\{0\}\cup\mcT_h\right)+\lambda_{\bbeta}\left(\left\vert\bbeta^*_{\mcT_h}+\wt\bu_{\mcT_h}\right\vert_1-\left\vert\bbeta^*_{\mcT_h}\right\vert_1\right)\\
&=&\delta L\left(\wt \bu_{\mcT_h};\{0\}\cup\mcT_h\right)+ \bS\left(\bbeta^*_{\mcT_h};\{0\}\cup\mcT_h\right)^\rmT\wt\bu_{\mcT_h}+
\lambda_{\bbeta}\left(\left\vert\bbeta^*_{\mcT_h}+\wt\bu_{\mcT_h}\right\vert_1-\left\vert\bbeta^*_{\mcT_h}\right\vert_1\right)\\
&\ge& \min\left\{\frac{\kappa_0}4\left\vert\wt\bu_{\mcT_h}\right\vert^2_2,\frac{3\kappa_0^3}{16m_0}\left\vert\wt\bu_{\mcT_h}\right\vert_2\right\}-B\sqrt{\frac{\log p}{n_{\mcT_h}+n_0}}\left\vert\wt\bu_{\mcT_h}\right\vert_1-\frac12\lambda_{\bbeta}\left\vert\wt\bu_{\mcT_h}\right\vert_1
\\
&&-\lambda_{\bbeta}\left(\left\vert\bbeta^*_{\mcT_h}+\wt\bu_{\mcT_h}\right\vert_1-\left\vert\bbeta^*_{\mcT_h}\right\vert_1\right)\\
&\ge&
\frac{\kappa_0}4\left\vert\wt\bu_{\mcT_h}\right\vert^2_2-B\sqrt{\frac{\log p}{n_{\mcT_h}+n_0}}\left\vert\wt\bu_{\mcT_h}\right\vert_1-\frac32\lambda_{\bbeta}\left\vert\wt\bu_{\mcT_h,\mcS_0}\right\vert_1
+\frac12\lambda_{\bbeta}\left\vert\wt\bu_{\mcT_h,\mcS_0^c}\right\vert_1-2\lambda_{\bbeta} Ch
\end{eqnarray*}
with probability greater than $1-p^{-2}-2p^{-1}$,
where the last inequality follows from a similar form of (\ref{cone})
and condition
$$ \sqrt{\frac{s_0\log p}{n_{\mcT_h}+n_0}}+\sqrt{h}\left(\frac{\log p}{n_{\mcT_h}+n_0}\right)^{1/4}\to 0. $$
As a  result, we have
\begin{eqnarray*}
G\left(\wt\bu_{\mcT_h}\right)
&\ge&
\frac{\kappa_0}4\left\vert\wt\bu_{\mcT_h}\right\vert^2_2-B\sqrt{\frac{\log p}{n_{\mcT_h}+n_0}}\left(4\sqrt {s_0}\left\vert\wt\bu_{\mcT_h}\right\vert_2+4Ch\right)-\frac32\lambda_{\bbeta}\sqrt {s_0}\left\vert\wt\bu_{\mcT_h}\right\vert_2-2\lambda_{\bbeta} Ch\\
&=&\left\{\frac{\kappa_0}4\left\vert\wt\bu_{\mcT_h}\right\vert_2-\left(4+\frac32C_{\bbeta}\right)B\sqrt{\frac{s_0\log p}{n_{\mcT_h}+n_0}}\right\}\left\vert\wt\bu_{\mcT_h}\right\vert_2-(4+2C_{\bbeta})BCh\sqrt{\frac{\log p}{n_{\mcT_h}+n_0}}\\
&\ge &\sqrt{3\kappa_0BCh}\left(\frac{\log p}{n_{\mcT_h}+n_0}\right)^{1/4}\left\vert\wt\bu_{\mcT_h}\right\vert_2-(4+2C_{\bbeta})BCh\sqrt{\frac{\log p}{n_{\mcT_h}+n_0}}>0,
\end{eqnarray*}	
which contradicts with (\ref{Gu}) and hence (\ref{suppose}) does not hold.
Therefore, it implies that
\begin{eqnarray*}
\left\vert\wh\bu_{\mcT_h}\right\vert_2\lesssim \sqrt{\frac{s_0\log p}{n_{\mcT_h}+n_0}}+\sqrt{h}\left(\frac{\log p}{n_{\mcT_h}+n_0}\right)^{1/4}
\end{eqnarray*}
and
\begin{eqnarray*}
\left\vert\wh\bu_{\mcT_h}\right\vert_1=\left\vert\wh\bu_{\mcT_h,\mathcal{S}_0}\right\vert_1+\left\vert\wh\bu_{\mcT_h,\mathcal{S}_0^c}\right\vert_1\le
4\left\vert\wh\bu_{\mcT_h,\mathcal{S}_0}\right\vert_1+4Ch\le 4\sqrt {s_0}\left\vert\wh\bu_{\mcT_h,\mathcal{S}_0}\right\vert_2+4Ch
\lesssim s_0\sqrt{\frac{\log p}{n_{\mcT_h}+n_0}}+h,
\end{eqnarray*}
with probability at least $1-p^{-2}-2p^{-1}$.
\end{proof}

\subsection{Proof of Theorem \ref{theorem2.2}}

For simplicity, we denote
$
\delta  L(\bv;\{0\})=  L\left(\bbeta_0+\bv;\{0\}\right)- L\left(\bbeta_0;\{0\}\right)- \bS\left(\bbeta_0;\{0\}\right)^\rmT\bv.
$
\begin{lemma}
\label{lemmaA4}
Under assumptions \ref{C1} and \ref{C3}, it holds that
\begin{eqnarray*}
E\left(\delta  L(\bv+\bbeta-\bbeta_0;\{0\})\right)\bigg\vert_{\bbeta=\wh\bbeta_{\mcT_h}}\ge \min\left\{\frac{\kappa_0}4\vert\wh\bbeta_{\mcT_h}+\bv-\bbeta_0\vert_2^2,\frac{3\kappa_0^3}{16m_0}\vert\wh\bbeta_{\mcT_h}+\bv-\bbeta_0\vert_2\right\}.
\end{eqnarray*}
Moreover, under assumptions \ref{C1}--\ref{C3}, the sample version satisfies the following inequality
\begin{eqnarray*}
\delta L(\bv+\wh\bbeta_{\mcT_h}-\bbeta_0;\{0\})
&\ge& \min\left\{\frac{\kappa_0}4\left\vert\wh\bbeta_{\mcT_h}+\bv-\bbeta_0\right\vert_2^2,\frac{3\kappa_0^3}{16m_0}\left\vert\wh\bbeta_{\mcT_h}+\bv-\bbeta_0\right\vert_2\right\}
\\&&-B\sqrt{\frac{\log p}{n_0}}\left\vert\wh\bbeta_{\mcT_h}+\bv-\bbeta_0\right\vert_1
\end{eqnarray*}
with probability more than $1-p^{-2}$, for any $\wh\bbeta_{\mcT_h}+\bv-\bbeta_0\in \mathcal{D}(a_n,\mcS)$.
\end{lemma}

\begin{proof}
By the Knight equation,
it holds that
\begin{eqnarray*}
\delta  L(\bv+\bbeta-\bbeta_0;\{0\})&=&  L(\bbeta+\bv;\{0\})- L(\bbeta_0;\{0\})- \bS(\bbeta_0;\{0\})^\rmT(\bbeta+\bv-\bbeta_0)\\	&=&\frac1{n_0}\sum_{i=1}^{n_0}\Bigg[\rho_\tau\left(Y_{0i}-\bZ_{0i}^\rmT(\bbeta+\bv)\right)-\rho_\tau\left(Y_{0i}-\bZ_{0i}^\rmT\bbeta_0\right)\\
&&\hspace{1in}-\left\{\tau-I\left(Y_{0i}-\bZ_{0i}^\rmT\bbeta_0\le 0\right)\right\}\bZ_{0i}^\rmT(\bbeta+\bv-\bbeta_0)\Bigg]\\
&=&\frac1{n_0}\sum_{i=1}^{n_0}
\int^{\mathbf{Z}_{0i}^\rmT(\bbeta+\mathbf{v}-\bbeta_0)}_0
\left\{I\left(Y_{0i}-\bZ_{0i}^\rmT\bbeta_0\le z\right)-I\left(Y_{0i}-\bZ_{0i}^\rmT\bbeta_0\le 0\right)\right\}\rmd z.
\end{eqnarray*}
An analogous methodology in the proof of Lemma \ref{lemmaA2} indicates that
\begin{eqnarray*}
E\left(\delta  L(\bv+\bbeta-\bbeta_0;\{0\})\right)\bigg\vert_{\bbeta=\wh\bbeta_{\mcT_h}}\ge \min\left\{\frac{\kappa_0}4\left\vert\wh\bbeta_{\mcT_h}+\bv-\bbeta_0\right\vert_2^2,\frac{3\kappa_0^3}{16m_0}\left\vert\wh\bbeta_{\mcT_h}+\bv-\bbeta_0\right\vert_2\right\}.
\end{eqnarray*}
On the other hand, note that
$$ \int^{\mathbf{Z}_{0i}^\rmT(\bbeta+\mathbf{v}-\bbeta_0)}_0\left\{I\left(Y_{0i}-\bZ_{0i}^\rmT\bbeta_0\le z\right)-I\left(Y_{0i}-\bZ_{0i}^\rmT\bbeta_0\le 0\right)\right\}\rmd z $$
takes value in $[0,B\vert\bbeta+\mathbf{v}-\bbeta_0\vert_1]$, which implies that
the integration is a sub-Gaussian random variable with parameter at most $B\vert\bbeta+\mathbf{v}-\bbeta_0\vert_1/2$ 
\citep{wainwright2019high}.
Moreover, the Hoeffding inequality entails that
\begin{eqnarray*}
&&P\left(\delta  L(\bbeta+\bv-\bbeta_0;\{0\})-E\left\{\delta  L(\bbeta+\bv-\bbeta_0;\{0\})\right\}\le -t\right)\\&=&
P\Bigg(\frac1{n_0}\sum_{i=1}^{n_0}\Bigg\{\int^{\mathbf{Z}_{0i}^\rmT(\bbeta+\mathbf{v}-\bbeta_0)}_0\left\{I\left(Y_{0i}-\bZ_{0i}^\rmT\bbeta_0\le z\right)-I\left(Y_{0i}-\bZ_{0i}^\rmT\bbeta_0\le 0\right)\right\}\rmd z\\
&&\hspace{0.5in}-E\int^{\mathbf{Z}_{0}^\rmT(\bbeta+\mathbf{v}-\bbeta_0)}_0\left\{I\left(Y_{0}-\bZ_{0}^\rmT\bbeta_0\le z\right)-I\left(Y_{0}-\bZ_{0}^\rmT\bbeta_0\le 0\right)\right\}\rmd z\le -t\Bigg\}\Bigg)\\
&\le &\exp\left(-\frac{2n_0t^2}{B^2\vert\bbeta+\mathbf{v}-\bbeta_0\vert_1^2}\right),
\end{eqnarray*}
for any $\bbeta+\bv-\bbeta_0\in \mathcal{D}(a_n,\mcS)$.
Taking $t=B\vert\bbeta+\mathbf{v}-\bbeta_0\vert_1\sqrt{\log p/n_0}$ and $\bbeta=\wh\bbeta_{\mcT_h}$, and it concludes that
\begin{eqnarray*}
&&\delta L\left(\wh\bbeta_{\mcT_h}+\bv-\bbeta_0;\{0\}\right)\\&\ge& E\left\{\delta L(\bbeta+\bv-\bbeta_0;\{0\})\right\}\bigg\vert_{\balpha=\wh\bbeta_{\mcT_h}}-B\sqrt{\frac{\log p}{n_0}}\left\vert\wh\bbeta_{\mcT_h}+\bv-\bbeta_0\right\vert_1\\
&\ge& \min\left\{\frac{\kappa_0}4\left\vert\wh\bbeta_{\mcT_h}+\bv-\bbeta_0\right\vert_2^2,\frac{3\kappa_0^3}{16m_0}\left\vert\wh\bbeta_{\mcT_h}+\bv-\bbeta_0\right\vert_2\right\}
-B\sqrt{\frac{\log p}{n_0}}\left\vert\wh\bbeta_{\mcT_h}+\bv-\bbeta_0\right\vert_1
\end{eqnarray*}
with probability more than $1-p^{-2}$.
\end{proof}

\begin{lemma}
\label{lemmaA5}
Under assumption \ref{C2}, it holds that
\begin{eqnarray*}
\left\vert  \bS(\bbeta_0;\{0\})\right\vert_\infty\le 2B\sqrt{\frac{\log p}{n_0}},
\end{eqnarray*}
with probability at least $1-2p^{-1}$.
Moreover, under assumptions \ref{C1}--\ref{C2}, for any $\bbeta\in\mathbb{R}^p$, we have
\begin{eqnarray*}
\left\vert \bS(\bbeta;\{0\})\right\vert_\infty\le 2B\sqrt{\frac{\log p}{n_0}}+m_0B^2\vert\bbeta-\bbeta_0\vert_1
\end{eqnarray*}
with probability at least $1-2p^{-1}$.
\end{lemma}

\begin{proof}
Under condition \ref{C2}, for any $j\in\{1,\ldots, p\}$, we have
$$ \left\vert\left\{\tau-I\left(Y_{0i}-\bZ_{0i}^\rmT\bbeta_0\le 0\right)\right\}Z_{0i,j}\right\vert \le B. $$
According to the Hoeffding inequality, we have
\begin{eqnarray}
\label{S0}
P\left(\left\vert\bS(\bbeta;\{0\})\right\vert_\infty>t\right)
&=&P\Bigg(\Bigg\vert\frac{1}{n_0}\sum_{i=1}^{n_0}\left\{\tau-I\left(Y_{0i}-\bZ_{0i}^\rmT\bbeta_0\le 0\right)\right\}\bZ_{0i}\nonumber\\
&&\hspace{1in}-E\left[\frac{1}{n_0}\sum_{i=1}^{n_0}\left\{\tau-I\left(Y_{0i}-\bZ_{0i}^\rmT\bbeta_0\le 0\right)\right\}\bZ_{0i}\right]\Bigg\vert_\infty>t\Bigg)\nonumber\\
&\le&\sum_{j=1}^pP\Bigg(\Bigg\vert\frac{1}{n_0}\sum_{i=1}^{n_k}\left\{\tau-I\left(Y_{0i}-\bZ_{0i}^\rmT\bbeta_0\le 0\right)\right\}Z_{0i,j}\nonumber\\
&&\hspace{1in}-E\left[\frac{1}{n_0}\sum_{i=1}^{n_k}\left\{\tau-I\left(Y_{0i}-\bZ_{0i}^\rmT\bbeta_0\le 0\right)\right\}Z_{0i,j}\right]\Bigg\vert>t\Bigg)\nonumber\\
&\le &2p\exp\left(-\frac{n_0t^2}{2B^2}\right).
\end{eqnarray}
Moreover, we have
\begin{eqnarray*}
&&P\left(\left\vert \bS(\bbeta;\{0\})\right\vert_\infty>t+m_0B^2\vert\bbeta-\bbeta_0\vert_1\right)\\
&\le&P\Bigg(\Bigg\vert\frac{1}{n_0}\sum_{i=1}^{n_0}\left\{\tau-I\left(Y_{0i}-\bZ_{0i}^\rmT\bbeta\le 0\right)\right\}\bZ_{0i}-E\left[\frac{1}{n_0}\sum_{i=1}^{n_0}\left\{\tau-I\left(Y_{0i}-\bZ_{0i}^\rmT\bbeta\le 0\right)\right\}\bZ_{0i}\right]\Bigg\vert_\infty\\&&
\hspace{0.8in}>t+m_0B^2\vert\bbeta-\bbeta_0\vert_1-\left\vert E\left[\frac{1}{n_0}\sum_{i=1}^{n_0}\left\{\tau-I\left(Y_{0}-\bZ_{0}^\rmT\bbeta\le 0\right)\right\}\bZ_{0}\right]\right\vert_\infty\Bigg)\\
&=&P\Bigg(\Bigg\vert\frac{1}{n_0}\sum_{i=1}^{n_0}\left\{\tau-I\left(Y_{0i}-\bZ_{0i}^\rmT\bbeta\le 0\right)\right\}\bZ_{0i}-E\left[\frac{1}{n_0}\sum_{i=1}^{n_k}\left\{\tau-I\left(Y_{0i}-\bZ_{0i}^\rmT\bbeta\le 0\right)\right\}\bZ_{0i}\right]\Bigg\vert_\infty\\&&
\hspace{0.8in}>t+m_0B^2\vert\bbeta-\bbeta_0\vert_1-\left\vert E\left[\frac{1}{n_0}\sum_{i=1}^{n_0}
\left\{F_0\left(\bZ_{0}^\rmT\bbeta_0\mid \bZ_{0}\right)-F_0\left(\bZ_{0}^\rmT\bbeta\mid \bZ_{0}\right)\right\}\bZ_{0}\right]\right\vert_\infty\Bigg)\\
&\le&\sum_{j=1}^pP\Bigg(\Bigg\vert\frac{1}{n_0}\sum_{i=1}^{n_0}\left\{\tau-I\left(Y_{0i}-\bZ_{0i}^\rmT\bbeta\le 0\right)\right\}Z_{0i,j}-E\left[\left\{\tau-I\left(Y_{0i}-\bZ_{0i}^\rmT\bbeta\le 0\right)\right\}Z_{0i,j}\right]\Bigg\vert>t\Bigg)\\
&\le &2p\exp\left(-\frac{n_0t^2}{2B^2}\right).
\end{eqnarray*}
Set $t=2B\sqrt{\log p/n_0}$
and the proof is completed.
\end{proof}

{\bf Proof of Theorem \ref{theorem2.2}}
\begin{proof}
When $\lambda_{\bdelta}\ge 2\left\vert \bS\left(\bbeta_0;\{0\}\right)\right\vert_\infty$,
under assumptions \ref{C1} and \ref{C4}, it holds by Lemmas \ref{lemmaA1} and \ref{lemmaA5} that
\begin{eqnarray}
\label{D'}
\delta L\left(\wh\bdelta_{\mcT_h}+\wh\bbeta_{\mcT_h}-\bbeta_0;\{0\}\right)&=&  L\left(\wh\bbeta_0;\{0\}\right)- L\left(\bbeta_0;\{0\}\right)- \bS\left(\bbeta_0;\{0\}\right)^\rmT\left(\wh\bbeta_0-\bbeta_0\right)
\nonumber\\&\le &\lambda_{\bdelta}\left\vert\bbeta_0-\wh\bbeta_{\mcT_h}\right\vert_1-\lambda_{\bdelta}\left\vert\wh\bbeta_0-\wh\bbeta_{\mcT_h}\right\vert_1
+\frac12\lambda_{\bdelta}\left\vert\wh\bbeta_0-\bbeta_0\right\vert_1\nonumber\\
&\le &\lambda_{\bdelta}\left\vert\bbeta_0-\wh\bbeta_{\mcT_h}\right\vert_1-\lambda_{\bdelta}\left\vert\wh\bbeta_0-\wh\bbeta_{\mcT_h}\right\vert_1
+\frac12\lambda_{\bdelta}\left(\left\vert\wh\bbeta_0-\wh\bbeta_{\mcT_h}\right\vert_1+\left\vert\bbeta_0-\wh\bbeta_{\mcT_h}\right\vert_1\right)\nonumber\\
&=&\frac32\lambda_{\bdelta}\left\vert\bbeta_0-\wh\bbeta_{\mcT_h}\right\vert_1-\frac12\lambda_{\bdelta}\left\vert\wh\bbeta_0-\wh\bbeta_{\mcT_h}\right\vert_1\nonumber\\
&\le&\frac32\lambda_{\bdelta}\left\vert\bbeta_0-\bbeta^*_{\mcT_h}\right\vert_1+\frac32\lambda_{\bdelta}\left\vert\wh\bu_{\mcT_h}\right\vert_1
-\frac12\lambda_{\bdelta}\left\vert\wh\bbeta_0-\wh\bbeta_{\mcT_h}\right\vert_1\nonumber\\
&\le &\frac32\lambda_{\bdelta} Ch+\frac32\lambda_{\bdelta} \left\vert\wh\bu_{\mcT_h}\right\vert_1-\frac12\lambda_{\bdelta}\left\vert\wh\bbeta_0-\wh\bbeta_{\mcT_h}\right\vert_1,
\end{eqnarray}
with probability at least $1-2p^{-1}$.
Note the fact that
$\delta L\left(\wh\bdelta_{\mcT_h}+\wh\bbeta_{\mcT_h}-\bbeta_0;\{0\}\right)\ge 0$, and it holds that $\left\vert\wh\bbeta_0-\wh\bbeta_{\mcT_h}\right\vert_1\le3Ch+3\left\vert\wh\bu_{\mcT_h}\right\vert_1$.
As a result, we have
\begin{eqnarray*}
\left\vert\wh\bbeta_0-\bbeta_0\right\vert_1&\le& \left\vert\bbeta_0-\wh\bbeta_{\mcT_h}\right\vert_1+\left\vert\wh\bbeta_0-\wh\bbeta_{\mcT_h}\right\vert_1
\\&\le& \left(\left\vert\bbeta_0-\bbeta^*_{\mcT_h}\right\vert_1+\left\vert
\wh\bu_{\mcT_h}\right\vert_1\right)+\left(3Ch+3\left\vert\wh\bu_{\mcT_h}\right\vert_1\right)
\\&=&4Ch+4\left\vert\wh\bu_{\mcT_h}\right\vert_1,
\end{eqnarray*}
which along with Theorem 1 concludes (\ref{beta-l1-bound}).
On the other hand, (\ref{D'}) also implies that
\begin{eqnarray*}
0&\le &\delta\wh L^{(0)}\left(\wh\bdelta_{\mcT_h},\wh\bbeta_{\mcT_h}\right)\\
&\le&\lambda_{\bdelta}\left\vert\bbeta_0-\wh\bbeta_{\mcT_h}\right\vert_1-\lambda_{\bdelta}\left\vert\wh\bbeta_0-\wh\bbeta_{\mcT_h}\right\vert_1
+\frac12\lambda_{\bdelta}\left\vert\wh\bbeta_0-\bbeta_0\right\vert_1\\
&=& \lambda_{\bdelta}\left\vert\bbeta_{0,\mathcal{S}_0}-\wh\bbeta_{\mcT_h,\mathcal{S}_0}\right\vert_1+
\lambda_{\bdelta}\left\vert\bbeta_{0,\mathcal{S}_0^c}-\wh\bbeta_{\mcT_h,\mathcal{S}_0^c}\right\vert_1
-\lambda_{\bdelta}\left\vert\wh\bbeta_{0,\mathcal{S}_0}-\wh\bbeta_{\mcT_h,\mathcal{S}_0}\right\vert_1
\\&&-\lambda_{\bdelta}\left\vert\wh\bbeta_{0,\mathcal{S}_0^c}-\wh\bbeta_{\mcT_h,\mathcal{S}_0^c}\right\vert_1
+\frac12\lambda_{\bdelta}\left\vert\wh\bbeta_0-\bbeta_0\right\vert_1\\
&\le &\lambda_{\bdelta}\left(\left\vert\bbeta_{0,\mathcal{S}_0}-\wh\bbeta_{\mcT_h,\mathcal{S}_0}\right\vert_1
-\left\vert\wh\bbeta_{0,\mathcal{S}_0}-\wh\bbeta_{\mcT_h,\mathcal{S}_0}\right\vert_1\right)-
\lambda_{\bdelta}\left(\left\vert\bbeta_{0,\mathcal{S}_0^c}-\wh\bbeta_{\mcT_h,\mathcal{S}_0^c}\right\vert_1
+\left\vert\wh\bbeta_{0,\mathcal{S}_0^c}-\wh\bbeta_{\mcT_h,\mathcal{S}_0^c}\right\vert_1\right)\\
&&+2\lambda_{\bdelta}\left\vert\bbeta_{0,\mathcal{S}_0^c}-\wh\bbeta_{\mcT_h,\mathcal{S}_0^c}\right\vert_1
+\frac12\lambda_{\bdelta}\left\vert\wh\bbeta_0-\bbeta_0\right\vert_1\\
&\le &\lambda_{\bdelta}\left\vert\bbeta_{0,\mathcal{S}_0}-\wh\bbeta_{0,\mathcal{S}_0}\right\vert_1-
\lambda_{\bdelta}\left\vert\bbeta_{0,\mathcal{S}_0^c}-\wh\bbeta_{0,\mathcal{S}_0^c}\right\vert_1
+2\lambda_{\bdelta}\left\vert\bbeta_{0,\mathcal{S}_0^c}-\bbeta^*_{\mcT_h,\mathcal{S}_0^c}\right\vert_1
\\&&+2\lambda_{\bdelta}\left\vert\bbeta^*_{\mcT_h,\mathcal{S}_0^c}
-\wh\bbeta_{\mcT_h,\mathcal{S}_0^c}\right\vert_1
+\frac12\lambda_{\bdelta}\left\vert\wh\bbeta_0-\bbeta_0\right\vert_1\\
&\le &\frac32\lambda_{\bdelta}\left\vert\bbeta_{0,\mathcal{S}_0}-\wh\bbeta_{0,\mathcal{S}_0}\right\vert_1-
\frac12\lambda_{\bdelta}\left\vert\bbeta_{0,\mathcal{S}_0^c}-\wh\bbeta_{0,\mathcal{S}_0^c}\right\vert_1
+2\lambda_{\bdelta} Ch+2\lambda_{\bdelta}\left\vert\bbeta^*_{\mcT_h,\mathcal{S}_0^c}-\wh\bbeta_{\mcT_h,\mathcal{S}_0^c}\right\vert_1
\\
&\le &\frac32\lambda_{\bdelta}\left\vert\bbeta_{0,\mathcal{S}_0}-\wh\bbeta_{0,\mathcal{S}_0}\right\vert_1-
\frac12\lambda_{\bdelta}\left\vert\bbeta_{0,\mathcal{S}_0^c}-\wh\bbeta_{0,\mathcal{S}_0^c}\right\vert_1
+10 \lambda_{\bdelta} Ch\\&&+
\frac{(96+36C_{\bbeta})\lambda_{\bdelta} Bs_0}{\kappa_0}\sqrt\frac{\log p}{n_{\mcT_h}+n_0}+12\lambda_{\bdelta} \sqrt{\frac{(4+2C_{\bbeta})BChs_0}{\kappa_0}}\left(\frac{\log p}{n_{\mcT_h}+n_0}\right)^{1/4}
\end{eqnarray*}
with probability at least $1-p^{-2}-4p^{-1}$ by the contradiction of (\ref{suppose}) and the fact of $\wh\bu_{\mcT_h}\in\mathcal{D}(4Ch,\mcS_0)$, which yields that $$\wh\bbeta_0-\bbeta_0\in\mathcal{D}\left(20 Ch+
\frac{(192+72C_{\bbeta}) Bs_0}{\kappa_0}\sqrt\frac{\log p}{n_{\mcT_h}+n_0}+24 \sqrt{\frac{(4+2C_{\bbeta})BChs_0}{\kappa_0}}\left(\frac{\log p}{n_{\mcT_h}+n_0}\right)^{1/4},\mcS_0\right)$$ in probability.
Thus, using Lemma \ref{lemmaA4}, it holds with probability more than $1-2p^{-2}-4p^{-1}$ that
\begin{eqnarray*}
\min\left\{\frac{\kappa_0}4\left\vert\wh\bbeta_0-\bbeta_0\right\vert^2_2,\frac{3\kappa_0^3}{16m_0}\left\vert\wh\bbeta_0-\bbeta_0\right\vert_2\right\}
&\le& B\sqrt{\frac{\log p}{n_0}}\left(4Ch+4\left\vert\wh\bu_{\mcT_h}\right\vert_1\right)
+ \frac32\lambda_{\bdelta} Ch+\frac32\lambda_{\bdelta} \left\vert\wh\bu_{\mcT_h}\right\vert_1\\
&\lesssim& \sqrt{\frac{\log p}{n_0}}\left(s_0\sqrt{\frac{\log p}{n_{\mcT_h}+n_0}}+h\right).
\end{eqnarray*}
Since $$ \sqrt{\frac{\log p}{n_0}}\left(s_0\sqrt{\frac{\log p}{n_{\mcT_h}+n_0}}+h\right)\to 0, $$
then
\begin{eqnarray*}
\min\left\{\frac{\kappa_0}4\left\vert\wh\bbeta_0-\bbeta_0\right\vert^2_2,\frac{3\kappa_0^3}{16m_0}\left\vert\wh\bbeta_0-\bbeta_0\right\vert_2\right\}=
\frac{\kappa_0}4\left\vert\wh\bbeta_0-\bbeta_0\right\vert^2_2
\end{eqnarray*}
if the sample size is sufficiently large.
As a result, the conclusion is drawn by (\ref{beta-l2-bound}) that
\begin{equation*}
\left\vert\wh\bbeta_0-\bbeta_0\right\vert_2\lesssim
\sqrt h\left(\frac{\log p}{n_0}\right)^{1/4}+\left(\frac{s_0\log p}{n_0}\right)^{1/4}\left(\frac{s_0\log p}{n_{\mcT_h}+n_0}\right)^{1/4}
\end{equation*}
with probability tending to $1$.
\end{proof}

\subsection{Proof of Theorem \ref{theorem2.3}}
\begin{proof}
For any $\bbeta\in\{\bbeta: \vert\bbeta-\bbeta_0\vert_1\le h\}$,
denote $\wh\bu=\wh\bbeta_{\mcT_h}-\bbeta$.
Lemma \ref{lemmaA3} implies that the event $\lambda_{\bbeta}\ge 2\left\vert\bS(\bbeta; \{0\}\cup\mcT_h)\right\vert_\infty$ occurs with probability more than $1-2p^{-1}$,
on which we suppose that
\begin{eqnarray}
\label{suppose*}
\left\vert\wh\bu\right\vert_2\ge \frac{112B}{\kappa_0}\sqrt{\frac{s_0\log p}{n_{\mcT_h}+n_0}}+4\sqrt{\frac{3Bh}{\kappa_0}}\left(\frac{\log p}{n_{\mcT_h}+n_0}\right)^{1/4} +\frac{96m_0B^2\sqrt {s_0}h}{\kappa_0}+4\sqrt{\frac{2m_0}{\kappa_0}}B h.
\end{eqnarray}
Since $L\left(\balpha+\wh\bu, \{0\}\cup\mcT_h\right)+\lambda_{\bbeta}\left\vert\balpha+\wh\bu\right\vert_1\le L\left(\balpha, \{0\}\cup\mcT_h\right)+\lambda_{\bbeta}\vert\balpha\vert_1$,
then we have
\begin{eqnarray}
\label{cone*}
0&\le& L\left(\bbeta+\wh\bu; \{0\}\cup\mcT_h\right)- L\left(\bbeta; \{0\}\cup\mcT_h\right)-\bS\left(\bbeta; \{0\}\cup\mcT_h\right)^\rmT\wh\bu\nonumber\\
&\le& -\lambda_{\bbeta}\left\vert\bbeta+\wh\bu\right\vert_1+\lambda_{\bbeta}\left\vert\bbeta\right\vert_1+\frac12\lambda_{\bbeta}\left\vert\wh\bu\right\vert_1\nonumber\\
&= &-\lambda_{\bbeta}\left(\left\vert\wh\bbeta_{\mcT_h,\mcS_0}\right\vert_1+\left\vert\wh\bbeta_{\mcT_h,\mcS_0^c}\right\vert_1\right)
+\lambda_{\bbeta}\left(\left\vert\bbeta_{\mcS_0}\right\vert_1
+\left\vert\bbeta_{\mcS_0^c}\right\vert_1\right)+\frac12\lambda_{\bbeta}\left\vert\wh\bu\right\vert_1\nonumber\\
&= &\lambda_{\bbeta}\left(\left\vert\bbeta_{\mcS_0}\right\vert_1-\left\vert\wh\bbeta_{\mcT_h,\mcS_0}\right\vert_1\right)
-\lambda_{\bbeta}\left(\left\vert\bbeta_{\mcS_0^c}\right\vert_1
+\left\vert\wh\bbeta_{\mcT_h,\mcS_0^c}\right\vert_1\right)
+2\lambda_{\bbeta}\left\vert\bbeta_{\mcS_0^c}\right\vert_1+\frac12\lambda_{\bbeta}\left\vert\wh\bu\right\vert_1\nonumber\\
&\le&\lambda_{\bbeta}\left\vert\wh\bu_{\mcS_0}\right\vert_1-\lambda_{\bbeta}\left\vert\wh\bu_{\mcS_0^c}\right\vert_1
+2\lambda_{\bbeta}\left\vert\bbeta_{\mcS_0^c}-\bbeta_{0,\mcS_0^c}\right\vert_1
+\frac12\lambda_{\bbeta}\left\vert\wh\bu_{\mcS_0^c}\right\vert_1+\frac12\lambda_{\bbeta}\left\vert\wh\bu_{\mcS_0}\right\vert_1\nonumber\\
&\le&\frac32\lambda_{\bbeta}\left\vert\wh\bu_{\mcS_0}\right\vert_1-\frac12\lambda_{\bbeta}\left\vert\wh\bu_{\mcS_0^c}\right\vert_1+2\lambda_{\bbeta} h
\end{eqnarray}
with probability more than $1-2p^{-1}$.
Denote
\begin{eqnarray*}
G_{\bbeta}(\bu)= L(\bbeta+\bu;\{0\}\cup\mcT_h)- L(\bbeta;\{0\}\cup\mcT_h)+\lambda_{\bbeta}(\vert\bbeta+\bu\vert_1-\vert\bbeta\vert_1).
\end{eqnarray*}
Since $\delta L(\wh\bu;\{0\}\cup\mcT_h)\ge0$, then
$\wh\bu\in\mathcal{D}(4h,\mcS_0)$. Thus,
there exists $t\in(0,1]$ such that $\wt\bu=t\wh\bu$ with
$$\vert\wt\bu\vert_2=\frac{112B}{\kappa_0}\sqrt{\frac{s_0\log p}{n_{\mcT_h}+n_0}}+4\sqrt{\frac{3Bh}{\kappa_0}}\left(\frac{\log p}{n_{\mcT_h}+n_0}\right)^{1/4} +\frac{96m_0B^2\sqrt {s_0}h}{\kappa_0}+4\sqrt{\frac{2m_0}{\kappa_0}}B h.$$
The convexity of $G_{\bbeta}(\bu)$ entails that
\begin{eqnarray}
\label{Gu*}
G_{\bbeta}(\wt\bu)=G_{\bbeta}(t\wh\bu+(1-t){\bf 0})\le tG_{\bbeta}(\wh\bu)+(1-t)G_{\bbeta}({\bf 0})=tG_{\bbeta}(\wh\bu)\le 0
\end{eqnarray}
for any $t>0$.
On the other hand, we have
\begin{eqnarray*}
	\left\vert\wt\bu_{\mcS_0^c}\right\vert_1=t\left\vert\wh\bu_{\mcS_0^c}\right\vert_1
	\le 3t\left\vert\wh\bu_{\mcS_0}\right\vert_1+4t h
	\le 3\left\vert\wt\bu_{\mcS_0}\right\vert_1+4 h.
\end{eqnarray*}
Thus, it yields that $\wt \bu\in\mathcal{D}(4h,\mcS_0)$.
Mimicking the proof of Lemma \ref{lemmaA2}, we conclude that
\begin{eqnarray*}
&&L(\bbeta+\wt\bu; \{0\}\cup\mcT_h)-L(\bbeta; \{0\}\cup\mcT_h)-\bS(\bbeta; \{0\}\cup\mcT_h)^\rmT\wt\bu
\\&\ge& \min\left\{\frac{\kappa_0}4\vert\wt\bu\vert_2^2,\frac{3\kappa_0^3}{16m_0}\vert\wt\bu\vert_2\right\}-B\sqrt{\frac{\log p}{n_{\mcT_h}+n_0}}\vert\wt\bu\vert_1
\end{eqnarray*}
with probability more than $1-p^{-2}$.
Consequently, coupled with Lemma \ref{lemmaA3}, for sufficiently large sample size, it holds with probability more than $1-p^{-2}-2p^{-1}$ that
\begin{eqnarray*}
G_{\bbeta}(\wt\bu)&=&L(\bbeta+\wt\bu;\{0\}\cup\mcT_h)- L(\bbeta;\{0\}\cup\mcT_h)+\lambda_{\bbeta}(\vert\bbeta+\wt\bu\vert_1-\vert\bbeta\vert_1)\\
&\ge& \min\left\{\frac{\kappa_0}4\vert\wt\bu\vert^2_2,\frac{3\kappa_0^3}{16m_0}\vert\wt\bu\vert_2\right\}-B\sqrt{\frac{\log p}{n_{\mcT_h}+n_0}}\vert\wt\bu\vert_1-\frac12\lambda_{\bbeta}\vert\wt\bu\vert_1-\lambda_{\bbeta}(\vert\bbeta+\wt\bu\vert_1-\vert\bbeta\vert_1)\\
&\ge&
\frac{\kappa_0}4\vert\wt\bu\vert^2_2-B\sqrt{\frac{\log p}{n_{\mcT_h}+n_0}}\vert\wt\bu\vert_1-\frac32\lambda_{\bbeta}\vert\wt\bu_{\mcS_0}\vert_1+\frac12\lambda_{\bbeta}\vert\wt\bu_{\mcS_0^c}\vert_1-2\lambda_{\bbeta} h,
\end{eqnarray*}
where the last inequality is followed from a similar form of (\ref{cone*})
and the fact that
$$ \sqrt{\frac{s_0\log p}{n_{\mcT_h}+n_0}}+\sqrt {s_0}h\to0. $$
As a result, we have with probability at least $1-p^{-2}-2p^{-1}$ that
\begin{eqnarray*}
G_{\bbeta}(\wt\bu)
&\ge&
\frac{\kappa_0}4\vert\wt\bu\vert^2_2-B\sqrt{\frac{\log p}{n_{\mcT_h}+n_0}}(4\sqrt {s_0}\vert\wt\bu\vert_2+4h)-6\lambda_{\bbeta}\sqrt {s_0}\vert\wt\bu\vert_2-2\lambda_{\bbeta} h\\
&=&\left\{\frac{\kappa_0}4\vert\wt\bu\vert_2-28B\sqrt{\frac{s_0\log p}{n_{\mcT_h}+n_0}}-24m_0B^2h\sqrt {s_0}\right\}\vert\wt\bu\vert_2-12Bh\sqrt{\frac{\log p}{n_{\mcT_h}+n_0}}-8m_0B^2h^2\\
&\ge &\left(\sqrt{3\kappa_0Bh}\left(\frac{\log p}{n_{\mcT_h}+n_0}\right)^{1/4}+\sqrt{\kappa_0m_0}Bh\right)\vert\wt\bu\vert_2-12Bh\sqrt{\frac{\log p}{n_{\mcT_h}+n_0}}-8m_0B^2h^2>0,
\end{eqnarray*}	
which contradicts with (\ref{Gu*}) and hence (\ref{suppose*}) does not hold.
Therefore, it implies that
\begin{eqnarray*}
\vert\wh\bu\vert_2\lesssim \sqrt{\frac{s_0\log p}{n_{\mcT_h}+n_0}}+\sqrt {s_0}h,
\end{eqnarray*}
and
\begin{eqnarray}
\label{u-bound*}
\vert\wh\bu\vert_1\le
4\vert\wh\bu_{\mathcal{S}_0}\vert_1+4h\le 4\sqrt {s_0}\vert\wh\bu_{\mathcal{S}_0}\vert_2+4h
\lesssim s_0\sqrt{\frac{\log p}{n_{\mcT_h}+n_0}}+s_0h,
\end{eqnarray}
with probability at least $1-p^{-2}-2p^{-1}$.

When $\lambda_{\bdelta}\ge 2\left\vert\bS(\bbeta_0;\{0\})\right\vert_\infty$, it holds that
\begin{eqnarray}
\label{D'_debias}
\delta L\left(\wh\bdelta_{\mcT_h}+\wh\bbeta_{\mcT_h}-\bbeta_0;\{0\}\right)&=& L\left(\wh\bbeta_0;\{0\}\right)-L\left(\bbeta_0;\{0\}\right)- \bS\left(\bbeta_0;\{0\}\right)^\rmT\left(\wh\bbeta_0-\bbeta_0\right)
\nonumber\\&\le &\lambda_{\bdelta}\left\vert\bbeta_0-\wh\bbeta_{\mcT_h}\right\vert_1
-\lambda_{\bdelta}\left\vert\wh\bbeta_0-\wh\bbeta_{\mcT_h}\right\vert_1+\frac12\lambda_{\bdelta}\left\vert\wh\bbeta_0-\bbeta_0\right\vert_1\nonumber\\
&\le &\lambda_{\bdelta}\left\vert\bbeta_0-\wh\bbeta_{\mcT_h}\right\vert_1-\lambda_{\bdelta}\left\vert\wh\bbeta_0-\wh\bbeta_{\mcT_h}\right\vert_1
+\frac12\lambda_{\bdelta}\left(\left\vert\wh\bbeta_0-\wh\bbeta_{\mcT_h}\right\vert_1+\left\vert\bbeta_0-\wh\bbeta_{\mcT_h}\right\vert_1\right)\nonumber\\
&=&\frac32\lambda_{\bdelta}\left\vert\bbeta_0-\wh\bbeta_{\mcT_h}\right\vert_1-\frac12\lambda_{\bdelta}\left\vert\wh\bbeta_0-\wh\bbeta_{\mcT_h}\right\vert_1\nonumber\\
&\le&\frac32\lambda_{\bdelta}\left\vert\bbeta_0-\bbeta\right\vert_1
+\frac32\lambda_{\bdelta}\left\vert\wh\bu\right\vert_1-\frac12\lambda_{\bdelta}\left\vert\wh\bbeta_0-\wh\bbeta_{\mcT_h}\right\vert_1\nonumber\\
&\le &\frac32\lambda_{\bdelta} h+\frac32\lambda_{\bdelta} \vert\wh\bu\vert_1-\frac12\lambda_{\bdelta}\left\vert\wh\bbeta_0-\wh\bbeta_{\mcT_h}\right\vert_1.
\end{eqnarray}
Note the fact that
$\delta L\left(\wh\bdelta_{\mcT_h}+\wh\bbeta_{\mcT_h}-\bbeta_0;\{0\}\right)\ge 0$, and it holds that $\left\vert\wh\bbeta_0-\wh\bbeta_{\mcT_h}\right\vert_1\le3h+3\vert\wh\bu\vert_1$.
As a result, we have
\begin{eqnarray*}
\left\vert\wh\bbeta_0-\bbeta_0\right\vert_1\le \left\vert\bbeta_0-\wh\bbeta_{\mcT_h}\right\vert_1+\left\vert\wh\bbeta_0-\wh\bbeta_{\mcT_h}\right\vert_1\le \left(\left\vert\bbeta_0-\bbeta\right\vert_1+\left\vert
\wh\bu\right\vert_1\right)+(3h+3\vert\wh\bu\vert_1)\le 4h+4\vert\wh\bu\vert_1,
\end{eqnarray*}
which along with (\ref{u-bound*}) implies (\ref{beta-l1-bound*}).
On the other hand, (\ref{D'_debias})  implies that
\begin{eqnarray*}
0&\le &\delta L\left(\wh\bbeta_0-\bbeta_0;\{0\}\right)\\
&\le&\lambda_{\bdelta}\left\vert\bbeta_0-\wh\bbeta_{\mcT_h}\right\vert_1-\lambda_{\bdelta}\left\vert\wh\bbeta_0-\wh\bbeta_{\mcT_h}\right\vert_1
+\frac12\lambda_{\bdelta}\left\vert\wh\bbeta_0-\bbeta_0\right\vert_1\\
&\le & \lambda_{\bdelta}\left\vert\bbeta_{0,\mathcal{S}_0}-\wh\bbeta_{\mcT_h,\mathcal{S}_0}\right\vert_1+
\lambda_{\bdelta}\left\vert\bbeta_{0,\mathcal{S}_0^c}-\wh\bbeta_{\mcT_h,\mathcal{S}_0^c}\right\vert_1
\\&&-\lambda_{\bdelta}\left\vert\wh\bbeta_{0,\mathcal{S}_0}-\wh\bbeta_{\mcT_h,\mathcal{S}_0}\right\vert_1
-\lambda_{\bdelta}\left\vert\wh\bbeta_{0,\mathcal{S}_0^c}-\wh\bbeta_{\mcT_h,\mathcal{S}_0^c}\right\vert_1
+\frac12\lambda_{\bdelta}\left\vert\wh\bbeta_0-\bbeta_0\right\vert_1\\
&\le &\lambda_{\bdelta}\left(\left\vert\bbeta_{0,\mathcal{S}_0}-\wh\bbeta_{\mcT_h,\mathcal{S}_0}\right\vert_1
-\left\vert\wh\bbeta_{0,\mathcal{S}_0}-\wh\bbeta_{\mcT_h,\mathcal{S}_0}\right\vert_1\right)-
\lambda_{\bdelta}\left(\left\vert\bbeta_{0,\mathcal{S}_0^c}-\wh\bbeta_{\mcT_h,\mathcal{S}_0^c}\right\vert_1
+\left\vert\wh\bbeta_{0,\mathcal{S}_0^c}-\wh\bbeta_{\mcT_h,\mathcal{S}_0^c}\right\vert_1\right)\\&&
+2\lambda_{\bdelta}\left\vert\bbeta_{0,\mathcal{S}_0^c}-\wh\bbeta_{\mcT_h,\mathcal{S}_0^c}\right\vert_1
+\frac12\lambda_{\bdelta}\left\vert\wh\bbeta_0-\bbeta_0\right\vert_1\\
&\le &\lambda_{\bdelta}\left\vert\bbeta_{0,\mathcal{S}_0}-\wh\bbeta_{0,\mathcal{S}_0}\right\vert_1-
\lambda_{\bdelta}\left\vert\bbeta_{0,\mathcal{S}_0^c}-\wh\bbeta_{0,\mathcal{S}_0^c}\right\vert_1
\\&&+2\lambda_{\bdelta}\left\vert\bbeta_{0,\mathcal{S}_0^c}-\bbeta_{0,\mathcal{S}_0^c}\right\vert_1
+2\lambda_{\bdelta}\left\vert\bbeta_{0,\mathcal{S}_0^c}-\wh\bbeta_{\mcT_h,\mathcal{S}_0^c}\right\vert_1
+\frac12\lambda_{\bdelta}\left\vert\wh\bbeta_0-\bbeta_0\right\vert_1\\
&\le &\frac32\lambda_{\bdelta}\left\vert\bbeta_{0,\mathcal{S}_0}-\wh\bbeta_{0,\mathcal{S}_0}\right\vert_1-
\frac12\lambda_{\bdelta}\left\vert\bbeta_{0,\mathcal{S}_0^c}-\wh\bbeta_{0,\mathcal{S}_0^c}\right\vert_1
+2\lambda_{\bdelta} h+2\lambda_{\bdelta}\left\vert\wh\bu_{\mathcal{S}_0^c}\right\vert_1
\\
&\le &\frac32\lambda_{\bdelta}\left\vert\bbeta_{0\mathcal{S}_0}-\wh\bbeta_{\mathcal{S}_0}\right\vert_1-
\frac12\lambda_{\bdelta}\left\vert\bbeta_{0\mathcal{S}_0^c}-\wh\bbeta_{\mathcal{S}_0^c}\right\vert_1
+2 \lambda_{\bdelta} h+
\frac{672\lambda_{\bdelta} Bs_0}{\kappa_0}\sqrt\frac{\log p}{n_{\mcT_h}+n_0}\\
&&+24\lambda_{\bdelta} \sqrt{\frac{3Bhs_0}{\kappa_0}}\left(\frac{\log p}{n_{\mcT_h}+n_0}\right)^{1/4}+\frac{576\lambda_{\bdelta} m_0B^2s_0h}{\kappa_0}+24\lambda_{\bdelta}\sqrt{\frac{2m_0s_0}{\kappa_0}}Bh
\end{eqnarray*}
with probability at least $1-p^{-2}-4p^{-1}$ by Lemma \ref{lemmaA5}, which yields that
\begin{eqnarray*}
&&\wh\bbeta_0-\bbeta_0\in\\&&\mathcal{D}\left(4 h+
\frac{1344 Bs_0}{\kappa_0}\sqrt\frac{\log p}{n_{\mcT_h}+n_0}
 +48\sqrt{\frac{3Bhs_0}{\kappa_0}}\left(\frac{\log p}{n_{\mcT_h}+n_0}\right)^{1/4}+\frac{1152 m_0B^2s_0h}{\kappa_0}+48\sqrt{\frac{2m_0s_0}{\kappa_0}}Bh,\mcS_0\right)
\end{eqnarray*}
in probability.
Thus, using Lemma \ref{lemmaA4}, it holds with probability more than $1-2p^{-2}-4p^{-1}$ that
\begin{eqnarray*}
\min\left\{\frac{\kappa_0}4\left\vert\wh\bbeta_0-\bbeta_0\right\vert^2_2,\frac{3\kappa_0^3}{16m_0}\left\vert\wh\bbeta_0-\bbeta_0\right\vert_2\right\}
&\le& B\sqrt{\frac{\log p}{n_0}}(4h+4\vert\wh\bu\vert_1)
+ \frac32\lambda_{\bdelta} h+\frac32\lambda_{\bdelta} \vert\wh\bu\vert_1\\
&\lesssim& \sqrt{\frac{s_0\log p}{n_0}}\sqrt{\frac{s_0\log p}{n_{\mcT_h}+n_0}}+s_0h\sqrt{\frac{\log p}{n_0}}.
\end{eqnarray*}
Since $$ \sqrt{\frac{s_0\log p}{n_0}}\sqrt{\frac{s_0\log p}{n_{\mcT_h}+n_0}}+s_0h\sqrt{\frac{\log p}{n_0}}\to 0, $$
then
\begin{eqnarray*}
\min\left\{\frac{\kappa_0}4\left\vert\wh\bbeta_0-\bbeta_0\right\vert^2_2,\frac{3\kappa_0^3}{16m_0}\left\vert\wh\bbeta_0-\bbeta_0\right\vert_2\right\}=
\frac{\kappa_0}4\left\vert\wh\bbeta_0-\bbeta_0\right\vert^2_2
\end{eqnarray*}
if the sample size is sufficiently large.
As a result, the conclusion is drawn by (\ref{beta-l2-bound*}) that
\begin{equation*}
\left\vert\wh\bbeta_0-\bbeta_0\right\vert_2\lesssim
\sqrt {s_0h}\left(\frac{\log p}{n_0}\right)^{1/4}+\left(\frac{s_0\log p}{n_0}\right)^{1/4}\left(\frac{s_0\log p}{n_{\mcT_h}+n_0}\right)^{1/4}.
\end{equation*}
\end{proof}

\subsection{Proof of Theorem \ref{theorem2.4}}

Denote the oracle estimation of the crude step as $\wh\bbeta_{\mcT_h,{\rm ora}}$ with
\begin{eqnarray}
\label{op_oracle}
\wh\bbeta_{\mcT_h,{\rm ora}}&\in&\arg\min
L\left(\bbeta;\{0\}\cup\mcT_h\right)+\lambda_{\bbeta}\vert\bbeta\vert_1,\nonumber\\
&&{\rm subject\ to\ } \bbeta_{\mcS_{\mcT_h}^c}={\bf 0}.
\end{eqnarray}

\begin{lemma}
\label{lemmaA6}
Under assumptions \ref{C1}--\ref{C3} and \ref{C5}--\ref{C6},
with $$ a_n=4h',\ \mcS=\mcS_{\mcT_h} $$
in assumption \ref{C3},
if
$$ \lambda_{\bbeta}\ge \left\{\frac{8D}{\gamma}\sqrt{\frac{s'}{\log p}}+ \frac{32\sqrt{2}}{\gamma}\right\}B\sqrt{\frac{\log p}{n_{\mcT_h}+n_0}}, $$
where the positive constant $D$ is determined by (\ref{determineD}),
$$ \sqrt{s'\lambda_{\bbeta}}\to0, \frac{s'\log p}{\lambda_{\bbeta}\left(n_{\mcT_h}+n_0\right)}\to 0,\ {\rm and}\ h'<\frac{\gamma\lambda_{\bbeta}}{8m_0B^2},$$
then
with probability more than $1-12p^{-1}-2p^{-2}$, $\wh\bbeta_{\mcT_h,{\rm ora}}$ is a solution of problem (\ref{op_trans}).
\end{lemma}

\begin{proof}
According to the Karush--Kuhn--Tucker condition of (\ref{op_trans}), we have
\begin{eqnarray}
\label{KKT_trans}
 \bS\left(\wh\bbeta_{\mcT_h};\{0\}\cup\mcT_h\right)+\lambda_{\bbeta}{\rm \bf  sign}\left(\wh\bbeta_{\mcT_h}\right)={\bf 0}_{p},
\end{eqnarray}
where
\begin{eqnarray*}
{\rm sign}(t)
\left\{
\begin{aligned}
&=-1,\ &t<0,\\
&\in[-1,1],\ &t=0,\\
&=1,\ &t>0,
\end{aligned}
\right.
\end{eqnarray*}
with ${\bf sign}\left(\wh\bbeta_{\mcT_h}\right)=\left({\rm sign}\left(\wh\beta_{\mcT_h,1}\right),\ldots,{\rm sign}\left(\wh\beta_{\mcT_h,p}\right)\right)^\rmT$.
Thus, $\wh\bbeta_{\mcT_h,{\rm ora}}$ is the solution of (\ref{op_trans}) if and only if it satisfies (\ref{KKT_trans}), which can be written as
\begin{eqnarray*}
\label{op_oracle1*}
\wh \bS\left(\wh\bbeta_{\mcT_h,{\rm ora}},\{0\}\cup\mcT_h\right)-\wh \bS\left(\bbeta^*_{\mcT_h},\{0\}\cup\mcT_h\right)+\wh \bS\left(\bbeta^*_{\mcT_h},\{0\}\cup\mcT_h\right)+\lambda_{\bbeta}{\rm \bf  sign}\left(\wh\bbeta_{\mcT_h,{\rm ora}}\right)={\bf 0}_{p}.
\end{eqnarray*}
Denote
\begin{eqnarray*}
\bR_1=\sum_{k\in\{0\}\cup\mcT_h}\pi_kE\left[\bZ_{k}\left\{I\left(Y_{k}-\bZ_{k}^{\rmT}\bbeta\le 0\right)-I\left(Y_{k}-\bZ_{k}^{\rmT}\bbeta^*_{\mcT_h}\le 0\right)\right\}\right]\Big\vert_{\bbeta=\wh\bbeta_{\mcT_h,{\rm ora}}}-\bH^{*}_{\mcT_h}\left(\wh\bbeta_{\mcT_h,{\rm ora}}-\bbeta^*_{\mcT_h}\right)
\end{eqnarray*}
and
\begin{eqnarray*}
\bR_2&=&\frac1{n_0+n_{\mcT_h}}\sum_{k\in\{0\}\cup\mcT_h}\sum_{i=1}^{n_k}\bZ_{ki}\left\{I\left(Y_{ki}-\bZ_{ki}^{\rmT}\wh\bbeta_{\mcT_h,{\rm ora}}\le 0\right)-I\left(Y_{ki}-\bZ_{ki}^{\rmT}\bbeta^*_{\mcT_h}\le 0\right)\right\}\\
&&-\sum_{k\in\{0\}\cup\mcT_h}\pi_kE\left[\bZ_{k}\left\{I\left(Y_{k}-\bZ_{k}^{\rmT}\bbeta\le 0\right)-I\left(Y_{k}-\bZ_{k}^{\rmT}\bbeta^*_{\mcT_h}\le 0\right)\right\}\right]\Big\vert_{\bbeta=\wh\bbeta_{\mcT_h,{\rm ora}}}.
\end{eqnarray*}
As a result, (\ref{op_oracle1*}) can be rewritten as
\begin{eqnarray*}
\bH^{*}_{\mcT_h}(\wh\bbeta_{\mcT_h,{\rm ora}}-\bbeta^*_{\mcT_h})
+\bR_1+\bR_2+\wh \bS(\bbeta^*_{\mcT_h},\{0\}\cup\mcT_h)+\lambda_{\bbeta}{\rm \bf  sign}(\wh\bbeta_{\mcT_h})={\bf 0}_{p},
\end{eqnarray*}
which is further exhibited in the form of partitioned matrices as
\begin{eqnarray*}
\label{fenkuai}
&&\left(
\begin{matrix}
\bH^{*}_{\mcT_h,\mcS_{\mcT_h}\mcS_{\mcT_h}}   & \bH^{*}_{\mcT_h,\mcS_{\mcT_h}\mcS^c_{\mcT_h}}\\
\bH^{*}_{\mcT_h,\mcS^c_{\mcT_h}\mcS_{\mcT_h}} & \bH^{*}_{\mcT_h,\mcS^c_{\mcT_h}\mcS^c_{\mcT_h}}
\end{matrix}
\right)
\left(
\begin{matrix}
\wh\bbeta_{\mcT_h,{\rm ora},\mcS_{\mcT_h}}-\bbeta^*_{\mcT_h,\mcS_{\mcT_h}}\\
{\bf 0}_{\vert\mcS^c_{\mcT_h}\vert}-\bbeta^*_{\mcT_h,\mcS^c_{\mcT_h}}
\end{matrix}
\right)
+\left(
\begin{matrix}
\bR_{1\mcS_{\mcT_h}}+\bR_{2\mcS_{\mcT_h}}\\
\bR_{1\mcS^c_{\mcT_h}}+\bR_{2\mcS^c_{\mcT_h}}
\end{matrix}
\right)
\\&&+\left(
\begin{matrix}
\bS_{\mcS_{\mcT_h}}\left(\bbeta^*_{\mcT_h};\{0\}\cup\mcT_h\right)\\
\bS_{\mcS_{\mcT_h}}\left(\bbeta^*_{\mcT_h};\{0\}\cup\mcT_h\right)
\end{matrix}
\right)
\\
&=&-\lambda_{\bbeta}
\left(
\begin{matrix}
{\rm\bf sign}\left(\wh\bbeta_{\mcT_h,{\rm ora},\mcS_{\mcT_h}}\right)\\
\wh\bz
\end{matrix}
\right),
\end{eqnarray*}
where $\wh\bz\in[-1,1]^{\left\vert\mcS^c_{\mcT_h}\right\vert}$.
Note that the first part of partition is  the Karush--Kuhn--Tucker condition of (\ref{op_oracle}) that
\begin{eqnarray}
\label{mat1}
&&\bH^{*}_{\mcT_h,\mcS_{\mcT_h}\mcS_{\mcT_h}}\left(\wh\bbeta_{\mcT_h,{\rm ora},\mcS_{\mcT_h}}-\bbeta^*_{\mcT_h,\mcS_{\mcT_h}}\right)-\bH^{*}_{\mcT_h,\mcS_{\mcT_h}\mcS^c_{\mcT_h}}\bbeta^*_{\mcT_h,\mcS^c_{\mcT_h}}
\nonumber\\&&+\bR_{1\mcS_{\mcT_h}}+\bR_{2\mcS_{\mcT_h}}+\bS_{\mcS_{\mcT_h}}\left(\bbeta^*_{\mcT_h};\{0\}\cup\mcT_h\right)\nonumber\\
&=&-\lambda_{\bbeta}{\bf sign}\left(\wh\bbeta_{\mcT_h,{\rm ora},\mcS_{\mcT_h}}\right).
\end{eqnarray}
Consequently, we only need to verify the sufficient form of the second part of partition that
\begin{eqnarray}
\label{mat2}
&&\left\vert\bS_{\mcS_{\mcT_h}^c}\left(\wh\bbeta_{\mcT_h,{\rm ora}};\{0\}\cup\mcT_h\right)\right\vert_\infty\nonumber\\
&=&\left\vert\bH^{*}_{\mcT_h,\mcS^c_{\mcT_h}\mcS_{\mcT_h}}\left(\wh\bbeta_{\mcT_h,{\rm ora},\mcS_{\mcT_h}}-\bbeta^*_{\mcT_h,\mcS_{\mcT_h}}\right)-\bH^{*}_{\mcT_h,\mcS^c_{\mcT_h}\mcS^c_{\mcT_h}}\bbeta^*_{\mcT_h,\mcS^c_{\mcT_h}}
\right.\nonumber\\&&\left.+\bR_{1\mcS^c_{\mcT_h}}+\bR_{2\mcS^c_{\mcT_h}}+ \bS_{\mcS_{\mcT_h}^c}\left(\bbeta^*_{\mcT_h};\{0\}\cup\mcT_h\right)\right\vert_\infty\nonumber\\
&<&\lambda_{\bbeta}.
\end{eqnarray}
Based on (\ref{mat1}), (\ref{mat2}) is represented as
\begin{eqnarray*}
&&\Bigg\vert\bH^{*}_{\mcT_h,\mcS^c_{\mcT_h}\mcS_{\mcT_h}}
\left(\bH^{*}_{\mcT_h,\mcS_{\mcT_h}\mcS_{\mcT_h}}\right)^{-1}\left(\bH^{*}_{\mcT_h,\mcS_{\mcT_h}\mcS^c_{\mcT_h}}\!\bbeta^*_{\mcT_h,\mcS^c_{\mcT_h}}
-\bR_{1\mcS_{\mcT_h}}-\bR_{2\mcS_{\mcT_h}}-\wh \bS_{\mcS_{\mcT_h}}\left(\bbeta^*_{\mcT_h},\{0\}\cup\mcT_h\right)\right.
\\&&\left.-\lambda_{\bbeta}{\bf sign}\left(\wh\bbeta_{\mcT_h,{\rm ora},\mcS_{\mcT_h}}\right)\right)-\bH^{*}_{\mcT_h,\mcS^c_{\mcT_h}\mcS^c_{\mcT_h}}\bbeta^*_{\mcT_h,\mcS^c_{\mcT_h}}
+\bR_{1\mcS^c_{\mcT_h}}+\bR_{2\mcS^c_{\mcT_h}}+\bS_{\mcS_{\mcT_h}^c}\left(\bbeta^*_{\mcT_h};\{0\}\cup\mcT_h\right)\Bigg\vert_\infty
\\&<&\lambda_{\bbeta}.
\end{eqnarray*}
On the contrary, under assumption \ref{C6}, the probability that (\ref{mat2}) does not hold can be controlled by
\begin{eqnarray*}
&&P\left(\left\vert\bH^{*}_{\mcT_h,\mcS^c_{\mcT_h}\mcS_{\mcT_h}}\left(\wh\bbeta_{\mcT_h,{\rm ora},\mcS_{\mcT_h}}-\bbeta^*_{\mcT_h,\mcS_{\mcT_h}}\right)-\bH^{*}_{\mcT_h,\mcS^c_{\mcT_h}\mcS^c_{\mcT_h}}\bbeta^*_{\mcT_h,\mcS^c_{\mcT_h}}
+\bR_{1\mcS^c_{\mcT_h}}+\bR_{2\mcS^c_{\mcT_h}}\right\vert_\infty\ge\lambda_{\bbeta}\right)\\
&\le&P\left(\left\Vert\bH^{*}_{\mcT_h,\mcS^c_{\mcT_h}\mcS_{\mcT_h}}\left(\bH^{*}_{\mcT_h,\mcS_{\mcT_h}\mcS_{\mcT_h}}\right)^{-1}\right\Vert_\infty
\left\vert\bH^{*}_{\mcT_h,\mcS_{\mcT_h}\mcS^c_{\mcT_h}}\bbeta^*_{\mcT_h,\mcS^c_{\mcT_h}}\right\vert_\infty\ge\frac{\gamma\lambda_{\bbeta}}8\right)\\
&&+P\left(\left\Vert\bH^{*}_{\mcT_h,\mcS^c_{\mcT_h}\mcS_{\mcT_h}}\left(\bH^{*}_{\mcT_h,\mcS_{\mcT_h}\mcS_{\mcT_h}}\right)^{-1}\right\Vert_\infty
\left\vert\bR_{1\mcS_{\mcT_h}}\right\vert_\infty\ge\frac{\gamma\lambda_{\bbeta}}8\right)\\
&&+P\left(\left\Vert\bH^{*}_{\mcT_h,\mcS^c_{\mcT_h}\mcS_{\mcT_h}}\left(\bH^{*}_{\mcT_h,\mcS_{\mcT_h}\mcS_{\mcT_h}}\right)^{-1}\right\Vert_\infty
\left\vert\bR_{2\mcS_{\mcT_h}}\right\vert_\infty\ge\frac{\gamma\lambda_{\bbeta}}8\right)\\
&&+P\left(\left\Vert\bH^{*}_{\mcT_h,\mcS^c_{\mcT_h}\mcS_{\mcT_h}}\left(\bH^{*}_{\mcT_h,\mcS_{\mcT_h}\mcS_{\mcT_h}}\right)^{-1}\right\Vert_\infty
\left\vert \bS_{\mcS_{\mcT_h}}\left(\bbeta^*_{\mcT_h};\{0\}\cup\mcT_h\right)\right\vert_\infty\ge\frac{\gamma\lambda_{\bbeta}}8\right)\\
&&+P\left(\lambda_{\bbeta}\left\Vert\bH^{*}_{\mcT_h,\mcS^c_{\mcT_h}\mcS_{\mcT_h}}\left(\bH^{*}_{\mcT_h,\mcS_{\mcT_h}\mcS_{\mcT_h}}\right)^{-1}\right\Vert_\infty
\left\vert{\bf sign}\left(\wh\bbeta_{\mcT_h,{\rm ora},\mcS_{\mcT_h}}\right)\right\vert_\infty\ge (1-\gamma)\lambda_{\bbeta}\right)\\
&&+P\left(
\left\vert\bH^{*}_{\mcT_h,\mcS^c_{\mcT_h}\mcS^c_{\mcT_h}}\bbeta^*_{\mcT_h,\mcS^c_{\mcT_h}}\right\vert_\infty\ge \frac{\gamma\lambda_{\bbeta}}8\right)\\
&&+P\left(
\left\vert\bR_{1\mcS^c_{\mcT_h}}\right\vert_\infty\ge \frac{\gamma\lambda_{\bbeta}}8\right)\\
&&+P\left(
\left\vert\bR_{2\mcS^c_{\mcT_h}}\right\vert_\infty\ge \frac{\gamma\lambda_{\bbeta}}8\right)\\
&&+P\left(
\left\vert\bS_{\mcS_{\mcT_h}}(\bbeta^*_{\mcT_h};\{0\}\cup\mcT_h)\right\vert_\infty\ge \frac{\gamma\lambda_{\bbeta}}8\right)\\
&=&(I_1)+(I_2)+(I_3)+(I_4)+(I_5)+(I_6)+(I_7)+(I_8)+(I_9),
\end{eqnarray*}
where  terms $(I_1)$--$(I_9)$ are self-defined above and will be expired when
the proof of this lemma is completed.
Note that
\begin{eqnarray*}
(I_1)&=&P\left(\left\Vert\bH^{*}_{\mcT_h,\mcS^c_{\mcT_h}\mcS_{\mcT_h}}\left(\bH^{*}_{\mcT_h,\mcS_{\mcT_h}\mcS_{\mcT_h}}\right)^{-1}\right\Vert_\infty
\left\vert\bH^{*}_{\mcT_h,\mcS_{\mcT_h}\mcS^c_{\mcT_h}}\bbeta^*_{\mcT_h,\mcS^c_{\mcT_h}}\right\vert_\infty\ge\frac{\gamma\lambda_{\bbeta}}8\right)\\
&\le&P\left(\left\Vert\bH^{*}_{\mcT_h,\mcS^c_{\mcT_h}\mcS_{\mcT_h}}\left(\bH^{*}_{\mcT_h,\mcS_{\mcT_h}\mcS_{\mcT_h}}\right)^{-1}\right\Vert_\infty
\left\Vert\bH^{*}_{\mcT_h,\mcS_{\mcT_h}\mcS^c_{\mcT_h}}\right\Vert_{\max}\left\vert\bbeta^*_{\mcT_h,\mcS^c_{\mcT_h}}\right\vert_1\ge\frac{\gamma\lambda_{\bbeta}}8\right)\\
&\le&P\left((1-\gamma)m_0B^2h'\ge\frac{\gamma\lambda_{\bbeta}}8\right)=0,
\end{eqnarray*}
\begin{eqnarray*}
(I_5)=P\left(\lambda_{\bbeta}\left\Vert\bH^{*}_{\mcT_h,\mcS^c_{\mcT_h}\mcS_{\mcT_h}}\left(\bH^{*}_{\mcT_h,\mcS_{\mcT_h}\mcS_{\mcT_h}}\right)^{-1}\right\Vert_\infty
\left\vert{\bf sign}\left(\wh\bbeta_{\mcT_h,{\rm ora},\mcS_{\mcT_h}}\right)\right\vert_\infty\ge (1-\gamma)\lambda_{\bbeta}\right)=0,
\end{eqnarray*}
and
\begin{eqnarray*}
(I_6)&=&P\left(
\left\vert\bH^{*}_{\mcT_h,\mcS^c_{\mcT_h}\mcS^c_{\mcT_h}}\bbeta^*_{\mcT_h,\mcS^c_{\mcT_h}}\right\vert_\infty\ge \frac{\gamma\lambda_{\bbeta}}8\right)
\le P\left(
\left\Vert\bH^{*}_{\mcT_h,\mcS^c_{\mcT_h}\mcS^c_{\mcT_h}}\right\Vert_{\max}\left\vert\bbeta^*_{\mcT_h,\mcS^c_{\mcT_h}}\right\vert_1\ge \frac{\gamma\lambda_{\bbeta}}8\right)
\\&\le& P\left(
m_0B^2h'\ge \frac{\gamma\lambda_{\bbeta}}8\right)=0.
\end{eqnarray*}
Moreover, according to (\ref{grad_hoef}), we obtain
\begin{eqnarray*}
(I_4)&=&P\left(\left\Vert\bH^{*}_{\mcT_h,\mcS^c_{\mcT_h}\mcS_{\mcT_h}}\left(\bH^{*}_{\mcT_h,\mcS_{\mcT_h}\mcS_{\mcT_h}}\right)^{-1}\right\Vert_\infty
\left\vert\bS_{\mcS_{\mcT_h}}\left(\bbeta^*_{\mcT_h};\{0\}\cup\mcT_h\right)\right\vert_\infty\ge\frac{\gamma\lambda_{\bbeta}}8\right)\\
&\le&P\left(
\left\vert\bS_{\mcS_{\mcT_h}}\left(\bbeta^*_{\mcT_h};\{0\}\cup\mcT_h\right)\right\vert_\infty\ge\frac{\gamma\lambda_{\bbeta}}{8(1-\gamma)}\right)\\
&\le&2p\exp\left(-\frac{(n_{\mcT_h}+n_0)\gamma^2\lambda_{\bbeta}^2}{128B^2(1-\gamma)^2}\right)\le 2p^{-1}\\
\end{eqnarray*}
and
\begin{eqnarray*}
(I_9)=P\left(
\left\vert\bS_{\mcS_{\mcT_h}}(\bbeta^*_{\mcT_h};\{0\}\cup\mcT_h)\right\vert_\infty\ge \frac{\gamma\lambda_{\bbeta}}8\right)
\le 2p\exp\left(-\frac{(n_{\mcT_h}+n_0)\gamma^2\lambda_{\bbeta}^2}{128B^2}\right)\le 2p^{-1},
\end{eqnarray*}
when $$\lambda_{\bbeta}\ge\frac{16B(1-\gamma)}{\gamma}\sqrt{\frac{\log p}{n_{\mcT_h}+n_0}}.$$
On the other hand, under assumption \ref{C1}, we have
\begin{eqnarray*}
&&P\left(\vert\bR_1\vert_\infty\ge t\right)\\
&=&P\left(\left\vert \sum_{k\in\{0\}\cup\mcT_h}\pi_kE\left[\bZ_{k}\left\{F_k\left(\bZ_{k}^{\rmT}\bbeta\mid \bZ_{k}\right)
-F_k\left(\bZ_{k}^{\rmT}\bbeta^*_{\mcT_h}\mid \bZ_{k}\right)\right\}\right]\Big\vert_{\bbeta=\wh\bbeta_{\mcT_h,{\rm ora}}}-\bH^{*}_{\mcT_h}\left(\wh\bbeta_{\mcT_h,{\rm ora}}-\bbeta^*_{\mcT_h}\right)\right\vert_\infty\ge t\right)\\
&=&P\Bigg(\Bigg\vert \sum_{k\in\{0\}\cup\mcT_h}\pi_kE\left[\bZ_{k}\bZ_{k}^{\rmT}\int^1_0f_k
\left(\bZ_{k}^{\rmT}\bbeta^*_{\mcT_h}+t\bZ_{k}^{\rmT}\left(\bbeta-\bbeta^*_{\mcT_h}\right)\mid \bZ_{k}\right)\rmd t\right]\bigg\vert_{\bbeta=\wh\bbeta_{\mcT_h,{\rm ora}}}\left(\wh\bbeta_{\mcT_h,{\rm ora}}-\bbeta^*_{\mcT_h}\right)\\
&&\hspace{1in}-\bH^{*}_{\mcT_h}\left(\wh\bbeta_{\mcT_h,{\rm ora}}-\bbeta^*_{\mcT_h}\right)\Bigg\vert_\infty\ge t\Bigg)\\
&=&P\Bigg(\Bigg\vert \sum_{k\in\{0\}\cup\mcT_h}\pi_kE\Bigg[\int^1_0\left\{f_k
\left(\bZ_{k}^{\rmT}\bbeta^*_{\mcT_h}\!+\!t\bZ_{k}^{\rmT}\left(\bbeta\!-\!\bbeta^*_{\mcT_h}\right)\rmd t\mid \bZ_{k}\right)\!-\!f_k
\left(\bZ_{k}^{\rmT}\bbeta^*_{\mcT_h}\mid \bZ_{k}\right)\right\}\rmd t\bZ_{k}\bZ_{k}^{\rmT}\Bigg]\bigg\vert_{\bbeta=\wh\bbeta_{\mcT_h,{\rm ora}}}\\
&&\hspace{1in}\left(\wh\bbeta_{\mcT_h,{\rm ora}}\!-\!\bbeta^*_{\mcT_h}\right)\Bigg\vert_\infty
\ge t\Bigg)\\
&\le&P\Bigg(\Bigg\vert \sum_{k\in\{0\}\cup\mcT_h}\pi_k m_0B\left\vert\wh\bbeta_{\mcT_h,{\rm ora}}-\bbeta^*_{\mcT_h}\right\vert_1E\left[\bZ_{k}\bZ_{k}^{\rmT}
\right]\left(\wh\bbeta_{\mcT_h,{\rm ora}}-\bbeta^*_{\mcT_h}\right)\Bigg\vert_\infty
\ge t\Bigg)\\
&\le&P\Bigg(m_0B^3\left\vert\wh\bbeta_{\mcT_h,{\rm ora}}-\bbeta^*_{\mcT_h}\right\vert^2_1
\ge t\Bigg)
\end{eqnarray*}
for any $t>0$. Thus, it holds that
\begin{eqnarray*}
(I_2)\le P\left(\left\vert\bR_1\right\vert_\infty\ge\frac{\gamma\lambda_{\bbeta}}{8(1-\gamma)}\right)
\le P\left(\left\vert\wh\bbeta_{\mcT_h,{\rm ora}}-\bbeta^*_{\mcT_h}\right\vert^2_1\ge\frac{\gamma\lambda_{\bbeta}}{8m_0B^3(1-\gamma)}\right)
\end{eqnarray*}
and
\begin{eqnarray*}
(I_7)\le P\left(\left\vert\bR_1\right\vert_\infty\ge\frac{\gamma\lambda_{\bbeta}}{8}\right)\le
P\left(\left\vert\wh\bbeta_{\mcT_h,{\rm ora}}-\bbeta^*_{\mcT_h}\right\vert^2_1
\ge\frac{\gamma\lambda_{\bbeta}}{8m_0B^3}\right).
\end{eqnarray*}
Lemma \ref{lemmaA3} indicates that
\begin{eqnarray*}
\left\vert\bS\left(\bbeta^*_{\mcT_h};\{0\}\cup\mcT_h\right)\right\vert_\infty
\le 2B\sqrt{\frac{\log p}{n_{\mcT_h}+n_0}}
\end{eqnarray*}
with probability more than $1-2p^{-1}$.
Therefore, when $\lambda_{\bbeta}\ge 4B\sqrt{\log p/(n_{\mcT_h}+n_0)}$, it holds that
\begin{eqnarray}
\label{Hessian}
&&\delta L\left(\wh\bbeta_{\mcT_h,{\rm ora}}-\bbeta^*_{\mcT_h}; \{0\}\cup\mcT_h\right)\nonumber\\&\le& L\left(\wh\bbeta_{\mcT_h,{\rm ora}}; \{0\}\cup\mcT_h\right)-L\left(\bbeta^*_{\mcT_h}; \{0\}\cup\mcT_h\right)-  \bS(\bbeta^*_{\mcT_h}; \{0\}\cup\mcT_h)^\rmT\left(\wh\bbeta_{\mcT_h,{\rm ora}}-\bbeta^*_{\mcT_h}\right)\nonumber\\
&\le& -\lambda_{\bbeta}\left\vert\wh\bbeta_{\mcT_h,{\rm ora}}\right\vert_1
+\lambda_{\bbeta}\left\vert\bbeta^*_{\mcT_h}\right\vert_1+\frac12\lambda_{\bbeta}\left\vert\wh\bbeta_{\mcT_h,{\rm ora}}-\bbeta^*_{\mcT_h}\right \vert_1\nonumber\\
&\le&\frac32\lambda_{\bbeta}\left\vert\wh\bbeta_{\mcT_h,{\rm ora},\mcS_{\mcT_h}}-\bbeta^*_{\mcT_h,\mcS_{\mcT_h}}\right\vert_1
+\frac32\lambda_{\bbeta}\left\vert\bbeta^*_{\mcT_h,\mcS^c_{\mcT_h}}\right\vert_1\nonumber\\
&\le &\frac32\lambda_{\bbeta}\left\vert\wh\bbeta_{\mcT_h,{\rm ora},\mcS_{\mcT_h}}-\bbeta^*_{\mcT_h,\mcS_{\mcT_h}}\right\vert_1+\frac32\lambda_{\bbeta} h'
\end{eqnarray}
with probability more than $1-2p^{-1}$.
Since $\wh\bbeta_{\mcT_h,{\rm ora}}-\bbeta^*_{\mcT_h}\in\mathcal{D}(h',\mcS_{\mcT_h})$, then according to Lemma \ref{lemmaA2} and assumption \ref{C3}, it yields with probability more than $1-p^{-2}$ that
\begin{eqnarray*}
\delta L\left(\wh\bbeta_{\mcT_h,{\rm ora}}\!-\!\bbeta^*_{\mcT_h}; \{0\}\cup\mcT_h\right)
&\ge& \min\left\{\frac{\kappa_0}4\left\vert\wh\bbeta_{\mcT_h,{\rm ora}}\!-\!\bbeta^*_{\mcT_h}\right\vert_2^2,\frac{3\kappa_0^3}{16m_0}\left\vert\wh\bbeta_{\mcT_h,{\rm ora}}-\bbeta^*_{\mcT_h}\right\vert_2\right\}\\
&&-\!B\sqrt{\frac{\log p}{n_{\mcT_h}+n_0}}\left\vert\wh\bbeta_{\mcT_h,{\rm ora}}-\bbeta^*_{\mcT_h}\right\vert_1,
\end{eqnarray*}
which implies that
\begin{eqnarray*}
\left\vert\wh\bbeta_{\mcT_h,{\rm ora}}-\bbeta^*_{\mcT_h}\right\vert_2\le \frac{4B}{\kappa_0}\sqrt{\frac{s'\log p}{n_{\mcT_h}+n_0}}+\frac6{\kappa_0}\lambda_{\bbeta}\sqrt{s'}+\sqrt{\frac{7\lambda_{\bbeta} h'}{\kappa_0}}
\end{eqnarray*}
with probability more than $1-2p^{-1}-p^{-2}$ by combining (\ref{Hessian}) and mimicking the proof of Theorem \ref{theorem2.1}.
For sufficiently large $n_{\mcT_h},n_0$ and $p$, it holds that
$$ \frac{4B}{\kappa_0}\sqrt{\frac{s'\log p}{n_{\mcT_h}+n_0}}+\frac6{\kappa_0}\lambda_{\bbeta}\sqrt{s'}+\sqrt{\frac{7\lambda_{\bbeta} h'}{\kappa_0}}\le\sqrt{\frac{\gamma\lambda_{\bbeta}}{8m_0B^3}}. $$
As a result, we have
$$ (I_2)\le 2p^{-1}+p^{-2}\ {\rm and}\ (I_7)\le 2p^{-1}+p^{-2}. $$
On the other hand, to investigate the tail probability of item $\bR_2$, we denote
the function class $\mathcal{F}_j=\{f_{\bbeta,j}\mid \bbeta_{\mcS^c_{\mcT_h}}={\bf 0}\}$,
where $$f_{\bbeta,j}(Y,\bZ)=Z_j\{I(Y-\bZ^\rmT\bbeta\le 0)-I(Y-\bZ^\rmT\bbeta^*_{\mcT_h}\le 0)\}-E(Z_j\{I(Y-\bZ^\rmT\bbeta\le 0)-I(Y-\bZ^\rmT\bbeta^*_{\mcT_h}\le 0)\})$$ and $Z_j$ is the $j$-th component of $\bZ$.
Since the Vapnik--Chervonenkis dimension of $\mathcal{F}_j$ is less than $s'$, then Lemma 19.15 in 
\citet{van1998asymptotic} entails that
$$ \sup_Q N(\epsilon B,\mathcal{F}_j,L_2(Q))\le Ms'(16e)^{s'}\epsilon^{1-2s'} $$
for some positive constant $M>0$, where
the covering number $N(t,\mathcal{F},L_r(Q))$
is the minimal number of $L_r(Q)$-balls of radius $t$ needed to cover the function class $\mathcal{F}$.
Thus, Lemma 19.35 in 
\citet{van1998asymptotic} implies that
\begin{eqnarray}
\label{determineD}
&&E\left\{\sup_{\bbeta: \bbeta_{\mcS^c_{\mcT_h}}={\bf 0}}
\frac1{n_0+n_{\mcT_h}}\sum_{k\in\{0\}\cup\mcT_h}\sum_{i=1}^{n_k}f_{\bbeta,j}\left(Y_{ki},\bZ_{ki}\right)\right\}\nonumber\\
&\le&\sum_{k\in\{0\}\cup\mcT_h}\pi_kE\left\{\sup_{\bbeta: \bbeta_{\mcS^c_{\mcT_h}}={\bf 0}}
\frac1{n_k}\sum_{i=1}^{n_k}f_{\bbeta,j}\left(Y_{ki},\bZ_{ki}\right)\right\}\nonumber\\
&\le&\sum_{k\in\{0\}\cup\mcT_h}\pi_kn_k^{-1/2} BJ(1,\mathcal{F}_j,L_2)\nonumber\\
&=&B\sqrt{\frac K{n_0+n_{\mcT_h}}}\int^1_0 \sqrt{\log\sup_Q N\left(\epsilon B,\mathcal{F}_j,L_2(Q)\right)}\rmd \epsilon\nonumber\\
&\le&DB\sqrt{s'}\left(n_0+n_{\mcT_h}\right)^{-1/2}
\end{eqnarray}
for some constant $D>0$.
Employing Theorem 3.26 in 
\citet{wainwright2019high}, it holds that
\begin{eqnarray*}
&&P\left(\vert\bR_2\vert_\infty\ge DB\sqrt{s'}(n_0+n_{\mcT_h})^{-1/2}+t\right)\\&\le&
\sum_{j=1}^pP\Bigg(\Bigg\vert\frac1{n_0+n_{\mcT_h}}\sup_{\bbeta: \bbeta_{\mcS^c_{\mcT_h}}={\bf 0}}\sum_{k\in\{0\}\cup\mcT_h}\sum_{i=1}^{n_k}f_{\bbeta,j}\left(Y_{ki},\bZ_{ki}\right)\Bigg\vert
\ge DB\sqrt{s'}(n_0+n_{\mcT_h})^{-1/2}+t\Bigg)\\
&\le&2p\exp\left(-\frac{(n_{\mcT_h}+n_0)t^2}{16B^2}\right).
\end{eqnarray*}
By taking $t=4L\sqrt{2\log p/(n_{\mcT_h}+n_0)}$, we obtain
$$ \vert\bR_2\vert_\infty< DB\sqrt{\frac{s'}{n_0+n_{\mcT_h}}}+ 4B\sqrt{\frac{2\log p}{n_{\mcT_h}+n_0}} $$
with probability more than $1-2p^{-1}$.

As a result, when
$$ \lambda_{\bbeta}\ge \frac{8DB}{\gamma}\sqrt{\frac{s'}{n_0+n_{\mcT_h}}}+ \frac{32B}{\gamma}\sqrt{\frac{2\log p}{n_{\mcT_h}+n_0}} , $$
we have
\begin{eqnarray*}
(I_3)&=&P\left(\left\Vert\bH^{*}_{\mcT_h,\mcS^c_{\mcT_h}\mcS_{\mcT_h}}\left(\bH^{*}_{\mcT_h,\mcS_{\mcT_h}\mcS_{\mcT_h}}\right)^{-1}\right\Vert_\infty
\left\vert\bR_{2\mcS_{\mcT_h}}\right\vert_\infty\ge\frac{\gamma\lambda_{\bbeta}}8\right)
\\&\le& P\left(
\left\vert\bR_{2\mcS_{\mcT_h}}\right\vert_\infty\ge
\frac{\gamma\lambda_{\bbeta}}{8(1-\gamma)}\right)
\\&\le& 2p^{-1}
\end{eqnarray*}
and
\begin{eqnarray*}
(I_8)=P\left(
\left\vert\bR_{2\mcS^c_{\mcT_h}}\right\vert_\infty\ge \frac{\gamma\lambda_{\bbeta}}8\right)
&\le&2p^{-1}.
\end{eqnarray*}
Combining the pieces, we conclude that (\ref{mat2}) holds with probability more than $1-12p^{-1}-2p^{-2}$,
which concludes the proof of Lemma \ref{lemmaA6}.
\end{proof}

\begin{lemma}
\label{lemmaA7}
Under assumptions \ref{C1}--\ref{C3} and \ref{C5}--\ref{C6},
with $$ a_n=4h',\ \mcS=\mcS_{\mcT_h} $$
in assumption \ref{C3},
if
$$ \lambda_{\bbeta}\ge \left\{\frac{8D}{\gamma}\sqrt{\frac{s'}{\log p}}+ \frac{32\sqrt{2}}{\gamma}\right\}B\sqrt{\frac{\log p}{n_{\mcT_h}+n_0}}, $$
where the positive constant $D$ is determined by (\ref{determineD}),
$$ \sqrt{s'\lambda_{\bbeta}}\to0, \frac{s'\log p}{\lambda_{\bbeta}\left(n_{\mcT_h}+n_0\right)}\to 0,\ {\rm and}\ h'<\frac{\gamma\lambda_{\bbeta}}{8m_0B^2},$$
then
with probability more than $1-12p^{-1}-2p^{-2}$, any solution $\wh\bbeta_{\mcT_h}$ of problem (\ref{op_trans})
satisfies $\wh\bbeta_{\mcT_h,\mcS_{\mcT_h}^c}={\bf 0}$.
\end{lemma}

\begin{proof}
From Lemma \ref{lemmaA6}, the event that the oracle estimator $\wh\bbeta_{\mcT_h,{\rm ora}}$ is a solution of (\ref{op_trans})
occurs with probability more than $1-12p^{-1}-2p^{-2}$.
On this event, to demonstrate that any solution $\wh\bbeta_{\mcT_h}$ of problem (\ref{op_trans})
satisfies $\wh\bbeta_{\mcT_h,\mcS_{\mcT_h}^c}={\bf 0}$,
it is sufficient to verify that
\begin{eqnarray*}
L(\bbeta;\{0\}\cup\mcT_h)+\lambda_{\bbeta}\vert\bbeta\vert_1> L\left(\wh\bbeta_{\mcT_h,{\rm ora}};\{0\}\cup\mcT_h\right)+\lambda_{\bbeta}\left\vert\wh\bbeta_{\mcT_h,{\rm ora}}\right\vert_1
\end{eqnarray*}
for any $\bbeta$ with $\bbeta_{\mcS_{\mcT_h}^c}\neq {\bf 0}_{\vert\mcS_{\mcT_h}^c\vert}$.
We first show that
\begin{eqnarray*}
L(\bbeta,\{0\}\cup\mcT_h)+\lambda_{\bbeta}\vert\bbeta\vert_1> L\left(\left(\bbeta_{\mcS_{\mcT_h}}^\rmT,{\bf 0}^\rmT\right)^\rmT,\{0\}\cup\mcT_h\right)+\lambda_{\bbeta}\left\vert\bbeta_{\mcS_{\mcT_h}}\right\vert_1
\end{eqnarray*}
for any $\bbeta\in \mathcal{B}\left(\wh\bbeta_{\mcT_h,{\rm ora}},r\right)$ with $\bbeta_{\mcS_{\mcT_h}^c}\neq {\bf 0}$,
where $\mathcal{B}\left(\wh\bbeta_{\mcT_h,{\rm ora}},r\right)$ represents the neighbourhood of $\ell_1$-ball with radius $r$.
The Knight equation entails that
\begin{eqnarray*}
&&L(\bbeta;\{0\}\cup\mcT_h)\\
&=&L\left(\wh\bbeta_{\mcT_h,{\rm ora}};\{0\}\cup\mcT_h\right)+\bS\left(\wh\bbeta_{\mcT_h,{\rm ora}};\{0\}\cup\mcT_h\right)^\rmT\left(\bbeta-\wh\bbeta_{\mcT_h,{\rm ora}}\right)\\
&&+\frac1{n_{\mcT_h}+n_0}\sum_{k\in\{0\}\cup\mcT_h}\sum_{i=1}^{n_k}\int^{\mathbf{Z}_{ki}^{\rmT}\left(\bbeta-\wh\bbeta_{\mcT_h,{\rm ora}}\right)}_0\left\{I\left(Y_{ki}-\bZ_{ki}^{\rmT}\wh\bbeta_{\mcT_h,{\rm ora}}\le z\right)-I\left(Y_{ki}-\bZ_{ki}^{\rmT}\wh\bbeta_{\mcT_h,{\rm ora}}\le 0\right)\right\}\rmd z
\end{eqnarray*}
and
\begin{eqnarray*}
&&L\left(\left(\bbeta_{\mcS_{\mcT_h}}^\rmT,{\bf 0}^\rmT\right)^\rmT;\{0\}\cup\mcT_h\right)\\
&=&L\left(\wh\bbeta_{\mcT_h,{\rm ora}};\{0\}\cup\mcT_h\right)+\bS\left(\wh\bbeta_{\mcT_h,{\rm ora}};\{0\}\cup\mcT_h\right)^\rmT\left(\left(\bbeta_{\mcS_{\mcT_h}}^\rmT,{\bf 0}^\rmT\right)^\rmT-\wh\bbeta_{\mcT_h,{\rm ora}}\right)+\frac1{n_{\mcT_h}+n_0}\sum_{k\in\{0\}\cup\mcT_h}\sum_{i=1}^{n_k}\\
&&\int^{\mathbf{Z}_{ki}^{\rmT}\left(\left(\bbeta_{\mcS_{\mcT_h}}^\rmT,{\bf 0}^\rmT\right)^\rmT-\wh\bbeta_{\mcT_h,{\rm ora}}\right)}_0\left\{I\left(Y_{ki}-\bZ_{ki}^{\rmT}\wh\bbeta_{\mcT_h,{\rm ora}}\le z\right)-I\left(Y_{ki}-\bZ_{ki}^{\rmT}\wh\bbeta_{\mcT_h,{\rm ora}}\le 0\right)\right\}\rmd z.
\end{eqnarray*}
Thus, we arrive at
\begin{eqnarray*}
&&\left\{ L\left(\bbeta;\{0\}\cup\mcT_h\right)+\lambda_{\bbeta}\vert\bbeta\vert_1\right\}-\left\{ L\left(\left(\bbeta_{\mcS_{\mcT_h}}^\rmT,{\bf 0}^\rmT\right)^\rmT;\{0\}\cup\mcT_h\right)+\lambda_{\bbeta}\left\vert\bbeta_{\mcS_{\mcT_h}}\right\vert_1\right\}\\
&=&\bS_{\mcS_{\mcT_h}^c}\left(\wh\bbeta_{\mcT_h,{\rm ora}};\{0\}\cup\mcT_h\right)^\rmT\bbeta_{\mcS_{\mcT_h}^c}+
\frac1{n_{\mcT_h}+n_0}\sum_{k\in\{0\}\cup\mcT_h}\sum_{i=1}^{n_k}\\
&&\int_{\mathbf{Z}_{ki}^{\rmT}\left(\left(\bbeta_{\mcS_{\mcT_h}}^\rmT,{\bf 0}^\rmT\right)^\rmT-\wh\bbeta_{\mcT_h,{\rm ora}}\right)}^{\mathbf{Z}_{ki}^{\rmT}\left(\bbeta-\wh\bbeta_{\mcT_h,{\rm ora}}\right)}\left\{I\left(Y_{ki}-\bZ_{ki}^{\rmT}\wh\bbeta_{\mcT_h,{\rm ora}}\le z\right)-I\left(Y_{ki}-\bZ_{ki}^{\rmT}\wh\bbeta_{\mcT_h,{\rm ora}}\le 0\right)\right\}\rmd z.
\end{eqnarray*}
Note that
\begin{eqnarray*}
&&\int_{\mathbf{Z}_{ki}^{\rmT}\left(\left(\bbeta_{\mcS_{\mcT_h}}^\rmT,{\bf 0}^\rmT\right)^\rmT-\wh\bbeta_{\mcT_h,{\rm ora}}\right)}^{\mathbf{Z}_{ki}^{\rmT}\left(\bbeta-\wh\bbeta_{\mcT_h,{\rm ora}}\right)}\left\{I\left(Y_{ki}-\bZ_{ki}^{\rmT}\wh\bbeta_{\mcT_h,{\rm ora}}\le z\right)-I\left(Y_{ki}-\bZ_{ki}^{\rmT}\wh\bbeta_{\mcT_h,{\rm ora}}\le 0\right)\right\}\rmd z\\
&=&\left(\bZ_{ki}^{\rmT}\bbeta-Y_{ki}\right)I\left(Y_{ki}-\bZ_{ki}^{\rmT}\bbeta\le0\right)
-\left(\bZ_{ki,\mcS_{\mcT_h}}^{\rmT}\bbeta_{\mcS_{\mcT_h}}-Y_{ki}\right)I\left(Y_{ki}-\bZ_{ki,\mcS_{\mcT_h}}^{\rmT}\bbeta_{\mcS_{\mcT_h}}\le0\right)\\
&&-\bZ_{ki,\mcS_{\mcT_h}^c}^{\rmT}\bbeta_{\mcS_{\mcT_h}^c}I\left(Y_{ki}-\bZ_{ki}^{\rmT}\wh\bbeta_{\mcT_h,{\rm ora}}\le 0\right)\\
&=&\int^{\mathbf{Z}_{ki,\mcS_{\mcT_h}^c}^{\rmT}\bbeta_{\mcS_{\mcT_h}^c}}_0I\left(Y_{ki}-\bZ_{ki,\mcS_{\mcT_h}}^{\rmT}\bbeta_{\mcS_{\mcT_h}}\le z\right)\rmd z
-\bZ_{ki,\mcS_{\mcT_h}^c}^{\rmT}\bbeta_{\mcS_{\mcT_h}^c}I\left(Y_{ki}-\bZ_{ki}^{\rmT}\wh\bbeta_{\mcT_h,{\rm ora}}\le 0\right)\\
&=&\int^{\mathbf{Z}_{ki,\mcS_{\mcT_h}^c}^{\rmT}\bbeta_{\mcS_{\mcT_h}^c}}_0\left\{I\left(Y_{ki}-\bZ_{ki,\mcS_{\mcT_h}}^{\rmT}\bbeta_{\mcS_{\mcT_h}}\le z\right)-I\left(Y_{ki}-\bZ_{ki,\mcS_{\mcT_h}}^{\rmT}\bbeta_{\mcS_{\mcT_h}}\le 0\right)\right\}\rmd z\\&&
+\bZ_{ki,\mcS_{\mcT_h}^c}^{\rmT}\bbeta_{\mcS_{\mcT_h}^c}\left\{I\left(Y_{ki}-\bZ_{ki,\mcS_{\mcT_h}}^{\rmT}\bbeta_{\mcS_{\mcT_h}}\le 0\right)-I\left(Y_{ki}-\bZ_{ki}^{\rmT}\wh\bbeta_{\mcT_h,{\rm ora}}\le 0\right)\right\}.
\end{eqnarray*}
Denote $F_{k,\bbeta}(\cdot)$ as the distribution function of $Y_{ki}-\bZ_{ki}^{\rmT}\bbeta$
and $\wt F_{k,\bbeta}(\cdot)$ as the empirical distribution function of $\left\{F_{k,\bbeta}\left(Y_{ki}-\bZ_{ki}^{\rmT}\bbeta\right)\right\}_{i=1}^{n_k}$.
Set $\wh \Psi_{k,\bbeta}(t)=n_k^{1/2}\left(\wt F_{k,\bbeta}(t)-t\right)$.
Since
$$ \int^{\mathbf{Z}_{ki,\mcS_{\mcT_h}^c}^{\rmT}\bbeta_{\mcS_{\mcT_h}^c}}_0\left\{I\left(Y_{ki}-\bZ_{ki,\mcS_{\mcT_h}}^{\rmT}\bbeta_{\mcS_{\mcT_h}}\le z\right)-I\left(Y_{ki}-\bZ_{ki,\mcS_{\mcT_h}}^{\rmT}\bbeta_{\mcS_{\mcT_h}}\le 0\right)\right\}\rmd z\ge 0, $$
Theorem 1.2 in 
\citet{zuijlen1978properties} implies that
\begin{small}
\begin{eqnarray*}
&&P\Bigg(\frac1{n_{\mcT_h}+n_0}\sum_{k\in\{0\}\cup\mcT_h}\sum_{i=1}^{n_k}\int^{\mathbf{Z}_{ki,\mcS_{\mcT_h}^c}^{\rmT}\bbeta_{\mcS_{\mcT_h}^c}}_0
\left\{I\left(Y_{ki}-\bZ_{ki,\mcS_{\mcT_h}}^{\rmT}\bbeta_{\mcS_{\mcT_h}}\le z\right)-I\left(Y_{ki}-\bZ_{ki,\mcS_{\mcT_h}}^{\rmT}\bbeta_{\mcS_{\mcT_h}}\le 0\right)\right\}\rmd z\\
&&>\sum_{k\in\{0\}\cup\mcT_h}\pi_kt_k\left\vert \bbeta_{\mcS_{\mcT_h}^c}\right\vert_1\Bigg)\\
&\le &\sum_{k\in\{0\}\cup\mcT_h}P\Bigg(\frac1{n_{k}}\sum_{i=1}^{n_k}\frac1{\left\vert \bbeta_{\mcS_{\mcT_h}^c}\right\vert_1}\int^{\mathbf{Z}_{ki,\mcS_{\mcT_h}^c}^{\rmT}\bbeta_{\mcS_{\mcT_h}^c}}_0
\!\!\left\{I\left(Y_{ki}\!-\!\bZ_{ki,\mcS_{\mcT_h}}^{\rmT}\bbeta_{\mcS_{\mcT_h}}\le z\right)\!-\!I\left(Y_{ki}\!-\!\bZ_{ki,\mcS_{\mcT_h}}^{\rmT}\bbeta_{\mcS_{\mcT_h}}\le 0\right)\right\}\rmd z\!>\!t_k\Bigg)\\
&\le &\sum_{k\in\{0\}\cup\mcT_h}\!\!\!\!P\left(\sup_{\bbeta:\left\vert \bbeta_{\mcS_{\mcT_h}^c}\right\vert_1\le r}\frac1{n_{k}}\sum_{i=1}^{n_k}\!\!\frac1{\left\vert \bbeta_{\mcS_{\mcT_h}^c}\right\vert_1}\!\!\left(\int_{0}^{B\left\vert \bbeta_{\mcS_{\mcT_h}^c}\right\vert_1}+\int_{0}^{-B\left\vert \bbeta_{\mcS_{\mcT_h}^c}\right\vert_1}\right)\right.\\
&&\left\{I\left(F_{k,\left(\bbeta_{\mcS_{\mcT_h}}^\rmT,{\bf 0}^\rmT\right)}\left(Y_{ki}\!-\!\bZ_{ki,\mcS_{\mcT_h}}^{\rmT}\bbeta_{\mcS_{\mcT_h}}\right)\le F_{k,\left(\bbeta_{\mcS_{\mcT_h}}^\rmT,{\bf 0}^\rmT\right)}(z)\right)\right.\\
&&\left.\left.-I\left(F_{k,\left(\bbeta_{\mcS_{\mcT_h}}^\rmT,{\bf 0}^\rmT\right)}\left(Y_{ki}\!-\!\bZ_{ki,\mcS_{\mcT_h}}^{\rmT}\bbeta_{\mcS_{\mcT_h}}\right)\le F_{k,\left(\bbeta_{\mcS_{\mcT_h}}^\rmT,{\bf 0}^\rmT\right)}(0)\right)\right\}\rmd z>t_k\right)\\
&=&\sum_{k\in\{0\}\cup\mcT_h}P\left(\sup_{\bbeta:\left\vert \bbeta_{\mcS_{\mcT_h}^c}\right\vert_1\le r}\frac1{\left\vert \bbeta_{\mcS_{\mcT_h}^c}\right\vert_1}\left(\int_{0}^{B\left\vert \bbeta_{\mcS_{\mcT_h}^c}\right\vert_1}+\int_{0}^{-B\left\vert \bbeta_{\mcS_{\mcT_h}^c}\right\vert_1}\right)\right.\\
&&\left.\left\{\wt F_{k,\left(\bbeta_{\mcS_{\mcT_h}}^\rmT,{\bf 0}^\rmT\right)}\left( F_{k,\left(\bbeta_{\mcS_{\mcT_h}}^\rmT,{\bf 0}^\rmT\right)}(z)\right)-\wt F_{k,\left(\bbeta_{\mcS_{\mcT_h}}^\rmT,{\bf 0}^\rmT\right)}\left( F_{k,\left(\bbeta_{\mcS_{\mcT_h}}^\rmT,{\bf 0}^\rmT\right)}(0)\right)\right\}\rmd z>t_k\right)\\
&= &\sum_{k\in\{0\}\cup\mcT_h}\!\!\!\!P\left(\sup_{\bbeta:\left\vert \bbeta_{\mcS_{\mcT_h}^c}\right\vert_1\le r}\!\frac1{\left\vert \bbeta_{\mcS_{\mcT_h}^c}\right\vert_1}\!\left(\int_{0}^{B\left\vert \bbeta_{\mcS_{\mcT_h}^c}\right\vert_1}+\int_{0}^{-B\left\vert \bbeta_{\mcS_{\mcT_h}^c}\right\vert_1}\right)\!\!\left\{\left\{F_{k,\left(\bbeta_{\mcS_{\mcT_h}}^\rmT,{\bf 0}^\rmT\right)}(z)-F_{k,\left(\bbeta_{\mcS_{\mcT_h}}^\rmT,{\bf 0}^\rmT\right)}(0)\right\}\right.\right.\\
&&\left.\left.+\wh\Psi_{k,\left(\bbeta_{\mcS_{\mcT_h}}^\rmT,{\bf 0}^\rmT\right)}\left(F_{k,\left(\bbeta_{\mcS_{\mcT_h}}^\rmT,{\bf 0}^\rmT\right)}(z)\right)\!\!-\!\!\wh\Psi_{k,\left(\bbeta_{\mcS_{\mcT_h}}^\rmT,{\bf 0}^\rmT\right)}\left(F_{k,\left(\bbeta_{\mcS_{\mcT_h}}^\rmT,{\bf 0}^\rmT\right)}(0)\right)\right\}\rmd z>n_k^{1/2}t_k\right)\\
&\le &\sum_{k\in\{0\}\cup\mcT_h}\!\!\!\!P\Bigg(m_0B^2r+\sup_{\bbeta:\left\vert \bbeta_{\mcS_{\mcT_h}^c}\right\vert_1\le r}\frac1{\left\vert \bbeta_{\mcS_{\mcT_h}^c}\right\vert_1}\left(\int_{0}^{B\left\vert \bbeta_{\mcS_{\mcT_h}^c}\right\vert_1}+\int_{0}^{-B\left\vert \bbeta_{\mcS_{\mcT_h}^c}\right\vert_1}\right)\\
&&\!\!\left\{\wh\Psi_{k,\left(\bbeta_{\mcS_{\mcT_h}}^\rmT,{\bf 0}^\rmT\right)}\left(F_{k,\left(\bbeta_{\mcS_{\mcT_h}}^\rmT,{\bf 0}^\rmT\right)}(z)\right)\!\!-\!\!\wh\Psi_{k,\left(\bbeta_{\mcS_{\mcT_h}}^\rmT,{\bf 0}^\rmT\right)}^{(k)}\left(F_{k,\left(\bbeta_{\mcS_{\mcT_h}}^\rmT,{\bf 0}^\rmT\right)}(0)\right)\right\}\rmd z>n_k^{1/2}t_k\Bigg)\\
&\le &\sum_{k\in\{0\}\cup\mcT_h}\!\!\!\!P\left(\sup_{\bbeta:\left\vert \bbeta_{\mcS_{\mcT_h}^c}\right\vert_1\le r}\sup_{z\in [-Br,Br]}\right.\\
&&\left.\left\vert\wh\Psi_{k,\left(\bbeta_{\mcS_{\mcT_h}}^\rmT,{\bf 0}^\rmT\right)}\left(F_{k,\left(\bbeta_{\mcS_{\mcT_h}}^\rmT,{\bf 0}^\rmT\right)}(z)\right)\!\!-\!\!\wh\Psi_{k,\left(\bbeta_{\mcS_{\mcT_h}}^\rmT,{\bf 0}^\rmT\right)}\left(F_{k,\left(\bbeta_{\mcS_{\mcT_h}}^\rmT,{\bf 0}^\rmT\right)}(0)\right)\right\vert>\frac{n_k^{1/2}t_k}{2B}-\frac{m_0Br}2\right)\\
&\le & \sum_{k\in\{0\}\cup\mcT_h}\frac{8MB^3r}{(n_k^{1/2}t_k-m_0B^2r)^2}.
\end{eqnarray*}
\end{small}
for some positive constant $M$.
Taking $t_k=n_k^{-1/2}m_0B^2r^{1/2-\eta}$ whereas $\ k=\in\{0\}\cup\mcT_h$, for a sufficiently small positive constant $\eta$,
we conclude that
\begin{eqnarray*}
&&P\Bigg(\frac1{n_{\mcT_h}+n_0}\sum_{k\in\{0\}\cup\mcT_h}\sum_{i=1}^{n_k}\int^{\mathbf{Z}_{ki,\mcS_{\mcT_h}^c}^{\rmT}\bbeta_{\mcS_{\mcT_h}^c}}_0\!\!\left\{I\left(Y_{ki}-\bZ_{ki,\mcS_{\mcT_h}}^{\rmT}\bbeta_{\mcS_{\mcT_h}}\le z\right)-I\left(Y_{ki}-\bZ_{ki,\mcS_{\mcT_h}}^{\rmT}\bbeta_{\mcS_{\mcT_h}}\le 0\right)\right\}\rmd z\\
&&>\frac{Km_0B^2r^{1/2-\eta}}{\sqrt{n_{\mcT_h}+n_0}}\left\vert \bbeta_{\mcS_{\mcT_h}^c}\right\vert_1\Bigg)\\
&\le&P\Bigg(\frac1{n_{\mcT_h}+n_0}\sum_{k\in\{0\}\cup\mcT_h}\sum_{i=1}^{n_k}\int^{\mathbf{Z}_{ki,\mcS_{\mcT_h}^c}^{\rmT}\bbeta_{\mcS_{\mcT_h}^c}}_0\!\!\left\{I\left(Y_{ki}-\bZ_{ki,\mcS_{\mcT_h}}^{\rmT}\bbeta_{\mcS_{\mcT_h}}\le z\right)-I\left(Y_{ki}-\bZ_{ki,\mcS_{\mcT_h}}^{\rmT}\bbeta_{\mcS_{\mcT_h}}\le 0\right)\right\}\rmd z\\
&&>\sum_{k\in\{0\}\cup\mcT_h}\frac{m_0B^2r^{1/2-\eta}\sqrt{n_k}}{n_{\mcT_h}+n_0}\left\vert \bbeta_{\mcS_{\mcT_h}^c}\right\vert_1\Bigg)\\
&\le& \frac{8KMr^{2\eta}}{m_0^2B(1-r^{1/2+\eta})^2}.
\end{eqnarray*}
On the other hand, using the Markov inequality and Cauchy--Schwarz inequality, we obtain that
\begin{eqnarray*}
&&P\Bigg(\left\vert\frac1{n_{\mcT_h}+n_0}\sum_{k\in\{0\}\cup\mcT_h}\sum_{i=1}^{n_k}\bZ_{ki,\mcS_{\mcT_h}^c}^{\rmT}\bbeta_{\mcS_{\mcT_h}^c}\left\{I\left(Y_{ki}-\bZ_{ki,\mcS_{\mcT_h}}^{\rmT}\bbeta_{\mcS_{\mcT_h}}\le 0\right)-I\left(Y_{ki}-\bZ_{ki}^{\rmT}\wh\bbeta_{\mcT_h,{\rm ora}}\le 0\right)\right\}\right\vert>t\Bigg)\\
&\le &E\left\{\left(\frac1{n_{\mcT_h}+n_0}\!\!\sum_{k\in\{0\}\cup\mcT_h}\!\!\sum_{i=1}^{n_k}\bZ_{ki,\mcS_{\mcT_h}^c}^{\rmT}\bbeta_{\mcS_{\mcT_h}^c}
\!\!\left\{I\left(Y_{ki}-\bZ_{ki,\mcS_{\mcT_h}}^{\rmT}\bbeta_{\mcS_{\mcT_h}}\le 0\right)\!-\!I\left(Y_{ki}-\bZ_{ki}^{\rmT}\wh\bbeta_{\mcT_h,{\rm ora}}\le 0\right)\right\}\right)^2\right\}/t^2\\
&\le &E\Bigg(\left\{\frac1{n_{\mcT_h}+n_0}\sum_{k\in\{0\}\cup\mcT_h}\sum_{i=1}^{n_k}\left(\bZ_{ki,\mcS_{\mcT_h}^c}^{\rmT}\bbeta_{\mcS_{\mcT_h}^c}\right)^2\right\}\\
&&\left\{\frac1{n_{\mcT_h}+n_0}\sum_{k\in\{0\}\cup\mcT_h}\sum_{i=1}^{n_k}\left\vert I\left(Y_{ki}-\bZ_{ki,\mcS_{\mcT_h}}^{\rmT}\bbeta_{\mcS_{\mcT_h}}\le 0\right)-I\left(Y_{ki}-\bZ_{ki}^{\rmT}\wh\bbeta_{\mcT_h,{\rm ora}}\le 0\right)\right\vert\right\}\Bigg)/t^2.
\end{eqnarray*}
Since
\begin{eqnarray*}
&&I\left(Y_{ki}-\bZ_{ki,\mcS_{\mcT_h}}^{\rmT}\bbeta_{\mcS_{\mcT_h}}\le -Br\right)\le I\left(Y_{ki}-\bZ_{ki,\mcS_{\mcT_h}}^{\rmT}\bbeta_{\mcS_{\mcT_h}}\le \bZ^{\rmT}_{ki}\left(\wh\bbeta_{\mcT_h,{\rm ora}}-\left(\bbeta_{\mcS_{\mcT_h}}^\rmT,{\bf 0}^\rmT\right)\right)\right)\\
&=&
I\left(Y_{ki}-\bZ^{\rmT}_{ki}\wh\bbeta_{\mcT_h,{\rm ora}}\le 0\right)\le I\left(Y_{ki}-\bZ_{ki,\mcS_{\mcT_h}}^{\rmT}\bbeta_{\mcS_{\mcT_h}}\le Br\right)
\end{eqnarray*}
and
\begin{eqnarray*}
&&I\left(Y_{ki}-\bZ_{ki,\mcS_{\mcT_h}}^{\rmT}\bbeta_{\mcS_{\mcT_h}}\le -Br\right)\le I\left(Y_{ki}-\bZ_{ki,\mcS_{\mcT_h}}^{\rmT}\bbeta_{\mcS_{\mcT_h}}\le 0\right)
\le I\left(Y_{ki}-\bZ_{ki,\mcS_{\mcT_h}}^{\rmT}\bbeta_{\mcS_{\mcT_h}}\le Br\right),
\end{eqnarray*}
then
\begin{eqnarray*}
&&\left\vert I\left(Y_{ki}-\bZ_{ki,\mcS_{\mcT_h}}^{\rmT}\bbeta_{\mcS_{\mcT_h}}\le 0\right)-I\left(Y_{ki}-\bZ_{ki}^{\rmT}\wh\bbeta_{\mcT_h,{\rm ora}}\le 0\right)\right\vert\\
&\le& I\left(Y_{ki}-\bZ_{ki,\mcS_{\mcT_h}}^{\rmT}\bbeta_{\mcS_{\mcT_h}}\le Br\right)-I\left(Y_{ki}-\bZ_{ki,\mcS_{\mcT_h}}^{\rmT}\bbeta_{\mcS_{\mcT_h}}\le-Br\right)
\end{eqnarray*}
for any $i=1,\ldots,n_k$ and $\ k=1,\ldots,K.$
Consequently, we have
\begin{eqnarray*}
&&P\Bigg(\left\vert\frac1{n_{\mcT_h}+n_0}\sum_{k\in\{0\}\cup\mcT_h}\sum_{i=1}^{n_k}\bZ_{ki,\mcS_{\mcT_h}^c}^{\rmT}\bbeta_{\mcS_{\mcT_h}^c}\left\{I\left(Y_{ki}-\bZ_{ki,\mcS_{\mcT_h}}^{\rmT}\bbeta_{\mcS_{\mcT_h}}\le 0\right)-I\left(Y_{ki}-\bZ_{ki}^{\rmT}\wh\bbeta_{\mcT_h,{\rm ora}}\le 0\right)\right\}\right\vert>t\Bigg)\\
&\le &\frac{B^2\left\vert\bbeta_{\mcS_{\mcT_h}^c}\right\vert_1^2}{t^2}E\left[\frac1{n_{\mcT_h}+n_0}\sum_{k\in\{0\}\cup\mcT_h}\sum_{i=1}^{n_k}\left\{I\left(Y_{ki}-\bZ_{ki,\mcS_{\mcT_h}}^{\rmT}\bbeta_{\mcS_{\mcT_h}}\le Br\right)-I\left(Y_{ki}-\bZ_{ki,\mcS_{\mcT_h}}^{\rmT}\bbeta_{\mcS_{\mcT_h}}\le-Br\right)\right\}\right]\\
&\le &\frac{B^2\left\vert\bbeta_{\mcS_{\mcT_h}^c}\right\vert_1^2}{t^2}\sum_{k\in\{0\}\cup\mcT_h}\pi_k\left(F_{k,\left(\bbeta_{\mcS_{\mcT_h}}^\rmT,{\bf 0}^\rmT\right)^\rmT}(Br)-F_{k,\left(\bbeta_{\mcS_{\mcT_h}}^\rmT,{\bf 0}^\rmT\right)^\rmT}(-Br)\right)\\
&\le &2m_0B^3r\left\vert\bbeta_{\mcS_{\mcT_h}^c}\right\vert_1^2/t^2.
\end{eqnarray*}
Taking $t=m_0L^2r^{1/2-\eta}\left\vert\bbeta_{\mcS_{\mcT_h}}\right\vert_1$ for a sufficiently small positive constant $\eta$,
we conclude that
\begin{eqnarray*}
&&P\Bigg(\left\vert\frac1{n_{\mcT_h}+n_0}\sum_{k\in\{0\}\cup\mcT_h}\sum_{i=1}^{n_k}\bZ_{ki,\mcS_{\mcT_h}^c}^{\rmT}\bbeta_{\mcS_{\mcT_h}^c}\left\{I\left(Y_{ki}-\bZ_{ki,\mcS_{\mcT_h}}^{\rmT}\bbeta_{\mcS_{\mcT_h}}\le 0\right)-I\left(Y_{ki}-\bZ_{ki}^{\rmT}\wh\bbeta_{\mcT_h,{\rm ora}}\le 0\right)\right\}\right\vert\\
&&>m_0B^2r^{1/2-\eta}\left\vert\bbeta_{\mcS_{\mcT_h}}\right\vert_1\Bigg)\\
&\le &\frac{2r^{2\eta}}{m_0B}.
\end{eqnarray*}
Besides,
Lemma \ref{lemmaA6} implies that (\ref{mat2}) holds with probability at least $1-12p^{-1}-2p^{-2}$.
Denote $S_j(\bbeta;\mcT)$
as the $j$-th component of $\bS(\bbeta;\mcT)$ for some index set $\mcT\subset\{1,\ldots,K\}$.
As a result, we can always find a sufficient small $r>0$ such that
\begin{eqnarray*}
&&\left\{L(\bbeta;\{0\}\cup\mcT_h)+\lambda_{\bbeta}\vert\bbeta\vert_1\right\}-\left\{L\left(\left(\bbeta_{\mcS_{\mcT_h}}^\rmT,{\bf 0}^\rmT\right)^\rmT;\{0\}\cup\mcT_h\right)+\lambda_{\bbeta}\left\vert\bbeta_{\mcS_{\mcT_h}}\right\vert_1\right\}\\
&>&\sum_{j \in\mcS_{\mcT_h}^c}\left( S_j\left(\wh\bbeta_{\mcT_h,{\rm ora}};\{0\}\cup\mcT_h\right)+\left(\lambda_{\bbeta}-m_0B^2r^{1/2-\eta}-\frac{Km_0B^2r^{1/2-\eta}}{\sqrt{n_{\mcT_h}+n_0}}\right){\rm sign}(\beta_j)\right)\beta_j
\\&>&0,
\end{eqnarray*}
which holds with probability at least $1-12p^{-1}-2p^{-2}-8KMr^{\eta}/\left(m_0^2B(1-r^{1/2+\eta})^2\right)-2r^{2\eta}/(m_0B)$.
Furthermore, for any $\bbeta$ with $\bbeta_{\mcS_{\mcT_h}^c}\neq {\bf 0}$, there exists a vector
$\bbeta^\ddag=s\bbeta+(1-s)\wh\bbeta_{\mcT_h,{\rm ora}}$ with $\bbeta^\ddag\in \mathcal{B}\left(\wh\bbeta_{\mcT_h,{\rm ora}},r\right)$.
The convexity of the object function in (\ref{op_trans}) entails that
\begin{eqnarray*}
&&s\left(L(\bbeta;\{0\}\cup\mcT_h)+\lambda_{\bbeta}\vert\bbeta\vert_1\right)+(1-s)\left(L\left(\wh\bbeta_{\mcT_h,{\rm ora}};\{0\}\cup\mcT_h\right)+\lambda_{\bbeta}\left\vert\wh\bbeta_{\mcT_h,{\rm ora}}\right\vert_1\right)
\\&\ge& L\left(\bbeta^\ddag;\{0\}\cup\mcT_h\right)+\lambda_{\bbeta}\left\vert\bbeta^\ddag\right\vert_1.
\end{eqnarray*}
The above inequality implies that $\bbeta$ is not the solution of (\ref{op_trans}) since
\begin{eqnarray*}
&&L(\bbeta;\{0\}\cup\mcT_h)+\lambda_{\bbeta}\vert\bbeta\vert_1\\&\ge&\frac1s\left\{ L\left(\bbeta^\ddag;\{0\}\cup\mcT_h\right)+\lambda_{\bbeta}\left\vert\bbeta^\ddag\right\vert_1-(1-s)\left(L\left(\wh\bbeta_{\mcT_h,{\rm ora}};\{0\}\cup\mcT_h\right)+\lambda_{\bbeta}\left\vert\wh\bbeta_{\mcT_h,{\rm ora}}\right\vert_1\right)\right\}\\
&\ge& L(\bbeta^\ddag;\{0\}\cup\mcT_h)+\lambda_{\bbeta}\vert\bbeta^\ddag\vert_1>
L\left(\wh\bbeta_{\mcT_h,{\rm ora}};\{0\}\cup\mcT_h\right)+\lambda_{\bbeta}\left\vert\wh\bbeta_{\mcT_h,{\rm ora}}\right\vert_1,
\end{eqnarray*}
which concludes the proof,  implying that all the solutions of (\ref{op_trans}) satisfies
$\wh\bbeta_{\mcT_h,\mcS_{\mcT_h}^c}={\bf 0}$ with probability at least $1-12p^{-1}-2p^{-2}$.
\end{proof}

The oracle estimation of the refined  step can be defined as
\begin{eqnarray}
\label{op_oracle_deb}
\wh{\bdelta}_{\mcT_h,{\rm ora}}\in\arg\min_{\bdelta:\bdelta_{\bar{\mcS}^c}={\bf 0}}
L\left(\wh\bbeta_{\mcT_h,{\rm ora}}+\bdelta;\{0\}\right)+\lambda_{\bdelta}\vert\bdelta\vert_1.
\end{eqnarray}

\begin{lemma}
\label{lemmaA8}
Under assumptions \ref{C1}--\ref{C3}, \ref{C5}, and \ref{C6},
with $$ a_n=4h',\ \mcS=\mcS_{\mcT_h} $$
and $$ a_n=0,\ \mcS=\bar{\mcS} $$
in assumption \ref{C3},
if
$$ \lambda_{\bbeta}\ge \left\{\frac{8D}{\gamma}\sqrt{\frac{s'}{\log p}}+ \frac{32\sqrt{2}}{\gamma}\right\}B\sqrt{\frac{\log p}{n_{\mcT_h}+n_0}}$$
and
$$ \lambda_{\bdelta}\ge \left\{\frac{6\wt D}{\gamma}\sqrt{\frac{s_0+s'}{\log p}}+ \frac{24\sqrt{2}}{\gamma}\right\}B\sqrt{\frac{\log p}{n_0}}, $$
where the positive constant $\wt D$ is determined by (\ref{determineD_deb}),
$$ \sqrt{s'\lambda_{\bbeta}}\to0,\ \frac{s'\log p}{\lambda_{\bbeta}\left(n_{\mcT_h}+n_0\right)}\to 0,\ \sqrt{(s'+s_0)\lambda_{\bdelta}}\to0, \frac{(s'+s_0)\log p}{\lambda_{\bdelta}n_0}\to 0,
\ {\rm and}\ h'<\frac{\gamma\lambda_{\bbeta}}{8m_0B^2},$$
then
with probability more than $1-24p^{-1}-4p^{-2}$, $\wh\bbeta_{\mcT_h,{\rm ora}}$ is a solution of problem (\ref{op_debias}).
\end{lemma}

\begin{proof}
According to the Karush--Kuhn--Tucker condition of (\ref{op_debias}), we have
\begin{eqnarray}
\label{KKT_debias}
\bS\left(\wh\bbeta_{\mcT_h,{\rm ora}}+\wh\bdelta_{\mcT_h};\{0\}\right)+\lambda_{\bdelta}{\rm \bf  sign}\left(\wh\bdelta_{\mcT_h}\right)={\bf 0}.
\end{eqnarray}
Thus, $\wh\bdelta_{\mcT_h,{\rm ora}}$ is the solution of (\ref{op_debias}) if and only if it satisfies (\ref{KKT_debias}), which can be written as
\begin{eqnarray}
\label{op_oracle1}
\bS\left(\wh\bbeta_{\mcT_h,{\rm ora}}+\wh\bdelta_{\mcT_h,{\rm ora}};\{0\}\right)-\bS\left(\bbeta_0;\{0\}\right)+\bS\left(\bbeta_0;\{0\}\right)+\lambda_{\bdelta}{\rm \bf  sign}\left(\wh\bdelta_{\mcT_h,{\rm ora}}\right)={\bf 0}.
\end{eqnarray}
In the following, we denote
\begin{eqnarray*}
\bR_3&=&E\left[\bZ_{0}\left\{I\left(Y_{0}-\bZ_{0}^\rmT\bbeta\le 0\right)-I\left(Y_{0}-\bZ_{0}^\rmT\bbeta_0\le 0\right)\right\}\right]\bigg\vert_{\bbeta=\wh\bbeta_{\mcT_h,{\rm ora}}+\wh\bdelta_{\mcT_h,\rm ora}}\\
&&-\bH^{*}_0\left(\wh\bbeta_{\mcT_h,{\rm ora}}+\wh\bdelta_{\mcT_h,\rm ora}-\bbeta_0\right)
\end{eqnarray*}
and
\begin{eqnarray*}
\bR_4&=&\frac1{n_0}\sum_{i=1}^{n_0}\bZ_{0i}\left\{I\left(Y_{0i}-\bZ_{0i}^\rmT\left(\wh\bbeta_{\mcT_h,{\rm ora}}+\wh\bdelta_{\mcT_h,\rm ora}\right)\le 0\right)-I\left(Y_{0i}-\bZ_{0i}^\rmT\bbeta_0\le 0\right)\right\}\\
&&-E\left[\bZ_0\left\{I\left(Y_{0}-\bZ_{0}^\rmT\bbeta\le 0\right)-I\left(Y_{0}-\bZ_{0}^\rmT\bbeta_0\le 0\right)\right\}\right]\bigg\vert_{\bbeta=\wh\bbeta_{\mcT_h,{\rm ora}}+\wh\bdelta_{\mcT_h,\rm ora}}.
\end{eqnarray*}
As a result, (\ref{op_oracle1}) can be rewritten as
\begin{eqnarray*}
\bH^{*}_0\left(\wh\bbeta_{\mcT_h,{\rm ora}}+\wh\bdelta_{\mcT_h,\rm ora}-\bbeta_0\right)
+\bR_3+\bR_4+\bS(\bbeta_0;\{0\})+\lambda_{\bdelta}{\rm \bf  sign}\left(\wh\bdelta_{\mcT_h,\rm ora}\right)={\bf 0},
\end{eqnarray*}
which is further exhibited in the form of partitioned matrices as
\begin{eqnarray*}
\label{fenkuai}
&&\left(
\begin{matrix}
\bH^{*}_{0,\bar\mcS\bar\mcS}   & \bH^{*}_{0,\bar\mcS\bar{\mcS}^c}\\
\bH^{*}_{0,\bar{\mcS}^c\bar\mcS} & \bH^{*}_{0,\bar{\mcS}^c\bar{\mcS}^c}
\end{matrix}
\right)
\left(
\begin{matrix}
\wh\bbeta_{\mcT_h,{\rm ora},\bar\mcS}+\wh\bdelta_{\mcT_h,\rm ora,\bar\mcS}-\bbeta_{0,\bar{\mcS}}\\
{\bf 0}
\end{matrix}
\right)
+\left(
\begin{matrix}
\bR_{3\bar\mcS}+\bR_{4\bar\mcS}\\
\bR_{3\bar{\mcS}^c}+\bR_{4\bar{\mcS}^c}
\end{matrix}
\right)
+\left(
\begin{matrix}
\bS_{\bar{\mcS}}(\bbeta_0;\{0\})\\
\bS_{\bar{\mcS}^c}(\bbeta_0;\{0\})
\end{matrix}
\right)
\\&=&-\lambda_{\bdelta}
\left(
\begin{matrix}
{\rm\bf sign}_{\bar\mcS}\left(\wh\bdelta_{\mcT_h,\rm ora}\right)\\
\wt\bz
\end{matrix}
\right),
\end{eqnarray*}
where $\wt\bz\in[-1,1]^{\left\vert\bar{\mcS}^c\right\vert}$.
Note that the first part of partition is  the Karush--Kuhn--Tucker condition of (\ref{op_oracle_deb}) that
\begin{eqnarray}
\label{mat1_deb}
\bH^{*}_{0,\bar\mcS\bar\mcS}\left(\wh\bbeta_{\mcT_h,{\rm ora},\bar\mcS}+\wh\bdelta_{\mcT_h,\rm ora,\bar\mcS}-\bbeta_{0,\bar{\mcS}}\right)
+\bR_{3\bar\mcS}+\bR_{4\bar\mcS}+\bS_{\bar\mcS}(\bbeta_0;\{0\})=-\lambda_{\bdelta}{\bf sign}_{\bar\mcS}\left(\wh\bdelta_{\mcT_h,\rm ora}\right).
\end{eqnarray}
Consequently, we only need to verify the sufficient form of the second part of partition that
\begin{eqnarray}
\label{mat2_deb}
&&\left\vert\bS_{\bar{\mcS}^c}\left(\wh\bbeta_{\mcT_h,{\rm ora}}+\wh\bdelta_{\mcT_h,\rm ora};\{0\}\right)\right\vert_\infty
\nonumber\\&=&\left\vert\bH^{*}_{0,\bar{\mcS}^c\bar\mcS}\left(\wh\bbeta_{\mcT_h,{\rm ora},\bar\mcS}+\wh\bdelta_{\mcT_h,\rm ora,\bar\mcS}-\bbeta_{0\bar{\mcS}}\right)
+\bR_{3\bar{\mcS}^c}+\bR_{4\bar{\mcS}^c}+\bS_{\bar{\mcS}^c}(\bbeta_0;\{0\})\right\vert_\infty
\nonumber\\&<&\lambda_{\bdelta}.
\end{eqnarray}
Using (\ref{mat1_deb}), (\ref{mat2_deb}) is rewritten as
\begin{eqnarray*}
&&\Bigg\vert\bH^{*}_{0,\bar{\mcS}^c\bar\mcS}\left(\bH^{*}_{0,\bar{\mcS}\bar\mcS}\right)^{-1}\left(
-\bR_{3\bar\mcS}-\bR_{4\bar\mcS}-\bS_{\bar\mcS}(\bbeta_0;\{0\})-\lambda_{\bdelta}{\bf sign}_{\bar\mcS}\left(\wh\bdelta_{\mcT_h,\rm ora}\right)\right)
\\&&+\bR_{3\bar{\mcS}^c}+\bR_{4\bar{\mcS}^c}+\bS_{\bar{\mcS}^c}(\bbeta_0;\{0\})\Bigg\vert_\infty
<\lambda_{\bdelta}.
\end{eqnarray*}
On the contrary, under assumption \ref{C6}, the probability that (\ref{mat2_deb}) does not hold can be controlled by
\begin{eqnarray*}
&&P\left(\left\vert\bH^{*}_{0,\bar{\mcS}^c\bar\mcS}\left(\wh\bbeta_{\mcT_h,{\rm ora},\bar\mcS}+\wh\bdelta_{\mcT_h,\rm ora,\bar\mcS}-\bbeta_{0,\bar{\mcS}}\right)
+\bR_{3\bar{\mcS}^c}+\bR_{4\bar{\mcS}^c}+\bS_{\bar{\mcS}^c}(\bbeta_0;\{0\})\right\vert_\infty\ge\lambda_{\bdelta}\right)\\
&\le&P\left(\left\Vert\bH^{*}_{0,\bar{\mcS}^c\bar\mcS}\left(\bH^{*}_{0,\bar{\mcS}\bar\mcS}\right)^{-1}\right\Vert_\infty
\left\vert\bR_{3\bar{\mcS}}\right\vert_\infty\ge\frac{\gamma\lambda_{\bdelta}}6\right)\\
&&+P\left(\left\Vert\bH^{*}_{0,\bar{\mcS}^c\bar\mcS}\left(\bH^{*}_{0,\bar{\mcS}\bar\mcS}\right)^{-1}\right\Vert_\infty
\left\vert\bR_{4\bar{\mcS}}\right\vert_\infty\ge\frac{\gamma\lambda_{\bdelta}}6\right)\\
&&+P\left(\left\Vert\bH^{*}_{0,\bar{\mcS}^c\bar\mcS}\left(\bH^{*}_{0,\bar{\mcS}\bar\mcS}\right)^{-1}\right\Vert_\infty
\left\vert\bS_{\bar\mcS}(\bbeta_0;\{0\})\right\vert_\infty\ge\frac{\gamma\lambda_{\bdelta}}6\right)\\
&&+P\left(\lambda_{\bdelta}\left\Vert\bH^{*}_{0,\bar{\mcS}^c\bar\mcS}\left(\bH^{*}_{0,\bar{\mcS}\bar\mcS}\right)^{-1}\right\Vert_\infty
\left\vert{\bf sign}_{\bar\mcS}\left(\wh\bdelta_{\mcT_h,\rm ora}\right)\right\vert_\infty\ge (1-\gamma)\lambda_{\bdelta}\right)\\
&&+P\left(
\left\vert\bR_{3\bar{\mcS}^c}\right\vert_\infty\ge \frac{\gamma\lambda_{\bdelta}}6\right)\\
&&+P\left(
\left\vert\bR_{4\bar{\mcS}^c}\right\vert_\infty\ge \frac{\gamma\lambda_{\bdelta}}6\right)\\
&&+P\left(
\left\vert\bS_{\bar{\mcS}^c}(\bbeta_0;\{0\})\right\vert_\infty\ge \frac{\gamma\lambda_{\bdelta}}6\right)\\
&=&(I_{1})+(I_{2})+(I_{3})+(I_{4})+(I_{5})+(I_{6})+(I_{7}),
\end{eqnarray*}
where terms $(I_{1})$--$(I_{7})$ are self-defined above
and will be expired when
the proof of this lemma is completed.
Note that
\begin{eqnarray*}
(I_{4})&=&P\left(\lambda_{\bdelta}\left\Vert\bH^{*}_{0,\bar{\mcS}^c\bar\mcS}\left(\bH^{*}_{0,\bar{\mcS}\bar\mcS}\right)^{-1}\right\Vert_\infty
\left\vert{\bf sign}_{\bar\mcS}\left(\wh\bdelta_{\mcT_h,\rm ora}\right)\right\vert_\infty\ge (1-\gamma)\lambda_{\bdelta}\right)=0.
\end{eqnarray*}
Moreover, according to (\ref{S0}), we obtain
\begin{eqnarray*}
(I_{3})&=&P\left(\left\Vert\bH^{*}_{0,\bar{\mcS}^c\bar\mcS}\left(\bH^{*}_{0,\bar{\mcS}\bar\mcS}\right)^{-1}\right\Vert_\infty
\left\vert\bS_{\bar\mcS}(\bbeta_0;\{0\})\right\vert_\infty\ge\frac{\gamma\lambda_{\bdelta}}6\right)\\
&\le&P\left(
\left\vert\bS_{\bar\mcS}(\bbeta_0;\{0\})\right\vert_\infty\ge\frac{\gamma\lambda_{\bdelta}}{6(1-\gamma)}\right)\\
&\le&2p\exp\left(-\frac{n_0\gamma^2\lambda_{\bdelta}^2}{72B^2(1-\gamma)^2}\right)\le 2p^{-1}
\end{eqnarray*}
and
\begin{eqnarray*}
(I_{7})=P\left(
\left\vert\bS_{\bar{\mcS}^c}(\bbeta_0;\{0\})\right\vert_\infty\ge \frac{\gamma\lambda_{\bdelta}}6\right)
\le 2p\exp\left(-\frac{n_0\gamma^2\lambda_{\bdelta}^2}{72B^2}\right)\le 2p^{-1}.
\end{eqnarray*}
On the other hand, we have
\begin{eqnarray*}
&&P\left(\vert\bR_3\vert_\infty\ge t\right)\\
&=&P\Bigg(\Bigg\vert E\left[\bZ_{0}\left\{I\left(Y_{0}-\bZ_{0}^\rmT\bbeta\le 0\right)-I\left(Y_{0}-\bZ_{0}^\rmT\bbeta_0\le 0\right)\right\}\right]\bigg\vert_{\bbeta=\wh\bbeta_{\mcT_h,{\rm ora}}+\wh\bdelta_{\mcT_h,\rm ora}}\\
&&\hspace{1in}-\bH^{*}_0\left(\wh\bbeta_{\mcT_h,{\rm ora}}+\wh\bdelta_{\mcT_h,\rm ora}-\bbeta_0\right)\Bigg\vert_\infty\ge t\Bigg)\\
&=&P\Bigg(\Bigg\vert E\left[\bZ_{0}\bZ_{0}^\rmT\int^1_0f_0
\left(\bZ_{0}^\rmT\bbeta_0+t\bZ_{0}^\rmT(\bbeta-\bbeta_0)\mid\bZ_{0}\right)\rmd t\right]\bigg\vert_{\bbeta=\wh\bbeta_{\mcT_h,{\rm ora}}+\wh\bdelta_{\mcT_h,\rm ora}}\left(\wh\bbeta_{\mcT_h,{\rm ora}}+\wh\bdelta_{\mcT_h,\rm ora}-\bbeta_0\right)\\
&&\hspace{1in}-\bH^{*}_0\left(\wh\bbeta_{\mcT_h,{\rm ora}}+\wh\bdelta_{\mcT_h,\rm ora}-\bbeta_0\right)\Bigg\vert_\infty\ge t\Bigg)\\
&=&P\Bigg(\Bigg\vert E\Bigg[\bZ_{0}\bZ_{0}^\rmT\int^1_0\left\{f_0
\left(\bZ_{0}^\rmT\bbeta_0\!+\!t\bZ_{0}^\rmT(\bbeta\!-\!\bbeta_0)\mid \bZ_{0}\right)\!-\!f_0
\left(\bZ_{0}^\rmT\bbeta_0\mid \bZ_{0}\right)\right\}\rmd t\Bigg]\bigg\vert_{\bbeta=\wh\bbeta_{\mcT_h,{\rm ora}}+\wh\bdelta_{\mcT_h,\rm ora}}\\
&&\left(\wh\bbeta_{\mcT_h,{\rm ora}}+\wh\bdelta_{\mcT_h,\rm ora}-\bbeta_0\right)\Bigg\vert_\infty
\ge t\Bigg)\\
&\le&P\Bigg(\Bigg\vert m_0B\left\vert\wh\bbeta_{\mcT_h,{\rm ora}}+\wh\bdelta_{\mcT_h,\rm ora}-\bbeta_0\right\vert_1E\left[\bZ_{0}\bZ_{0}^\rmT
\right]\left(\wh\bbeta_{\mcT_h,{\rm ora}}+\wh\bdelta_{\mcT_h,\rm ora}-\bbeta_0\right)\Bigg\vert_\infty
\ge t\Bigg)\\
&\le&P\Bigg(m_0B^3\left\vert\wh\bbeta_{\mcT_h,{\rm ora}}+\wh\bdelta_{\mcT_h,\rm ora}-\bbeta_0\right\vert^2_1
\ge t\Bigg)
\end{eqnarray*}
for any $t>0$. Thus, it holds that
\begin{eqnarray*}
(I_{1})\le P\left(\left\vert\bR_3\right\vert_\infty\ge\frac{\gamma\lambda_{\bdelta}}{6(1-\gamma)}\right)
\le P\left(\left\vert\wh\bbeta_{\mcT_h,{\rm ora}}+\wh\bdelta_{\mcT_h,\rm ora}-\bbeta_0\right\vert^2_1\ge\frac{\gamma\lambda_{\bdelta}}{6m_0B^3(1-\gamma)}\right)
\end{eqnarray*}
and
\begin{eqnarray*}
(I_{5})\le P\left(\left\vert\bR_3\right\vert_\infty\ge\frac{\gamma\lambda_{\bdelta}}{6}\right)\le
P\left(\left\vert\wh\bbeta_{\mcT_h,{\rm ora}}+\wh\bdelta_{\mcT_h,\rm ora}-\bbeta_0\right\vert^2_1
\ge\frac{\gamma\lambda_{\bdelta}}{6m_0B^3}\right).
\end{eqnarray*}
Lemma \ref{lemmaA5} indicates that
\begin{eqnarray*}
\left\vert\bS(\bbeta_0;\{0\})\right\vert_\infty
\le 2B\sqrt{\frac{\log p}{n_0}}
\end{eqnarray*}
with probability more than $1-2p^{-1}$.
Therefore, when $\lambda_{\bdelta}\ge 4B\sqrt{\log p/n_0}$, it holds that
\begin{eqnarray}
\label{Hessian_deb}
&&\delta L\left(\wh\bbeta_{\mcT_h,{\rm ora}}+\wh\bdelta_{\mcT_h,\rm ora}-\bbeta_0;\{0\}\right)\nonumber\\&\le&L\left(\wh\bbeta_{\mcT_h,{\rm ora}}+\wh\bdelta_{\mcT_h,\rm ora};\{0\}\right)- L(\bbeta_0;\{0\})- \bS(\bbeta_0;\{0\})^\rmT\left(\wh\bbeta_{\mcT_h,{\rm ora}}+\wh\bdelta_{\mcT_h,\rm ora}-\bbeta_0\right)\nonumber\\
&\le& -\lambda_{\bdelta}\left\vert\wh\bdelta_{\mcT_h,\rm ora}\right\vert_1
+\lambda_{\bdelta}\left\vert\bbeta_0-\wh\bbeta_{\mcT_h,{\rm ora}}\right\vert_1+\frac12\lambda_{\bdelta}\left\vert\wh\bbeta_{\mcT_h,{\rm ora}}+\wh\bdelta_{\mcT_h,\rm ora}-\bbeta_0\right\vert_1\nonumber\\
&\le&\frac32\lambda_{\bdelta}\left\vert\wh\bbeta_{\mcT_h,{\rm ora}}+\wh\bdelta_{\mcT_h,\rm ora}-\bbeta_0\right\vert_1
\end{eqnarray}
with probability more than $1-2p^{-1}$.
Since $\wh\bbeta_{\mcT_h,{\rm ora}}+\wh\bdelta_{\mcT_h,\rm ora}-\bbeta_0\in\mathcal{D}(0,\bar{\mcS})$, then according to Lemma \ref{lemmaA4}, it yields with probability more than $1-p^{-2}$ that
\begin{eqnarray*}
&&\delta\wh L^{(0)}\left(\wh\bdelta_{\mcT_h,\rm ora}, \wh\bbeta_{\mcT_h,{\rm ora}}\right)
\\&\ge& \min\left\{\frac{\kappa_0}4\left\vert\wh\bbeta_{\mcT_h,{\rm ora}}+\wh\bdelta_{\mcT_h,\rm ora}-\bbeta^{*}\right\vert_2^2,\frac{3\kappa_0^3}{16m_0}\left\vert\wh\bbeta_{\mcT_h,{\rm ora}}+\wh\bdelta_{\mcT_h,\rm ora}-\bbeta^{*}\right\vert_2\right\}
\\&&-B\sqrt{\frac{\log p}{n_0}}\left\vert\wh\bbeta_{\mcT_h,{\rm ora}}+\wh\bdelta_{\mcT_h,\rm ora}-\bbeta^{*}\right\vert_1,
\end{eqnarray*}
which implies that
\begin{eqnarray*}
\left\vert\wh\bbeta_{\mcT_h,{\rm ora}}+\wh\bdelta_{\mcT_h,\rm ora}-\bbeta_0\right\vert_2\le \frac{4B}{\kappa_0}\sqrt{\frac{(s_0+s')\log p}{n_0}}+\frac6{\kappa_0}\lambda_{\bdelta}\sqrt{s'+s_0}
\end{eqnarray*}
with probability more than $1-2p^{-1}-p^{-2}$ by combining (\ref{Hessian_deb}) and mimicking the proof of Theorem \ref{theorem2.1}.
For sufficiently large $n_{\mcT_h},n_0$ and $p$, it holds that
$$ \frac{4B}{\kappa_0}\sqrt{\frac{(s_0+s')\log p}{n_0}}+\frac6{\kappa_0}\lambda_{\bdelta}\sqrt{s'+s_0}\le\sqrt{\frac{\gamma\lambda_{\bdelta}}{6m_0B^3}}. $$
As a result, we have
$$ (I_{1})\le 2p^{-1}+p^{-2}\ {\rm and}\ (I_{5})\le 2p^{-1}+p^{-2}. $$
Denote $\wt{\mathcal{F}}_j=\{\wt f_{\bdelta,j}\mid \bdelta_{\bar{\mcS}^c}={\bf 0}\}$,
where
\begin{eqnarray*}
\wt f_{\bdelta,j}(Y,\bZ)&=&Z_j\left\{I\left(Y-\bZ^\rmT\left(\wh\bbeta_{\mcT_h}+\bdelta\right)\le 0\right)-I\left(Y-\bZ^\rmT\bbeta_0\le 0\right)\right\}\\
&&-E\left(Z_j\left\{I\left(Y-\bZ^\rmT\left(\wh\bbeta_{\mcT_h}+\bdelta\right)\le 0\right)-I(Y-\bZ^\rmT\bbeta_0\le 0)\right\}\right).
\end{eqnarray*}
Since the Vapnik--Chervonenkis dimension of $\wt{\mathcal{F}}_j$ is less than $s_0+s'$, then Lemma 19.15 in 
\citet{van1998asymptotic}
entails that
$$ \sup_Q N\left(\epsilon B,\wt{\mathcal{F}}_j,L_2(Q)\right)\le \wt M(s_0+s')(16e)^{s_0+s'}\epsilon^{1-2s_0-2s'} $$
for some positive constant $\wt M>0$.
Thus, Lemma 19.35 in 
\citet{van1998asymptotic}
implies that
\begin{eqnarray}
\label{determineD_deb}
&&E\left\{\sup_{\bdelta: \bdelta_{\bar{\mcS}^c}={\bf 0}_{\vert\bar{\mcS}^c\vert}}
\frac1{n_0}\sum_{i=1}^{n_0}\wt f_{\bdelta,j}(Y_{0i},\bZ_{0i})\right\}\nonumber\\
&\le&n_0^{-1/2} BJ(1,\wt{\mathcal{F}}_j,L_2)\nonumber
\\&=& \frac{B}{\sqrt{n_0}}\int^1_0 \sqrt{\log\sup_Q N(\epsilon B,\wt{\mathcal{F}}_j,L_2(Q))}\rmd \epsilon\nonumber\\
&\le&\wt DB\sqrt{s_0+s'}n_0^{-1/2}
\end{eqnarray}
for some constant $\wt D>0$.
Employing Theorem 3.26 in 
\citet{wainwright2019high}, it holds that
\begin{eqnarray*}
&&P(\vert\bR_4\vert_\infty\ge \wt DB\sqrt{s_0+s'}n_0^{-1/2}+t)\\&\le&
\sum_{j=1}^pP\Bigg(\Bigg\vert\frac1{n_0}\sup_{\bdelta: \bdelta_{\bar{\mcS}^c}={\bf 0}}\sum_{i=1}^{n_0}\wt f_{\bdelta,j}(Y_{0i},\bZ_{0i})\Bigg\vert
\ge n_0^{-1/2}\wt DB\sqrt{s_0+s'}+t\Bigg)\\
&\le&2p\exp\left(-\frac{n_0t^2}{16B^2}\right).
\end{eqnarray*}
By taking $t=4B\sqrt{2\log p/n_0}$, we obtain
$$ \vert\bR_2\vert_\infty< \wt DB\sqrt{\frac{s_0+s'}{n_0}}+ 4B\sqrt{\frac{2\log p}{n_0}} $$
with probability more than $1-2p^{-1}$.

As a result, when
$$ \lambda_{\bdelta}\ge \frac{6\wt DB}{\gamma}\sqrt{\frac{s_0+s'}{n_0}}+ \frac{24B}{\gamma}\sqrt{\frac{2\log p}{n_0}} , $$
we have
\begin{eqnarray*}
(I_{2})=P\left(\left\Vert\bH^{*}_{0,\bar{\mcS}^c\bar{\mcS}}\left(\bH^{*}_{0,\bar{\mcS}\bar{\mcS}}\right)^{-1}\right\Vert_\infty
\left\vert\bR_{4\bar{\mcS}}\right\vert_\infty\ge\frac{\gamma\lambda_{\bdelta}}6\right)
\le P\left(
\left\vert\bR_{4\bar{\mcS}}\right\vert_\infty\ge\frac{\gamma\lambda_{\bdelta}}{6(1-\gamma)}\right)\le 2p^{-1}
\end{eqnarray*}
and
\begin{eqnarray*}
(I_{6})=P\left(
\left\vert\bR_{4\bar{\mcS}^c}\right\vert_\infty\ge \frac{\gamma\lambda_{\bdelta}}6\right)
&\le&2p^{-1}.
\end{eqnarray*}
Combining the pieces, we conclude that (\ref{mat2_deb}) holds with probability more than $1-24p^{-1}-4p^{-2}$,
which concludes the proof of Lemma \ref{lemmaA8}.
\end{proof}

\begin{lemma}
\label{lemmaA9}
Under assumptions \ref{C1}--\ref{C3}, \ref{C5}, and \ref{C6},
with $$ a_n=4h',\ \mcS=\mcS_{\mcT_h} $$
and $$ a_n=0,\ \mcS=\bar{\mcS} $$
in assumption \ref{C3},
if
$$ \lambda_{\bbeta}\ge \left\{\frac{8D}{\gamma}\sqrt{\frac{s'}{\log p}}+ \frac{32\sqrt{2}}{\gamma}\right\}B\sqrt{\frac{\log p}{n_{\mcT_h}+n_0}}, $$
and
$$ \lambda_{\bdelta}\ge \left\{\frac{6\wt D}{\gamma}\sqrt{\frac{s_0+s'}{\log p}}+ \frac{24\sqrt{2}}{\gamma}\right\}B\sqrt{\frac{\log p}{n_0}}, $$
where the positive constant $\wt D$ is determined by (\ref{determineD_deb}),
$$ \sqrt{s'\lambda_{\bbeta}}\to0,\ \frac{s'\log p}{\lambda_{\bbeta}\left(n_{\mcT_h}+n_0\right)}\to 0,\ \sqrt{(s'+s_0)\lambda_{\bdelta}}\to0, \frac{(s'+s_0)\log p}{\lambda_{\bdelta}n_0}\to 0,
\ {\rm and}\ h'<\frac{\gamma\lambda_{\bbeta}}{8m_0B^2},$$
then
with probability more than $1-24p^{-1}-4p^{-2}$, any solution $\wh\bdelta_{\mcT_h}$ of problem (\ref{op_debias})
satisfies $\wh\bdelta_{\mcT_h,\bar{\mcS}^c}={\bf 0}$.
\end{lemma}

\begin{proof}
From Lemma \ref{lemmaA8}, when $\wh\bbeta_{\mcT_h}$ is an oracle solution, the event that the oracle estimator $\wh\bdelta_{\mcT_h,\rm ora}$ is a solution of (\ref{op_debias})
occurs with probability more than $1-24p^{-1}-4p^{-2}$.
On this event, to demonstrate that any solution $\wh\bdelta$ of problem (\ref{op_debias})
satisfies $\wh\bdelta_{\bar{\mcS}^c}={\bf 0}$,
it is sufficient to verify that
\begin{eqnarray*}
L\left(\wh\bbeta_{\mcT_h}+\bdelta;\{0\}\right)+\lambda_{\bdelta}\vert\bdelta\vert_1> L\left(\wh\bbeta_{\mcT_h}+\wh\bdelta_{\mcT_h,\rm ora};\{0\}\right)+\lambda_{\bdelta}\left\vert\wh\bdelta_{\mcT_h,\rm ora}\right\vert_1
\end{eqnarray*}
for any $\bdelta$ with $\bdelta_{\bar{\mcS}^c}\neq {\bf 0}$.
We first show that
\begin{eqnarray*}
L\left(\wh\bbeta_{\mcT_h}+\bdelta;\{0\}\right)+\lambda_{\bdelta}\vert\bdelta\vert_1> L\left(\wh\bbeta_{\mcT_h}+\left(\bdelta_{\bar{\mcS}}^\rmT,{\bf 0}^\rmT\right)^\rmT;\{0\}\right)+\lambda_{\bdelta}\left\vert\bdelta_{\bar{\mcS}}\right\vert_1
\end{eqnarray*}
for any $\bdelta\in \mathcal{B}\left(\wh\bdelta_{\mcT_h,\rm ora},r\right)$ with $\bdelta_{\bar{\mcS}^c}\neq {\bf 0}$.
The Knight equation entails that
\begin{eqnarray*}
&&L\left(\wh\bbeta_{\mcT_h}+\bdelta\right)\\
&=&L\left(\wh\bbeta_{\mcT_h}+\wh\bdelta_{\mcT_h,\rm ora};\{0\}\right)+\bS\left(\wh\bbeta_{\mcT_h}+\wh\bdelta_{\mcT_h,\rm ora};\{0\}\right)^\rmT\left(\bdelta-\wh\bdelta_{\mcT_h,\rm ora}\right)\\
&&+\frac1{n_0}\sum_{i=1}^{n_0}\int^{\mathbf{Z}_{0i}^\rmT\left(\bdelta-\wh\bdelta_{\mcT_h,\rm ora}\right)}_0\left\{I\left(Y_{0i}-\bZ_{0i}^\rmT\left(\wh\bbeta_{\mcT_h}+\wh\bdelta_{\mcT_h,\rm ora}\right)\le z\right)-I\left(Y_{0i}-\bZ_{0i}^\rmT\left(\wh\bbeta_{\mcT_h}+\wh\bdelta_{\mcT_h,\rm ora}\right)\le 0\right)\right\}\rmd z
\end{eqnarray*}
and
\begin{eqnarray*}
&&L\left(\wh\bbeta_{\mcT_h}+\left(\bdelta_{\bar{\mcS}}^\rmT,{\bf 0}^\rmT\right)^\rmT;\{0\}\right)\\
&=&L\left(\wh\bbeta_{\mcT_h}+\wh\bdelta_{\mcT_h,\rm ora};\{0\}\right)+\bS\left(\wh\bbeta_{\mcT_h}+\wh\bdelta_{\mcT_h,\rm ora};\{0\}\right)^\rmT\left(\left(\bdelta_{\bar{\mcS}}^\rmT,{\bf 0}_{\vert\bar{\mcS}^c\vert}^\rmT\right)^\rmT-\wh\bdelta_{\mcT_h,\rm ora}\right)\\
&&+\frac1{n_0}\sum_{i=1}^{n_0}\!\!\int^{\mathbf{Z}_{0i}^\rmT\left(\left(\bdelta_{\bar{\mcS}}^\rmT,{\bf 0}^\rmT\right)^\rmT-\wh\bdelta_{\mcT_h,\rm ora}\right)}_0\!\!\left\{I\left(Y_{0i}\!-\!\bZ_{0i}^\rmT\left(\wh\bbeta_{\mcT_h}\!+\!\wh\bdelta_{\mcT_h,\rm ora}\right)\le z\right)\!-\!I\left(Y_{0i}\!-\!\bZ_{0i}^\rmT\left(\wh\bbeta_{\mcT_h}\!+\!\wh\bdelta_{\mcT_h,\rm ora}\right)\le 0\right)\right\}\rmd z.
\end{eqnarray*}
Thus, we arrive at
\begin{eqnarray*}
&&\left\{L\left(\wh\bbeta_{\mcT_h}+\bdelta;\{0\}\right)+\lambda_{\bdelta}\vert\bdelta\vert_1\right\}-\left\{ L\left(\wh\bbeta_{\mcT_h}+\left(\bdelta_{\bar{\mcS}}^\rmT,{\bf 0}^\rmT\right)^\rmT;\{0\}\right)+\lambda_{\bdelta}\left\vert\bdelta_{\bar{\mcS}}\right\vert_1\right\}\\
&=&\bS_{\bar\mcS^c}\left(\wh\bbeta_{\mcT_h}+\wh\bdelta_{\mcT_h,\rm ora};\{0\}\right)^\rmT\bdelta_{\bar\mcS^c}+
\frac1{n_0}\sum_{i=1}^{n_0}\\
&&\int_{\mathbf{Z}_{0i}^\rmT\left(\left(\bdelta_{\bar{\mcS}}^\rmT,{\bf 0}^\rmT\right)^\rmT-\wh\bdelta_{\mcT_h,\rm ora}\right)}^{\mathbf{Z}_{0i}^\rmT(\bdelta-\wh\bdelta_{\mcT_h,\rm ora})}\left\{I\left(Y_{0i}-\bZ_{0i}^\rmT\left(\wh\bbeta_{\mcT_h}+\wh\bdelta_{\mcT_h,\rm ora}\right)\le z\right)-I\left(Y_{0i}-\bZ_{0i}^\rmT\left(\wh\bbeta_{\mcT_h}+\wh\bdelta_{\mcT_h,\rm ora}\right)\le 0\right)\right\}\rmd z.
\end{eqnarray*}
Note that
\begin{eqnarray*}
&&\int_{\mathbf{Z}_{0i}^\rmT\left(\left(\bdelta_{\bar{\mcS}}^\rmT,{\bf 0}^\rmT\right)^\rmT-\wh\bdelta_{\mcT_h,\rm ora}\right)}^{\mathbf{Z}_{0i}^\rmT\left(\bdelta-\wh\bdelta_{\mcT_h,\rm ora}\right)}\left\{I\left(Y_{0i}-\bZ_{0i}^\rmT\left(\wh\bbeta_{\mcT_h}+\wh\bdelta_{\mcT_h,\rm ora}\right)\le z\right)-I\left(Y_{0i}-\bZ_{0i}^\rmT\left(\wh\bbeta_{\mcT_h}+\wh\bdelta_{\mcT_h,\rm ora}\right)\le 0\right)\right\}\rmd z\\
&=&\left(\bZ_{0i}^\rmT\left(\wh\bbeta_{\mcT_h}+\bdelta\right)-Y_{0i}\right)I\left(Y_{0i}-\bZ_{0i}^\rmT\left(\wh\bbeta_{\mcT_h}+\bdelta\right)\le0\right)
\\&&-\left(\bZ^{\rmT}_{0i}\wh\bbeta_{\mcT_h}+\bZ^{\rmT}_{0i,\bar{\mcS}}\bdelta_{\bar{\mcS}}-Y_{0i}\right)
I\left(Y_{0i}-\bZ^{\rmT}_{0i}\wh\bbeta_{\mcT_h}-\bZ^{\rmT}_{0i,\bar{\mcS}}\bdelta_{\bar{\mcS}}\le0\right)\\
&&-\bZ^{\rmT}_{0i,\bar{\mcS}^c}\bdelta_{\bar{\mcS}^c}I\left(Y_{0i}-\bZ_{0i}^\rmT\left(\wh\bbeta_{\mcT_h}+\wh\bdelta_{\mcT_h,\rm ora}\right)\le 0\right)\\
&=&\int^{\mathbf{Z}_{0i,\bar{\mcS}^c}^{\rmT}\bdelta_{\bar{\mcS}^c}}_0I\left(Y_{0i}-\bZ^{\rmT}_{0i}\wh\bbeta_{\mcT_h}-\bZ^{\rmT}_{0i,\bar{\mcS}}\bdelta_{\bar{\mcS}}\le z\right)\rmd z
-\bZ^{\rmT}_{0i,\bar{\mcS}^c}\bdelta_{\bar{\mcS}^c}I\left(Y_{0i}-\bZ_{0i}^\rmT\left(\wh\bbeta_{\mcT_h}+\wh\bdelta_{\mcT_h,\rm ora}\right)\le 0\right)\\
&=&\int^{\mathbf{Z}_{0i,\bar{\mcS}^c}^{\rmT}\bdelta_{\bar{\mcS}^c}}_0\left\{I\left(Y_{0i}-\bZ^{\rmT}_{0i}\wh\bbeta_{\mcT_h}-\bZ^{\rmT}_{0i,\bar{\mcS}}\bdelta_{\bar{\mcS}}\le z\right)-I\left(Y_{0i}-\bZ^{\rmT}_{0i}\wh\bbeta_{\mcT_h}-\bZ^{\rmT}_{0i,\bar{\mcS}}\bdelta_{\bar{\mcS}}\le 0\right)\right\}\rmd z\\&&
+\bZ^{\rmT}_{0i,\bar{\mcS}^c}\bdelta_{\bar{\mcS}^c}\left\{I\left(Y_{0i}-\bZ^{\rmT}_{0i}\wh\bbeta_{\mcT_h}-\bZ^{\rmT}_{0i,\bar{\mcS}}\bdelta_{\bar{\mcS}}\le 0\right)-I\left(Y_{0i}-\bZ_{0i}^\rmT\left(\wh\bbeta_{\mcT_h}+\wh\bdelta_{\mcT_h,\rm ora}\right)\le 0\right)\right\}.
\end{eqnarray*}
Since
$$ \int^{\mathbf{Z}_{0i,\bar{\mcS}^c}^{\rmT}\bdelta_{\bar{\mcS}^c}}_0\left\{I\left(Y_{0i}-\bZ^{\rmT}_{0i}\wh\bbeta_{\mcT_h}-\bZ^{\rmT}_{0i,\bar{\mcS}}\bdelta_{\bar{\mcS}}\le z\right)-I\left(Y_{0i}-\bZ^{\rmT}_{0i}\wh\bbeta_{\mcT_h}-\bZ^{\rmT}_{0i,\bar{\mcS}}\bdelta_{\bar{\mcS}}\le 0\right)\right\}\rmd z\ge 0, $$
Theorem 1.2 in 
\citet{zuijlen1978properties} implies that
\begin{eqnarray*}
&&P\Bigg(\frac1{n_0}\!\sum_{i=1}^{n_0}\!\!\frac1{\left\vert \bdelta_{\bar{\mcS}^c}\right\vert_1}\!\!\int^{\mathbf{Z}_{0i,\bar{\mcS}^c}^{\rmT}\bdelta_{\bar{\mcS}^c}}_0\!\!\!\!\left\{I\left(Y_{0i}\!\!-\!\bZ^{\rmT}_{0i}\wh\bbeta_{\mcT_h}\!-\!\bZ^{\rmT}_{0i,\bar{\mcS}}
\bdelta_{\bar{\mcS}}\!\le\! z\right)\!\!-\!I\left(Y_{0i}\!\!-\!\bZ^{\rmT}_{0i}\wh\bbeta_{\mcT_h}\!-\!\bZ^{\rmT}_{0i,\bar{\mcS}}\bdelta_{\bar{\mcS}}\!\le\! 0\right)\right\}\rmd z\!>\!t\Bigg)\\
&\le &P\left(\wh\bbeta_{\mcT_h,\mcS_{\mcT_h}^c}\neq {\bf 0}\right)+P\Bigg(\sup_{\bbeta: \bbeta_{\mcS_{\mcT_h}^c}={\bf 0}}\sup_{\bdelta:\vert \bdelta_{\bar{\mcS}^c}\vert_1\le r}\frac1{n_{0}}\sum_{i=1}^{n_0}\frac1{\vert \bdelta_{\bar{\mcS}^c}\vert_1}\left(\int_{0}^{B\left\vert \bdelta_{\bar{\mcS}^c}\right\vert_1}+\int_{0}^{-B\left\vert \bdelta_{\bar{\mcS}^c}\right\vert_1}\right)\\
&&
\Bigg\{I\left(F_{0,\bbeta+\left(\bdelta_{\bar{\mcS}}^\rmT,{\bf 0}^\rmT\right)}\left(Y_{0i}\!\!-\!\bZ^{\rmT}_{0i}\bbeta\!-\!\bZ^{\rmT}_{0i,\bar{\mcS}}
\bdelta_{\bar{\mcS}}\right)\le F_{0,\bbeta+\left(\bdelta_{\bar{\mcS}}^\rmT,{\bf 0}^\rmT\right)}(z)\right)\\
&&-I\left(F_{0,\bbeta+\left(\bdelta_{\bar{\mcS}}^\rmT,{\bf 0}^\rmT\right)}\left(Y_{0i}\!\!-\!\bZ^{\rmT}_{0i}\bbeta\!-\!\bZ^{\rmT}_{0i,\bar{\mcS}}
\bdelta_{\bar{\mcS}}\right)\le F_{0,\bbeta+\left(\bdelta_{\bar{\mcS}}^\rmT,{\bf 0}^\rmT\right)}(0)\right)\Bigg\}\rmd z>t\Bigg)\\
&=&12p^{-1}+2p^{-2}+P\Bigg(\sup_{\bbeta: \bbeta_{\mcS_{\mcT_h}^c}={\bf 0}}\sup_{\bdelta:\vert \bdelta_{\bar{\mcS}^c}\vert_1\le r}\frac1{n_{0}}\sum_{i=1}^{n_0}\frac1{\vert \bdelta_{\bar{\mcS}^c}\vert_1}\left(\int_{0}^{B\vert \bdelta_{\bar{\mcS}^c}\vert_1}+\int_{0}^{-B\vert \bdelta_{\bar{\mcS}^c}\vert_1}\right)\\
&&\Bigg\{\wt F_{0,\bbeta+\left(\bdelta_{\bar{\mcS}}^\rmT,{\bf 0}^\rmT\right)}\left( F_{0,\bbeta+\left(\bdelta_{\bar{\mcS}}^\rmT,{\bf 0}^\rmT\right)}(z)\right)-\wt F_{0,\bbeta+\left(\bdelta_{\bar{\mcS}}^\rmT,{\bf 0}^\rmT\right)}\left( F_{0,\bbeta+\left(\bdelta_{\bar{\mcS}}^\rmT,{\bf 0}^\rmT\right)}(0)\right)\Bigg\}\rmd z>t\Bigg)\\
&=&P\Bigg(\sup_{\bbeta: \bbeta_{\mcS_{\mcT_h}^c}={\bf 0}}\sup_{\bdelta:\vert \bdelta_{\bar{\mcS}^c}\vert_1\le r}\!\frac1{\vert \bdelta_{\bar{\mcS}^c}\vert_1}\!\left(\int_{0}^{B\vert \bdelta_{\bar{\mcS}^c}\vert_1}+\int_{0}^{-B\vert \bdelta_{\bar{\mcS}^c}\vert_1}\right)\!\!\Bigg\{\left\{F_{0,\bbeta+\left(\bdelta_{\bar{\mcS}}^\rmT,{\bf 0}^\rmT\right)}(z)-F_{0,\bbeta+\left(\bdelta_{\bar{\mcS}}^\rmT,{\bf 0}^\rmT\right)}(0)\right\}\\
&&+\wh\Psi_{0,\bbeta+\left(\bdelta_{\bar{\mcS}}^\rmT,{\bf 0}^\rmT\right)}\left(F_{0,\bbeta+\left(\bdelta_{\bar{\mcS}}^\rmT,{\bf 0}^\rmT\right)}(z)\right)\!\!-\!\!\wh\Psi_{0,\bbeta+\left(\bdelta_{\bar{\mcS}}^\rmT,{\bf 0}^\rmT\right)}\left(F_{0,\bbeta+\left(\bdelta_{\bar{\mcS}}^\rmT,{\bf 0}^\rmT\right)}(0)\right)\Bigg\}\rmd z>n_0^{1/2}t\Bigg)+12p^{-1}+2p^{-2}\\
&\le&P\Bigg(m_0B^2r+\sup_{\bbeta: \bbeta_{\mcS_{\mcT_h}^c}={\bf 0}}\sup_{\bdelta:\vert \bdelta_{\bar{\mcS}^c}\vert_1\le r}\!\frac1{\vert \bdelta_{\bar{\mcS}^c}\vert_1}\!\left(\int_{0}^{B\vert \bdelta_{\bar{\mcS}^c}\vert_1}+\int_{0}^{-B\vert \bdelta_{\bar{\mcS}^c}\vert_1}\right)\\
&&\left\{\wh\Psi_{0,\bbeta+\left(\bdelta_{\bar{\mcS}}^\rmT,{\bf 0}^\rmT\right)}\left(F_{0,\bbeta+\left(\bdelta_{\bar{\mcS}}^\rmT,{\bf 0}^\rmT\right)}(z)\right)\!\!-\!\!\wh\Psi_{0,\bbeta+\left(\bdelta_{\bar{\mcS}}^\rmT,{\bf 0}^\rmT\right)}\left(F_{0,\bbeta+\left(\bdelta_{\bar{\mcS}}^\rmT,{\bf 0}^\rmT\right)}(0)\right)\right\}\rmd z>n_0^{1/2}t\Bigg)+12p^{-1}+2p^{-2}\\
&\le & P\Bigg(\sup_{\bbeta: \bbeta_{\mcS_{\mcT_h}^c}={\bf 0}}\sup_{\bdelta:\vert \bdelta_{\bar{\mcS}^c}\vert_1\le r}\sup_{z\in[-Br,Br]}\\
&&\left\vert\wh\Psi_{0,\bbeta+\left(\bdelta_{\bar{\mcS}}^\rmT,{\bf 0}^\rmT\right)}\left(F_{0,\bbeta+\left(\bdelta_{\bar{\mcS}}^\rmT,{\bf 0}^\rmT\right)}(z)\right)\!\!-\!\!\wh\Psi_{0,\bbeta+\left(\bdelta_{\bar{\mcS}}^\rmT,{\bf 0}^\rmT\right)}\left(F_{0,\bbeta+\left(\bdelta_{\bar{\mcS}}^\rmT,{\bf 0}^\rmT\right)}(0)\right)\right\vert>\frac{n_0^{1/2}t}{2B}-\frac{m_0Br}{2}\Bigg)\\&&+12p^{-1}+2p^{-2}\\
&\le & \frac{8MB^3r}{(n_0^{1/2}t-m_0B^2r)^2}+12p^{-1}+2p^{-2}
\end{eqnarray*}
for some positive constant $M$.
Taking $t=n_0^{-1/2}m_0B^2r^{1/2-\eta}$ for a sufficiently small positive constant $\eta$,
we conclude that
\begin{eqnarray*}
&&P\Bigg(\frac1{n_0}\sum_{i=1}^{n_0}\int^{\mathbf{Z}_{0i,\bar{\mcS}^c}^{\rmT}\bdelta_{\bar{\mcS}^c}}_0
\left\{I\left(Y_{0i}-\bZ^{\rmT}_{0i}\wh\bbeta_{\mcT_h}-\bZ^{\rmT}_{0i,\bar{\mcS}}
\bdelta_{\bar{\mcS}}\le z\right)-I\left(Y_{0i}-\bZ^{\rmT}_{0i}\wh\bbeta_{\mcT_h}-\bZ^{\rmT}_{0i,\bar{\mcS}}
\bdelta_{\bar{\mcS}}\le 0\right)\right\}\rmd z\\
&&>\frac{m_0B^2r^{1/2-\eta}}{\sqrt{n_0}}\vert \bdelta_{\bar{\mcS}^c}\vert_1\Bigg)\\
&\le& \frac{8Mr^{2\eta}}{m_0^2B(1-r^{1/2+\eta})^2}.
\end{eqnarray*}
On the other hand, it follows from the Markov inequality and Cauchy--Schwarz inequality that
\begin{eqnarray*}
&&P\Bigg(\left\vert\frac1{n_0}\sum_{i=1}^{n_0}\bZ^{\rmT}_{0i,\bar{\mcS}^c}\bdelta_{\bar{\mcS}^c}
\left\{I\left(Y_{0i}-\bZ^{\rmT}_{0i}\wh\bbeta_{\mcT_h}-\bZ^{\rmT}_{0i,\bar{\mcS}}\bdelta_{\bar{\mcS}}\le 0\right)-I\left(Y_{0i}-\bZ_{0i}^\rmT\left(\wh\bbeta_{\mcT_h}+\wh\bdelta_{\mcT_h,\rm ora}\right)\le 0\right)\right\}\right\vert>t\Bigg)\\
&\le &E\left\{\left(\frac1{n_0}\sum_{i=1}^{n_k}\bZ^{\rmT}_{0i,\bar{\mcS}^c}\bdelta_{\bar{\mcS}^c}
\left\{I\left(Y_{0i}-\bZ^{\rmT}_{0i}\wh\bbeta_{\mcT_h}-\bZ^{\rmT}_{0i,\bar{\mcS}}\bdelta_{\bar{\mcS}}\le 0\right)-I\left(Y_{0i}-\bZ_{0i}^\rmT\left(\wh\bbeta_{\mcT_h}+\wh\bdelta_{\mcT_h,\rm ora}\right)\le 0\right)\right\}\right)^2\right\}/t^2\\
&\le &E\Bigg(\left\{\frac1{n_0}\sum_{i=1}^{n_0}\left(\bZ^{\rmT}_{0i,\bar{\mcS}^c}\bdelta_{\bar{\mcS}^c}\right)^2\right\}\\
&&\times\frac1{n_0}\sum_{i=1}^{n_0}\left\vert I\left(Y_{0i}-\bZ^{\rmT}_{0i}\wh\bbeta_{\mcT_h}-\bZ^{\rmT}_{0i,\bar{\mcS}}\bdelta_{\bar{\mcS}}\le 0\right)-I\left(Y_{0i}-\bZ_{0i}^\rmT\left(\wh\bbeta_{\mcT_h}+\wh\bdelta_{\mcT_h,\rm ora}\right)\le 0\right)\right\vert\Bigg)/t^2.
\end{eqnarray*}
Since
\begin{eqnarray*}
&&I\left(Y_{0i}\!-\!\bZ^{\rmT}_{0i}\wh\bbeta_{\mcT_h}\!-\!\bZ^{\rmT}_{0i,\bar{\mcS}}\bdelta_{\bar{\mcS}}\le -Br\right)\le I\left(Y_{0i}\!-\!\bZ^{\rmT}_{0i}\wh\bbeta_{\mcT_h}\!-\!\bZ^{\rmT}_{0i,\bar{\mcS}}\bdelta_{\bar{\mcS}}\le \bZ^{\rmT}_{0i}\left(\wh\bdelta_{\mcT_h,\rm ora}-\left(\bdelta_{\bar{\mcS}}^\rmT,{\bf 0}^\rmT\right)\right)\right)\\
&=&
I\left(Y_{0i}\!-\!\bZ_{0i}^\rmT\left(\wh\bbeta_{\mcT_h}\!+\!\wh\bdelta_{\mcT_h,\rm ora}\right)\le 0\right)\le I\left(Y_{0i}\!-\!\bZ^{\rmT}_{0i}\wh\bbeta_{\mcT_h}\!-\!\bZ^{\rmT}_{0i,\bar{\mcS}}\bdelta_{\bar{\mcS}}\le Br\right)
\end{eqnarray*}
and
\begin{eqnarray*}
&&I\left(Y_{0i}\!-\!\bZ^{\rmT}_{0i}\wh\bbeta_{\mcT_h}\!-\!\bZ^{\rmT}_{0i,\bar{\mcS}}\bdelta_{\bar{\mcS}}\le -Br\right)\le I\left(Y_{0i}\!-\!\bZ^{\rmT}_{0i}\wh\bbeta_{\mcT_h}\!-\!\bZ^{\rmT}_{0i,\bar{\mcS}}\bdelta_{\bar{\mcS}}\le 0\right)
\\&\le&I\left(Y_{0i}\!-\!\bZ^{\rmT}_{0i}\wh\bbeta_{\mcT_h}\!-\!\bZ^{\rmT}_{0i,\bar{\mcS}}\bdelta_{\bar{\mcS}}\le Br\right),
\end{eqnarray*}
then
\begin{eqnarray*}
&&\left\vert I\left(Y_{0i}-\bZ^{\rmT}_{0i}\wh\bbeta_{\mcT_h}-\bZ^{\rmT}_{0i,\bar{\mcS}}\bdelta_{\bar{\mcS}}\le 0\right)-I\left(Y_{0i}-\bZ_{0i}^\rmT\left(\wh\bbeta_{\mcT_h}+\wh\bdelta_{\mcT_h,\rm ora}\right)\le 0\right)\right\vert\\
&\le& I\left(Y_{0i}-\bZ^{\rmT}_{0i}\wh\bbeta_{\mcT_h}-\bZ^{\rmT}_{0i,\bar{\mcS}}\bdelta_{\bar{\mcS}}\le Br\right)-I\left(Y_{0i}-\bZ^{\rmT}_{0i}\wh\bbeta_{\mcT_h}-\bZ^{\rmT}_{0i,\bar{\mcS}}\bdelta_{\bar{\mcS}}\le -Br\right)
\end{eqnarray*}
for any $i=1,\ldots,n_k$ and $\ k=1,\ldots,K.$
Consequently, we have
\begin{eqnarray*}
&&P\Bigg(\left\vert\frac1{n_0}\sum_{i=1}^{n_0}\bZ^{\rmT}_{0i,\bar{\mcS}^c}\bdelta_{\bar{\mcS}^c}
\left\{I\left(Y_{0i}\!-\!\bZ^{\rmT}_{0i}\wh\bbeta_{\mcT_h}-\bZ^{\rmT}_{0i,\bar{\mcS}}\bdelta_{\bar{\mcS}}\le 0\right)-I\left(Y_{0i}-\bZ_{0i}^\rmT\left(\wh\bbeta_{\mcT_h}+\wh\bdelta_{\mcT_h,\rm ora}\right)\le 0\right)\right\}\right\vert>t\Bigg)\\
&\le &\frac{B^2\vert\bdelta_{\bar{\mcS}^c}\vert_1^2}{t^2}E\left[\frac1{n_0}
\sum_{i=1}^{n_k}\left\{I\left(Y_{0i}-\bZ^{\rmT}_{0i}\wh\bbeta_{\mcT_h}-\bZ^{\rmT}_{0i,\bar{\mcS}}\bdelta_{\bar{\mcS}}\le Br\right)-I\left(Y_{0i}-\bZ^{\rmT}_{0i}\wh\bbeta_{\mcT_h}-\bZ^{\rmT}_{0i,\bar{\mcS}}\bdelta_{\bar{\mcS}}\le -Br\right)\right\}\right]\\
&\le &\frac{B^2\vert\bdelta_{\bar{\mcS}^c}\vert_1^2}{t^2}\left(F_{0,\wh\bbeta_{\mcT_h}+\left(\bdelta_{\bar{\mcS}}^\rmT,{\bf 0}^\rmT\right)^\rmT}(Br)-F_{0,\wh\bbeta_{\mcT_h}+\left(\bdelta_{\bar{\mcS}}^\rmT,{\bf 0}^\rmT\right)^\rmT}(-Br)\right)\\
&\le &2m_0L^3r\vert\bdelta_{\bar{\mcS}^c}\vert_1^2/t^2.
\end{eqnarray*}
Taking $t=m_0B^2r^{1/2-\eta}\left\vert\bdelta_{\bar{\mcS}^c}\right\vert_1$ for a sufficiently small positive constant $\eta$,
we conclude that
\begin{eqnarray*}
&&P\Bigg(\left\vert\frac1{n_0}\sum_{i=1}^{n_0}\bZ^{\rmT}_{0i,\bar{\mcS}^c}\bdelta_{\bar{\mcS}^c}
\left\{I\left(Y_{0i}-\bZ^{\rmT}_{0i}\wh\bbeta_{\mcT_h}-\bZ^{\rmT}_{0i,\bar{\mcS}}
\bdelta_{\bar{\mcS}}\le 0\right)-I\left(Y_{0i}-\bZ_{0i}^\rmT\left(\wh\bbeta_{\mcT_h}+\wh\bdelta_{\mcT_h,\rm ora}\right)\le 0\right)\right\}\right\vert\\
&&>m_0B^2r^{1/2-\eta}\left\vert\bdelta_{\bar{\mcS}^c}\right\vert_1\Bigg)\\
&\le &\frac{2r^{2\eta}}{m_0B}.
\end{eqnarray*}
Besides,
Lemma \ref{lemmaA8} implies that (\ref{mat2_deb}) holds with probability at least $1-24p^{-1}-4p^{-2}$.
As a result, we can always find a sufficient small $r>0$ such that
\begin{eqnarray*}
&&\left\{L\left(\wh\bbeta_{\mcT_h}+\bdelta;\{0\}\right)+\lambda_{\bdelta}\vert\bdelta\vert_1\right\}-\left\{L\left(\wh\bbeta_{\mcT_h}+\left(\bdelta_{\bar{\mcS}}^\rmT,{\bf 0}^\rmT\right)^\rmT;\{0\}\right)+\lambda_{\bdelta}\vert\bdelta_{\bar{\mcS}}\vert_1\right\}\\
&>&\sum_{j \in\bar{\mcS}^c}\left( S_j\left(\wh\bbeta_{\mcT_h}+\wh \bdelta_{\rm ora};\{0\}\right)+\left(\lambda_{\bdelta}-m_0B^2r^{1/2-\eta}-\frac{m_0B^2r^{1/2-\eta}}{\sqrt{ n_0}}\right){\rm sign}(\delta_j)\right)\delta_j
\\&>&0,
\end{eqnarray*}
which holds with probability at least $1-24p^{-1}-4p^{-2}-8Mr^{\eta}/(m_0^2B(1-r^{1/2+\eta})^2)-2r^{2\eta}/(m_0B)$.
Furthermore, for any $\bdelta$ with $\bdelta_{\bar{\mcS}^c}\neq {\bf 0}$, there exists a vector
$\bdelta^\ddag=s\bdelta+(1-s)\wh\bdelta_{\mcT_h,\rm ora}$ with $\bdelta^\ddag\in \mathcal{B}\left(\wh\bdelta_{\mcT_h,\rm ora},r\right)$.
The convexity of the object function in (\ref{op_debias}) entails that
\begin{eqnarray*}
&&s\left(L\left(\wh\bbeta_{\mcT_h}+\bdelta;\{0\}\right)+\lambda_{\bdelta}\vert\bdelta\vert_1\right)+(1-s)\left(L\left(\wh\bbeta_{\mcT_h}+\wh\bdelta_{\mcT_h,\rm ora};\{0\}\right)+\lambda_{\bdelta}\vert\wh\bdelta_{\mcT_h,\rm ora}\vert_1\right)
\\&\ge& L\left(\wh\bbeta_{\mcT_h}+\bdelta^\ddag;\}0\}\right)+\lambda_{\bdelta}\vert\bdelta^\ddag\vert_1.
\end{eqnarray*}
The above inequality implies that $\bdelta$ is not the solution of (\ref{op_debias}) since
\begin{eqnarray*}
&&L\left(\wh\bbeta_{\mcT_h}+\bdelta;\{0\}\right)+\lambda_{\bdelta}\vert\bdelta\vert_1\\&\ge&\frac1s\left\{ L\left(\wh\bbeta_{\mcT_h}+\bdelta^\ddag;\{0\}\right)+\lambda_{\bdelta}\vert\bdelta^\ddag\vert_1-(1-s)\left(L\left(\wh\bbeta_{\mcT_h}+\wh\bdelta_{\mcT_h,\rm ora};\{0\}\right)+\lambda_{\bdelta}\vert\wh\bdelta_{\mcT_h,\rm ora}\vert_1\right)\right\}
\\&\ge& L\left(\wh\bbeta_{\mcT_h}+\bdelta^\ddag;\{0\}\right)+\lambda_{\bdelta}\vert\bdelta^\ddag\vert_1>
L\left(\wh\bbeta_{\mcT_h}+\wh\bdelta_{\mcT_h,\rm ora};\{0\}\right)+\lambda_{\bdelta}\vert\wh\bdelta_{\mcT_h,\rm ora}\vert_1,
\end{eqnarray*}
which concludes the proof, implying that all the solutions of (\ref{op_debias}) satisfies
$\wh\bdelta_{\mcT_h,\bar{\mcS}^c}= {\bf 0}$ with probability at least $1-24p^{-1}-4p^{-2}$.
\end{proof}

{\bf Proof of Theorem \ref{theorem2.4}}
\begin{proof}
It follows directly from  Lemmas \ref{lemmaA6}--\ref{lemmaA9}
and thus is omitted.
\end{proof}

\subsection{Proof of Theorem \ref{theorem2.5}}
\begin{proof}
Note that
\begin{eqnarray*}
P\left(\mcS_0\subset\wh\mcS\right)=P\left(\wh\beta_{0j}\neq0, j\in \mcS_0\right).
\end{eqnarray*}
Under assumptions \ref{C1}--\ref{C3} and \ref{C7}, and using
 Theorem \ref{theorem2.3}, for any $j\in\mcS_0$, there exists a constant $c>0$ such that
\begin{eqnarray*}
&&\left\vert\wh\beta_{0j}\right\vert=\left\vert\beta_{0j}+\wh\beta_{0j}-\beta_{0j}\right\vert\ge \vert\beta_{0j}\vert-\left\vert\wh\beta_{0j}-\beta_{0j}\right\vert\ge
\vert\beta_{0j}\vert-\left\vert\wh\bbeta_0-\bbeta_0\right\vert_2\\
&\ge& \min_{j\in\mcS_0}\vert\beta_{0j}\vert-c\sqrt{s_0}\left(\frac{\log p}{n_0}\right)^{1/4}\max\left\{h^2,\frac{\log p}{n_{\mcT_h}+n_0}\right\}^{1/4}>0
\end{eqnarray*}
with probability at least $1-2p^{-1}-2p^{-2}$.
Therefore, it yields that
\begin{eqnarray*}
P\left(\mcS_0\subset\wh\mcS\right)&=&P\left(\wh\beta_{0j}\neq0, j\in \mcS_0\right)
=P\left(\min_{j\in\mcS_0}\left\vert\wh\beta_{0j}\right\vert>0\right)\\
&\ge&P\left(\min_{j\in\mcS_0}\vert\beta_{0j}\vert-\left\vert\wh\bbeta_0-\bbeta_0\right\vert_2>0\right)\\
&\ge&P\left(\min_{j\in\mcS_0}\vert\beta_{0j}\vert-c\sqrt{s_0}\max\{h^2,\log p/(n_{\mcT_h}+n_0)\}^{1/4}>0,\right.\\&&\left.\left\vert\wh\bbeta_0-\bbeta_0\right\vert_2\lesssim\sqrt{s_0}\max\{h^2,\log p/(n_{\mcT_h}+n_0)\}^{1/4}\right)\\
&\ge&1-2p^{-1}-2p^{-2}.
\end{eqnarray*}
\end{proof}

\section{Asymptotic Distribution}
\label{appendixB}
\setcounter{equation}{0}
\def\theequation{B.\arabic{equation}}

\subsection{Proof of Theorem \ref{theorem3.1}}

\begin{lemma}
\label{lemmaB1}
Under conditions in Theorem \ref{theorem2.4} and assumption \ref{C4},
it holds that
\begin{eqnarray*}
&&\left\Vert\frac1{n_k}\sum_{i=1}^{n_k}\wh f_{ki}\left(\wh\bbeta_0\right)\bZ_{ki}\bZ_{ki}^{\rmT}-\frac1{n_k}\sum_{i=1}^{n_k}\wh f_{ki}(\bbeta_0)\bZ_{ki}\bZ_{ki}^{\rmT}\right\Vert_{\max}\\
&=&O_P\left(\sum_{k\in\{0\}\cup\mcT_h}\frac{\pi_k}{b_k}\Bigg\{\left(s_0\sqrt{\frac{\log p}{n_{\mcT_h}+n_0}}+h\right)+\sqrt{\frac{s_0+s'}{n_k}}
+\sqrt{\frac{\log p}{n_k}}\Bigg\}\right).
\end{eqnarray*}
\end{lemma}

\begin{proof}
We conduct our demonstration on the event that $\wh\bbeta_{0,\bar{\mcS}^c}={\bf 0}$,
which occurs with probability at least $1-24p^{-1}-4p^{-2}$ following from Lemma \ref{lemmaA8}.
Note that
\begin{eqnarray*}
2b_k\left(\wh f_{ki}\left(\wh\bbeta_0\right)-\wh f_{ki}(\bbeta_0)\right)
=I\left(-b_k\le Y_{ki}-\bZ_{ki}^{\rmT}\wh\bbeta_0\le b_k\right)-I\left(-b_k\le Y_{ki}-\bZ_{ki}^{\rmT}\bbeta_0\le b_k\right).
\end{eqnarray*}
Denote $\mathcal{G}_{k,jl}=\left\{g_{\bbeta,k,jl}(Y,\bZ)\mid \bbeta_{\bar{\mcS}^c}={\bf 0}\right\}$,
where
$$g_{\bbeta,k,jl}(Y,\bZ)=\frac{Z_jZ_l}{2b_k}\left\{I\left(-b_k\le Y-\bZ^{\rmT}\bbeta\le b_k\right)-I\left(-b_k\le Y-\bZ^{\rmT}\bbeta_0\le b_k\right)\right\}.$$
Since the Vapnik--Chervonenkis dimension of $\mathcal{G}_{k,jl}$ is less than $s_0+s'$, then Lemma 19.15 in 
\citet{van1998asymptotic} entails that
$$ \sup_Q N\left(\epsilon B^2,\mathcal{G}_{k,jl},L_2(Q)\right)\le M_0(s_0+s')(16e)^{s_0+s'}\epsilon^{1-2s_0-2s'} $$
for some positive constant $M_0>0$.
Thus, Lemma 19.35 in 
\citet{van1998asymptotic} implies that
\begin{eqnarray*}
&&E\left\{\sup_{\bbeta: \bbeta_{\bar{\mcS}^c}={\bf 0}}
\frac1{n_0+n_{\mcT_h}}\sum_{k\in\{0\}\cup\mcT_h}\sum_{i=1}^{n_k}
\left[g_{\bbeta,k,jl}\left(Y_{ki},\bZ_{ki}\right)-E\left\{g_{\bbeta,k,jl}\left(Y_{k},\bZ_{k}\right)\right\}\right]\right\}\\
&\le&\sum_{k\in\{0\}\cup\mcT_h}\pi_kE\left\{\sup_{\bbeta: \bbeta_{\bar{\mcS}^c}
={\bf 0}}
\frac1{n_k}\sum_{i=1}^{n_k}\left[g_{\bbeta,k,jl}\left(Y_{ki},\bZ_{ki}\right)-E\left\{g_{\bbeta,k,jl}\left(Y_{k},\bZ_{k}\right)\right\}\right]\right\}\\
&\le&\sum_{k\in\{0\}\cup\mcT_h}\frac{\pi_kB^2}{2b_k\sqrt{n_k}} J(1,\mathcal{G}_{k,jl},L_2)
\\&=& \sum_{k\in\{0\}\cup\mcT_h}\frac{\pi_kB^2}{2b_k\sqrt{n_k}}\int^1_0 \sqrt{\log\sup_Q N\left(\epsilon B^2,\mathcal{G}_{k,jl},L_2(Q)\right)}\rmd \epsilon\\
&\le&\sum_{k\in\{0\}\cup\mcT_h}\frac{\pi_kD_0B^2}{b_k}\sqrt{\frac{s_0+s'}{n_k}}
\end{eqnarray*}
for some constant $D_0>0$.
Employing Theorem 3.26 in 
\citet{wainwright2019high}, it holds that
\begin{eqnarray*}
&&P\Bigg(\Bigg\Vert\frac1{n_0+n_{\mcT_h}}\sum_{k\in\{0\}\cup\mcT_h}\sum_{i=1}^{n_k}\left\{\wh f_{ki}\left(\wh\bbeta_0\right)\bZ_{ki}\bZ_{ki}^{\rmT}-\wh f_{ki}(\bbeta_0)\bZ_{ki}\bZ_{ki}^{\rmT}\right\}\\
&&-\sum_{k\in\{0\}\cup\mcT_h}\!\!\pi_kE\left\{\wh f_{ki}(\bbeta)\bZ_{ki}\bZ_{ki}^{\rmT}-\wh f_{ki}(\bbeta_0)\bZ_{ki}\bZ_{ki}^{\rmT}\right\}\Big\vert_{\bbeta=\wh\bbeta_0}\Bigg\Vert_{\max}\!\ge\! \sum_{k\in\{0\}\cup\mcT_h}\pi_k\left(\frac{D_0B^2}{b_k}\sqrt{\frac{s_0+s'}{n_k}}+t_k\right)\Bigg)\\&\le&
\sum_{j=1}^p\sum_{l=1}^pP\Bigg(\Bigg\vert\frac1{n_0+n_{\mcT_h}}\sup_{\bbeta: \bbeta_{\bar{\mcS}^c}={\bf 0}}\sum_{k\in\{0\}\cup\mcT_h}\sum_{i=1}^{n_k}
\left[g_{\bbeta,k,jl}\left(Y_{ki},\bZ_{ki}\right)-E\left\{g_{\bbeta,k,jl}\left(Y_{k},\bZ_{k}\right)\right\}\right]\Bigg\vert\\
&&\ge \sum_{k\in\{0\}\cup\mcT_h}\pi_k\left(\frac{D_0B^2}{b_k}\sqrt{\frac{s_0+s'}{n_k}}+t_k\right)\Bigg)+P\left(\wh\bbeta_{0,\bar{\mcS}^c}\neq{\bf 0}\right)\\
&\le&\sum_{j=1}^p\!\sum_{l=1}^p\!\sum_{k\in\{0\}\cup\mcT_h}\!\!\!P\Bigg(\Bigg\vert\frac1{n_k}\!\sup_{\bbeta: \bbeta_{\bar{\mcS}^c}={\bf 0}}\!\sum_{i=1}^{n_k}\!
\left[g_{\bbeta,k,jl}\left(Y_{ki},\bZ_{ki}\right)\!-\!E\left\{g_{\bbeta,k,jl}\left(Y_{k},\bZ_{k}\right)\right\}\right]\Bigg\vert\!\ge\! \frac{D_0B^2}{b_k}\sqrt{\frac{s_0+s'}{n_k}}\!+\!t_k\Bigg)\\
&&+P\left(\wh\bbeta_{0,\bar{\mcS}^c}\neq{\bf 0}\right)\\
&\le&2Kp^2\exp\left(-\frac{n_kb_k^2t_k^2}{4B^4}\right)+24p^{-1}+4p^{-2}.
\end{eqnarray*}
By taking $$t_k=\frac{4B^2}{b_k}\sqrt{\frac{\log p}{n_k}},$$
we obtain with probability greater than $1-24p^{-1}-(4+2K)p^{-2}$ that
\begin{eqnarray*}
&&\Bigg\Vert\frac1{n_0+n_{\mcT_h}}\!\!\sum_{k\in\{0\}\cup\mcT_h}\!\!\sum_{i=1}^{n_k}\!\!\left\{\left(\wh f_{ki}\left(\wh\bbeta_0\right)\!-\!\wh f_{ki}(\bbeta_0)\right)\bZ_{ki}\bZ_{ki}^{\rmT}\!-\!E\left\{\left(\wh f_{ki}(\bbeta)\!-\!\wh f_{ki}(\bbeta_0)\right)\bZ_{ki}\bZ_{ki}^{\rmT}\right\}\Big\vert_{\bbeta=\wh\bbeta_0}\right\}\Bigg\Vert_{\max}\\
&<& \sum_{k\in\{0\}\cup\mcT_h}\pi_k\left(\frac{D_0B^2}{b_k}\sqrt{\frac{s_0+s'}{n_k}}+\frac{4B^2}{b_k}\sqrt{\frac{\log p}{n_k}}\right).
\end{eqnarray*}
Since
\begin{eqnarray*}
\left\vert E\left(\wh f_{ki}(\bbeta)\!-\!\wh f_{ki}(\bbeta_0)\right)\right\vert&=&\Bigg\vert\frac1{2b_k}E\left\{F_k\left(\bZ_{k}^{\rmT}\bbeta+b_k\mid \bZ_{k}\right)
-F_k\left(\bZ_{k}^{\rmT}\bbeta-b_k\mid \bZ_{k}\right)\right\}\\
&&-\frac1{2b_k}E\left\{F_k\left(\bZ_{k}^{\rmT}\bbeta_0+b_k\mid \bZ_{k}\right)
-F_k\left(\bZ_{k}^{\rmT}\bbeta_0-b_k\mid \bZ_{k}\right)\right\} \Bigg\vert\\
&\le &\frac{m_0B}{b_k}\vert\bbeta-\bbeta_0\vert_1,
\end{eqnarray*}
by combining Theorem \ref{theorem2.2}, it implies that for come constant $c>0$,
\begin{eqnarray*}
&&\Bigg\Vert\frac1{n_0+n_{\mcT_h}}\sum_{k\in\{0\}\cup\mcT_h}\sum_{i=1}^{n_k}\left\{\left(\wh f_{ki}\left(\wh\bbeta_0\right)-\wh f_{ki}(\bbeta_0)\right)\bZ_{ki}\bZ_{ki}^{\rmT}\right\}\Bigg\Vert_{\max}\\
&<& \sum_{k\in\{0\}\cup\mcT_h}\pi_k\Bigg\{\frac{cm_0B^3}{b_k}\left(s_0\sqrt{\frac{\log p}{n_{\mcT_h}+n_0}}+h\right)+\frac{D_0B^2}{b_k}\sqrt{\frac{s_0+s'}{n_k}}+\frac{4B^2}{b_k}\sqrt{\frac{\log p}{n_k}}\Bigg\}
\end{eqnarray*}
with probability at least $1-26p^{-1}-(6+2K)p^{-2}$.
\end{proof}

{\bf Proof of Theorem \ref{theorem3.1}}
\begin{proof}
We derive an upper bound of $\left\Vert\wh \bH_{\mcT_h}-\bH_{\mcT_h}(\bbeta_0)\right\Vert_{\max}$ as follows. Specifically,
\begin{eqnarray}
\label{H_decom}
&&\left\Vert\wh \bH_{\mcT_h}-\bH_{\mcT_h}(\bbeta_0)\right\Vert_{\max}\nonumber\\
&=&\left\Vert\frac1{n_{\mcT_h}+n_0}\sum_{k\in \{0\}\cup\mcT_h}\sum_{i=1}^{n_k}\left\{\wh f_{ki}\left(\wh\bbeta_0\right)\bZ_{ki}\bZ_{ki}^{\rmT}-E\left(f_k(\bZ^{\rmT}_{ki}\bbeta_0\mid \bZ_{ki})\bZ_{ki}\bZ_{ki}^{\rmT}\right)\right\}\right\Vert_{\max}\nonumber\\
&\le &\left\Vert\frac1{n_{\mcT_h}+n_0}\sum_{k\in \{0\}\cup\mcT_h}\sum_{i=1}^{n_k}\wh f_{ki}\left(\wh\bbeta_0\right)\bZ_{ki}\bZ_{ki}^{\rmT}-\frac1{n_{\mcT_h}+n_0}\sum_{k\in \{0\}\cup\mcT_h}\sum_{i=1}^{n_k}\wh f_{ki}(\bbeta_0)\bZ_{ki}\bZ_{ki}^{\rmT}\right\Vert_{\max}\nonumber\\
&&+\left\Vert\frac1{n_{\mcT_h}+n_0}\sum_{k\in \{0\}\cup\mcT_h}\sum_{i=1}^{n_k}\left\{\wh f_{ki}(\bbeta_0)\bZ_{ki}\bZ_{ki}^{\rmT}-E\left(\wh f_{ki}(\bbeta_0)\bZ_{ki}\bZ_{ki}^{\rmT}\right)\right\}\right\Vert_{\max}\nonumber\\
&&+\left\Vert \frac1{n_{\mcT_h}+n_0}\sum_{k\in \{0\}\cup\mcT_h}\sum_{i=1}^{n_k}\left\{E\left(\wh f_{ki}(\bbeta_0)\bZ_{ki}\bZ_{ki}^{\rmT}\right)-E\left(f_k(\bZ^{\rmT}_{ki}\bbeta_0\mid \bZ_{ki})\bZ_{ki}\bZ_{ki}^{\rmT}\right)\right\}\right\Vert_{\max}.
\end{eqnarray}
Combining Lemma \ref{lemmaB1}, it is sufficient to dominate the last two terms of (\ref{H_decom}) to show Theorem \ref{theorem3.1}.
It follows from the Hoeffding inequality that
\begin{eqnarray*}
&&P\left(\left\Vert\frac1{n_{\mcT_h}+n_0}\sum_{k\in \{0\}\cup\mcT_h}\sum_{i=1}^{n_k}\left\{\wh f_{ki}(\bbeta_0)\bZ_{ki}\bZ_{ki}^{\rmT}-E\left(\wh f_{ki}(\bbeta_0)\bZ_{ki}\bZ_{ki}^{\rmT}\right)\right\}\right\Vert_{\max}>\sum_{k\in \{0\}\cup\mcT_h}\pi_kt_k\right)\\
&\le &\sum_{k\in \{0\}\cup\mcT_h}P\left(\left\Vert\frac1{n_k}\sum_{i=1}^{n_k}\left\{\wh f_{ki}(\bbeta_0)\bZ_{ki}\bZ_{ki}^{\rmT}-E\left(\wh f_{ki}(\bbeta_0)\bZ_{ki}\bZ_{ki}^{\rmT}\right)\right\}\right\Vert_{\max}>t_k\right)\\
&\le &\sum_{m=1}^p\sum_{j=1}^p\sum_{k\in \{0\}\cup\mcT_h}P\left(\left\Vert\frac1{n_k}\sum_{i=1}^{n_k}\left\{\wh f_{ki}(\bbeta_0)Z_{ki,m}Z_{ki,j}-E\left(\wh f_{ki}(\bbeta_0)Z_{ki,m}Z_{ki,j}\right)\right\}\right\Vert_{\max}>t_k\right)\\
&\le &2p^2\sum_{k\in \{0\}\cup\mcT_h}\exp\left(-\frac{8n_kb_k^2t_k^2}{B^4}\right).
\end{eqnarray*}
By taking
$$ t_k=\frac{B^2}{b_k}\sqrt{\frac{\log p}{2 n_k}}, $$
we arrive at
$$ \left\Vert\frac1{n_{\mcT_h}+n_0}\sum_{k\in \{0\}\cup\mcT_h}\sum_{i=1}^{n_k}\left\{\wh f_{ki}(\bbeta_0)\bZ_{ki}\bZ_{ki}^{\rmT}-E\left(\wh f_{ki}(\bbeta_0)\bZ_{ki}\bZ_{ki}^{\rmT}\right)\right\}\right\Vert_{\max}\le \sum_{k\in \{0\}\cup\mcT_h}\pi_k\frac{B^2}{b_k}\sqrt{\frac{\log p}{2 n_k}} $$
with probability more than $1-2Kp^{-2}$.

On the other hand,
we have
\begin{eqnarray*}
&&\left\Vert \frac1{n_{\mcT_h}+n_0}\sum_{k\in \{0\}\cup\mcT_h}\sum_{i=1}^{n_k}\left\{E\left(\wh f_{ki}(\bbeta_0)\bZ_{ki}\bZ_{ki}^{\rmT}\right)-E\left(f_k(\bZ^{\rmT}_{k}\bbeta_0\mid \bZ_{k})\bZ_{k}\bZ_{k}^{\rmT}\right)\right\}\right\Vert_{\max}\\
&=&\Bigg\Vert \sum_{k\in \{0\}\cup\mcT_h}\pi_kE\left\{\left(\frac{F_k\left(\bZ_{k}^{\rmT}\bbeta_0+b_k\mid \bZ_{k}\right)
\!-\!F_k\left(\bZ_{k}^{\rmT}\bbeta_0-b_k\mid \bZ_{k}\right)}{2b_k}\!-\!f_k(\bZ^{\rmT}_{k}\bbeta_0\mid \bZ_{k})\right)\bZ_{ki}\bZ_{k}^{\rmT}\right\}\Bigg\Vert_{\max}\\
&\le&m_0B^2\sum_{k\in \{0\}\cup\mcT_h}\pi_kb_k.
\end{eqnarray*}
As a consequence, it holds that
\begin{eqnarray*}
&&\left\Vert\wh \bH_{\mcT_h}-\bH_{\mcT_h}(\bbeta_0)\right\Vert_{\max}\\
&\le&
\sum_{k\in \{0\}\cup\mcT_h}\pi_k\Bigg\{\frac{B^2}{b_k}\sqrt{\frac{\log p}{2 n_k}}+
m_0B^2b_k+\frac{cm_0L^3}{b_k}\left(s_0\sqrt{\frac{\log p}{n_{\mcT_h}+n_0}}+h\right)+\frac{D_0B^2}{b_k}\sqrt{\frac{s_0+s'}{n_k}}+\frac{4B^2}{b_k}\sqrt{\frac{\log p}{n_k}}\Bigg\}
\end{eqnarray*}
with probability more than $1-26p^{-1}-(6+4K)p^{-2}$.
Further, when
\begin{eqnarray*}
b_k&=&O\left(\sqrt{\sqrt{\frac{\log p}{ n_k}}+\left(s_0\sqrt{\frac{\log p}{n_{\mcT_h}+n_0}}+h\right)+
\sqrt{\frac{s_0+s'}{n_k}}}\right)\\
&=&O\left(\left(\frac{\log p}{ n_k}\right)^{1/4}+s_0^{1/2}\left(\frac{\log p}{n_{\mcT_h}+n_0}\right)^{1/4}+h^{1/2}+
\left(\frac{s_0+s'}{n_k}\right)^{1/4}\right),
\end{eqnarray*}
the error bound of $\left\Vert\wh \bH_{\mcT_h}-\bH_{\mcT_h}(\bbeta_0)\right\Vert_{\max}$ achieves a faster convergence rate that
\begin{eqnarray*}
\left\Vert\wh \bH_{\mcT_h}-\bH_{\mcT_h}(\bbeta_0)\right\Vert_{\max}
\le
O_P\left(s_0^{1/2}\left(\frac{\log p}{n_{\mcT_h}+n_0}\right)^{1/4}+h^{1/2}\right).
\end{eqnarray*}
\end{proof}

\subsection{Proof of Theorem \ref{theorem3.2}}

\begin{lemma}
\label{lemmaB2}
Set
\begin{eqnarray*}
b_k=O\left(\left(\frac{\log p}{ n_k}\right)^{1/4}+s_0^{1/2}\left(\frac{\log p}{n_{\mcT_h}+n_0}\right)^{1/4}+h^{1/2}+
\left(\frac{s_0+s'}{n_k}\right)^{1/4}\right),\ k\in\{0\}\cup\mcT_h.
\end{eqnarray*}
Assume that
$$ s_{\mcT_h,m}\left(s_0^{1/2}\left(\frac{\log p}{n_{\mcT_h}+n_0}\right)^{1/4}+h^{1/2}\right)=o(1). $$
Then under conditions in Theorem \ref{theorem2.4} and assumptions \ref{C4} and \ref{C8},
there exists a positive constant $\kappa_m$ such that
\begin{equation*}
\min_{\mathbf{\bv}\in\mathcal{D}\left(0,\mcS_{\mcT_h,m}\right)} \frac{\mathbf{v}^{\rmT}
\wh\bH_{\mcT_h,-m,-m}\mathbf{v}}{\vert\mathbf{v}\vert_2^2}\ge \kappa_m
\end{equation*}
holds with probability approaching to one.
\end{lemma}

\begin{proof}
For any $\bv\in\mathcal{D}(0,\mcS_{\mcT_h,m})$, we have
$$ \vert\bv\vert_1=\left\vert\bv_{\mcS_{\mcT_h,m}^c}\right\vert_1+\left\vert\bv_{\mcS_{\mcT_h,m}}\right\vert_1\le 4\left\vert\bv_{\mcS_{\mcT_h,m}}\right\vert_1
\le 4\sqrt{s_{\mcT_h,m}}\left\vert\bv_{\mcS_{\mcT_h,m}}\right\vert_2\le 4\sqrt{s_{\mcT_h,m}}\vert\bv\vert_2. $$
Note that under assumption \ref{C4}, Lemma \ref{lemmaA1} entails that
\begin{eqnarray*}
&&\left\Vert\bH^*_{\mcT_h,-m,-m}-\bH_{\mcT_h,-m,-m}\left(\bbeta_0\right)\right\Vert_{\max}
\\&\le& \left\Vert\bH^*_{\mcT_h}-\bH_{\mcT_h}\left(\bbeta_0\right)\right\Vert_{\max}\\
&=&\left\Vert\sum_{k\in\{0\}\cup\mcT_h}\pi_kE\left[\left\{f_k\left(\bZ_{k}^\rmT\bbeta^*_{\mcT_h}\mid \bZ_{k}\right)-f_k\left(\bZ_{k}^\rmT\bbeta_0\mid \bZ_{k}\right)\right\}
\bZ_{k}\bZ_{k}^{\rmT}\right]\right\Vert_{\max}\\
&\le &m_0B^3Ch.
\end{eqnarray*}
Hence, based on Theorem \ref{theorem3.1}, for some positive constant $c$, it holds that
\begin{eqnarray*}
&&P\left(\frac{\vert\mathbf{v}^{\rmT}
\left(\wh\bH_{\mcT_h,-m,-m}-\bH^*_{\mcT_h,-m,-m}\right)\mathbf{v}\vert}{\vert\mathbf{v}\vert_2^2}> 16cs_{\mcT_h,m}\left(s_0^{1/2}\left(\frac{\log p}{n_{\mcT_h}+n_0}\right)^{1/4}+h^{1/2}\right)\right)\\
&\leq& P\left(\frac{\vert\mathbf{v}\vert_1^2
\left\Vert\wh\bH_{\mcT_h,-m,-m}-\bH^*_{\mcT_h,-m,-m}\right\Vert_{\max}}{\vert\mathbf{v}\vert_2^2}> 16cs_{\mcT_h,m}\left(s_0^{1/2}\left(\frac{\log p}{n_{\mcT_h}+n_0}\right)^{1/4}+h^{1/2}\right)\right)\\
&\leq&P\left(
\left\Vert\wh\bH_{\mcT_h,-m,-m}-\bH_{\mcT_h,-m,-m}\left(\bbeta_0\right)\right\Vert_{\max}> \frac{c}{2}\left(s_0^{1/2}\left(\frac{\log p}{n_{\mcT_h}+n_0}\right)^{1/4}+h^{1/2}\right)\right)\\
&&+P\left(
\left\Vert\bH^*_{\mcT_h,-m,-m}-\bH_{\mcT_h,-m,-m}\left(\bbeta_0\right)\right\Vert_{\max}> \frac{c}{2}\left(s_0^{1/2}\left(\frac{\log p}{n_{\mcT_h}+n_0}\right)^{1/4}+h^{1/2}\right)\right)\\
&\leq& 26p^{-1}+(6+4K)p^{-2},
\end{eqnarray*}
which implies that
\begin{eqnarray*}
\frac{\left\vert\mathbf{v}^{\rmT}
\left(\wh\bH_{\mcT_h,-m,-m}-\bH^*_{\mcT_h,-m,-m}\mathbf{v}\right)\right\vert}{\vert\mathbf{v}\vert_2^2}=O_P\left(s_{\mcT_h,m}\left(s_0^{1/2}\left(\frac{\log p}{n_{\mcT_h}+n_0}\right)^{1/4}+h^{1/2}\right)\right).
\end{eqnarray*}
As
$$ s_{\mcT_h,m}\left(s_0^{1/2}\left(\frac{\log p}{n_{\mcT_h}+n_0}\right)^{1/4}+h^{1/2}\right)=o(1), $$
when $n$ and $p$ are large enough, it holds that
\begin{eqnarray*}
\frac{\left\vert\mathbf{v}^{\rmT}
\left(\wh\bH_{\mcT_h,-m,-m}-\bH^*_{\mcT_h,-m,-m}\right)\mathbf{v}\right\vert}{\vert\mathbf{v}\vert_2^2}\leq \frac{\kappa_{0m}}4
\end{eqnarray*}
with probability tending to 1.
Hence, following the triangle inequality, we obtain
\begin{eqnarray*}
\frac{\bv^{\rmT} \wh\bH_{\mcT_h,-m,-m}\bv}{\vert\bv\vert_2^2}
\geq\frac{\bv^{\rmT} \bH^*_{\mcT_h,-m,-m}\bv}{\vert\bv\vert_2^2}
-\frac{\left\vert\bv^{\rmT} \left(\wh\bH_{\mcT_h,-m,-m}-\bH^*_{\mcT_h,-m,-m}\right)\bv\right\vert}{\vert\bv\vert_2^2}\geq \frac34\kappa_{0m},
\end{eqnarray*}
which completes the proof by taking $\kappa_m=3\kappa_{0m}/4$.
\end{proof}

\begin{lemma}
\label{lemmaB3}
Set \begin{eqnarray*}
b_k=O\left(\left(\frac{\log p}{ n_k}\right)^{1/4}+s_0^{1/2}\left(\frac{\log p}{n_{\mcT_h}+n_0}\right)^{1/4}+h^{1/2}+
\left(\frac{s_0+s'}{n_k}\right)^{1/4}\right),\ k\in\{0\}\cup\mcT_h.
\end{eqnarray*}
Assume that
$$ s_{\mcT_h,m}\left(s_0^{1/2}\left(\frac{\log p}{n_{\mcT_h}+n_0}\right)^{1/4}+h^{1/2}\right)=o(1). $$
Under conditions in Theorem \ref{theorem2.4} and assumptions \ref{C4} and \ref{C8},
if
$$\lambda_m\ge4C_0(1+r_{\mcT_h,m})\left(s_0^{1/2}\left(\frac{\log p}{n_{\mcT_h}+n_0}\right)^{1/4}+h^{1/2}\right),$$
where the positive constant $C_0$ is specified by (\ref{determineC}),
then it holds that
\begin{eqnarray*}
\left\vert\wh\bgamma_{\mcT_h,m}-\bgamma_{\mcT_h,m}\right\vert_2=O_P\left(\lambda_m\sqrt{s_{\mcT_h,m}}\right)
\end{eqnarray*}
and
\begin{eqnarray*}
\left\vert\wh\bgamma_{\mcT_h,m}-\bgamma_{\mcT_h,m}\right\vert_1=O_P\left(\lambda_ms_{\mcT_h,m}\right).
\end{eqnarray*}
\end{lemma}

\begin{proof}
By the definition of $\wh\bgamma_{\mcT_h,m}$ in (\ref{gamma_trans}), we have
\begin{eqnarray*}
&&\left(\wh\bgamma_{\mcT_h,m}\right)^\rmT\wh\bH_{\mcT_h,-m,-m}\wh\bgamma_{\mcT_h,m}-2\wh\bH_{\mcT_h,m,-m}\wh\bgamma_{\mcT_h,m}
+\lambda_m\vert\wh\bgamma_{\mcT_h,m}\vert_1\\
&\le&
\left(\bgamma_{\mcT_h,m}\right)^\rmT\wh\bH_{\mcT_h,-m,-m}\bgamma_{\mcT_h,m}-2\wh\bH_{\mcT_h,m,-m}\bgamma_{\mcT_h,m}
+\lambda_m\vert\bgamma_{\mcT_h,m}\vert_1.
\end{eqnarray*}
By some basic algebra, we arrive at
\begin{eqnarray}
\label{basic}
&&\left(\wh\bgamma_{\mcT_h,m}-\bgamma_{\mcT_h,m}\right)^\rmT\wh\bH_{\mcT_h,-m,-m}\left(\wh\bgamma_{\mcT_h,m}-\bgamma_{\mcT_h,m}\right)\nonumber
\\&\le& 2\left(\wh\bgamma_{\mcT_h,m}-\bgamma_{\mcT_h,m}\right)^\rmT\left(\wh\bH_{\mcT_h,-m,m}-\wh\bH_{\mcT_h,-m,-m}\bgamma_{\mcT_h,m}\right)
+\lambda_m\left(\left\vert\bgamma_{\mcT_h,m}\right\vert_1-\left\vert\wh\bgamma_{\mcT_h,m}\right\vert_1\right).
\end{eqnarray}
Since
\begin{eqnarray}
\label{determineC}
&&\left\vert\wh\bH_{\mcT_h,-m,m}-\wh\bH_{\mcT_h,-m,-m}\bgamma_{\mcT_h,m}\right\vert_\infty
\nonumber\\
&\le& \left\vert\bH^*_{\mcT_h,-m,m}-\bH^*_{\mcT_h,-m,-m}\bgamma_{\mcT_h,m}\right\vert_\infty+\left\vert\wh\bH_{\mcT_h,-m,m}-\bH^*_{\mcT_h,-m,m}\right\vert_\infty
\nonumber\\&&+\left\vert\wh\bH_{\mcT_h,-m,-m}\bgamma_{\mcT_h,m}-\bH^*_{\mcT_h,-m,-m}\bgamma_{\mcT_h,m}\right\vert_\infty\nonumber\\
&\le &0+\left\vert\wh\bH_{\mcT_h,-m,m}-\bH^*_{\mcT_h,-m,m}\right\vert_\infty+
\left\Vert\wh\bH_{\mcT_h,-m,-m}-\bH^*_{\mcT_h,-m,-m}\right\Vert_{\max}\left\vert\bgamma_{\mcT_h,m}\right\vert_1\nonumber\\
&\le &C_0(1+r_{\mcT_h,m})\left(s_0^{1/2}\left(\frac{\log p}{n_{\mcT_h}+n_0}\right)^{1/4}+h^{1/2}\right)
\end{eqnarray}
 with probability more than $1-26p^{-1}-(6+4K)p^{-2}$ for some constant $C_0>0$
and if
$$\lambda_m\ge4C_0(1+r_{\mcT_h,m})\left(s_0^{1/2}\left(\frac{\log p}{n_{\mcT_h}+n_0}\right)^{1/4}+h^{1/2}\right),$$
(\ref{basic}) entails that
\begin{eqnarray}
\label{RSC_m}
0&\le& 2\left(\wh\bgamma_{\mcT_h,m}-\bgamma_{\mcT_h,m}\right)^\rmT\left(\wh\bH_{\mcT_h,-m,m}-\wh\bH_{\mcT_h,-m,-m}\bgamma_{\mcT_h,m}\right)
+\lambda_m\left(\left\vert\bgamma_{\mcT_h,m}\right\vert_1-\left\vert\wh\bgamma_{\mcT_h,m}\right\vert_1\right)\nonumber\\
&\le &2\left\vert\wh\bgamma_{\mcT_h,m}-\bgamma_{\mcT_h,m}\right\vert_1\left\vert\wh\bH_{\mcT_h,-m,m}-\wh\bH_{\mcT_h,-m,-m}\bgamma_{\mcT_h,m}\right\vert_\infty+
\lambda_m\left(\left\vert\bgamma_{\mcT_h,m}\right\vert_1-\left\vert\wh\bgamma_{\mcT_h,m}\right\vert_1\right)\nonumber\\
&\le &\frac12\lambda_m\left\vert\wh\bgamma_{\mcT_h,m}-\bgamma_{\mcT_h,m}\right\vert_1
+\lambda_m\left(\left\vert\bgamma_{\mcT_h,m}\right\vert_1-\left\vert\wh\bgamma_{\mcT_h,m}\right\vert_1\right)\nonumber\\
&=&\frac12\lambda_m\left\vert\wh\bgamma_{\mcT_h,m,\mcS_{\mcT_h,m}^c}\right\vert_1
+\frac12\lambda_m\left\vert\wh\bgamma_{\mcT_h,m,\mcS_{\mcT_h,m}}-\bgamma_{\mcT_h,m,\mcS_{\mcT_h,m}}\right\vert_1
+\lambda_m\left\vert\bgamma_{\mcT_h,m,\mcS_{\mcT_h,m}}\right\vert_1
\nonumber\\&&-\lambda_m\left\vert\wh\bgamma_{\mcT_h,m,\mcS_{\mcT_h,m}}\right\vert_1
-\lambda_m\left\vert\wh\bgamma_{\mcT_h,m,\mcS_{\mcT_h,m}^c}\right\vert_1
\nonumber\\
&\le& \frac32\lambda_m\left\vert\wh\bgamma_{\mcT_h,m,\mcS_{\mcT_h,m}}-\bgamma_{\mcT_h,m,\mcS_{\mcT_h,m}}\right\vert_1
-\frac12\lambda_m\left\vert\wh\bgamma_{\mcT_h,m,\mcS_{\mcT_h,m}^c}-\bgamma_{\mcT_h,m,\mcS_{\mcT_h,m}^c}\right\vert_1,
\end{eqnarray}
which implies that $\wh\bgamma_{\mcT_h,m}-\bgamma_{\mcT_h,m}\in\mathcal{D}(0,\mcS_{\mcT_h,m})$.
As a consequence, it follows from Lemma \ref{lemmaB2} and (\ref{basic}) that
\begin{eqnarray*}
&&\kappa_m\left\vert\wh\bgamma_{\mcT_h,m}-\bgamma_{\mcT_h,m}\right\vert_2^2
\\&\le&\left(\wh\bgamma_{\mcT_h,m}-\bgamma_{\mcT_h,m}\right)^\rmT\wh\bH_{\mcT_h,-m,-m}\left(\wh\bgamma_{\mcT_h,m}-\bgamma_{\mcT_h,m}\right)\\
&\le&2\left\vert\wh\bgamma_{\mcT_h,m}-\bgamma_{\mcT_h,m}\right\vert_1\left\vert\wh\bH_{\mcT_h,-m,m}-\wh\bH_{\mcT_h,-m,-m}\bgamma_{\mcT_h,m}\right\vert_\infty+
\lambda_m\left(\left\vert\bgamma_{\mcT_h,m}\right\vert_1-\left\vert\wh\bgamma_{\mcT_h,m}\right\vert_1\right)\\
&\le&\frac32\lambda_m\left\vert\wh\bgamma_{\mcT_h,m}-\bgamma_{\mcT_h,m}\right\vert_1\\
&\le &6\lambda_m\sqrt{s_{\mcT_h,m}}
\left\vert\wh\bgamma_{\mcT_h,m}-\bgamma_{\mcT_h,m}\right\vert_2
\end{eqnarray*}
with probability at least $1-26p^{-1}-(6+4K)p^{-2}$.
Therefore, it concludes that
\begin{eqnarray*}
\left\vert\wh\bgamma_{\mcT_h,m}-\bgamma_{\mcT_h,m}\right\vert_2=O_P\left(\lambda_m\sqrt{s_{\mcT_h,m}}\right)
\end{eqnarray*}
and
\begin{eqnarray*}
\left\vert\wh\bgamma_{\mcT_h,m}-\bgamma_{\mcT_h,m}\right\vert_1=O_P\left(\lambda_ms_{\mcT_h,m}\right).
\end{eqnarray*}
\end{proof}

\begin{lemma}
\label{lemmaB4}
Set
\begin{eqnarray*}
h_0=O\left(\left(\frac{\log p}{ n_0}\right)^{1/4}+s_0^{1/2}\left(\frac{\log p}{n_{\mcT_h}+n_0}\right)^{1/4}+h^{1/2}+
\left(\frac{s_0+s'}{n_0}\right)^{1/4}\right).
\end{eqnarray*}
Assume that
$$ \left(s_0+s_{\mcT_h,m}\right)\left(s_0^{1/2}\left(\frac{\log p}{n_{\mcT_h}+n_0}\right)^{1/4}+h^{1/2}\right)=o(1). $$
Under conditions in Theorem \ref{theorem2.4} and assumptions \ref{C4} and \ref{C9},
if
$$\lambda_m\ge4C_0(1+r_{\mcT_h,m})\left(s_0^{1/2}\left(\frac{\log p}{n_{\mcT_h}+n_0}\right)^{1/4}+h^{1/2}\right),$$
where the positive constant $C_0$ is specified by (\ref{determineC}),
then there exists a positive constant $\kappa_m$ such that
\begin{equation*}
\min_{\mathbf{\bv}\in\mathcal{D}\left(0,\bar{\mcS}_{\mcT_h,m}\right)} \frac{\mathbf{v}^{\rmT}
\wh\bH_{0,-m,-m}\mathbf{v}}{\vert\mathbf{v}\vert_2^2}\ge \kappa_m
\end{equation*}
holds with probability approaching to one.
\end{lemma}

\begin{proof}
For any $\bv\in\mathcal{D}(0,\bar{\mcS}_{\mcT_h,m})$, we have
$$ \vert\bv\vert_1=\left\vert\bv_{\left(\bar{\mcS}_{\mcT_h,m}\right)^c}\right\vert_1+\left\vert\bv_{\bar{\mcS}_{\mcT_h,m}}\right\vert_1\le 4\left\vert\bv_{\bar{\mcS}_{\mcT_h,m}}\right\vert_1
\le 4\sqrt{s_0+s_{\mcT_h,m}}\left\vert\bv_{\bar{\mcS}_{\mcT_h,m}}\right\vert_2\le 4\sqrt{s_0+s_{\mcT_h,m}}\vert\bv\vert_2. $$
Hence, mimicking Theorem \ref{theorem3.1}, for some positive constant $C_0$, it holds that
\begin{eqnarray*}
&&P\left(\frac{\left\vert\mathbf{v}^{\rmT}
\left(\wh\bH_{0,-m,-m}-\bH_{0,-m,-m}^*\right)\mathbf{v}\right\vert}{\vert\mathbf{v}\vert_2^2}> 16C_0\left(s_0+s_{\mcT_h,m}\right)\left(s_0^{1/2}\left(\frac{\log p}{n_{\mcT_h}+n_0}\right)^{1/4}+h^{1/2}\right)\right)\\
&\leq& P\left(\frac{\vert\mathbf{v}\vert_1^2
\left\Vert\wh\bH_0-\bH_0^*\right\Vert_{\max}}{\vert\mathbf{v}\vert_2^2}> 16C_0\left(s_0+s_{\mcT_h,m}\right)\left(s_0^{1/2}\left(\frac{\log p}{n_{\mcT_h}+n_0}\right)^{1/4}+h^{1/2}\right)\right)\\
&\leq&P\left(
\left\Vert\wh\bH_0-\bH_0^*\right\Vert_{\max}> C_0\left(s_0^{1/2}\left(\frac{\log p}{n_{\mcT_h}+n_0}\right)^{1/4}+h^{1/2}\right)\right)\\
&\leq& 26p^{-1}+10p^{-2},
\end{eqnarray*}
which implies that
\begin{eqnarray*}
\frac{\left\vert\mathbf{v}^{\rmT}
\left(\wh\bH_{0,-m,-m}-\bH^*_{0,-m,-m}\right)\mathbf{v}\right\vert}{\vert\mathbf{v}\vert_2^2}=O_P\left(\left(s_0+s_{\mcT_h,m}\right)\left(s_0^{1/2}\left(\frac{\log p}{n_{\mcT_h}+n_0}\right)^{1/4}+h^{1/2}\right)\right).
\end{eqnarray*}
As
$$ \left(s_0+s_{\mcT_h,m}\right)\left(s_0^{1/2}\left(\frac{\log p}{n_{\mcT_h}+n_0}\right)^{1/4}+h^{1/2}\right)=o(1), $$
when $n$ and $p$ are large enough, it holds that
\begin{eqnarray*}
\frac{\left\vert\mathbf{v}^{\rmT}
\left(\wh\bH_{0,-m,-m}-\bH^*_{0,-m,-m}\right)\mathbf{v}\right\vert}{\vert\mathbf{v}\vert_2^2}\leq \frac{\kappa_{0m}}4
\end{eqnarray*}
with probability tending to 1.
Hence, following the triangle inequality, we obtain
\begin{eqnarray*}
\frac{\bv^{\rmT} \wh\bH_{0,-m,-m}\bv}{\vert\bv\vert_2^2}
\geq\frac{\bv^{\rmT} \bH^*_{0,-m,-m}\bv}{\vert\bv\vert_2^2}
-\frac{\left\vert\bv^{\rmT} \left(\wh\bH_{0,-m,-m}-\bH^*_{0,-m,-m}\right)\bv\right\vert}{\vert\bv\vert_2^2}\geq \frac34\kappa_{0m},
\end{eqnarray*}
which completes the proof by taking $\kappa_m=3\kappa_{0m}/4$.
\end{proof}

{\bf Proof of Theorem \ref{theorem3.2}}
\begin{proof}
Using the definition of $\wh\bzeta_m$ in (\ref{gamma_debias}), we have
\begin{eqnarray*}
&&\left(\wh\bgamma_{\mcT_h,m}+\wh\bzeta_m\right)^\rmT\wh\bH_{0,-m,-m}\left(\wh\bgamma_{\mcT_h,m}+\wh\bzeta_m\right)-2\wh\bH_{0,m,-m}\left(\wh\bgamma_{\mcT_h,m}+\wh\bzeta_m\right)
+\lambda_m'\left\vert\wh\bzeta_m\right\vert_1\\
&\le& \left(\wh\bgamma_{\mcT_h,m}+\bgamma_{0,m}-\bgamma_{\mcT_h,m}\right)^{\rmT}\wh\bH_{0,-m,-m}\left(\wh\bgamma_{\mcT_h,m}+\bgamma_{0,m}-\bgamma_{\mcT_h,m}\right)
\\&&-2\wh\bH_{0,m,-m}\left(\wh\bgamma_{\mcT_h,m}+\bgamma_{0,m}-\bgamma_{\mcT_h,m}\right)
+\lambda_m'\left\vert\bgamma_{0,m}-\bgamma_{\mcT_h,m}\right\vert_1,
\end{eqnarray*}
which implies that
\begin{eqnarray}
\label{basic_debias}
&&(\wh\bgamma_{0,m}-\bgamma_{0,m}-\wh\bgamma_{\mcT_h,m}+\bgamma_{\mcT_h,m})^\rmT\wh\bH_{0,-m,-m}
\left(\wh\bgamma_{0,m}-\bgamma_{0,m}-\wh\bgamma_{\mcT_h,m}+\bgamma_{\mcT_h,m}\right)\\
&\le& 2\left(\wh\bgamma_{0,m}-\bgamma_{0,m}-\wh\bgamma_{\mcT_h,m}+\bgamma_{\mcT_h,m}\right)^\rmT\left\{\wh\bH_{0,-m,m}-\wh\bH_{0,-m,-m}
\left(\bgamma_{0,m}+\wh\bgamma_{\mcT_h,m}-\bgamma_{\mcT_h,m}\right)\right\}\nonumber
\\&&+\lambda_m'\left(\left\vert\bgamma_{0,m}-\bgamma_{\mcT_h,m}\right\vert_1-\left\vert\wh\bzeta_m\right\vert_1\right).\nonumber
\end{eqnarray}
Since
\begin{eqnarray}
\label{determineC'}
&&\left\vert\wh\bH_{0,-m,m}-\wh\bH_{0,-m,-m}\left(\bgamma_{0,m}+\wh\bgamma_{\mcT_h,m}-\bgamma_{\mcT_h,m}\right)\right\vert_\infty\nonumber\\
&\le& \left\vert\wh\bH_{0,-m,m}-\wh\bH_{0,-m,-m}\bgamma_{0,m}\right\vert_\infty
+\left\vert\wh\bH_{0,-m,-m}\left(\wh\bgamma_{\mcT_h,m}-\bgamma_{\mcT_h,m}\right)\right\vert_\infty\nonumber\\
&\le & \left(1+r_{0,m}\right)\left\Vert\wh\bH_{0}-\bH^*_{0}\right\Vert_{\max}+\frac{B^2}{2b_0}\left\vert\wh\bgamma_{\mcT_h,m}-\bgamma_{\mcT_h,m}\right\vert_1\nonumber
\\&\le&C'\left(r_{0,m}\left(s_0^{1/2}\left(\frac{\log p}{n_{\mcT_h}+n_0}\right)^{1/4}+h^{1/2}\right)+\frac{\lambda_ms_{\mcT_h,m}}{b_0}\right),
\end{eqnarray}
for some constant $C'>0$,
then if
$$ \lambda_m'\ge 4C'\left(r_{0,m}\left(s_0^{1/2}\left(\frac{\log p}{n_{\mcT_h}+n_0}\right)^{1/4}+h^{1/2}\right)+\frac{\lambda_ms_{\mcT_h,m}}{b_0}\right), $$
using the arguments similar to (\ref{RSC_m}) yields that
\begin{eqnarray*}
\frac32\lambda_m'\left\vert\left(\wh\bgamma_{0,m}-\bgamma_{0,m}-\wh\bgamma_{\mcT_h,m}+\bgamma_{\mcT_h,m}\right)_{\bar{\mcS}_{\mcT_h,m}}\right\vert_1-
\frac12\lambda_m'\left\vert\left(\wh\bgamma_{0,m}-\bgamma_{0,m}-\wh\bgamma_{\mcT_h,m}+\bgamma_{\mcT_h,m}\right)_{\bar{\mcS}^c_{\mcT_h,m}}\right\vert_1\ge 0,
\end{eqnarray*}
which entails that
$$ \wh\bgamma_{0,m}-\bgamma_{0,m}-\wh\bgamma_{\mcT_h,m}+\bgamma_{\mcT_h,m}\in\mathcal{D}(0,\bar{\mcS}_{\mcT_h,m}). $$
Thus, following Lemma \ref{lemmaB4} and (\ref{basic_debias}), it holds that
\begin{eqnarray*}
&&\kappa_m\left\vert\wh\bgamma_{0,m}-\bgamma_{0,m}-\wh\bgamma_{\mcT_h,m}+\bgamma_{\mcT_h,m}\right\vert_2^2\\
&\le&\left(\wh\bgamma_{0,m}-\bgamma_{0,m}-\wh\bgamma_{\mcT_h,m}+\bgamma_{\mcT_h,m}\right)^\rmT\wh\bH_{0,-m,-m}
\left(\wh\bgamma_{0,m}-\bgamma_{0,m}-\wh\bgamma_{\mcT_h,m}+\bgamma_{\mcT_h,m}\right)\\
&\le& 2\left(\wh\bgamma_{0,m}-\bgamma_{0,m}-\wh\bgamma_{\mcT_h,m}+\bgamma_{\mcT_h,m}\right)^\rmT\left\{\wh\bH_{0,-m,m}-\wh\bH_{0,-m,-m}
\left(\bgamma_{0,m}+\wh\bgamma_{\mcT_h,m}-\bgamma_{\mcT_h,m}\right)\right\}
\\&&+\lambda_m'\left(\left\vert\bgamma_{0,m}-\bgamma_{\mcT_h,m}\right\vert_1-\left\vert\wh\bzeta_m\right\vert_1\right)\\
&\le &\frac32\lambda_m'\left\vert\wh\bgamma_{0,m}-\bgamma_{0,m}-\wh\bgamma_{\mcT_h,m}+\bgamma_{\mcT_h,m}\right\vert_1\\
&\le&6\lambda_m'\sqrt{s_0+s_{\mcT_h,m}}\left\vert\wh\bgamma_{0,m}-\bgamma_{0,m}-\wh\bgamma_{\mcT_h,m}+\bgamma_{\mcT_h,m}\right\vert_2,
\end{eqnarray*}
which implies that
$$ \left\vert\wh\bgamma_{0,m}-\bgamma_{0,m}-\wh\bgamma_{\mcT_h,m}+
\bgamma_{\mcT_h,m}\right\vert_2=O_P\left(\lambda_m'\sqrt{s_0+s_{\mcT_h,m}}\right) $$
and
$$ \left\vert\wh\bgamma_{0,m}-\bgamma_{0,m}-\wh\bgamma_{\mcT_h,m}+
\bgamma_{\mcT_h,m}\right\vert_1=O_P\left(\lambda_m'\left(s_0+s_{\mcT_h,m}\right)\right). $$
Consequently, we can conclude that
$$ \left\vert\wh\bgamma_{0,m}-\bgamma_{0,m}\right\vert_2=
O_P\left(\lambda_m'\sqrt{s_0+s_{\mcT_h,m}}+\lambda_m\sqrt{s_{\mcT_h,m}}\right) $$
and
$$ \left\vert\wh\bgamma_{0,m}-\bgamma_{0,m}\right\vert_1=O_P\left(\lambda_m'\left(s_0+s_{\mcT_h,m}\right)+\lambda_ms_{\mcT_h,m}\right). $$
\end{proof}

\subsection{Proof of Theorem \ref{theorem3.3}}

\begin{lemma}
\label{lemmaB5}
Assume
$$\frac{s_0\log p}{\sqrt{n_0}}=o(1)\ {\rm and}\ h=o\left(n_0^{-1/4}\right)$$
as  $n_0,p\to+\infty$.
Under conditions in Theorem \ref{theorem2.2}, it holds
\begin{eqnarray*}
\sqrt{n_0}\bvarphi_{0,m}^\rmT \left\{\bS\left(\wh\bbeta_0;\{0\}\right)-\bS(\bbeta_0;\{0\})-\bH^{*}_0\left(\wh\bbeta_0-\bbeta_0\right)\right\}=o_P(1),
\end{eqnarray*}
where $\bvarphi_{0,m}=\Gamma_{m,1}(-\bgamma_{0,m})$.
\end{lemma}

\begin{proof}
For any $\bdelta\in{\mathbb R}^p$,  rewrite
\begin{eqnarray*}
&&\bvarphi_{0,m}^\rmT \left\{\bS(\bbeta_0+\bdelta;\{0\})-\bS(\bbeta_0;\{0\})-\bH^{*}_0\bdelta\right\}\\
&=&\bvarphi_{0,m}^\rmT \left[\bS(\bbeta_0+\bdelta;\{0\})\!-\!\bS(\bbeta_0;\{0\})\!-\!E\left\{\bS(\bbeta_0+\bdelta;\{0\})\!-\!\bS(\bbeta_0;\{0\})\right\}\right.
\\&&\left.+E\left\{\bS(\bbeta_0+\bdelta;\{0\})-\bS(\bbeta_0;\{0\})\right\}-\bH^{*}_0\bdelta\right].
\end{eqnarray*}
We first show
\begin{eqnarray}
\label{I1}
\left.\sqrt{n}\bvarphi_{0,m}^\rmT \left[\bS(\bbeta_0+\bdelta;\{0\})-\bS(\bbeta_0;\{0\})-E\left\{\bS(\bbeta_0+\bdelta;\{0\})-
\bS(\bbeta_0;\{0\})\right\}\right]\right\vert_{\bdelta=\wh\bbeta_0-\bbeta_0}=o_P(1)
\nonumber\\
\end{eqnarray}
by utilizing the Hoeffding inequality and Theorem 3 of 
\citet{chen2003estimation}.
As conditions (3.1) and (3.3) of Theorem 3 in 
\citet{chen2003estimation}
are automatically
satisfied, we only verify condition (3.2).  Under assumption \ref{C2}
and for any $t>0$  as long as $\vert\bdelta\vert_1<t$, we have
\begin{eqnarray*}
&&I\left(Y_{0}-\bZ_{0}^\rmT\bbeta_0\leq -Bt\right)\leq I\left(Y_{0}-\bZ_{0}^\rmT\bbeta_0\leq \bZ_{0}^\rmT\bdelta\right)
\\&=&I\left(Y_{0}-\bZ_{0}^\rmT(\bbeta_0+\bdelta)\leq 0\right)\leq I\left(Y_{0}-\bZ_{0}^\rmT\bbeta_0\leq Bt\right)
\end{eqnarray*}
and
\begin{eqnarray*}
I\left(Y_{0}-\bZ_{0}^\rmT\bbeta_0\leq -Bt\right)\leq I\left(Y_{0}-\bZ_{0}^\rmT\bbeta_0\leq 0\right)
\leq I\left(Y_{0}-\bZ_{0}^\rmT\bbeta_0\leq Bt\right).
\end{eqnarray*}
Hence, it holds
\begin{eqnarray*}
&&\sup_{\bdelta: \vert\bdelta\vert_1<t}\left\vert I\left(Y_{0}-\bZ_{0}^\rmT(\bbeta_0+\bdelta)\leq 0\right)
-I\left(Y_{0}-\bZ_{0}^\rmT\bbeta_0\leq 0\right)\right\vert\\
&<&I\left(Y_{0}-\bZ_{0}^\rmT\bbeta_0\leq Bt\right)
-I\left(Y_{0}-\bZ_{0}^\rmT\bbeta_0\leq -Bt\right).
\end{eqnarray*}
It follows from  assumptions \ref{C1} and \ref{C3} and some basic calculations that
\begin{eqnarray*}
\label{B1}
&&\left\{E\left(\sup_{\bdelta: \vert\bdelta\vert_1<t}\left\vert \bvarphi_{0,m}^\rmT
\bZ_{0}\left\{I\left(Y_{0}-\bZ_{0}^\rmT(\bbeta_0+\bdelta)\leq 0\right)
-I\left(Y_{0}-\bZ_{0}^\rmT\bbeta_0\leq 0\right)\right\}\right\vert^2\right)\right\}^{1/2}\nonumber\\
&\leq& \left(1+r_{0,m}\right)B E\left\{E\left(\sup_{\bdelta: \vert\bdelta\vert_1<t}\left\{I\left(Y_{0}-\bZ_{0}^\rmT(\bbeta_0+\bdelta)\leq 0\right)
-I\left(Y_{0}-\bZ_{0}^\rmT\bbeta_0\leq 0\right)\right\}^2\mid \bZ_{0}\right)\right\}^{1/2}\nonumber\\
&\leq&\left(1+r_{0,m}\right)B E\left\{E\left(
\left\{I\left(Y_{0}-\bZ_{0}^\rmT\bbeta_0\leq Bt\right)
-I\left(Y_{0}-\bZ_{0}^\rmT\bbeta_0\leq -Bt\right)\right\}\mid  \bZ_{0}\right)\right\}^{1/2}\nonumber\\
&\leq&\left(1+r_{0,m}\right)B E\left\{F_0\left(\bZ_{0}^\rmT\bbeta_0+ Bt\mid \bZ_{0}\right)
-F_0\left(\bZ_{0}^\rmT\bbeta_0- Bt\mid \bZ_{0}\right)
\right\}^{1/2}\nonumber\\
&\leq&\left(1+r_{0,m}\right)B(2Bm_0t)^{1/2}\nonumber
\\&=&\sqrt 2B^{3/2}\left(1+r_{0,m}\right)m_0^{1/2}\sqrt{t}.
\end{eqnarray*}
Thus, condition (3.2) of Theorem 3 in 
\citet{chen2003estimation}
is satisfied,
which implies with probability at least $1-p^{-2}-2p^{-1}$ that
\begin{eqnarray*}
&&\bvarphi_{0,m}^\rmT \left\{\bS\left(\wh\bbeta_0;\{0\}\right)-\bS(\bbeta_0;\{0\})\right\}\\
&=& \frac{1}{n_0}\sum_{i=1}^n
\bvarphi_{0,m}^\rmT \bZ_{0i}\left\{I\left(Y_{0i}-\bZ_{0i}^\rmT\wh\bbeta_0\leq 0\right)
-I\left(Y_{0i}-\bZ_{0i}^\rmT\bbeta_0\leq 0\right)\right\}\\
&=& \left.\bvarphi_{0,m}^\rmT E\left[\bZ_{0}\left\{I\left(Y_{0}-\bZ_{0}^\rmT(\bbeta_0+\bdelta)\leq 0\right)
-I\left(Y_{0}-\bZ_{0}^\rmT\bbeta_0\leq 0\right)\right\}\right]\right\vert_{\bdelta=\wh\bbeta_0-\bbeta_0}+o_P\left(\frac1{\sqrt {n_0}}\right)\\
&=&\left.\bvarphi_{0,m}^\rmT E\left\{\bS(\bbeta_0+\bdelta;\{0\})-\bS(\bbeta_0;\{0\})\right\}\right\vert_{\bdelta=\wh\bbeta_0-\bbeta_0}+o_P\left(\frac1{\sqrt {n_0}}\right)
\end{eqnarray*}
by setting
\begin{eqnarray*}\label{tSTAR}
t=O\left(s_0\sqrt{\frac{\log p}{n_{\mcT_h}+n_0}}+h\right)=o(1).
\end{eqnarray*}
On the other hand, it follows from the Taylor expansion and Theorem \ref{theorem2.2} that
\begin{eqnarray*}
\left.\bvarphi_{0,m}^{\rmT}\left[E\left\{\bS\left(\bbeta_0+\bdelta;\{0\}\right)-\bS(\bbeta_0;\{0\})\right\}-\bH^{*}_0\bdelta\right]
\right\vert_{\bdelta=\wh\bbeta_0-\bbeta_0}
=O_P\left(\left\vert\wh\bbeta_0-\bbeta_0\right\vert_2^2\right)=o_P\left(\frac{1}{\sqrt{n_0}}\right),
\end{eqnarray*}
which along with (\ref{I1}) concludes Lemma \ref{lemmaA5}.
\end{proof}

\begin{lemma}
\label{lemmaB6}
Assume
$$\frac{s_0\log p}{\sqrt{n_0}}=o(1)\ {\rm and}\ h=o\left(n_0^{-1/4}\right)$$
as $n_0,p\to+\infty$.
Under conditions in Theorem \ref{theorem2.2} and \ref{theorem3.2},
if
$$\sqrt{n_0}\left(\lambda_m'\left(s_0+s_{\mcT_h,m}\right)+\lambda_ms_{\mcT_h,m}\right)\left(s_0\sqrt{\frac{\log p}{n_0}}+h\right)=o(1),$$
then we have
\begin{eqnarray*}
\frac{\sqrt{n_0}\left(\wt\beta_{0m}-\beta_{0m}\right)H^{*}_{0,m\mid -m}}{\sigma_m}\xrightarrow{\mathcal{L}}N(0,1),
\end{eqnarray*}
where $H^{*}_{0,m\mid -m}=H^*_{0,m,m}-\bgamma_{0,m}^\rmT\bH^*_{0,-m,m}$ and $\sigma_m^2=\bvarphi_{0,m}^\rmT \bSigma_0\bvarphi_{0,m}$.
\end{lemma}

\begin{proof}
Utilizing the definition of $\wh\bvarphi_{0,m}$, we derive the decomposition
\begin{eqnarray}\label{tldbhatb}
&&\sqrt{n_0}\left(\wt\beta_{0m}-\beta_{0m}\right)H^{*}_{0,m\mid -m}\nonumber\\
&=&\sqrt{n_0}\left(\wh\beta_{0m}-\beta_{0m}\right)H^{*}_{0,m\mid -m}-
\sqrt{n_0}\wh\bvarphi_{0,m}^\rmT \bS\left(\wh\bbeta_0;\{0\}\right)
\nonumber\\&&-
\sqrt{n_0}\wh\bvarphi_{0,m}^\rmT \bS\left(\wh\bbeta_0;\{0\}\right)\left(\wh H_{0,m\mid -m}^{-1}H^{*}_{0,m\mid -m}-1\right)\nonumber\\
&=&-\sqrt{n_0}\bvarphi_{0,m}^\rmT \bS(\bbeta_0;\{0\})-
\sqrt{n_0}\left(\wh\bvarphi_{0,m}-\bvarphi_{0,m}\right)^\rmT\bS(\bbeta_0;\{0\})\nonumber\\
&&-\left[\sqrt{n_0}\wh\bvarphi_{0,m}^\rmT \left\{\bS\left(\wh\bbeta_0;\{0\}\right)-\bS(\bbeta_0;\{0\})\right\}
-\sqrt{n_0}\left(\wh\beta_{0m}-\beta_{0m}\right)H^{*}_{0,m\mid -m}\right]\nonumber\\
&&-\sqrt{n_0}\wh\bvarphi_{0,m}^\rmT \bS\left(\wh\bbeta_0;\{0\}\right)\left(\left(\wh H_{0,m\mid -m}\right)^{-1}H^{*}_{0,m\mid -m}-1\right).
\end{eqnarray}
Combining Lemma \ref{lemmaA5} and Theorem \ref{theorem3.2}, we have
\begin{eqnarray}\label{TMm1}
\sqrt{n_0}\left\vert\left(\wh\bvarphi_{0,m}-\bvarphi_{0,m}\right)^\rmT\bS(\bbeta_0;\{0\})\right\vert
&\leq&
\sqrt{n_0}\left\vert \wh\bvarphi_{0,m}-\bvarphi_{0,m}\right\vert_1\left\vert\bS(\bbeta_0;\{0\})\right\vert_\infty\nonumber\\
&=&O_P\left(\lambda_m'\left(s_0+s_{\mcT_h,m}\right)\sqrt{\log p}+\lambda_ms_{\mcT_h,m}\sqrt{\log p}
\right)\nonumber\\
&=&o_P(1).
\end{eqnarray}
On the other hand,
for any $\bv_1,\bv_2\in\mathbb{R}^p$,
since under assumption \ref{C2}, it holds
$$ \left\vert\bZ_{0i}\left\{I\left(Y_{0i}-\bZ_{0i}^\rmT\bv_1\right)-
I\left(Y_{0i}-\bZ_{0i}^\rmT\bv_2\right)\right\}\right\vert_\infty\le B, $$
then the Hoeffding inequality yields that
\begin{eqnarray*}
P\left(\left\vert \bS(\bv_1;\{0\})-\bS(\bv_2;\{0\})-E\{\bS(\bv_1;\{0\})-\bS(\bv_2;\{0\})\}\right\vert_\infty\ge t\right)\le 2p\exp\left(-\frac{n_0t^2}{2B^2}\right).
\end{eqnarray*}
By taking $t=2B\sqrt{\log p/n_0}$, it holds with probability more than $1-2p^{-1}$ that
\begin{eqnarray*}
\left\vert\bS\left(\wh\bbeta_0;\{0\}\right)-\bS(\bbeta_0;\{0\})\right\vert_\infty&\le& 2B\sqrt{\frac{\log p}{n_0}}+\vert E\{\bS(\wh\bbeta_0;\{0\})-\bS(\bbeta_0;\{0\})\}\vert_\infty\\
&\le& 2B\sqrt{\frac{\log p}{n_0}}+m_0B^2\vert\wh\bbeta_0-\bbeta_0\vert_1.
\end{eqnarray*}
Therefore, the definition of $\bvarphi_{0,m}$, Lemmas \ref{lemmaA5} and \ref{lemmaB5} and Theorem \ref{theorem3.2}
entail that
\begin{eqnarray}\label{TMm2}
&&\left\vert\sqrt{n_0}\wh\bvarphi_{0,m}^\rmT \left\{\bS\left(\wh\bbeta_0;\{0\}\right)-\bS(\bbeta_0;\{0\})\right\}-\sqrt{n_0}\left(\wh\beta_{0m}-\beta_{0m}\right)H^{*}_{0,m\mid -m}\right\vert\nonumber\\
&=& \left\vert\sqrt{n_0}\bvarphi_{0,m}^\rmT \left\{\bS\left(\wh\bbeta_0;\{0\}\right)-\bS(\bbeta_0;\{0\})-\bH^{*}_0\left(\wh\bbeta_0-\bbeta_0\right)\right\}
\right.\nonumber\\
&&+\left(\sqrt{n_0}\bvarphi_{0,m}^\rmT \bH^{*}_0\left(\wh\bbeta_0-\bbeta_0\right)-\sqrt{n_0}\left(\wh\beta_{0m}-\beta_{0m}\right)H^{*}_{0,m\mid -m}\right)
\nonumber\\
&&\left.+\sqrt{n_0}\left(\wh\bvarphi_{0,m}-\bvarphi_{0,m}\right)^{\rmT}\left\{\bS\left(\wh\bbeta_0;\{0\}\right)-\bS(\bbeta_0;\{0\})\right\}\right\vert\nonumber\\
&\leq&\left\vert\sqrt{n_0}\bvarphi_{0,m}^\rmT \left\{\bS\left(\wh\bbeta_0;\{0\}\right)-\bS(\bbeta_0;\{0\})-\bH^{*}_0\left(\wh\bbeta_0-\bbeta_0\right)\right\}\right\vert
+0\nonumber\\
&&+\sqrt{n_0}\left\vert\wh\bvarphi_{0,m}-\bvarphi_{0,m}\right\vert_1\left\vert\bS\left(\wh\bbeta_0;\{0\}\right)-\bS(\bbeta_0;\{0\})\right\vert_{\infty}\nonumber\\
&=&o_P(1)+\sqrt{n_0}O_P\left(\lambda_m'\left(s_0+s_{\mcT_h,m}\right)+\lambda_ms_{\mcT_h,m}\right)O_P\left(s_0\sqrt{\frac{\log p}{n_0}}+h\right)\nonumber\\&=&o_P(1).
\end{eqnarray}
Theorems \ref{theorem3.1} and  \ref{theorem3.2} imply that
\begin{eqnarray*}
\label{Bk_k}
&&\left\vert\wh H_{0,m\mid -m}-H^{*}_{0,m\mid -m}\right\vert\nonumber\\&\leq&\left\vert\wh H_{0,m,m}-H^{*}_{0,m,m}\right\vert+\left\vert\wh\bgamma_{0,m}^{\rmT}\wh\bH_{0,-m,m}-\bgamma_{0,m}^\rmT\wh\bH_{0,-m,m}
+\bgamma_{0,m}^\rmT\wh\bH_{0,-m,m}-\bgamma_{0,m}^\rmT\bH^{*}_{0,-m,m}\right\vert\nonumber\\
&\leq&\left(1+\left\vert\wh\bgamma_{0,m}-\bgamma_{0,m}\right\vert_1\right) \left\Vert\wh\bH_0\right\Vert_{\max}
+r_{0,m}\left\Vert\wh\bH_0-\bH^{*}_0\right\Vert_{\max}\nonumber\\
&\leq&\left(1+\left\vert\wh\bgamma_{0,m}-\bgamma_{0,m}\right\vert_1\right) \left\Vert\wh\bH_0-\bH^{*}_0\right\Vert_{\max}+\left\Vert\bH^{*}_0\right\Vert_{\max}\left\vert\wh\bgamma_{0,m}-\bgamma_{0,m}\right\vert_1
+r_{0,m}\left\Vert\wh\bH_0-\bH^{*}_0\right\Vert_{\max}\nonumber\\&=&O_P\left(\lambda_m'\left(s_0+s_{\mcT_h,m}\right)+\lambda_m's_{\mcT_h,m}
+s_0^{1/2}\left(\frac{\log p}{n_{\mcT_h}+n_0}\right)^{1/4}+h^{1/2}\right).
\end{eqnarray*}
As a result,
\begin{eqnarray}\label{TMm3}
&&\left\vert\sqrt{n_0}\wh\bvarphi_{0,m}^\rmT \bS\left(\wh\bbeta_0;\{0\}\right)\left(\left(\wh H_{0,m\mid -m}\right)^{-1}H^*_{0,m\mid -m}-1\right)\right\vert\nonumber\\
&\le& \sqrt{n_0}\left\vert\wh\bvarphi_{0,m}\right\vert_1\left\vert\bS\left(\wh\bbeta_0;\{0\}\right)\right\vert_{\infty}
\left\vert\left(\wh H_{0,m\mid -m}\right)^{-1}H^*_{0,m\mid -m}-1\right\vert
\nonumber\\&=&
O_P\left(\sqrt{n_0}\lambda_m'^2\left(s_0+s_{\mcT_h,m}\right)+\sqrt{n_0}\lambda_m'^2s_{\mcT_h,m}
+s_0^{1/2}\sqrt{n_0}\lambda_m'\left(\frac{\log p}{n_{\mcT_h}+n_0}\right)^{1/4}+\sqrt{n_0}\lambda_m'h^{1/2}\right).\nonumber\\
\end{eqnarray}
Combining (\ref{TMm1}), (\ref{TMm2}) and (\ref{TMm3}),
(\ref{tldbhatb}) is reduced to
\begin{eqnarray*}
\sqrt{n_0}\left(\wt\beta_{0m}-\beta_{0m}\right)H^{*}_{0,m\mid -m}
=-\sqrt{n_0}\bvarphi_{0,m}^\rmT \bS(\bbeta_0;\{0\})
+o_P(1),
\end{eqnarray*}
which concludes the proof by using the asymptotic normality of the $U$-statistic.
\end{proof}

{\bf Proof of Theorem \ref{theorem3.3}}
\begin{proof}
Recall that
$$\wh\bSigma_k=\frac1{n_k}\sum_{i=1}^{n_k} \bZ_{ki}\bZ_{ki}^{\rmT}\tau(1-\tau)\ {\rm and}\ \wh\bSigma_{\mcT_h}=\sum_{k\in\{0\}\cup\mcT_h}\pi_k\wh\bSigma_k, $$
for $k\in\{0\}\cup\mcT_h$.
The Hoeffding inequality entails that
\begin{eqnarray*}
&&P\left(\left\Vert\wh\bSigma_{\mcT_h}-\bSigma_{\mcT_h}\right\Vert_{\max}>\sum_{k\in\{0\}\cup\mcT_h}\pi_kt_k\right)
\\&\le& \sum_{k\in\{0\}\cup\mcT_h}P\left(\left\Vert\wh\bSigma_k-\bSigma_k\right\Vert_{\max}>t_k\right)\\
&\le &\sum_{k\in\{0\}\cup\mcT_h}\sum_{j=1}^p\sum_{m=1}^pP\left(\left\vert\frac1{n_k}\sum_{i=1}^{n_k}Z_{ki,j}Z_{ki,m}-E\left(Z_{ki,j}Z_{ki,m}\right)\right\vert>t_k\right)\\
&\le &\sum_{k\in\{0\}\cup\mcT_h}2p^2\exp\left(-\frac{n_kt_k^2}{2B^4}\right).
\end{eqnarray*}
By taking $t_k=2B^2\sqrt{2\log p/n_k}$, we obtain that
$$ \left\Vert\wh\bSigma_{\mcT_h}-\bSigma_{\mcT_h}\right\Vert_{\max}\le2B^2\sum_{k\in\{0\}\cup\mcT_h}\pi_k\sqrt{\frac{2\log p}{n_k}}=O_P\left(\sqrt{\frac{\log p}{n_0+n_{\mcT_h}}}\right) $$
with probability more than $1-2Kp^{-2}$.
Thus, combining Theorem \ref{theorem3.2}, it holds that
\begin{eqnarray*}
\wh \sigma_m^2-\sigma_m^2&=&\wh\bvarphi_{0,m}^\rmT \wh\bSigma_{\mcT_h}\wh\bvarphi_{0,m}-\bvarphi_{0,m}^\rmT \bSigma_0\bvarphi_{0,m}\\
&=&\wh\bvarphi_{0,m}^\rmT \wh\bSigma_{\mcT_h}\wh\bvarphi_{0,m}-\wh\bvarphi_{0,m}^\rmT \wh\bSigma_0\bvarphi_{0,m}
+\wh\bvarphi_{0,m}^\rmT \wh\bSigma_{\mcT_h}\bvarphi_{0,m}-\bvarphi_{0,m}^\rmT \wh\bSigma_{\mcT_h}\bvarphi_{0,m}
\\&&+\bvarphi_{0,m}^\rmT \wh\bSigma_{\mcT_h}\bvarphi_{0,m}-\bvarphi_{0,m}^\rmT \bSigma_0\bvarphi_{0,m}
\\&\le &\left\vert\wh\bvarphi_{0,m}\right\vert_1\left\Vert\wh\bSigma_{\mcT_h}\right\Vert_{\max}\left\vert\wh\bvarphi_{0,m}-\bvarphi_{0,m}\right\vert_1+
\left\vert\bvarphi_{0,m}\right\vert_1\left\Vert\wh\bSigma_{\mcT_h}\right\Vert_{\max}\left\vert\wh\bvarphi_{0,m}-\bvarphi_{0,m}\right\vert_1
\\&&+
\left\vert\bvarphi_{0,m}\right\vert_1^2\left\Vert\wh\bSigma_{\mcT_h}-\bSigma_0\right\Vert_{\max}\\
&\le &O_P\left(\lambda_m'\left(s_0+s_{\mcT_h,m}\right)+\lambda_m's_{\mcT_h,m}\right)+O_P\left(\sqrt{\frac{\log p}{n_0+n_{\mcT_h}}}\right)+O_P(h_{\max})\\&=&o_P(1).
\end{eqnarray*}
On the other hand,  assumption \ref{C10}, Theorems \ref{theorem3.1} and \ref{theorem3.2} imply that
\begin{eqnarray*}
&&\wh H_{\mcT_h,m\mid-m}- H^{*}_{0,m\mid-m}\\&=&\left(\wh H^{\mcT_h}_{m,m}-H^{*\mcT_h}_{m,m}\right)-\left(\wh \bH_{\mcT_h,m,-m}\wh\bgamma_{0,m}-\bH^{*}_{\mcT_h,m,-m}\wh\bgamma_{0,m}\right)
+\left(\bH^{*}_{\mcT_h,m,-m}\wh\bgamma_{0,m}-\bH^{*}_{\mcT_h,m,-m}\bgamma_{0,m}\right)\\&&+\left(H^{*}_{\mcT_h,m,m}-H^{*}_{0,m,m}\right)
+\left(\bH^*_{\mcT_h,m,-m}\bgamma_{0,m}-\bH^{*}_{0,m,-m}\bgamma_{0,m}\right)\\
&\le &O_P\left(\left\Vert\wh \bH_{\mcT_h}-\bH^{*}_{\mcT_h}\right\Vert_{\max}\right)+O_P\left(\left\Vert\wh\bgamma_{0,m}-\bgamma_{0,m}\right\Vert_{\max}\right)\\
&&+O\left(\max\left\{\left\vert H^{*}_{\mcT_h,m,m}-H^{*}_{0,m,m}\right\vert,\left\vert\left(\bH^{*}_{\mcT_h,m,-m}-\bH^{*}_{0,m,-m}\right)\bgamma_{0,m}\right\vert\right\}\right)\\
&\le &O_P\left(s_0^{1/2}\left(\frac{\log p}{n_{\mcT_h}+n_0}\right)^{1/4}+(s_0h)^{1/4}\left(\frac{\log p}{n_{\mcT_h}+n_0}\right)^{1/8}+h^{1/2}\right)
\\&&+O_P\left(\lambda_m'\left(s_0+s_{\mcT_h,m}\right)+\lambda_m's_{\mcT_h,m}\right)+O(h_1)\\
&=&o_P(1).
\end{eqnarray*}
As a result,  Theorem \ref{theorem3.3} follows immediately from the Slutsky theorem.
\end{proof}

\section{Detection Consistency}
\label{appendixC}
\setcounter{equation}{0}
\def\theequation{C.\arabic{equation}}

\subsection{Weak Consistency of Detection}
Denote $\bbeta_k^*$ as the pooled version of $\bbeta_k$ and $\bbeta_0$, which is the solution of equation
\begin{eqnarray*}
E\left\{\frac{2n_k}{2n_k+n_0}\bS(\bbeta_k^*;\{k\})
+\frac{n_0}{2n_k+n_0}\bS(\bbeta_k^*;\{0\})\right\}={\bf 0}.
\end{eqnarray*}
We need some assumptions to recover the $h$-transferable set $\mcT_h$ with high probability.
\begin{assumption}\label{C12}
(Identifiability condition of $\mcT_h$)
Denote $\mcT^c_h=\{1,\ldots,K\}\backslash \mcT_h$.
Assume that there exists a positive constant $\underline{\lambda}$ such that
\begin{eqnarray*}
\inf_{k\in\mcT_h^c}\lambda_{\min}\left(E\left[\int^1_0
f_k(\bZ_{k}^{\rmT}\bbeta_k^*+t\bZ_{k}^{\rmT}\left(\bbeta_k-\bbeta_k^*)\mid \bZ_{k}\right)\rmd t\bZ_{k}\bZ_{k}^{\rmT}\right]\right)=\underline{\lambda}>0.
\end{eqnarray*}
\end{assumption}

\begin{assumption}\label{C13}
For any $k\in\{1,\ldots,K\}$ and for some $s'$ and $h'>0$, it holds
\begin{eqnarray*}
\vert\bbeta_k\vert_0\le s'
\end{eqnarray*}
and
there exist index sets $\mcS_{k}\subset\{1,\ldots,p\},\ k=1,\ldots,K$, such that
$\left\vert\mcS_{k}\right\vert\le s'$
and $\left\vert\bbeta^*_{k,\mcS_{k}^c}\right\vert_1\le h'$ with $h'=o(1)$.
\end{assumption}

Assumption \ref{C12} guarantees that for $k\in \mcT_h^c$ there exists a uniformly sufficient large gap  between $\bbeta_k$ and $\bbeta_0$.
The gap between the parameters can be transmitted to the loss function, thereby identifying whether the source datasets are
in the $h$-transferable set $\mcT_h$.
Assumption \ref{C13} imposes a sparsity condition on $\bbeta_k$ and a weak sparsity condition on $\bbeta_k^*$ for
$k\in\mcT_h^c$, which is analogous to assumption \ref{C5}.

\begin{theorem}
\label{theorem4.1} {\rm (Weak consistency of detection)}
Under assumptions \ref{C1}--\ref{C3} and \ref{C12}--\ref{C13}, for any $k\in \mcT_h^c$ with $\mcS=\mcS_k\cup \mcS_0$ and $a_n=4h'$ in assumption \ref{C3},
if
$ \log p/n_0=o(1) $
and
\begin{eqnarray}
\label{cond_h_epsilon0}
\max\left\{\frac{(s_0+s')\log p}{n_0},h'\sqrt{\frac{\log p}{n_0}},\varepsilon_0\right\}
\lesssim\min\left\{\frac{n_k^2h^2}{(n_0^2+n_k^2)(s_0+s')},\frac{n_kh}{(n_0+n_k)\sqrt{s_0+s'}}\right\},
\end{eqnarray}
it holds that
\begin{eqnarray*}
P\left(\wh\mcT\subset\mcT_h\right)\to 1,
\end{eqnarray*}
when $\lambda_0=O\left(\sqrt{\log p/n_0}\right)$ and $\lambda_{\bbeta,k}=4B\sqrt{2\log p/(2n_k+n_0)}$.
\end{theorem}

It is evident that the detected transferable sources are
contained in the truly transferable ones with high probability,
revealing that the negative transfer can be effectively circumvented.
Furthermore,
a necessary condition to ensure $P(\wh\mcT\subset\mcT_h)\to 1$
is that
$$\frac{(s'+s_0)(n_k+n_0)}{n_k}\sqrt{\frac{\log p}{n_0}}\lesssim h.$$
Therefore, when $s'\asymp s_0$ and $n_0\lesssim n_k$, $h$ can attain the  rate  $s_0\sqrt{\log p/n_0}$,
which concludes that the convergence rate of the proposed transfer learning estimator $\wh{\bbeta}_0$ is not worse than the classical
lasso estimator.
The reason lies in the compromise of replacing the true parameter $\bbeta_0$ by its classical lasso estimator
$\wh\bbeta_{\rm lasso}$ to measure the distance between the candidate source datasets and the target.
Actually, restriction (\ref{cond_h_epsilon0}) entails that $h$ can not be too small from the perspective of  identification.
Otherwise, $h$ would be less than the rate of $\wh\bbeta_{\rm lasso}$ converging to $\bbeta_0$ and thus the transferable set is unidentifiable.

\subsection{Strong Consistency of Detection}
Furthermore, the strong consistency of detection can be achieved under additional assumptions.
\begin{assumption}\label{C14}
For any $k\in\mcT_h$, it holds that
\begin{eqnarray*}
\vert\bbeta_k^*-\bbeta_0\vert_2\lesssim \frac{h}{\sqrt{s'+s_0}}.
\end{eqnarray*}
\end{assumption}

\begin{assumption}\label{C15}
There exists a positive constant $C^*$ such that
$\sup_{k\in\mcT_h}\Vert\wt\bSigma_{k}^{*-1}\wt\bSigma_k\Vert_1\le C^*$,
where
\begin{eqnarray*}
\wt \bSigma^*_{k}&=&\frac{2n_k}{2n_k+n_0}E\left[\int^1_0
f_k\left(\bZ_{k}^{\rmT}\bbeta_0+t\bZ_{k}^{\rmT}(\bbeta^{*}_{k}-\bbeta_0)\mid \bZ_{k}\right)\rmd t\bZ_{k}\bZ_{k}^{\rmT}\right]\\
&&+\frac{n_0}{2n_k+n_0}E\left[\int^1_0
f_0\left(\bZ_{0}^{\rmT}\bbeta_0+t\bZ_{0}^{\rmT}(\bbeta^{*}_{k}-\bbeta_0)\mid \bZ_{0}\right)\rmd t\bZ_{0}\bZ_{0}^{\rmT}\right].
\end{eqnarray*}
\end{assumption}

Assumption \ref{C15} is similar to assumption \ref{C4},
under which we have the theorem below.

\begin{theorem}
\label{theorem4.2}
{\rm (Strong consistency of detection)} Under conditions in Theorem \ref{theorem4.1} and assumptions \ref{C14}--\ref{C15},
if
\begin{eqnarray}
\label{cond_h_epsilon0*}
\max\left\{h'\sqrt{\frac{\log p}{n_0}},\frac{s_0\log p}{n_0},\max_{k\in\mcT_h}\left\vert\bbeta_k^*-\bbeta_0\right\vert_2^2\right\}\lesssim \varepsilon_0,
\end{eqnarray}
then it holds
\begin{eqnarray*}
P\left(\wh\mcT=\mcT_h\right)\to 1.
\end{eqnarray*}
\end{theorem}

In fact, (\ref{cond_h_epsilon0}) also requires a proper rate of $\varepsilon_0$ in $\wh\mcT$
to control the false positive sources while (\ref{cond_h_epsilon0*}) to control the false negative sources.
Actually, when these two conditions hold simultaneously,
Theorem \ref{theorem4.2} implies that $\wh\mcT$ can exactly  recover $\mcT_h$ with probability tending to one. Hence, the proposed (double) transfer learning
procedure and corresponding inference procedure can be applicable in practice when the transferable sources are unknown.

\subsection{Proof of Weak Consistency}

Denote $l_0(\bbeta)=E\{\rho_{\tau}(Y_0-\bZ_0^{\rmT}\bbeta)\}$,
then for any $\mathcal{A}\subset\{1,\ldots,n_0\}$,
$E\{L_0(\bbeta;\mathcal{A})\}=l_0(\bbeta)$ based on the i.i.d. observations.

\begin{lemma}
\label{lemmaC1}
Under assumption \ref{C2}, it holds that
\begin{eqnarray}
\label{excess-loss}
\left\vert\left\{ L_0\left(\mathbf{v}_1, {\cal V}_{\rm te}\right)- L_0\left(\mathbf{v}_2, {\cal V}_{\rm te}\right)\right\}-\left\{l_0(\mathbf{v}_1)-l_0(\mathbf{v}_2)\right\}\right\vert\le 4B\sqrt{\frac{\log p}{n_0}}\vert\mathbf{v}_1-\mathbf{v}_2\vert_1
\end{eqnarray}
for any  $\mathbf{v}_1,\mathbf{v}_2\in\mathbb{R}^{p}$ with probability at least $1-4p^{-1}$.
\end{lemma}

\begin{proof}
By the Knight equation, note that
\begin{eqnarray*}
&& L_0(\mathbf{v}_1, {\cal V}_{\rm te})- L_0(\bbeta_0, {\cal V}_{\rm te})
\\&=&\frac2{n_0}\sum_{i\in {\cal V}_{\rm te}}\left\{\rho_\tau\left(Y_{0i}-\bZ_{0i}^\rmT\mathbf{v}_1\right)-\rho_\tau\left(Y_{0i}-\bZ_{0i}^\rmT\bbeta_0\right)\right\}\\
&=&\frac2{n_0}\sum_{i\in {\cal V}_{\rm te}}\left[-\bZ_{0i}^\rmT(\mathbf{v}_1-\bbeta_0)\left\{\tau-I\left(Y_{0i}-\bZ_{0i}^\rmT\bbeta_0\le 0\right)\right\}\right.\\
&&\left.+\int^{\mathbf{Z}_{0i}^\rmT(\mathbf{v}_1-\bbeta_0)}_0\left\{I\left(Y_{0i}-\bZ_{0i}^\rmT\bbeta_0\le z\right)-I\left(Y_{0i}-\bZ_{0i}^\rmT\bbeta_0\le 0\right)\right\}\rmd z\right].
\end{eqnarray*}
As a result, the same expansion of $ L_0(\mathbf{v}_2, {\cal V}_{\rm te})- L_0(\bbeta_0, {\cal V}_{\rm te})$ entails that
\begin{eqnarray*}
&& L_0(\mathbf{v}_1, {\cal V}_{\rm te})- L_0(\mathbf{v}_2, {\cal V}_{\rm te})
\\&=&\frac2{n_0}\sum_{i\in {\cal V}_{\rm te}}\bigg[-\bZ_{0i}^\rmT(\mathbf{v}_1-\mathbf{v}_2)\left\{\tau-I\left(Y_{0i}-\bZ_{0i}^\rmT\bbeta_0\le 0\right)\right\}\\
&&+\int^{\mathbf{Z}_{0i}^\rmT(\mathbf{v}_1-\bbeta_0)}_{\mathbf{Z}_{0i}^\rmT(\mathbf{v}_2-\bbeta_0)}\left\{I\left(Y_{0i}-\bZ_{0i}^\rmT\bbeta_0\le z\right)-I\left(Y_{0i}-\bZ_{0i}^\rmT\bbeta_0\le 0\right)\right\}\rmd z\bigg].
\end{eqnarray*}
Under assumption \ref{C2}, it holds that
\begin{eqnarray*}
\left\vert-\bZ_{0i}^\rmT(\mathbf{v}_1-\mathbf{v}_2)\left\{\tau-I\left(Y_{0i}-\bZ_{0i}^\rmT\bbeta_0\le 0\right)\right\}\right\vert\le B\vert\mathbf{v}_1-\mathbf{v}_2\vert_1
\end{eqnarray*}
and
\begin{eqnarray*}
\left\vert\int^{\mathbf{Z}_{0i}^\rmT(\mathbf{v}_1-\bbeta_0)}_{\mathbf{Z}_{0i}^\rmT(\mathbf{v}_2-\bbeta_0)}\left\{I\left(Y_{0i}-\bZ_{0i}^\rmT\bbeta_0\le z\right)-I\left(Y_{0i}-\bZ_{0i}^\rmT\bbeta_0\le 0\right)\right\}\rmd z\right\vert\le B\vert\mathbf{v}_1-\mathbf{v}_2\vert_1.
\end{eqnarray*}
Exploiting the Hoeffding inequality, we have
\begin{eqnarray*}
&&P\left(\left\vert L_0(\mathbf{v}_1, {\cal V}_{\rm te})- L_0(\mathbf{v}_2, {\cal V}_{\rm te})-\left\{ l_0(\mathbf{v}_1)-l_0(\mathbf{v}_2)\right\}\right\vert>t\right)\\
&\le &P\left(\left\vert-\frac2{n_0}\sum_{i\in {\cal V}_{\rm te}}\bZ_{0i}^\rmT(\mathbf{v}_1-\mathbf{v}_2)\left\{\tau-I\left(Y_{0i}-\bZ_{0i}^\rmT\bbeta_0\le 0\right)\right\}-0\right\vert>t/2\right)\\
&&+P\left(\left\vert\frac2{n_0}\sum_{i\in {\cal V}_{\rm te}}\int^{\mathbf{Z}_{0i}^\rmT(\mathbf{v}_1-\bbeta_0)}_{\mathbf{Z}_{0i}^\rmT(\mathbf{v}_2-\bbeta_0)}\left\{I\left(Y_{0i}-\bZ_{0i}^\rmT\bbeta_0\le z\right)-I\left(Y_{0i}-\bZ_{0i}^\rmT\bbeta_0\le 0\right)\right\}\rmd z\right.\right.
\\&&\hspace{.3in}\left.\left.-E\left\{\int^{\mathbf{Z}_{0}^\rmT(\mathbf{v}_1-\bbeta_0)}_{\mathbf{Z}_{0}^\rmT(\mathbf{v}_2-\bbeta_0)}\left\{I\left(Y_{0}-\bZ_{0}^\rmT\bbeta_0\le z\right)-I\left(Y_{0}-\bZ_{0}^\rmT\bbeta_0\le 0\right)\right\}\rmd z\right\}\right\vert>t/2\right)\\
&\le &4\exp\left(-\frac{n_0t^2}{16B^2\vert\mathbf{v}_1-\mathbf{v}_2\vert_1^2}\right).
\end{eqnarray*}
Take $t=4B\sqrt{\log p/n_0}\vert\mathbf{v}_1-\mathbf{v}_2\vert_1$
and (\ref{excess-loss}) is concluded.
\end{proof}

\begin{lemma}
\label{lemmaC2}
Under assumptions \ref{C1}--\ref{C3}, for any $\bbeta\in\mathbb{R}^p$ if $\bbeta-\bbeta_0\in\mathcal{D}(a_n,\mcS)$, it holds that
\begin{eqnarray*}
\min\left\{\frac{\kappa_0}4\vert\bbeta-\bbeta_0\vert_2^2,\frac{3\kappa_0^3}{16m_0}\vert\bbeta-\bbeta_0\vert_2\right\}\le l_0(\bbeta)-l_0(\bbeta_0)\le \frac{m_0 B^2}2\vert\bbeta-\bbeta_0\vert_2^2.
\end{eqnarray*}
\end{lemma}

\begin{proof}
Under assumptions \ref{C1} and \ref{C3}, following the similar procedures to show (\ref{Hessian-lower-bound}) in Lemma \ref{lemmaA2}, we can  conclude that
\begin{eqnarray*}
l_0(\bbeta)-l_0(\bbeta_0)\ge\min\left\{\frac{\kappa_0}4\vert\bbeta-\bbeta_0\vert_2^2,\frac{3\kappa_0^3}{16m_0}\vert\bbeta-\bbeta_0\vert_2\right\},
\end{eqnarray*}
for any $\bbeta-\bbeta_0\in\mathcal{D}(a_n,\mcS)$.
On the other hand, under assumptions \ref{C1} and \ref{C2}, the Knight equation implies that
\begin{eqnarray*}
l_0(\bbeta)-l_0(\bbeta_0)&=&E\left\{\rho_\tau\left(Y_{0}-\bZ_{0}^\rmT\bbeta\right)-\rho_\tau\left(Y_{0}-\bZ_{0}^\rmT\bbeta_0\right)\right\}\\
&=&E\left\{\int^{\mathbf{Z}_{0}^\rmT(\bbeta-\bbeta_0)}_0\left\{I\left(Y_{0}-\bZ_{0}^\rmT\bbeta_0\le z\right)-I\left(Y_{0}-\bZ_{0}^\rmT\bbeta_0\le 0\right)\right\}\rmd z\right\}\\
&=&E\left\{\int^{\mathbf{Z}_{0}^\rmT(\bbeta-\bbeta_0)}_0\left\{F_0\left(\bZ_{0}^\rmT\bbeta_0+ z\mid \bZ_{0}\right)
-F_0\left(\bZ_{0}^\rmT\bbeta_0\mid \bZ_{0}\right)\right\}\rmd z\right\}\\
&\le &E\left\{\int^{\mathbf{Z}_{0}^\rmT(\bbeta-\bbeta_0)}_0m_0z\rmd z\right\}\\
&\le &\frac{m_0 B^2}2\vert\bbeta-\bbeta_0\vert_2^2,
\end{eqnarray*}
which concludes Lemma \ref{lemmaC2}.
\end{proof}

\begin{lemma}
\label{lemmaC3}
Under assumptions \ref{C1}--\ref{C3}, and \ref{C12}--\ref{C13}, for any $k\in\mcT_h^c$, it holds that
\begin{eqnarray*}
\vert \bbeta_k^*-\bbeta_0 \vert_2\ge \frac{2n_k\underline{\lambda}h}{(n_0m_0B^2+2n_k\underline{\lambda})\sqrt{s_0+s^{'}}}.
\end{eqnarray*}
\end{lemma}

\begin{proof}
Recall that $\bbeta_k^*$ is the solution of the equation
\begin{eqnarray}
\label{beta_k}
\frac{n_0}2E\left(\left\{\tau-I\left(Y_{0}-\bZ_{0}^\rmT\bbeta\le 0\right)\right\}\bZ_{0}\right)+n_k E\left(\left\{\tau-I\left(Y_{k}-\bZ_{k}^{\rmT}\bbeta\le 0\right)\right\}\bZ_{k}\right)={\bf 0}.
\end{eqnarray}
Since $\tau=F_0\left(\bZ_{0}^\rmT\bbeta_0\mid \bZ_{0}\right)=F_k\left(\bZ_{k}^{\rmT}\bbeta_k\mid \bZ_{k}\right)$, then
(\ref{beta_k}) can be written as
\begin{eqnarray*}
&&\frac{n_0}2E\left(\left\{F_0\left(\bZ_{0}^\rmT\bbeta_0\mid\bZ_{0}\right)-F_0\left(\bZ_{0}^\rmT\bbeta_k^*\mid \bZ_{0}\right)\right\}\bZ_{0}\right)
\\&&+n_k E\left(\left\{F_k\left(\bZ_{k}^{\rmT}\bbeta_k\mid \bZ_{k}\right)-F_k\left(\bZ_{k}^{\rmT}\bbeta_k^*\mid \bZ_{k}\right)\right\}\bZ_{k}\right)\\&=&{\bf 0}.
\end{eqnarray*}
The Newton--Leibniz formula entails that
\begin{eqnarray*}
&&\frac{n_0}2E\left[\int^1_0f_0\left(\bZ_{0}^\rmT\bbeta_0+t\bZ_{0}^\rmT\left(\bbeta_k^*-\bbeta_0\right)\mid \bZ_{0}\right)\rmd t\bZ_{0}\bZ_{0}^\rmT\right]\left(\bbeta_k^*-\bbeta_0\right)\\
&=&n_kE\left[\int^1_0f_k\left(\bZ_{k}^{\rmT}\bbeta_k^*+t\bZ_{k}^{\rmT}\left(\bbeta_k-\bbeta_k^*\right)\mid \bZ_{k}\right)\rmd t\bZ_{k}\bZ_{k}^{\rmT}\right]\left(\bbeta_k-\bbeta_k^*\right).
\end{eqnarray*}
Therefore, under assumption \ref{C12}, we have
\begin{eqnarray*}
&&n_k\underline{\lambda}\left\vert\bbeta_k-\bbeta_k^*\right\vert_2\\&\le&n_k\left\vert E\left[\int^1_0f_k\left(\bZ_{k}^{\rmT}\bbeta_k^*+t\bZ_{k}^{\rmT}\left(\bbeta_k-\bbeta_k^*\right)\mid \bZ_{k}\right)\rmd t\bZ_{k}\bZ_{k}^{\rmT}\right]\left(\bbeta_k-\bbeta_k^*\right)\right\vert_2\\
&= &\left\vert \frac{n_0}2E\left[\int^1_0f_0\left(\bZ_{0}^\rmT\bbeta_0+t\bZ_{0}^\rmT\left(\bbeta_k^*-\bbeta_0\right)\mid\bZ_{0}\right)\rmd t\bZ_{0}\bZ_{0}^\rmT\right]\left(\bbeta_k^*-\bbeta_0\right)\right\vert_2\\
&\le&\frac{n_0}2\lambda_{\max}\left(E\left[\int^1_0f_0\left(\bZ_{0}^\rmT\bbeta_0+t\bZ_{0}^\rmT\left(\bbeta_k^*-\bbeta_0\right)\mid \bZ_{0}\right)\rmd t\bZ_{0}\bZ_{0}^\rmT\right]\right)\left\vert\bbeta_k^*-\bbeta_0\right\vert_2\\
&\le &\frac{n_0}2\left\Vert E\left[\int^1_0f_0\left(\bZ_{0}^\rmT\bbeta_0+t\bZ_{0}^\rmT\left(\bbeta_k^*-\bbeta_0\right)\mid \bZ_{0}\right)\rmd t\bZ_{0}\bZ_{0}^\rmT\right]\right\Vert_{\max}\left\vert\bbeta_k^*-\bbeta_0\right\vert_2\\
&\le &\frac{n_0}2m_0B^2\left\vert\bbeta_k^*-\bbeta_0\right\vert_2.
\end{eqnarray*}
Thus,
\begin{eqnarray*}
\left\vert\bbeta_k-\bbeta_0\right\vert_2\le \left\vert\bbeta_k-\bbeta_k^*\right\vert_2+\left\vert\bbeta_k^*-\bbeta_0\right\vert_2
\le \frac{n_0m_0B^2+2n_k\underline{\lambda}}{2n_k\underline{\lambda}}\left\vert\bbeta_k^*-\bbeta_0\right\vert_2.
\end{eqnarray*}
Moreover, for any $k\in\mcT_h^c$, it holds that
\begin{eqnarray*}
\left\vert\bbeta_k^*-\bbeta_0\right\vert_2\ge\frac{2n_k\underline{\lambda}}{n_0m_0B^2+2n_k\underline{\lambda}}\left\vert\bbeta_k-\bbeta_0\right\vert_2\ge
\frac{2n_k\underline{\lambda}h}{(n_0m_0B^2+2n_k\underline{\lambda})\sqrt{s_0+s^{'}}}.
\end{eqnarray*}
\end{proof}

{\bf Proof of Theorem \ref{theorem4.1}}
\begin{proof}
Make the decomposition of $ L_0\left(\wh\bbeta_{{\cal V}_{\rm tr},k},{\cal V}_{\rm te}\right)- L_0\left(\wh\bbeta_{\rm lasso},{\cal V}_{\rm te}\right)$ as
\begin{eqnarray*}
\label{decomposition}
&& L_0\left(\wh\bbeta_{{\cal V}_{\rm tr},k},{\cal V}_{\rm te}\right)- L_0\left(\wh\bbeta_{\rm lasso},{\cal V}_{\rm te}\right)\nonumber\\
&=&\left\{ L_0\left(\wh\bbeta_{{\cal V}_{\rm tr},k},{\cal V}_{\rm te}\right)\!-\! L_0(\bbeta_k^*,{\cal V}_{\rm te})\!-\!l_0\left(\wh\bbeta_{{\cal V}_{\rm tr},k}\right)\!+\!l_0(\bbeta_k^*)\right\}+\left\{l_0\left(\wh\bbeta_{{\cal V}_{\rm tr},k}\right)-l_0\left(\bbeta_k^*\right)\right\}\nonumber\\&&
-\left\{ L_0\left(\wh\bbeta_{\rm lasso},{\cal V}_{\rm te}\right)\!-\! L_0(\bbeta_0,{\cal V}_{\rm te})\!-\!l_0\left(\wh\bbeta_{\rm lasso}\right)\!+\!l_0(\bbeta_0)\right\}-\left\{l_0\left(\wh\bbeta_{\rm lasso}\right)-l_0\left(\bbeta_0\right)\right\}\nonumber\\&&
+\left\{ L_0\left(\bbeta_k^*,{\cal V}_{\rm te}\right)\!-\! L_0(\bbeta_0,{\cal V}_{\rm te})\!-\!l_0\left(\bbeta_k^*\right)\!+\!l_0(\bbeta_0)\right\}+\left\{l_0\left(\bbeta_k^*\right)-l_0\left(\bbeta_0\right)\right\}\nonumber\\
&=&(I_1)+(I_2)-(I_3)-(I_4)+(I_5)+(I_6),
\end{eqnarray*}
where terms $(I_1)$--$(I_6)$ are self-explained.
We plan to show that the main gap is due to $(I_6)$, and the other terms are  ignorable compared with $(I_6)$.
Lemma \ref{lemmaC1} entails that
\begin{eqnarray*}
P\left(\vert(I_1)\vert>4B\sqrt{\frac{\log p}{n_0}}\left\vert\wh\bbeta_{{\cal V}_{\rm tr},k}-\bbeta_k^*\right\vert_1\right)\le4p^{-1},
\end{eqnarray*}
\begin{eqnarray*}
P\left(\vert(I_3)\vert>4B\sqrt{\frac{\log p}{n_0}}\left\vert\wh\bbeta_{\rm lasso}-\bbeta_0\right\vert_1\right)\le4p^{-1},
\end{eqnarray*}
and
\begin{eqnarray*}
P\left(\vert(I_5)\vert>4B\sqrt{\frac{\log p}{n_0}}\left\vert\bbeta_k^*-\bbeta_0\right\vert_1\right)\le4p^{-1}.
\end{eqnarray*}
On the other hand, based on the fact that $\bbeta_k^*-\bbeta_0\in\mathcal{D}(h',\mcS_k\cup\mcS_0)$, Lemma \ref{lemmaC2} implies that
\begin{eqnarray*}
\vert(I_4)\vert\le \frac{m_0B^2}2\left\vert\wh\bbeta_{\rm lasso}-\bbeta_0\right\vert_2^2
\end{eqnarray*}
and
\begin{eqnarray*}
\vert(I_6)\vert\ge \min\left\{\frac{\kappa_0}4\left\vert\bbeta_k^*-\bbeta_0\right\vert_2^2,\frac{3\kappa_0^3}{16m_0}\left\vert\bbeta_k^*-\bbeta_0\right\vert_2\right\}.
\end{eqnarray*}
Furthermore, to bound $(I_2)$, by the Knight equation we have
\begin{eqnarray*}
\vert(I_2)\vert&=&\left\vert\frac2{n_0}\sum_{i\in {\cal V}_{\rm te}}E\left\{\rho_\tau\left(Y_{0}-\bZ_{0}^\rmT\wh\bbeta_{{\cal V}_{\rm tr},k}\right)-\rho_\tau\left(Y_{0}-\bZ_{0}^\rmT\bbeta_k^*\right)\right\}\right\vert\\
&=&\left\vert\frac2{n_0}\sum_{i\in {\cal V}_{\rm te}}E\left[-\bZ_{0}^\rmT\left(\wh\bbeta_{{\cal V}_{\rm tr},k}-\bbeta_k^*\right)\left\{\tau-I\left(Y_{0}-\bZ_{0}^\rmT\bbeta_k^*\le 0\right)\right\}\right.\right.\\
&&\hspace{.2in}\left.\left.+\int^{\mathbf{Z}_{0}^\rmT\left(\wh\bbeta_{{\cal V}_{\rm tr},k}-\bbeta_k^*\right)}_0\left\{I\left(Y_{0}-\bZ_{0}^\rmT\bbeta_k^*\le z\right)-I\left(Y_{0}-\bZ_{0}^\rmT\bbeta_k^*\le 0\right)\right\}\rmd z\right]\right\vert\\
&\le&\left\vert E\left[-\bZ_{0}^\rmT\left(\wh\bbeta_{{\cal V}_{\rm tr},k}-\bbeta_k^*\right)\left\{F_0\left(\bZ_{0}^\rmT\bbeta_0\mid \bZ_{0}\right)-F_0\left(\bZ_{0}^\rmT\bbeta_k\mid \bZ_{0}\right)\right\}\right]\right\vert
\\
&&\hspace{.2in}+\left\vert E\left[\int^{\mathbf{Z}_{0}^\rmT\left(\wh\bbeta_{{\cal V}_{\rm tr},k}-\bbeta_k^*\right)}_0\left\{F_0\left(\bZ_{0}^\rmT\bbeta_k^*+z\mid \bZ_{0}\right)-F_0\left(\bZ_{0}^\rmT\bbeta_k^*\mid \bZ_{0}\right)\right\}\rmd z\right]\right\vert\\
&\le&m_0B^2\left\vert\wh\bbeta_{{\cal V}_{\rm tr},k}-\bbeta_k^*\right\vert_2\left\vert\bbeta_k^*-\bbeta_0\right\vert_2+\frac{m_0B^2}2\left\vert\wh\bbeta_{{\cal V}_{\rm tr},k}-\bbeta_k^*\right\vert_2^2.
\end{eqnarray*}
Thus, we need to evaluate the rates of $\left\vert\wh\bbeta_{{\cal V}_{\rm tr},k}-\bbeta_k^*\right\vert_2$ and $\left\vert\wh\bbeta_{\rm lasso}-\bbeta_0\right\vert_2$.
Denote
$$ \wt L_{{\cal V}_{\rm tr},k}(\bbeta)=\frac{2n_k}{2n_k+n_0} L(\bbeta;\{k\})+\frac{n_0}{2n_k+n_0} L_0(\bbeta;{\cal V}_{\rm tr}) $$
and
$\wt \bS_{{\cal V}_{\rm tr},k}(\bbeta)=\nabla\wt L_{{\cal V}_{\rm tr},k}(\bbeta)$.
The Knight equation implies that
\begin{eqnarray*}
\wt L_{{\cal V}_{\rm tr},k}\left(\wh\bbeta_{{\cal V}_{\rm tr},k}\right)-\wt L_{{\cal V}_{\rm tr},k}\left(\bbeta_k^*\right)-\wt \bS_{{\cal V}_{\rm tr},k}\left(\bbeta_k^*\right)^\rmT\left(\wh\bbeta_{{\cal V}_{\rm tr},k}-\bbeta_k^*\right)\ge 0.
\end{eqnarray*}
Therefore, by denoting $\wh\bu_k=\wh\bbeta_{{\cal V}_{\rm tr},k}-\bbeta_k^*$, we have
\begin{eqnarray*}
0&\le& \wt L_{{\cal V}_{\rm tr},k}\left(\wh\bbeta_{{\cal V}_{\rm tr},k}\right)-\wt L_{{\cal V}_{\rm tr},k}\left(\bbeta_k^*\right)
-\wt \bS_{{\cal V}_{\rm tr},k}\left(\bbeta_k^*\right)^\rmT\left(\wh\bbeta_{{\cal V}_{\rm tr},k}-\bbeta_k^*\right)\\
&\le&-\lambda_{\bbeta,k}\left\vert\wh\bbeta_{{\cal V}_{\rm tr},k}\right\vert_1+\lambda_{\bbeta,k}\left\vert\bbeta_k^*\right\vert_1
+\frac12\lambda_{\bbeta,k}\left\vert\wh\bbeta_{{\cal V}_{\rm tr},k}-\bbeta_k^*\right\vert_1\\
&=&-\lambda_{\bbeta,k}\left(\left\vert\wh\bbeta_{{\cal V}_{\rm tr},k,\mcS_k}\right\vert_1+\left\vert\wh\bbeta_{{\cal V}_{\rm tr},k,\mcS_k^c}\right\vert_1\right)
+\lambda_{\bbeta,k}\left(\left\vert\bbeta_{k,\mcS_k}^*\right\vert_1+\left\vert\bbeta^*_{k,\mcS_k^c}\right\vert_1\right)
+\frac12\lambda_{\bbeta,k}\left\vert\wh\bbeta_{{\cal V}_{\rm tr},k}-\bbeta_k^*\right\vert_1\\
&\le&\lambda_{\bbeta,k}\left\vert\wh\bbeta_{{\cal V}_{\rm tr},k,\mcS_k}-\bbeta_{k,\mcS_k}^*\right\vert_1-\lambda_{\bbeta,k}\left\vert\wh\bbeta_{{\cal V}_{\rm tr},k,\mcS_k^c}
-\bbeta_{k,\mcS_k^c}^*\right\vert_1+\frac12\lambda_{\bbeta,k}\left\vert\wh\bbeta_{{\cal V}_{\rm tr},k}-\bbeta_k^*\right\vert_1
+2\lambda_{\bbeta,k}\left\vert\bbeta_{k,\mcS_k^c}^*\right\vert_1\\
&\le&\frac32\lambda_{\bbeta,k}\left\vert\wh\bu_{k,\mcS_k}\right\vert_1-\frac12\lambda_{\bbeta,k}\left\vert\wh\bu_{k,\mcS_k^c}\right\vert_1
+2\lambda_{\bbeta,k} h',
\end{eqnarray*}
where the second inequality is a parallel conclusion of Lemma \ref{lemmaA3} that
\begin{eqnarray*}
\label{grad**}
\left\vert \wt \bS_{{\cal V}_{\rm tr},k}\left(\bbeta_k^*\right)\right\vert_\infty\le \frac12\lambda_{\bbeta,k}
\end{eqnarray*}
with probability at least $1-2p^{-1}$ when $\lambda_{\bbeta,k}=4B\sqrt{2\log p/(2n_k+n_0)}$.
As a result, $\wh\bu_k\in\mathcal{D}(4h',\mcS_k)$ with probability more than $1-2p^{-1}$ as
\begin{eqnarray*}
\left\vert\wh\bu_{k,\mcS_k^c}\right\vert_1\le 3\left\vert\wh\bu_{k,\mcS_k}\right\vert_1+4h'.
\end{eqnarray*}
Set
\begin{eqnarray*}
\wt G_k(\bu)=\wt L_{{\cal V}_{\rm tr},k}\left(\bbeta_k^*+\bu\right)-\wt L_{{\cal V}_{\rm tr},k}\left(\bbeta_k^*\right)
+\lambda_{\bbeta,k}\left(\left\vert\bbeta_k^*+\bu\right\vert_1-\left\vert\bbeta_k^*\right\vert_1\right).
\end{eqnarray*}
Suppose that
\begin{eqnarray*}
\label{suppose_k}
\left\vert\wh\bu_k\right\vert_2> \frac{40B}{\kappa_0}\sqrt{\frac{2\log p}{2n_k+n_0}}+4\sqrt{\frac{3Bh'}{\kappa_0}}\left(\frac{2\log p}{2n_k+n_0}\right)^{1/4}.
\end{eqnarray*}
Then there exists $t\in[0,1]$ such that
$\wt\bu_k=t\wh\bu_k$ with
$$\left\vert\wt\bu_k\right\vert_2= \frac{40B}{\kappa_0}\sqrt{\frac{2\log p}{2n_k+n_0}}+4\sqrt{\frac{3Bh'}{\kappa_0}}\left(\frac{2\log p}{2n_k+n_0}\right)^{1/4}. $$
Note that $\wt\bu_k\in\mathcal{D}(4h',\mcS_k)$. Along with the fact that $\wt G_k\left(\wh\bu_k\right)\le0$, the convexity of $\wt G_k(\cdot)$ implies that
$\wt G_k\left(\wt\bu_k\right)\le 0$.
As a consequence, when $n_k,n_0,p$ are sufficiently large, a parallel conclusion of Lemma \ref{lemmaA2} yields with probability at least $1-p^{-2}$ that
\begin{eqnarray*}
\wt G_k\left(\wt\bu_k\right)
&\ge& \min\left\{\frac{\kappa_0}4\left\vert\wt\bu_k\right\vert_2^2,\frac{3\kappa_0^3}{16m_0}\left\vert\wt\bu_k\right\vert_2\right\}\!-\!B\sqrt{\frac{2\log p}{2n_k+n_0}}\left\vert\wt\bu_k\right\vert_1
\\&&-\frac32\lambda_{\bbeta,k}\left\vert\wt\bu_{k,\mcS_k}\right\vert_1
+\frac12\lambda_{\bbeta,k}\left\vert\wt\bu_{k,\mcS_k^c}\right\vert_1-2\lambda_{\bbeta,k} h'\\
&\ge &\frac{\kappa_0}4\left\vert\wt\bu_k\right\vert_2^2-B\sqrt{\frac{2\log p}{2n_k+n_0}}\left(4\left\vert\wt\bu_{k,\mcS_k}\right\vert_1+4h'\right)
\\&&-6B\sqrt{\frac{2\log p}{2n_k+n_0}}\left\vert\wt\bu_{k,\mcS_k}\right\vert_1-8B\sqrt{\frac{2\log p}{2n_k+n_0}} h'\\
&\ge &\frac{\kappa_0}4\left\vert\wt\bu_k\right\vert_2^2-10B\sqrt{\frac{2s'\log p}{2n_k+n_0}}\left\vert\wt\bu_k\right\vert_2-12B\sqrt{\frac{2\log p}{2n_k+n_0}} h'
\\&>&0,
\end{eqnarray*}
which results in the contradiction.
Therefore, we have
\begin{eqnarray*}
\left\vert\wh\bu_k\right\vert_2\le \frac{40B}{\kappa_0}\sqrt{\frac{2s'\log p}{2n_k+n_0}}+4\sqrt{\frac{3Bh'}{\kappa_0}}\left(\frac{2\log p}{2n_k+n_0}\right)^{1/4}
\end{eqnarray*}
and
\begin{eqnarray*}
\left\vert\wh\bu_k\right\vert_1
\le4\sqrt{s'}\left\vert\wh\bu_k\right\vert_2+4h'\lesssim s'\sqrt{\frac{\log p}{n_k+n_0}}+h'.
\end{eqnarray*}
\citet{belloni2011penalized} derived that there exist constants $c_1, c_2>0$ such that
\begin{eqnarray}
\label{HD-QR}
\left\vert\wh\bbeta_{\rm lasso}-\bbeta_0\right\vert_1\le c_1s_0\sqrt{\frac{\log p}{n_0}}
\end{eqnarray}
with probability more than $1-c_2p^{-2}$.
Therefore, combining the pieces, for any $k\in \mcT_h^c$, there exists a constant $c>0$ such that
\begin{eqnarray*}
&&P\left( L_0\left(\wh\bbeta_{{\cal V}_{\rm tr},k},{\cal V}_{\rm te}\right)- L_0\left(\wh\bbeta_{\rm lasso},{\cal V}_{\rm te}\right)\le \varepsilon_0 L_0\left(\wh\bbeta_{\rm lasso},{\cal V}_{\rm te}\right)\right)\\
&\le& P\left(\varepsilon_0  L_0\left(\wh\bbeta_{\rm lasso},{\cal V}_{\rm te}\right)\ge \vert(I_6)\vert-\vert(I_1)\vert-\vert(I_2)\vert-\vert(I_3)\vert-\vert(I_4)\vert-\vert(I_5)\vert
\right)\\
&\le &P\left(\varepsilon_0  L_0\left(\wh\bbeta_{\rm lasso},{\cal V}_{\rm te}\right)\ge \min\left\{\frac{\kappa_0}4\vert\bbeta_k^*-\bbeta_0\vert_2^2,\frac{3\kappa_0^3}{16m_0}\vert\bbeta_k^*-\bbeta_0\vert_2\right\}\right.\\
&&-c_3\left(\sqrt{\frac{s'\log p}{n_k+n_0}}+\sqrt{h'}\left(\frac{\log p}{n_k+n_0}\right)^{1/4}\right)\vert\bbeta_k^*-\bbeta_0\vert_2
\\&&-c_3\left(\!\sqrt{\frac{s'\log p}{n_k+n_0}}+\sqrt{h'}\left(\frac{\log p}{n_k+n_0}\right)^{1/4}\!\!\!+\!\sqrt{\frac{s_0\log p}{n_0}}\right)^2\\
&&\left.-c_3\sqrt{\frac{\log p}{n_0}}\left(s'\sqrt{\frac{\log p}{n_k+n_0}}+h'+s_0\sqrt{\frac{\log p}{n_0}}+\vert\bbeta_k^*-\bbeta_0\vert_1\right)
\right)\\
&&+\frac{1}{p^2}+\frac{14+2C_2}p.
\end{eqnarray*}
Since the rates of $h$ and $\varepsilon_0$ satisfy
$$ \max\left\{\frac{s_0\log p}{n_0},s'\sqrt{\frac{\log p}{n_0}}\sqrt{\frac{\log p}{n_k+n_0}},h'\sqrt{\frac{\log p}{n_0}},\varepsilon_0\right\}
\lesssim\min\left\{\frac{n_k^2h^2}{(n_0^2+n_k^2)(s_0+s')},\frac{n_kh}
{(n_0+n_k)\sqrt{s_0+s'}}\right\} $$
and
$$ \sqrt{\frac{(s'+s_0)\log p}{n_0}}
\lesssim\min\left\{\frac{n_kh}{(n_0+n_k)\sqrt{s_0+s'}},1\right\}$$
under conditions in Theorem \ref{theorem4.1},
we have
\begin{eqnarray*}
&&c_3\left(\sqrt{\frac{s'\log p}{n_k+n_0}}+\sqrt{h'}\left(\frac{\log p}{n_k+n_0}\right)^{1/4}\right)\vert\bbeta_k^*-\bbeta_0\vert_2
\\&&+c_3\left(\sqrt{\frac{s'\log p}{n_k+n_0}}+\sqrt{h'}\left(\frac{\log p}{n_k+n_0}\right)^{1/4}+\sqrt{\frac{s_0\log p}{n_0}}\right)^2\\
&&\left.+c_3\sqrt{\frac{\log p}{n_0}}\left(s'\sqrt{\frac{\log p}{n_k+n_0}}+h'+s_0\sqrt{\frac{\log p}{n_0}}+\vert\bbeta_k^*-\bbeta_0\vert_1\right)
\right)+\varepsilon_0  L_0\left(\wh\bbeta_{\rm lasso},{\cal V}_{\rm te}\right)\\
&\le& \min\left\{\frac{\kappa_0}8\vert\bbeta_k^*-\bbeta_0\vert_2^2,\frac{3\kappa_0^3}{32m_0}
\vert\bbeta_k^*-\bbeta_0\vert_2\right\}
\end{eqnarray*}
 when $n_k,n_0,p$ are sufficiently large.
As a  result, it holds according to Lemma \ref{lemmaC3} that
\begin{eqnarray*}
&&P\left( L_0\left(\wh\bbeta_{{\cal V}_{\rm tr},k},{\cal V}_{\rm te}\right)- L_0\left(\wh\bbeta_{\rm lasso},{\cal V}_{\rm te}\right)\le \varepsilon_0 L_0\left(\wh\bbeta_{\rm lasso},{\cal V}_{\rm te}\right)\right)\\
&\le &P\left(\varepsilon_0  L_0\left(\wh\bbeta_{\rm lasso},{\cal V}_{\rm te}\right)\ge \min\left\{\frac{\kappa_0}8\vert\bbeta_k^*-\bbeta_0\vert_2^2,
\frac{3\kappa_0^3}{32m_0}\vert\bbeta_k^*-\bbeta_0\vert_2\right\}\right)+
\frac{1}{p^2}+\frac{14+2C_2}p\\&\to& 0.
\end{eqnarray*}
\end{proof}

\subsection{Proof of Strong Consistency}

\begin{lemma}
\label{lemmaC4}
Under assumptions \ref{C1}--\ref{C3} and \ref{C15},
if
$$\max\left\{h'\sqrt{\frac{\log p}{n_0}},\frac{s_0\log p}{n_0},\max_{k\in\mcT_h}\left\vert\bbeta_k^*-\bbeta_0\right\vert_2^2\right\}\lesssim \varepsilon_0,$$
we have
\begin{eqnarray*}
P\left(\mcT_h\subset \wh\mcT\right)\to 1,
\end{eqnarray*}
when $\lambda_0=O\left(\sqrt{\log p/n_0}\right)$, $\lambda_{\bbeta,k}\ge 4B\sqrt{\log p/(n_k+n_0)}+2m_0B^2h$ and $\lambda_{\bdelta,k}\ge 4B\sqrt{\log p/n_0}$.
\end{lemma}

\begin{proof}
Make the decomposition of $ L_0\left(\wh\bbeta_{{\cal V}_{\rm tr},k},{\cal V}_{\rm te}\right)- L_0\left(\wh\bbeta_{\rm lasso},{\cal V}_{\rm te}\right)$ as
\begin{eqnarray*}
\label{decomposition}
&& L_0\left(\wh\bbeta_{{\cal V}_{\rm tr},k},{\cal V}_{\rm te}\right)- L_0\left(\wh\bbeta_{\rm lasso},{\cal V}_{\rm te}\right)\nonumber\\
&=&\left\{ L_0\left(\wh\bbeta_{{\cal V}_{\rm tr},k},{\cal V}_{\rm te}\right)- L_0\left(\bbeta_0,{\cal V}_{\rm te}\right)-l_0\left(\wh\bbeta_{{\cal V}_{\rm tr},k}\right)+l_0(\bbeta_0)\right\}
+\left\{l_0\left(\wh\bbeta_{{\cal V}_{\rm tr},k}\right)-l_0(\bbeta_0)\right\}\nonumber\\
&&+\left\{l_0(\bbeta_0)-l_0\left(\wh\bbeta_{\rm lasso}\right)\right\}+\left\{l_0\left(\wh\bbeta_{\rm lasso}\right)-l_0(\bbeta_0)+ L_0(\bbeta_0,{\cal V}_{\rm te})- L_0\left(\wh\bbeta_{\rm lasso},{\cal V}_{\rm te}\right)\right\}\nonumber\\
&=&(I_1)+(I_2)+(I_3)+(I_4),
\end{eqnarray*}
where terms $(I_1)$--$(I_4)$ are self-explained  and will be expired when
the proof of this lemma is completed.
Lemma \ref{lemmaC1} entails that
\begin{eqnarray*}
P\left(\vert(I_1)\vert>8Bt\left\vert\wh\bbeta_{{\cal V}_{\rm tr},k}-
\bbeta_0\right\vert_1\right)\le4\exp\left(-2n_0t^2\right)
\end{eqnarray*}
and
\begin{eqnarray*}
P\left(\vert(I_4)\vert>8Bt\left\vert\wh\bbeta_{\rm lasso}-\bbeta_0\right\vert_1\right)\le4\exp\left(-2n_0t^2\right).
\end{eqnarray*}
Besides, $(I_2)$ and $(I_3)$ can be bounded based on Lemma \ref{lemmaC2} as
\begin{eqnarray*}
\vert(I_2)\vert\le \frac{m_0B^2}2\left\vert\wh\bbeta_{{\cal V}_{\rm tr},k}-\bbeta_0\right\vert_2^2\ {\rm and}\
\vert(I_3)\vert\le \frac{m_0B^2}2\left\vert\wh\bbeta_{\rm lasso}-\bbeta_0\right\vert_2^2.
\end{eqnarray*}
Under assumption \ref{C15}, a parallel result of Theorem \ref{theorem2.1} entails that
\begin{eqnarray*}
\vert(I_2)\vert\le \frac{m_0B^2}2\left\vert\wh\bbeta_{{\cal V}_{\rm tr},k}-\bbeta_0\right\vert_2^2&\le&m_0B^2\left(\left\vert\wh\bbeta_{{\cal V}_{\rm tr},k}-\bbeta_k^*\right\vert_2^2+\left\vert\bbeta_k^*-\bbeta_0\right\vert_2^2\right)\\
&\lesssim& \frac{s_0\log p}{n_k+n_0}+h\sqrt{\frac{\log p}{n_k+n_0}}+\left\vert\bbeta_k^*-\bbeta_0\right\vert_2^2
\end{eqnarray*}
with probability at least $1-2p^{-2}-2p^{-1}$
and (\ref{HD-QR}) yields that
\begin{eqnarray*}
\vert(I_3)\vert\lesssim \frac{s_0\log p}{n_0}
\end{eqnarray*}
with probability greater than $1-c_2p^{-2}$.
Combining the pieces, under  assumption \ref{C14}, there exists a constant $c_4$ such that
\begin{eqnarray*}
&&P\left( L_0\left(\wh\bbeta_{{\cal V}_{\rm tr},k},{\cal V}_{\rm te}\right)- L_0\left(\wh\bbeta_{\rm lasso},{\cal V}_{\rm te}\right)\ge\varepsilon_0 L_0\left(\wh\bbeta_{\rm lasso},{\cal V}_{\rm te}\right)\right)\\
&\le &P\left(\varepsilon_0 L_0\left(\wh\bbeta_{\rm lasso},{\cal V}_{\rm te}\right)\le c_4\left(\sqrt{\frac{\log p}{n_0}}\left\vert\bbeta_k^*-\bbeta_0\right\vert_1+\frac{s_0\log p}{n_k+n_0}+h\sqrt{\frac{\log p}{n_k+n_0}}+\left\vert\bbeta_k^*-\bbeta_0\right\vert_2^2+\frac{s_0\log p}{n_0}\right)\right)\\
&&+\frac{2}{p^2}+\frac{10+2c_2}p.
\end{eqnarray*}
It follows from
$$ \max\left\{h'\sqrt{\frac{\log p}{n_0}},\frac{s_0\log p}{n_0},\max_{k\in\mcT_h}\left\vert\bbeta_k^*-\bbeta_0\right\vert_2^2\right\}\lesssim \varepsilon_0 $$
that
$$P\left( L_0\left(\wh\bbeta_{{\cal V}_{\rm tr},k},{\cal V}_{\rm te}\right)- L_0\left(\wh\bbeta_{\rm lasso},{\cal V}_{\rm te}\right)\ge\varepsilon_0 L_0\left(\wh\bbeta_{\rm lasso},{\cal V}_{\rm te}\right)\right)\to 0.$$
\end{proof}

{\bf Proof of Theorem \ref{theorem4.2}}
\begin{proof}
Following from Theorem \ref{theorem4.1} and Lemma \ref{lemmaC4}, we have
\begin{eqnarray*}
P(\wh\mcT\neq \mcT_h)
&=&P\Bigg(\bigcup_{k\in \mcT_h}\left\{ L_0\left(\wh\bbeta_{{\cal V}_{\rm tr},k},{\cal V}_{\rm te}\right)- L_0\left(\wh\bbeta_{\rm lasso},{\cal V}_{\rm te}\right)\ge\varepsilon_0 L_0\left(\wh\bbeta_{\rm lasso},{\cal V}_{\rm te}\right)\right\}\\
&&\hspace{.3in}\bigcup\bigcup_{k\in\mcT_h^c}\left\{ L_0\left(\wh\bbeta_{{\cal V}_{\rm tr},k},{\cal V}_{\rm te}\right)- L_0\left(\wh\bbeta_{\rm lasso},{\cal V}_{\rm te}\right)\le\varepsilon_0 L_0\left(\wh\bbeta_{\rm lasso},{\cal V}_{\rm te}\right)\right\}\Bigg)\\
&\le&\sum_{k\in \mcT_h}P\left( L_0\left(\wh\bbeta_{{\cal V}_{\rm tr},k},{\cal V}_{\rm te}\right)- L_0\left(\wh\bbeta_{\rm lasso},{\cal V}_{\rm te}\right)\ge\varepsilon_0 L_0\left(\wh\bbeta_{\rm lasso},{\cal V}_{\rm te}\right)\right)\\
&&+\sum_{k\in \mcT_h^c}P\left( L_0(\wh\bbeta_{{\cal V}_{\rm tr},k},{\cal V}_{\rm te})- L_0\left(\wh\bbeta_{\rm lasso},{\cal V}_{\rm te}\right)\le\varepsilon_0 L_0\left(\wh\bbeta_{\rm lasso},{\cal V}_{\rm te}\right)\right)\to0,
\end{eqnarray*}
which immediately completes the proof.
\end{proof}

\section{Additional Simulations}
\label{appendixD}
\setcounter{equation}{0}
\def\theequation{D.\arabic{equation}}
\setcounter{figure}{0}
\def\thefigure{D.\arabic{figure}}
\setcounter{table}{0}
\def\thetable{D.\arabic{table}}

\subsection{Simulations for Estimation Errors}

The averaged $\ell_2$ estimation errors with heterogenous $t_2$
and Gumbel model errors are depicted in Figures \ref{t2} and \ref{gumbel},
respectively. Simulation results show that the proposed transfer learning approach also delivers favorable performances.

\subsection{Simulations for Transferability Detection}
We further report the frequency of the detected transferable sources
that exactly match the truly transferable ones
in Table \ref{correct_fit}.
It can be seen that the proposed transferability detection procedure
can recover the truly transferable sources with very high probability.
Admittedly, a larger $h$ tends to result in a lower accuracy of detection
as a looser criterion is adopted to measure
the similarity between the target and source regression parameters.
Nevertheless, we observe that the frequency of event
$\{\wh\mcT\subset\mcT_h\}$ that happened among $1000$ replications is one in each configuration considered with simulation results omitted for brevity.
As a result, the proposed transferability detection procedure
kicks all the non-transferable sources out
by adopting a conservative strategy to avoid potential effect of the negative transfer.
Meanwhile, as show in Table \ref{true_select},
the averaged size of detected transferable sources is very close
to that of $\mcT_h$, showing that the proposed transferability detection procedure
substantially utilizes the truly transferable sources.

\subsection{Simulations for Confidence Interval}
We conduct additional simulations to evaluate the finite-sample performance
of the proposed Trans approach in terms of confidence interval
for the non-signal coefficient. It is evident from Figure \ref{cp50_75} that
the coverage proportions  of the proposed Trans approach are closer to
the nominal $95\%$ than that of Non-Trans approach for different heterogenous model errors.

\subsection{Simulations for Hypothesis Test}
We further examine the size and power of hypothesis test for individual regression parameter based on Theorem \ref{theorem3.3}.
Setting the significance level as $5\%$ and quantile level as $\tau=0.75$,
we specifically consider  hypothesis test
\begin{equation}
\label{trans_hp1}
H_0:\beta_{0,1}=3 \quad {\rm versus} \quad H_1:\beta_{0,1}=a_1\  (a_1\neq 3)
\end{equation}
for the signal covariate.
It can be seen from Tables \ref{norm_power1}, \ref{t2_power1} and \ref{gumbel_power1}
that when $\mcT_h = \varnothing$, i.e., all sources are non-transferable,
Oracle-Trans (reduced to Non-Trans) and Trans approaches
deliver slightly inflated type I errors.
On the contrary, Oracle-Trans and Trans approaches perform equally well and
both can control the type I errors with test sizes
being close to $5\%$ in all scenarios considered
when $\mcT_h\neq \varnothing$, i.e., some sources are transferable.
Accordingly, the powers are substantially enhanced and
eventually arrive at one when the gaps between null and alternative
hypothesises are gradually enlarged.
The test size of Pooled-Trans approach is larger than $5\%$
when $\vert\mcT_h\vert$ is smaller, i.e., more non-transferable sources
are utilized.

We also consider hypothesis test
\begin{equation}
\label{trans_hp2}
H_0:\beta_{0,50}=0 \quad {\rm versus} \quad  H_1:\beta_{0,50}=a_{50} \ (a_{50}\neq 0)
\end{equation}
for the non-signal covariate. Analogous conclusions can be drawn from
Tables \ref{norm_power50}, \ref{t2_power50} and \ref{gumbel_power50}.

\clearpage
\newpage

\begin{table}[tbp]
\centering
\caption{The frequency of event $\{\wh\mcT=\mcT_h\}$ that happened among $1000$ replications}
{\setlength{\tabcolsep}{0.7mm}
\begin{tabular}{rrrrrrrrrrrrr}
\hline	
\hline
 &  &\multicolumn{3}{c}{Normal} && \multicolumn{3}{c}{$t_2$} && \multicolumn{3}{c}{Gumbel} \\
\cline{3-5}
\cline{7-9}
\cline{11-13}
$h$& $\vert\mcT_h\vert$ & $\tau=0.25$
&$\tau=0.50$ &$\tau=0.75$& &$\tau=0.25$ &$\tau=0.50$
&$\tau=0.75$& &$\tau=0.25$&$\tau=0.50$ &$\tau=0.75$\\
\hline
3 &   0      &    1.000   & 1.000  &  1.000       & & 1.000    & 1.000    & 1.000    & & 1.000     & 1.000    & 1.000                                                  \\
  &   5      &    0.974   & 0.978  &  0.973       & & 0.980    & 0.984    & 0.981    & & 0.975     & 0.975    & 0.986                     \\
  &   10     &    0.946   & 0.959  &  0.947       & & 0.965    & 0.969    & 0.970    & & 0.966     & 0.970    & 0.966                 \\
  &   15     &    0.953   & 0.945  &  0.950       & & 0.964    & 0.966    & 0.967    & & 0.950     & 0.942    & 0.967                     \\
  &   20     &    0.936   & 0.943  &  0.939       & & 0.955    & 0.957    & 0.957    & & 0.945     & 0.946    & 0.948                       \\
6 &   0      &    1.000   & 1.000  &  1.000       & & 1.000    & 1.000    & 1.000    & & 1.000     & 1.000    & 1.000               \\
  &   5      &    0.933   & 0.930  &  0.934       & & 0.966    & 0.965    & 0.982    & & 0.927     & 0.949    & 0.963         \\
  &   10     &    0.911   & 0.902  &  0.896       & & 0.937    & 0.949    & 0.954    & & 0.895     & 0.918    & 0.947            \\
  &   15     &    0.886   & 0.874  &  0.875       & & 0.930    & 0.925    & 0.932    & & 0.873     & 0.900    & 0.924                    \\
  &   20     &    0.869   & 0.859  &  0.869       & & 0.937    & 0.916    & 0.918    & & 0.849     & 0.893    & 0.915                       \\
12&   0      &    1.000   & 1.000  &  1.000       & & 1.000    & 1.000    & 1.000    & & 1.000     & 1.000    & 1.000              \\
  &   5      &    0.780   & 0.773  &  0.765       & & 0.902    & 0.875    & 0.906    & & 0.739     & 0.816    & 0.868       \\
  &   10     &    0.700   & 0.685  &  0.678       & & 0.872    & 0.818    & 0.852    & & 0.670     & 0.734    & 0.838    \\
  &   15     &    0.648   & 0.622  &  0.640       & & 0.844    & 0.778    & 0.826    & & 0.623     & 0.708    & 0.821              \\
  &   20     &    0.597   & 0.582  &  0.589       & & 0.827    & 0.766    & 0.808    & & 0.571     & 0.648    & 0.786                  \\
\hline
\label{correct_fit}
\end{tabular}}
\end{table}

\clearpage
\newpage

\begin{table}[tbp]
\centering
\caption{Averaged size of sources identified as transferable  based on $1000$ replications}
{\setlength{\tabcolsep}{0.7mm}
\begin{tabular}{rrrrrrrrrrrrr}
\hline	
\hline
 &  &\multicolumn{3}{c}{Normal} && \multicolumn{3}{c}{$t_2$} && \multicolumn{3}{c}{Gumbel} \\
\cline{3-5}
\cline{7-9}
\cline{11-13}
$h$& $\vert\mcT_h\vert$ & $\tau=0.25$
&$\tau=0.50$ &$\tau=0.75$& &$\tau=0.25$ &$\tau=0.50$
&$\tau=0.75$& &$\tau=0.25$&$\tau=0.50$ &$\tau=0.75$\\
\hline
3 &   0      &    0.000   & 0.000  &  0.000     & & 0.000    & 0.000    & 0.000     & & 0.000     & 0.000    & 0.000                                                   \\
  &   5      &    4.964   & 4.967  &  4.967     & & 4.975    & 4.981    & 4.977     & & 4.968     & 4.967    & 4.982                      \\
  &   10     &    9.901   & 9.937  &  9.914     & & 9.941    & 9.941    & 9.951     & & 9.941     & 9.955    & 9.957                  \\
  &   15     &    14.902  & 14.888 &  14.902    & & 14.921   & 14.923   & 14.948    & & 14.896    & 14.894   & 14.935                     \\
  &   20     &    19.840  & 19.871 &  19.866    & & 19.908   & 19.924   & 19.924    & & 19.881    & 19.895   & 19.903                       \\
6 &   0      &    0.000   & 0.000  &  0.000     & & 0.000    & 0.000    & 0.000     & & 0.000     & 0.000    & 0.000                \\
  &   5      &    4.887   & 4.898  &  4.904     & & 4.954    & 4.956    & 4.978     & & 4.890     & 4.929    & 4.945          \\
  &   10     &    9.816   & 9.789  &  9.829     & & 9.912    & 9.916    & 9.930     & & 9.782     & 9.857    & 9.881             \\
  &   15     &    14.718  & 14.692 &  14.705    & & 14.852   & 14.820   & 14.867    & & 14.727    & 14.779   & 14.837                    \\
  &   20     &    19.596  & 19.634 &  19.616    & & 19.859   & 19.780   & 19.842    & & 19.568    & 19.700   & 19.755                       \\
12&   0      &    0.000   & 0.000  &  0.000     & & 0.000    & 0.000    & 0.000     & & 0.000     & 0.000    & 0.000               \\
  &   5      &    4.524   & 4.605  &  4.550     & & 4.834    & 4.808    & 4.829     & & 4.488     & 4.702    & 4.773        \\
  &   10     &    9.115   & 9.144  &  9.095     & & 9.699    & 9.594    & 9.632     & & 9.088     & 9.386    & 9.586     \\
  &   15     &    13.585  & 13.657 &  13.585    & & 14.570   & 14.427   & 14.468    & & 13.560    & 14.074   & 14.447              \\
  &   20     &    17.979  & 18.197 &  18.147    & & 19.374   & 19.264   & 19.315    & & 17.943    & 18.778   & 19.183                  \\
\hline
\label{true_select}
\end{tabular}}
\end{table}

\clearpage
\newpage
\begin{table}[tbp]
\begin{footnotesize}
\centering
\caption{The size and power under the significance level of $5\%$ with heterogenous normal model errors for hypothesis test (\ref{trans_hp1}) at the quantile level of $\tau=0.75$
}
{\setlength{\tabcolsep}{1.1mm}
\begin{tabular}{cccccccccccc}
\hline	
\hline
&&&& \multicolumn{8}{c}{\hspace{0in}Power}\\
\cline{5-12}
$h$& Method &$\vert\mcT_h\vert$   &Size
&$a_1=1.7$  &$a_1=2.0$
&$a_1=2.3$    &$a_1=2.7$  &$a_1=3.3$
&$a_1=3.7$  &$a_1=4.0$  &$a_1=4.3$   \\
\hline
3                &Oracle-Trans   &0        &0.075          &0.998 	&0.996 	&0.940 	&0.314       &0.666 	&0.987 	&0.999 	&1.000    \\
                 &               &5        &0.044          &1.000 	&0.996 	&0.906 	&0.219       &0.609 	&0.989 	&0.998 	&1.000    \\
                 &               &10       &0.051          &0.999 	&0.991 	&0.893 	&0.215       &0.595 	&0.983 	&0.997 	&1.000    \\
                 &               &15       &0.051          &1.000 	&0.993 	&0.897 	&0.241       &0.578 	&0.975 	&0.999 	&1.000    \\
                 &               &20       &0.042          &1.000 	&0.990 	&0.888 	&0.231       &0.558 	&0.973 	&0.997 	&1.000    \\
                 &Trans          &0        &0.075          &0.998 	&0.998 	&0.929 	&0.280       &0.676 	&0.991 	&0.999 	&1.000    \\
                 &               &5        &0.049          &1.000 	&0.995 	&0.899 	&0.202       &0.617 	&0.988 	&0.998 	&1.000    \\
                 &               &10       &0.062          &1.000 	&0.986 	&0.898 	&0.223       &0.603 	&0.977 	&0.997 	&1.000    \\
                 &               &15       &0.053          &0.999 	&0.993 	&0.888 	&0.235       &0.585 	&0.977 	&0.998 	&1.000    \\
                 &               &20       &0.053          &1.000 	&0.990 	&0.896 	&0.246       &0.556 	&0.971 	&0.998 	&0.999    \\
                 &Pooled-Trans   &0        &0.174          &0.999 	&0.987 	&0.897 	&0.505       &0.123 	&0.550 	&0.889 	&0.976    \\
                 &               &5        &0.200          &1.000 	&0.989 	&0.951 	&0.569       &0.170 	&0.685 	&0.946 	&0.993    \\
                 &               &10       &0.140          &1.000 	&0.995 	&0.958 	&0.550       &0.238 	&0.820 	&0.980 	&1.000    \\
                 &               &15       &0.030          &0.999 	&0.994 	&0.915 	&0.257       &0.447 	&0.962 	&1.000 	&1.000    \\
                 &               &20       &0.051          &0.999 	&0.990 	&0.891 	&0.240       &0.568 	&0.973 	&0.997 	&1.000    \\
\hline 	
6                &Oracle-Trans   &0        &0.097           &1.000 	&0.995 	&0.938 	&0.318        &0.677 	&0.994 	&0.999 	&1.000     \\
                 &               &5        &0.051           &0.993 	&0.990 	&0.893 	&0.227        &0.631 	&0.985 	&0.999 	&1.000     \\
                 &               &10       &0.054           &0.999 	&0.991 	&0.897 	&0.246        &0.594 	&0.981 	&0.996 	&0.999     \\
                 &               &15       &0.051           &0.999 	&0.988 	&0.911 	&0.249        &0.588 	&0.973 	&0.995 	&1.000     \\
                 &               &20       &0.058           &0.997 	&0.992 	&0.897 	&0.259        &0.569 	&0.971 	&0.994 	&1.000     \\
                 &Trans          &0        &0.096           &0.996 	&0.994 	&0.938 	&0.302        &0.673 	&0.993 	&0.998 	&1.000     \\
                 &               &5        &0.056           &0.998 	&0.993 	&0.899 	&0.225        &0.617 	&0.984 	&0.997 	&1.000     \\
                 &               &10       &0.063           &0.997 	&0.990 	&0.897 	&0.246        &0.601 	&0.981 	&0.996 	&1.000     \\
                 &               &15       &0.055           &0.998 	&0.992 	&0.905 	&0.251        &0.586 	&0.975 	&0.995 	&1.000     \\
                 &               &20       &0.060           &0.996 	&0.996 	&0.900 	&0.261        &0.569 	&0.973 	&0.996 	&0.999     \\
                 &Pooled-Trans   &0        &0.157           &0.994 	&0.988 	&0.922 	&0.510        &0.117 	&0.569 	&0.881 	&0.970     \\
                 &               &5        &0.195           &0.998 	&0.995 	&0.956 	&0.609        &0.165 	&0.695 	&0.935 	&0.985     \\
                 &               &10       &0.157           &0.998 	&0.996 	&0.970 	&0.589        &0.226 	&0.816 	&0.976 	&0.996     \\
                 &               &15       &0.031           &0.996 	&0.995 	&0.901 	&0.258        &0.471 	&0.974 	&1.000 	&1.000     \\
                 &               &20       &0.060           &0.996 	&0.989 	&0.899 	&0.269        &0.564 	&0.970 	&0.997 	&1.000     \\
\hline
12               &Oracle-Trans   &0        &0.088            &0.998 	&0.996 	&0.931 	&0.318         &0.676 	&0.986 	&1.000 	&1.000      \\
                 &               &5        &0.050            &1.000 	&0.990 	&0.905 	&0.240         &0.603 	&0.984 	&0.997 	&0.999      \\
                 &               &10       &0.051            &1.000 	&0.991 	&0.888 	&0.262         &0.560 	&0.966 	&0.999 	&1.000      \\
                 &               &15       &0.049            &0.998 	&0.994 	&0.892 	&0.262         &0.530 	&0.966 	&0.997 	&1.000      \\
                 &               &20       &0.053            &0.996 	&0.989 	&0.899 	&0.275         &0.528 	&0.959 	&0.999 	&0.999      \\
                 &Trans          &0        &0.098            &1.000 	&0.994 	&0.929 	&0.312         &0.686 	&0.994 	&0.999 	&1.000      \\
                 &               &5        &0.054            &0.999 	&0.994 	&0.906 	&0.248         &0.584 	&0.976 	&0.998 	&1.000      \\
                 &               &10       &0.055            &0.998 	&0.990 	&0.891 	&0.263         &0.573 	&0.973 	&0.995 	&0.999      \\
                 &               &15       &0.057            &0.998 	&0.991 	&0.896 	&0.276         &0.547 	&0.966 	&1.000 	&0.999      \\
                 &               &20       &0.058            &0.998 	&0.992 	&0.909 	&0.266         &0.520 	&0.964 	&0.993 	&1.000      \\
                 &Pooled-Trans   &0        &0.175            &0.998 	&0.984 	&0.914 	&0.512         &0.123 	&0.556 	&0.883 	&0.977      \\
                 &               &5        &0.221            &1.000 	&0.995 	&0.953 	&0.596         &0.188 	&0.675 	&0.925 	&0.992      \\
                 &               &10       &0.196            &0.998 	&0.997 	&0.963 	&0.572         &0.215 	&0.809 	&0.973 	&0.994      \\
                 &               &15       &0.051            &1.000 	&0.993 	&0.915 	&0.272         &0.446 	&0.970 	&0.998 	&1.000      \\
                 &               &20       &0.057            &0.998 	&0.991 	&0.912 	&0.267         &0.526 	&0.962 	&0.996 	&0.999      \\
\hline
\label{norm_power1}
\end{tabular}}
\end{footnotesize}
\end{table}

\clearpage
\newpage
\begin{table}[tbp]
\begin{footnotesize}
\centering
\caption{The size and power under the significance level of $5\%$ with
heterogenous $t_2$ model errors for hypothesis test (\ref{trans_hp1})
at the quantile level of $\tau=0.75$}
{\setlength{\tabcolsep}{1.1mm}
\begin{tabular}{cccccccccccc}
\hline	
\hline
&&&& \multicolumn{8}{c}{\hspace{0in}Power}\\
\cline{5-12}
$h$& Method&$\vert\mcT_h\vert$   &Size
&$a_1=1.7$  &$a_1=2.0$
&$a_1=2.3$    &$a_1=2.7$  &$a_1=3.3$
&$a_1=3.7$  &$a_1=4.0$  &$a_1=4.3$   \\
\hline
3&Oracle-Trans   &0        &0.074             &0.927 	&0.800 	&0.530 	&0.114          &0.403 	&0.866 	&0.972 	&0.996       \\
                 &               &5        &0.047             &0.896 	&0.748 	&0.410 	&0.058          &0.314 	&0.852 	&0.972 	&0.995       \\
                 &               &10       &0.049             &0.903 	&0.744 	&0.427 	&0.069          &0.313 	&0.839 	&0.961 	&0.990       \\
                 &               &15       &0.054             &0.889 	&0.741 	&0.412 	&0.064          &0.311 	&0.830 	&0.963 	&0.991       \\
                 &               &20       &0.053             &0.895 	&0.726 	&0.410 	&0.068          &0.328 	&0.828 	&0.965 	&0.989       \\
                 &Trans          &0        &0.069             &0.932 	&0.819 	&0.531 	&0.114          &0.400 	&0.860 	&0.979 	&0.995       \\
                 &               &5        &0.045             &0.909 	&0.742 	&0.433 	&0.066          &0.333 	&0.838 	&0.972 	&0.999       \\
                 &               &10       &0.046             &0.903 	&0.758 	&0.421 	&0.064          &0.324 	&0.850 	&0.960 	&0.991       \\
                 &               &15       &0.049             &0.896 	&0.739 	&0.412 	&0.068          &0.322 	&0.831 	&0.964 	&0.995       \\
                 &               &20       &0.058             &0.897 	&0.739 	&0.406 	&0.076          &0.324 	&0.829 	&0.967 	&0.991       \\
                 &Pooled-Trans   &0        &0.113             &0.920 	&0.791 	&0.608 	&0.279          &0.085 	&0.302 	&0.573 	&0.779       \\
                 &               &5        &0.140             &0.933 	&0.841 	&0.660 	&0.303          &0.098 	&0.358 	&0.623 	&0.831       \\
                 &               &10       &0.116             &0.936 	&0.841 	&0.662 	&0.289          &0.112 	&0.426 	&0.719 	&0.896       \\
                 &               &15       &0.037             &0.900 	&0.759 	&0.455 	&0.080          &0.176 	&0.723 	&0.934 	&0.988       \\
                 &               &20       &0.048             &0.901 	&0.745 	&0.409 	&0.077          &0.331 	&0.825 	&0.970 	&0.985       \\
\hline 	
6&Oracle-Trans   &0        &0.086              &0.926 	&0.815 	&0.501 	&0.132           &0.394 	&0.872 	&0.973 	&0.991        \\
                 &               &5        &0.036              &0.902 	&0.751 	&0.421 	&0.078           &0.303 	&0.854 	&0.968 	&0.992        \\
                 &               &10       &0.040              &0.905 	&0.759 	&0.410 	&0.082           &0.311 	&0.828 	&0.963 	&0.988        \\
                 &               &15       &0.038              &0.898 	&0.754 	&0.409 	&0.085           &0.321 	&0.832 	&0.957 	&0.987        \\
                 &               &20       &0.045              &0.904 	&0.771 	&0.409 	&0.090           &0.303 	&0.828 	&0.957 	&0.989        \\
                 &Trans          &0        &0.086              &0.931 	&0.811 	&0.485 	&0.129           &0.388 	&0.883 	&0.970 	&0.995        \\
                 &               &5        &0.038              &0.909 	&0.761 	&0.412 	&0.088           &0.298 	&0.841 	&0.969 	&0.989        \\
                 &               &10       &0.035              &0.906 	&0.759 	&0.410 	&0.076           &0.308 	&0.850 	&0.965 	&0.991        \\
                 &               &15       &0.037              &0.905 	&0.756 	&0.402 	&0.079           &0.312 	&0.833 	&0.957 	&0.988        \\
                 &               &20       &0.041              &0.891 	&0.756 	&0.395 	&0.084           &0.304 	&0.837 	&0.955 	&0.987        \\
                 &Pooled-Trans   &0        &0.122              &0.915 	&0.819 	&0.602 	&0.288           &0.084 	&0.296 	&0.577 	&0.801        \\
                 &               &5        &0.136              &0.924 	&0.841 	&0.665 	&0.304           &0.106 	&0.349 	&0.638 	&0.829        \\
                 &               &10       &0.109              &0.948 	&0.851 	&0.648 	&0.314           &0.116 	&0.437 	&0.725 	&0.879        \\
                 &               &15       &0.024              &0.910 	&0.766 	&0.422 	&0.097           &0.179 	&0.721 	&0.940 	&0.984        \\
                 &               &20       &0.037              &0.911 	&0.753 	&0.410 	&0.090           &0.305 	&0.830 	&0.955 	&0.988        \\
\hline
12&Oracle-Trans   &0        &0.086               &0.923 	&0.783 	&0.522 	&0.120            &0.676 	&0.890 	&0.961 	&0.996         \\
                 &               &5        &0.043               &0.897 	&0.742 	&0.428 	&0.078            &0.603 	&0.847 	&0.953 	&0.992         \\
                 &               &10       &0.042               &0.889 	&0.741 	&0.443 	&0.079            &0.560 	&0.833 	&0.937 	&0.987         \\
                 &               &15       &0.037               &0.886 	&0.741 	&0.433 	&0.080            &0.530 	&0.827 	&0.937 	&0.995         \\
                 &               &20       &0.041               &0.892 	&0.742 	&0.424 	&0.076            &0.528 	&0.823 	&0.924 	&0.987         \\
                 &Trans          &0        &0.081               &0.920 	&0.779 	&0.505 	&0.114            &0.686 	&0.883 	&0.963 	&0.994         \\
                 &               &5        &0.042               &0.891 	&0.746 	&0.439 	&0.085            &0.584 	&0.851 	&0.952 	&0.992         \\
                 &               &10       &0.044               &0.898 	&0.746 	&0.425 	&0.076            &0.573 	&0.825 	&0.937 	&0.989         \\
                 &               &15       &0.040               &0.888 	&0.727 	&0.432 	&0.079            &0.547 	&0.827 	&0.934 	&0.991         \\
                 &               &20       &0.039               &0.888 	&0.738 	&0.426 	&0.083            &0.520 	&0.821 	&0.927 	&0.986         \\
                 &Pooled-Trans   &0        &0.104               &0.899 	&0.801 	&0.618 	&0.280            &0.123 	&0.285 	&0.550 	&0.774         \\
                 &               &5        &0.119               &0.928 	&0.829 	&0.642 	&0.330            &0.188 	&0.352 	&0.624 	&0.829         \\
                 &               &10       &0.098               &0.925 	&0.833 	&0.647 	&0.287            &0.215 	&0.446 	&0.699 	&0.900         \\
                 &               &15       &0.024               &0.882 	&0.760 	&0.443 	&0.091            &0.446 	&0.745 	&0.917 	&0.986         \\
                 &               &20       &0.035               &0.887 	&0.738 	&0.431 	&0.092            &0.526 	&0.826 	&0.929 	&0.989         \\
\hline
\label{t2_power1}
\end{tabular}}
\end{footnotesize}
\end{table}

\clearpage
\newpage
\begin{table}[tbp]
\begin{footnotesize}
\centering
\caption{The size and power under the significance level of $5\%$ with
heterogenous Gumbel model errors for hypothesis test (\ref{trans_hp1})
at the quantile level of $\tau=0.75$}
{\setlength{\tabcolsep}{1.1mm}
\begin{tabular}{cccccccccccc}
\hline	
\hline
&&&& \multicolumn{8}{c}{\hspace{0in}Power}\\
\cline{5-12}
$h$& Method&$\vert\mcT_h\vert$   &Size
&$a_1=1.7$  &$a_1=2.0$
&$a_1=2.3$    &$a_1=2.7$  &$a_1=3.3$
&$a_1=3.7$  &$a_1=4.0$  &$a_1=4.3$   \\
\hline
3                &Oracle-Trans	&0	&0.083 &0.996  &0.967  &0.811  &0.272  	&0.411 	&0.891 	&0.989 	&1.000  \\
                 &	            &5	&0.046 &0.995  &0.950  &0.713  &0.163  	&0.386 	&0.902 	&0.989 	&1.000  \\
                 &	            &10 &0.052 &0.987  &0.945  &0.730  &0.170   &0.379  &0.880  &0.985  &0.999  \\
                 &	            &15 &0.061 &0.993  &0.940  &0.734  &0.175   &0.384  &0.865  &0.981  &0.999  \\
                 &	            &20 &0.062 &0.990  &0.950  &0.724  &0.181   &0.378  &0.868  &0.984  &1.000  \\
                 &Trans	        &0  &0.033 &0.990  &0.962  &0.805  &0.247   &0.423  &0.901  &0.988  &1.000  \\
                 &	            &5	&0.045 &0.993  &0.951  &0.735  &0.163 	&0.383 	&0.903 	&0.992 	&1.000  \\
                 &	            &10 &0.054 &0.991  &0.951  &0.718  &0.176   &0.394  &0.884  &0.988  &1.000  \\
                 &	            &15 &0.063 &0.990  &0.940  &0.715  &0.178   &0.395  &0.858  &0.979  &0.997  \\
                 &	            &20 &0.061 &0.989  &0.949  &0.711  &0.178   &0.379  &0.872  &0.984  &1.000  \\
                 &Pooled-Trans	&0	&0.154 &0.987  &0.950  &0.825  &0.443 	&0.105 	&0.369 	&0.709 	&0.899  \\
                 &	            &5	&0.197 &0.993  &0.973  &0.867  &0.514 	&0.134 	&0.457 	&0.794 	&0.954  \\
                 &	            &10 &0.176 &0.994  &0.975  &0.876  &0.481   &0.158  &0.576  &0.872  &0.983  \\
                 &	            &15 &0.043 &0.991  &0.963  &0.774  &0.207   &0.242  &0.815  &0.981  &0.999  \\
                 &	            &20 &0.060 &0.988  &0.944  &0.723  &0.182   &0.375  &0.867  &0.982  &0.999  \\
\hline 	
6                &Oracle-Trans   &0        &0.103        &0.995 	&0.972 	&0.792 	&0.286     &0.456 	&0.910 	&0.992 	&0.997  \\
                 &               &5        &0.060        &0.993 	&0.952 	&0.722 	&0.151     &0.418 	&0.912 	&0.991 	&0.998  \\
                 &               &10       &0.062        &0.991 	&0.964 	&0.720 	&0.172     &0.398 	&0.884 	&0.988 	&0.998  \\
                 &               &15       &0.068        &0.991 	&0.959 	&0.716 	&0.192     &0.407 	&0.892 	&0.985 	&0.998  \\
                 &               &20       &0.069        &0.991 	&0.958 	&0.701 	&0.187     &0.399 	&0.882 	&0.987 	&0.998  \\
                 &Trans          &0        &0.113        &0.989 	&0.967 	&0.778 	&0.252     &0.443 	&0.915 	&0.991 	&1.000  \\
                 &               &5        &0.054        &0.992 	&0.961 	&0.714 	&0.151     &0.420 	&0.911 	&0.991 	&0.997  \\
                 &               &10       &0.064        &0.993 	&0.957 	&0.711 	&0.170     &0.397 	&0.892 	&0.988 	&0.998  \\
                 &               &15       &0.062        &0.989 	&0.955 	&0.711 	&0.180     &0.396 	&0.892 	&0.984 	&0.999  \\
                 &               &20       &0.066        &0.993 	&0.956 	&0.704 	&0.195     &0.388 	&0.881 	&0.985 	&0.998  \\
                 &Pooled-Trans   &0        &0.195        &0.987 	&0.953 	&0.806 	&0.454     &0.113 	&0.401 	&0.689 	&0.873  \\
                 &               &5        &0.238        &0.992 	&0.967 	&0.879 	&0.518     &0.149 	&0.476 	&0.797 	&0.929  \\
                 &               &10       &0.188        &0.997 	&0.980 	&0.884 	&0.463     &0.180 	&0.582 	&0.876 	&0.975  \\
                 &               &15       &0.051        &0.995 	&0.957 	&0.743 	&0.214     &0.267 	&0.839 	&0.982 	&0.998  \\
                 &               &20       &0.069        &0.993 	&0.956 	&0.710 	&0.193     &0.403 	&0.882 	&0.983 	&0.998  \\
\hline
12               &Oracle-Trans   &0        &0.088         &0.996 	&0.961 	&0.806 	&0.268      &0.431 	&0.896 	&0.990 	&0.997   \\
                 &               &5        &0.050         &0.993 	&0.958 	&0.741 	&0.154      &0.383 	&0.896 	&0.994 	&0.990   \\
                 &               &10       &0.051         &0.997 	&0.952 	&0.730 	&0.171      &0.365 	&0.885 	&0.990 	&0.997   \\
                 &               &15       &0.049         &0.992 	&0.948 	&0.727 	&0.184      &0.363 	&0.872 	&0.991 	&0.997   \\
                 &               &20       &0.053         &0.995 	&0.953 	&0.735 	&0.201      &0.345 	&0.865 	&0.983 	&0.998   \\
                 &Trans          &0        &0.098         &0.992 	&0.965 	&0.784 	&0.244      &0.418 	&0.909 	&0.993 	&0.997   \\
                 &               &5        &0.054         &0.993 	&0.963 	&0.732 	&0.162      &0.376 	&0.909 	&0.990 	&0.997   \\
                 &               &10       &0.055         &0.993 	&0.956 	&0.738 	&0.186      &0.362 	&0.886 	&0.987 	&0.997   \\
                 &               &15       &0.057         &0.994 	&0.948 	&0.725 	&0.196      &0.358 	&0.874 	&0.989 	&0.997   \\
                 &               &20       &0.058         &0.996 	&0.950 	&0.737 	&0.210      &0.348 	&0.875 	&0.981 	&0.996   \\
                 &Pooled-Trans   &0        &0.175         &0.990 	&0.945 	&0.819 	&0.456      &0.120 	&0.383 	&0.670 	&0.889   \\
                 &               &5        &0.221         &0.994 	&0.973 	&0.872 	&0.518      &0.157 	&0.475 	&0.771 	&0.932   \\
                 &               &10       &0.196         &0.995 	&0.973 	&0.888 	&0.481      &0.160 	&0.602 	&0.865 	&0.974   \\
                 &               &15       &0.051         &0.994 	&0.953 	&0.769 	&0.196      &0.256 	&0.816 	&0.982 	&0.996   \\
                 &               &20       &0.057         &0.990 	&0.958 	&0.740 	&0.212      &0.333 	&0.873 	&0.984 	&0.997   \\
\hline
\label{gumbel_power1}
\end{tabular}}
\end{footnotesize}
\end{table}

\clearpage
\newpage

\begin{table}[H]
\begin{footnotesize}
\centering
\caption{The size and power under the significance level of $5\%$ with
heterogenous normal model errors for hypothesis test (\ref{trans_hp2})
at the quantile level of $\tau=0.75$}
{\setlength{\tabcolsep}{2.1mm}
\begin{tabular}{cccccccccc}
\hline	
\hline
&&& &\multicolumn{6}{c}{\hspace{0in}Power}\\
\cline{5-10}
$h$& Method&$\vert\mcT_h\vert$   &Size
 &$a_{50}=-1$   &$a_{50}=-0.7$
&$a_{50}=-0.3$    &$a_{50}=0.3$ &$a_{50}=0.7$ &$a_{50}=1$ \\
\hline
3                &Oracle-Trans    &0     &0.088     &1.000 	&1.000 	&0.947    	 	&0.953 	&1.000 	&1.000         \\
                 &                &5     &0.078     &1.000 	&1.000 	&0.932    	 	&0.940 	&0.999 	&1.000         \\
                 &                &10    &0.089     &1.000 	&0.999 	&0.930    	 	&0.923 	&0.997 	&1.000         \\
                 &                &15    &0.081     &1.000 	&0.999 	&0.913    	 	&0.903 	&0.998 	&1.000         \\
                 &                &20    &0.076     &0.998 	&0.996 	&0.918    	 	&0.911 	&0.996 	&1.000         \\
                 &Trans           &0     &0.098     &1.000 	&0.999 	&0.946    	 	&0.961 	&1.000 	&1.000         \\
                 &                &5     &0.081     &1.000 	&0.999 	&0.938    	 	&0.946 	&1.000 	&1.000         \\
                 &                &10    &0.076     &0.999 	&0.999 	&0.939    	 	&0.926 	&0.999 	&0.999         \\
                 &                &15    &0.079     &1.000 	&0.998 	&0.934    	 	&0.920 	&1.000 	&1.000         \\
                 &                &20    &0.077     &1.000 	&0.997 	&0.917    	 	&0.913 	&0.996 	&1.000         \\
                 &Pooled-Trans    &0     &0.003     &0.953 	&0.693 	&0.068    	 	&0.074 	&0.697 	&0.943         \\
                 &                &5     &0.005     &0.975 	&0.843 	&0.142    	 	&0.137 	&0.872 	&0.980         \\
                 &                &10    &0.007     &0.997 	&0.937 	&0.267    	 	&0.253 	&0.949 	&0.998         \\
                 &                &15    &0.044     &1.000 	&0.998 	&0.761    	 	&0.760 	&0.997 	&1.000         \\
                 &                &20    &0.072     &0.999 	&0.999 	&0.917    	 	&0.912 	&0.998 	&0.999         \\
\hline
6                &Oracle-Trans    &0     &0.119      &1.000 	&1.000 	&0.942     	 	&0.929 	&1.000 	&1.000           \\
                 &                &5     &0.066      &1.000 	&0.995 	&0.933     	 	&0.938 	&0.996 	&1.000           \\
                 &                &10    &0.071      &1.000 	&0.999 	&0.932     	 	&0.936 	&0.995 	&1.000           \\
                 &                &15    &0.068      &0.998 	&0.999 	&0.920     	 	&0.916 	&0.998 	&0.999           \\
                 &                &20    &0.069      &0.999 	&0.998 	&0.935     	 	&0.913 	&0.997 	&1.000           \\
                 &Trans           &0     &0.097      &1.000 	&1.000 	&0.950     	 	&0.945 	&0.999 	&1.000           \\
                 &                &5     &0.064      &1.000 	&1.000 	&0.929     	 	&0.939 	&0.999 	&1.000           \\
                 &                &10    &0.059      &0.999 	&1.000 	&0.935     	 	&0.929 	&0.996 	&1.000           \\
                 &                &15    &0.060      &0.999 	&0.996 	&0.914     	 	&0.913 	&0.998 	&1.000           \\
                 &                &20    &0.067      &1.000 	&0.999 	&0.906     	 	&0.908 	&0.999 	&1.000           \\
                 &Pooled-Trans    &0     &0.000      &0.951 	&0.694 	&0.076     	 	&0.086 	&0.708 	&0.954           \\
                 &                &5     &0.003      &0.990 	&0.833 	&0.144     	 	&0.149 	&0.849 	&0.981           \\
                 &                &10    &0.010      &0.997 	&0.932 	&0.259     	 	&0.262 	&0.942 	&0.999           \\
                 &                &15    &0.035      &1.000 	&0.998 	&0.741     	 	&0.746 	&0.998 	&0.999           \\
                 &                &20    &0.074      &0.999 	&0.998 	&0.923     	 	&0.919 	&0.997 	&0.999           \\

\hline
12               &Oracle-Trans    &0     &0.099       &1.000 	&0.999 	&0.940      	 	&0.939 	&1.000 	&1.000            \\
                 &                &5     &0.059       &0.999 	&0.998 	&0.937      	 	&0.931 	&0.999 	&1.000            \\
                 &                &10    &0.044       &1.000 	&1.000 	&0.934      	 	&0.935 	&0.998 	&0.999            \\
                 &                &15    &0.058       &0.999 	&0.999 	&0.930      	 	&0.901 	&0.997 	&1.000            \\
                 &                &20    &0.066       &1.000 	&1.000 	&0.922      	 	&0.905 	&0.999 	&1.000            \\
                 &Trans           &0     &0.099       &1.000 	&1.000 	&0.956      	 	&0.946 	&1.000 	&1.000            \\
                 &                &5     &0.055       &1.000 	&0.999 	&0.934      	 	&0.925 	&0.999 	&1.000            \\
                 &                &10    &0.057       &0.999 	&1.000 	&0.944      	 	&0.919 	&0.999 	&0.999            \\
                 &                &15    &0.050       &1.000 	&0.999 	&0.928      	 	&0.909 	&0.997 	&1.000            \\
                 &                &20    &0.063       &1.000 	&0.999 	&0.923      	 	&0.911 	&0.998 	&0.999            \\
                 &Pooled-Trans    &0     &0.002       &0.951 	&0.739 	&0.070      	 	&0.073 	&0.683 	&0.939            \\
                 &                &5     &0.003       &0.979 	&0.849 	&0.128      	 	&0.116 	&0.825 	&0.990            \\
                 &                &10    &0.001       &0.996 	&0.949 	&0.240      	 	&0.243 	&0.934 	&0.996            \\
                 &                &15    &0.033       &0.999 	&0.999 	&0.736      	 	&0.705 	&0.996 	&0.999            \\
                 &                &20    &0.060       &0.999 	&0.999 	&0.916      	 	&0.899 	&0.999 	&1.000            \\
\hline
\label{norm_power50}
\end{tabular}}
\end{footnotesize}
\end{table}

\clearpage
\newpage

\begin{table}[H]
\begin{footnotesize}
\centering
\caption{The size and power under the significance level of $5\%$ with
heterogenous $t_2$ model errors for hypothesis test (\ref{trans_hp2})
at the quantile level of $\tau=0.75$}
{\setlength{\tabcolsep}{2.1mm}
\begin{tabular}{cccccccccc}
\hline	
\hline
&&& &\multicolumn{6}{c}{\hspace{0in}Power}\\
\cline{5-10}
$h$& Method&$\vert\mcT_h\vert$   &Size
&$a_{50}=-1$   &$a_{50}=-0.7$
&$a_{50}=-0.3$    &$a_{50}=0.3$ &$a_{50}=0.7$ &$a_{50}=1$ \\
\hline
3                &Oracle-Trans    &0     &0.100        &0.988 	&0.971 	&0.712       	 	&0.689 	&0.970 	&0.990           \\
                 &                &5     &0.068        &0.986 	&0.961 	&0.678       	 	&0.656 	&0.944 	&0.984           \\
                 &                &10    &0.063        &0.986 	&0.955 	&0.679       	 	&0.666 	&0.946 	&0.983           \\
                 &                &15    &0.064        &0.987 	&0.954 	&0.700       	 	&0.676 	&0.957 	&0.985           \\
                 &                &20    &0.067        &0.983 	&0.945 	&0.690       	 	&0.668 	&0.946 	&0.987           \\
                 &Trans           &0     &0.086        &0.989 	&0.960 	&0.724       	 	&0.711 	&0.963 	&0.990           \\
                 &                &5     &0.071        &0.984 	&0.964 	&0.678       	 	&0.667 	&0.955 	&0.987           \\
                 &                &10    &0.072        &0.987 	&0.954 	&0.695       	 	&0.662 	&0.947 	&0.982           \\
                 &                &15    &0.069        &0.983 	&0.943 	&0.685       	 	&0.659 	&0.949 	&0.986           \\
                 &                &20    &0.068        &0.985 	&0.947 	&0.687       	 	&0.660 	&0.935 	&0.979           \\
                 &Pooled-Trans    &0     &0.004        &0.695 	&0.363 	&0.031       	 	&0.029 	&0.332 	&0.679           \\
                 &                &5     &0.007        &0.779 	&0.479 	&0.045       	 	&0.046 	&0.396 	&0.768           \\
                 &                &10    &0.010        &0.832 	&0.566 	&0.079       	 	&0.085 	&0.525 	&0.837           \\
                 &                &15    &0.038        &0.968 	&0.894 	&0.399       	 	&0.381 	&0.864 	&0.971           \\
                 &                &20    &0.065        &0.983 	&0.944 	&0.706       	 	&0.684 	&0.939 	&0.985           \\
\hline
6                &Oracle-Trans    &0     &0.105         &0.992 	&0.969 	&0.714        	 	&0.711 	&0.971 	&0.992            \\
                 &                &5     &0.057         &0.989 	&0.958 	&0.668        	 	&0.665 	&0.942 	&0.989            \\
                 &                &10    &0.057         &0.988 	&0.965 	&0.681        	 	&0.669 	&0.949 	&0.984            \\
                 &                &15    &0.058         &0.987 	&0.954 	&0.680        	 	&0.673 	&0.942 	&0.988            \\
                 &                &20    &0.054         &0.984 	&0.954 	&0.674        	 	&0.675 	&0.935 	&0.983            \\
                 &Trans           &0     &0.103         &0.992 	&0.973 	&0.722        	 	&0.715 	&0.963 	&0.992            \\
                 &                &5     &0.056         &0.983 	&0.961 	&0.678        	 	&0.676 	&0.951 	&0.987            \\
                 &                &10    &0.050         &0.988 	&0.962 	&0.667        	 	&0.674 	&0.949 	&0.985            \\
                 &                &15    &0.057         &0.985 	&0.956 	&0.674        	 	&0.659 	&0.939 	&0.992            \\
                 &                &20    &0.063         &0.979 	&0.953 	&0.675        	 	&0.670 	&0.946 	&0.987            \\
                 &Pooled-Trans    &0     &0.001         &0.666 	&0.343 	&0.030        	 	&0.040 	&0.330 	&0.674            \\
                 &                &5     &0.004         &0.743 	&0.437 	&0.048        	 	&0.490 	&0.416 	&0.746            \\
                 &                &10    &0.008         &0.825 	&0.553 	&0.082        	 	&0.083 	&0.540 	&0.818            \\
                 &                &15    &0.045         &0.962 	&0.876 	&0.407        	 	&0.436 	&0.867 	&0.973            \\
                 &                &20    &0.057         &0.983 	&0.956 	&0.690        	 	&0.679 	&0.938 	&0.986            \\
\hline
12               &Oracle-Trans    &0     &0.096          &0.987 	&0.966 	&0.716         	 	&0.725 	&0.962 	&0.989            \\
                 &                &5     &0.055          &0.975 	&0.953 	&0.630         	 	&0.672 	&0.942 	&0.989            \\
                 &                &10    &0.052          &0.977 	&0.946 	&0.639         	 	&0.674 	&0.931 	&0.980            \\
                 &                &15    &0.046          &0.980 	&0.941 	&0.668         	 	&0.666 	&0.941 	&0.987            \\
                 &                &20    &0.062          &0.979 	&0.948 	&0.665         	 	&0.686 	&0.933 	&0.981            \\
                 &Trans           &0     &0.105          &0.988 	&0.964 	&0.731         	 	&0.711 	&0.963 	&0.987            \\
                 &                &5     &0.060          &0.980 	&0.955 	&0.623         	 	&0.759 	&0.923 	&0.981            \\
                 &                &10    &0.054          &0.976 	&0.944 	&0.649         	 	&0.664 	&0.935 	&0.984            \\
                 &                &15    &0.061          &0.982 	&0.943 	&0.665         	 	&0.666 	&0.936 	&0.978            \\
                 &                &20    &0.049          &0.983 	&0.942 	&0.646         	 	&0.683 	&0.936 	&0.982            \\
                 &Pooled-Trans    &0     &0.004          &0.681 	&0.370 	&0.029         	 	&0.029 	&0.373 	&0.663            \\
                 &                &5     &0.003          &0.749 	&0.458 	&0.058         	 	&0.062 	&0.452 	&0.734            \\
                 &                &10    &0.010          &0.821 	&0.551 	&0.090         	 	&0.090 	&0.542 	&0.819            \\
                 &                &15    &0.035          &0.956 	&0.868 	&0.387         	 	&0.404 	&0.864 	&0.956            \\
                 &                &20    &0.054          &0.983 	&0.942 	&0.655         	 	&0.675 	&0.942 	&0.978            \\
\hline
\label{t2_power50}
\end{tabular}}
\end{footnotesize}
\end{table}

\clearpage
\newpage

\begin{table}[H]
\begin{footnotesize}
\centering
\caption{The size and power under the significance level of $5\%$ with
heterogenous Gumbel model errors for hypothesis test (\ref{trans_hp2})
at the quantile level of $\tau=0.75$}
{\setlength{\tabcolsep}{2.1mm}
\begin{tabular}{cccccccccc}
\hline	
\hline
&&& &\multicolumn{6}{c}{\hspace{0in}Power}\\
\cline{5-10}
$h$& Method&$\vert\mcT_h\vert$   &Size
&$a_{50}=-1$   &$a_{50}=-0.7$
&$a_{50}=-0.3$    &$a_{50}=0.3$ &$a_{50}=0.7$ &$a_{50}=1$ \\
\hline
3                &Oracle-Trans    &0     &0.102  &1.000 	&0.997 	&0.831 	 	&0.819 	&0.999 	&1.000       \\
                 &                &5     &0.075  &1.000 	&0.997 	&0.855 	 	&0.862 	&0.998 	&1.000       \\
                 &                &10    &0.076  &1.000 	&0.993 	&0.844 	 	&0.832 	&0.993 	&1.000       \\
                 &                &15    &0.068  &0.999 	&0.993 	&0.833 	 	&0.828 	&0.996 	&0.999       \\
                 &                &20    &0.072  &1.000 	&0.994 	&0.831 	 	&0.807 	&0.994 	&1.000       \\
                 &Trans           &0     &0.107  &0.999 	&0.997 	&0.848 	 	&0.838 	&0.996 	&1.000       \\
                 &                &5     &0.084  &1.000 	&0.996 	&0.840 	 	&0.854 	&0.996 	&1.000       \\
                 &                &10    &0.076  &1.000 	&0.996 	&0.849 	 	&0.832 	&0.995 	&1.000       \\
                 &                &15    &0.075  &1.000 	&0.994 	&0.837 	 	&0.819 	&0.996 	&1.000       \\
                 &                &20    &0.071  &1.000 	&0.991 	&0.824 	 	&0.812 	&0.993 	&0.998       \\
                 &Pooled-Trans    &0     &0.001  &0.842 	&0.536 	&0.050 	 	&0.043 	&0.551 	&0.873       \\
                 &                &5     &0.008  &0.911 	&0.688 	&0.085 	 	&0.085 	&0.697 	&0.936       \\
                 &                &10    &0.012  &0.978 	&0.811 	&0.156 	 	&0.153 	&0.861 	&0.980       \\
                 &                &15    &0.039  &0.997 	&0.979 	&0.590 	 	&0.565 	&0.985 	&0.999       \\
                 &                &20    &0.076  &0.999 	&0.991 	&0.822 	 	&0.809 	&0.996 	&1.000       \\
\hline
6                &Oracle-Trans    &0     &0.091   &0.999 	&0.998 	&0.852  	 	&0.845 	&1.000 	&1.000        \\
                 &                &5     &0.071   &0.999 	&0.997 	&0.868  	 	&0.863 	&0.995 	&0.999        \\
                 &                &10    &0.055   &0.999 	&0.998 	&0.856  	 	&0.832 	&0.996 	&1.000        \\
                 &                &15    &0.050   &0.997 	&0.996 	&0.831  	 	&0.833 	&0.996 	&1.000        \\
                 &                &20    &0.062   &0.998 	&0.994 	&0.835  	 	&0.833 	&0.991 	&0.998        \\
                 &Trans           &0     &0.096   &1.000 	&0.998 	&0.864  	 	&0.857 	&0.997 	&1.000        \\
                 &                &5     &0.067   &0.996 	&0.996 	&0.862  	 	&0.852 	&0.997 	&1.000        \\
                 &                &10    &0.061   &0.998 	&0.995 	&0.853  	 	&0.833 	&0.995 	&0.998        \\
                 &                &15    &0.050   &0.997 	&0.995 	&0.831  	 	&0.839 	&0.993 	&0.999        \\
                 &                &20    &0.059   &0.998 	&0.995 	&0.829  	 	&0.835 	&0.993 	&1.000        \\
                 &Pooled-Trans    &0     &0.002   &0.853 	&0.545 	&0.063  	 	&0.051 	&0.547 	&0.863        \\
                 &                &5     &0.007   &0.926 	&0.683 	&0.105  	 	&0.092 	&0.663 	&0.931        \\
                 &                &10    &0.008   &0.976 	&0.836 	&0.172  	 	&0.154 	&0.995 	&0.981        \\
                 &                &15    &0.035   &0.998 	&0.970 	&0.601  	 	&0.572 	&0.993 	&0.999        \\
                 &                &20    &0.055   &0.997 	&0.990 	&0.838  	 	&0.828 	&0.993 	&1.000        \\
\hline
12               &Oracle-Trans    &0     &0.118    &1.000 	&0.998 	&0.841   	 	&0.843 	&0.997 	&0.999         \\
                 &                &5     &0.059    &1.000 	&0.999 	&0.847   	 	&0.836 	&0.999 	&1.000         \\
                 &                &10    &0.057    &1.000 	&0.996 	&0.845   	 	&0.810 	&0.996 	&1.000         \\
                 &                &15    &0.051    &1.000 	&0.995 	&0.857   	 	&0.819 	&0.990 	&0.999         \\
                 &                &20    &0.045    &0.999 	&0.994 	&0.843   	 	&0.820 	&0.992 	&0.999         \\
                 &Trans           &0     &0.110    &0.999 	&0.998 	&0.854   	 	&0.855 	&1.000 	&1.000         \\
                 &                &5     &0.070    &1.000 	&0.999 	&0.848   	 	&0.847 	&1.000 	&0.999         \\
                 &                &10    &0.060    &0.999 	&0.996 	&0.859   	 	&0.819 	&0.992 	&1.000         \\
                 &                &15    &0.046    &1.000 	&0.994 	&0.841   	 	&0.804 	&0.994 	&1.000         \\
                 &                &20    &0.054    &1.000 	&0.995 	&0.849   	 	&0.818 	&0.994 	&1.000         \\
                 &Pooled-Trans    &0     &0.002    &0.869 	&0.543 	&0.056   	 	&0.046 	&0.529 	&0.850         \\
                 &                &5     &0.004    &0.932 	&0.684 	&0.093   	 	&0.086 	&0.663 	&0.920         \\
                 &                &10    &0.008    &0.982 	&0.827 	&0.161   	 	&0.148 	&0.812 	&0.979         \\
                 &                &15    &0.039    &0.998 	&0.984 	&0.590   	 	&0.549 	&0.981 	&1.000         \\
                 &                &20    &0.053    &0.998 	&0.996 	&0.853   	 	&0.824 	&0.992 	&1.000         \\
\hline
\label{gumbel_power50}
\end{tabular}}
\end{footnotesize}
\end{table}

\clearpage
\newpage
\begin{figure}[tbp]
\centerline{
\includegraphics[height=0.75\textwidth,width=\textwidth]{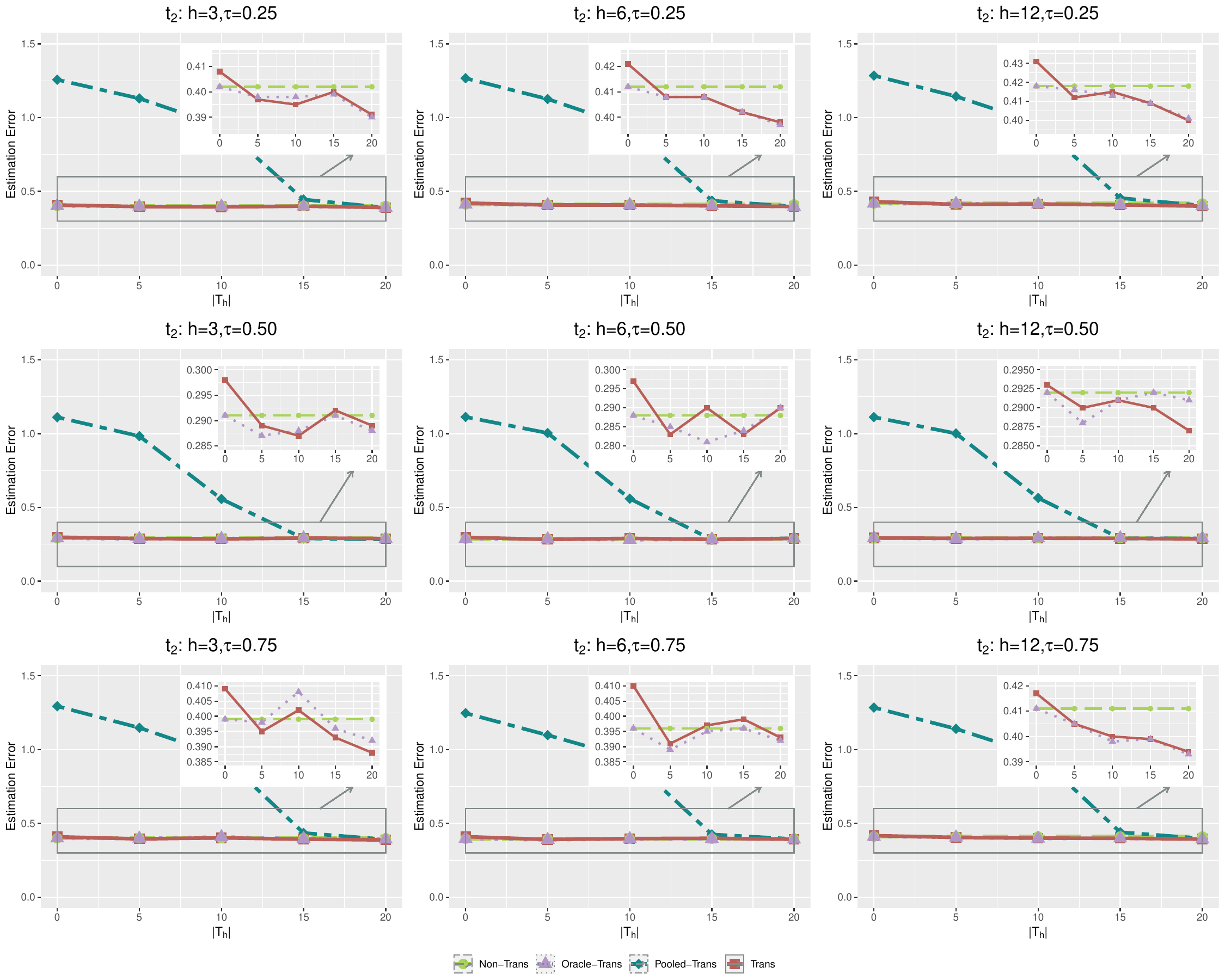}}
\caption{Averaged $\ell_2$ estimation errors with heterogenous $t_2$ model errors.}
\label{t2}
\end{figure}

\clearpage
\newpage

\begin{figure}[tbp]
\centerline{
\includegraphics[height=0.75\textwidth,width=\textwidth]{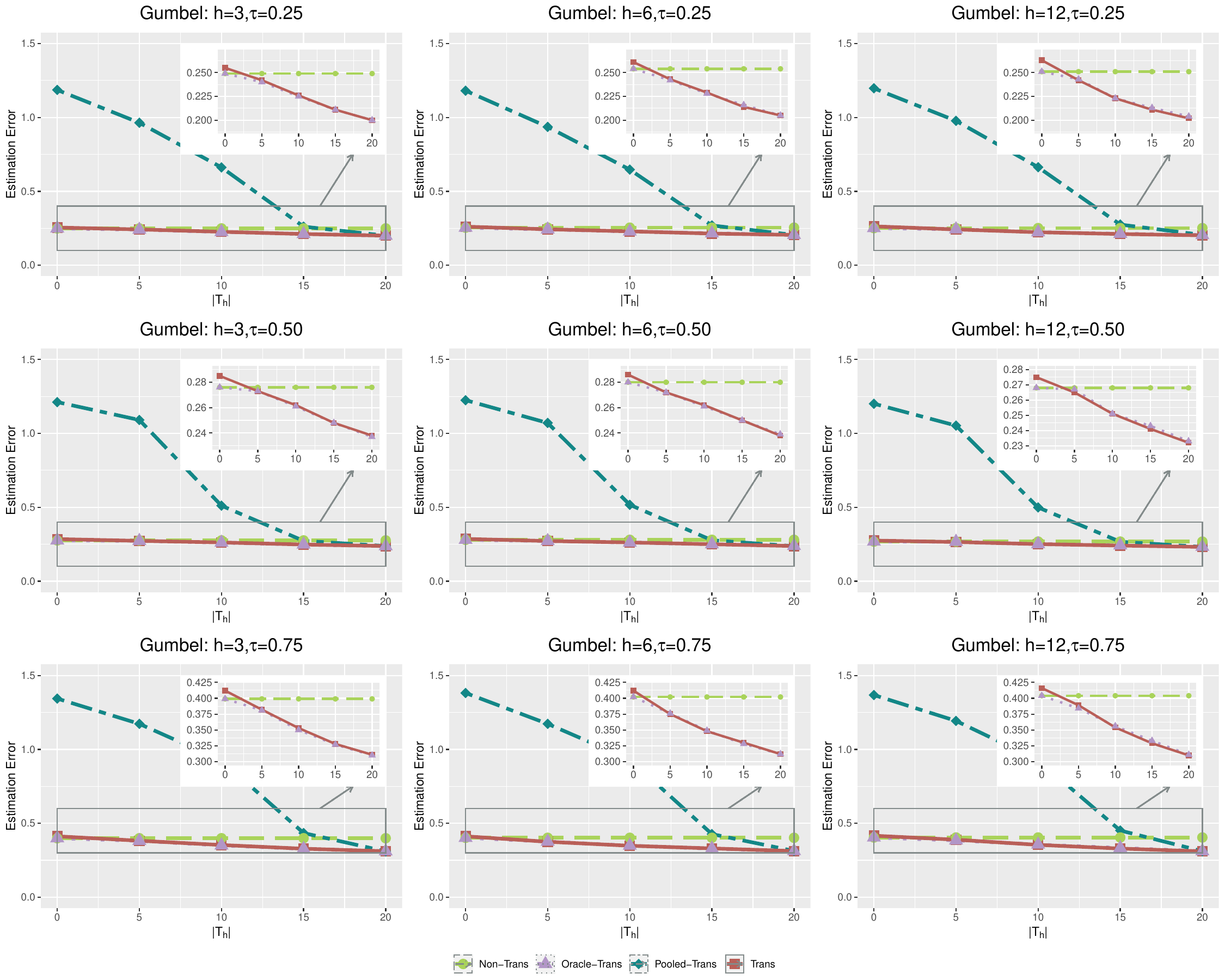}}
\caption{Averaged $\ell_2$ estimation errors with heterogenous Gumbel model errors.}
\label{gumbel}
\end{figure}

\clearpage
\newpage
\begin{figure}
\includegraphics[height=0.75\textwidth,width=\textwidth]{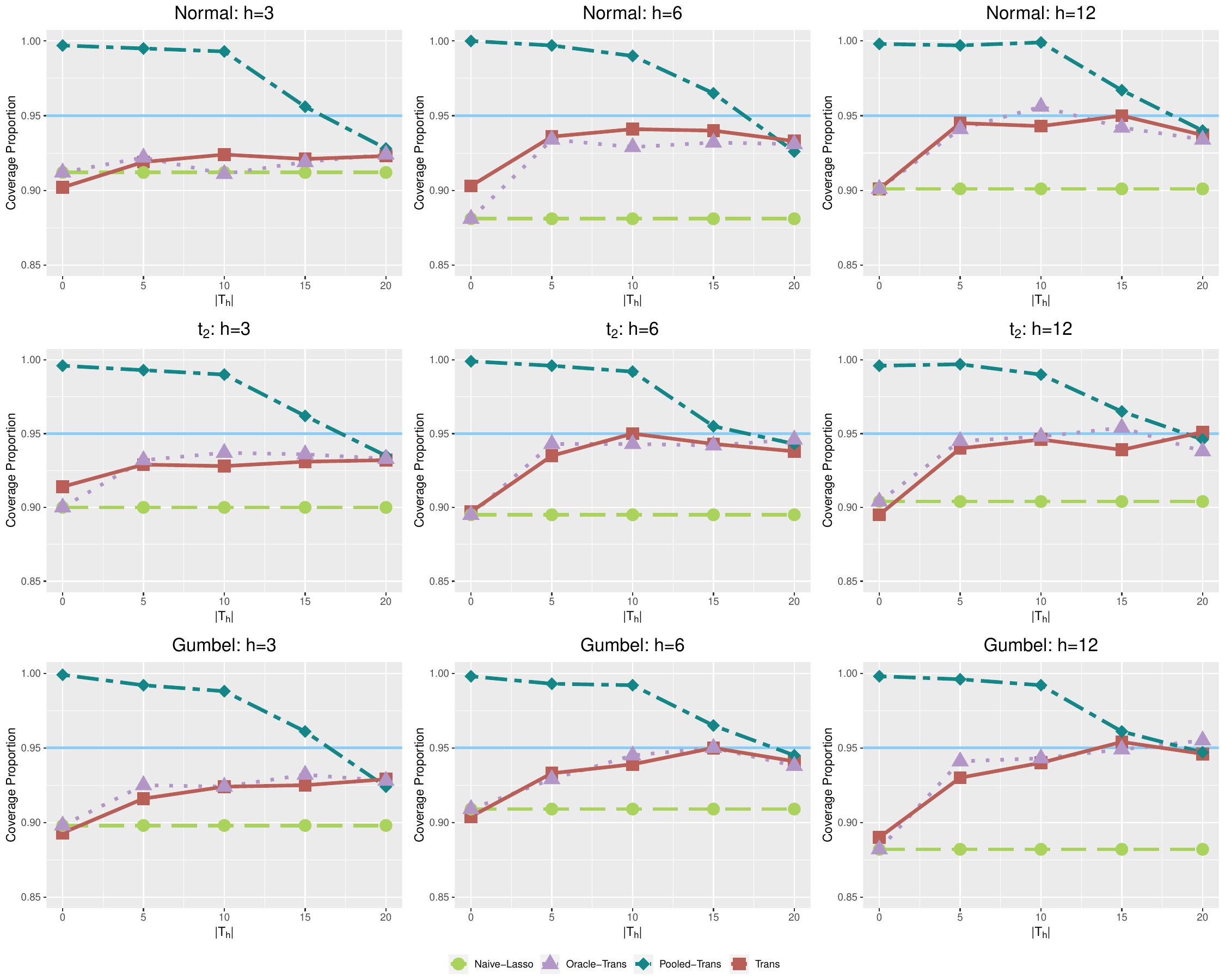}
\caption{Coverage proportion of confidence interval for the $50$-th component
 of quantile regression coefficients with $\tau=0.75$.}
\label{cp50_75}
\end{figure}

\addcontentsline{toc}{section}{References}
\bibliographystyle{mystyle}
\bibliography{ref}

\end{document}